\documentclass[10pt]{article} 
\usepackage[left=1in,top=1in,right=1in,bottom=1in,letterpaper]{geometry} 
\usepackage{ upgreek }
\usepackage[utf8]{inputenc}

\usepackage{amsmath,amsthm,amsfonts,amssymb}
\usepackage{mathtools}
\usepackage{bm}
\usepackage{xcolor}
\usepackage{bbm,dsfont}
\usepackage{enumitem}
\usepackage{subcaption}
\usepackage[most]{tcolorbox}
\usepackage{hyperref}
\usepackage{natbib}
\usepackage{adjustbox}
\usepackage{nicefrac}
\usepackage{thmtools}
\usepackage{enumitem}
\usepackage{wrapfig}

\makeatletter
\@addtoreset{equation}{section}
\makeatother

\hypersetup{
  colorlinks,
  linkcolor={red!50!black},
  citecolor={blue!50!black},
  urlcolor={blue!80!black}
}

\def\R{{\mathbb R}}
\def\N{{\mathbb N}}

\DeclarePairedDelimiter\floor{\lfloor}{\rfloor}

\DeclareMathOperator{\diag}{diag}

\DeclareMathOperator{\E}{\mathbb{E}}

\declaretheoremstyle[
  headfont=\color{blue}\normalfont\bfseries,
  bodyfont=\color{blue}\normalfont\itshape,
]{colored}

\declaretheoremstyle[
  headfont=\color{magenta}\normalfont\bfseries,
  bodyfont=\color{magenta}\normalfont\itshape,
]{mcolored}

\declaretheoremstyle[
  headfont=\color{purple}\normalfont\bfseries,
  bodyfont=\color{purple}\normalfont\itshape,
]{pcolored}

\declaretheoremstyle[
  headfont=\color{red}\normalfont\bfseries,
  bodyfont=\color{red}\normalfont\itshape,
]{rcolored}

\declaretheorem[
  name=Theorem,
]{theorem}

\declaretheorem[
  name=Corollary,
]{corollary}

\declaretheorem[
  name=Lemma,
]{lemma}

\declaretheorem[
  name=Proposition,
]{proposition}

\declaretheorem[
  name= Assumption,
]{ass}

\definecolor{change}{HTML}{C40000}

\newcommand{\bs}[1]{\boldsymbol{#1}}
\newcommand{\eq}[1]{\begin{align}#1\end{align}}

\def\cN{{\mathcal{N}}}
\newcommand{\mpr}{\mathbb{P}}
\DeclareMathOperator*{\argmax}{arg\,max}

\newcommand{\grad}{\nabla}
\newcommand*{\inner}[2]{\left  \langle #1,  #2 \right \rangle}
\DeclarePairedDelimiter\abs{\lvert}{\rvert}%
\DeclarePairedDelimiter\norm{\lVert}{\rVert}%

\def\Zin{\bs{Z}_{\mathrm{in}}}
\def\Zout{\bs{Z}_{\mathrm{out}}}
\def\sZin{\mathsf{Z}_{\mathrm{in}}}

\def\ztrig{\mathsf{z}_{\mathrm{trig}}}

\def\diag{\mathrm{diag}}
\def\tr{\mathrm{tr}}
\def\sym{\mathrm{sym}}
\def\V{\bs{V}}
\def\cF{\mathcal{F}}
\newcommand*{\indic}[1]{\mathbbm{1}_{  #1  }   }
\newcommand\independent{\protect\mathpalette{\protect\independenT}{\perp}}
\def\independenT#1#2{\mathrel{\rlap{$#1#2$}\mkern2mu{#1#2}}}

\def\be{\bs{e}}
\def\ba{\bs{\alpha}}
\def\bw{\bs{w}}
\def\bWkq{\bs{W}_{\mathrm{KQ}}}
\def\bWin{\bs{W}_{\mathrm{in}}}
\def\bN{\bs{N}}
\def\bX{\bs{X}}
\def\bx{\bs{x}}
\def\sX{\mathsf{X}}
\def\bztrig{\bs{z}_{\mathrm{trig}}}

\def\bZin{\bs{Z}_{\mathrm{in}}}
\def\FW{\mathrm{FW}}

\newtheorem*{theorem*}{Theorem}
\newtheorem*{proposition*}{Proposition}

\newcounter{relctr} 
\newcommand\labelrel[2]{%
  \begingroup
    \refstepcounter{relctr}%
    \stackrel{\tiny{(\alph{relctr})}}{\mathstrut{#1}}%
    \originallabel{#2}%
  \endgroup
}
\AtBeginDocument{\let\originallabel\label} 
\AtBeginEnvironment{proof}{\setcounter{relctr}{0}}

\newcommand{\bbt}{\bar \beta}

\mathtoolsset{showonlyrefs}
\newlength\TopY
\newlength\BotY
\newlength\TotalHeight

\usepackage{titlesec}
\titlespacing*{\paragraph}{0pt}{2pt plus 0.85pt minus 0.85pt}{1em}

\title{Learning to Recall with Transformers  Beyond Orthogonal Embeddings}
\author{
Nuri Mert Vural\thanks{University of Toronto and Vector Institute. \texttt{vural@cs.toronto.edu}. Work done while interning at the Flatiron Institute.},\,\,
Alberto Bietti\thanks{Flatiron Institute. \texttt{abietti@flatironinstitute.org}. },\,\,
Mahdi Soltanolkotabi\thanks{University of Southern California. \texttt{soltanol@usc.edu}. },\,\,
Denny Wu\thanks{New York University and Flatiron Institute. \texttt{dennywu@nyu.edu}.}
\vspace{-3mm}
}
\date{\today}

\begin{document}

\maketitle

\begin{abstract}

Modern large language models (LLMs) excel at tasks that require storing and retrieving knowledge, such as factual recall and question answering. Transformers are central to this capability because they can encode information during training and retrieve it at inference. Existing theoretical analyses typically study transformers under idealized assumptions such as infinite data or orthogonal embeddings. In realistic settings, however, models are trained on finite datasets with non-orthogonal (random) embeddings. We address this gap by analyzing a single-layer transformer with random embeddings trained with (empirical) gradient descent on a simple token-retrieval task, where the model must identify an informative token within a length-$L$ sequence and learn a one-to-one mapping from tokens to labels. Our analysis tracks the ``early phase'' of gradient descent and yields explicit formulas for the model’s storage capacity---revealing a multiplicative dependence between sample size $N$, embedding dimension $d$, and sequence length $L$. We validate these scalings numerically and further complement them with a lower bound for the underlying statistical problem, demonstrating that this multiplicative scaling is intrinsic under non-orthogonal embeddings. Code to reproduce all experiments is publicly available.\footnote{Code available at {\url{https://github.com/nurimertvural/learning-to-recall-experiments}}.
\vspace{-2mm}}
 
\end{abstract}

\allowdisplaybreaks

\section{Introduction}

\vspace{-1mm}
Large language models (LLMs) routinely answer knowledge questions with little or no external context, indicating that substantial factual information is stored in parameters and can be retrieved by suitable prompts \cite{petroni2019language,jiang2020can,roberts2020much}.
A deeper theoretical understanding of how such parametric memories are learned and accessed is increasingly important: it can guide scaling choices (e.g., trading off memory capacity against compute budgets,~\cite{carlini2022quantifying,allen2024physics}) and clarify failure modes (e.g., hallucination,~\cite{zucchet2025language,huang2025generalization}).
Motivated by empirical results documenting the prevalence of parametric factual recall and its scaling with model size \cite{allen2024physics,morris2025much}, recent theoretical works have begun to analyze the capacity and learning dynamics of transformers on controlled factual-recall tasks \cite{cabannes2024scaling,nichani2025factual}.

Many theoretical studies of transformer optimization work in population-dynamics settings and adopt simplifying assumptions, such as treating token embeddings as orthogonal or one-hot vectors (see, e.g., \cite{tian2023joma,chen2024unveiling,ghosal2024understanding}). While these choices do not always reflect practical applications, they make the mathematics, particularly gradient calculations, more tractable, and population analyses of this kind do not characterize the statistical or computational complexity of gradient-based learning. In factual-recall setups, strictly orthogonal embeddings are known not to be capacity-optimal, whereas random or non-orthogonal embeddings (i.e., \emph{superposition}) enable near-optimal factual storage \cite{nichani2025factual}. At the same time, removing the orthogonality assumption introduces token interference that leads to intricate optimization behavior (e.g., oscillatory trajectories \cite{cabannes2024learning}), and in practice, superposition-based, memory-efficient solutions can also be more difficult to train \cite{elhage2022toy}, which highlight a fundamental trade-off between optimization and statistical efficiency versus storage capacity.

Motivated by the above gaps, we aim to address the following question.

\begin{center}
\textit{Can we characterize the optimization and sample complexity of a transformer with non-orthogonal embeddings trained by gradient descent in the learning of a factual recall task?}
\end{center}

\subsection{Our Contributions}

In this paper, we analyze gradient-based learning of a single-layer transformer with an attention+MLP block and random embeddings on a 
synthetic task inspired by \cite{nichani2025factual}: the model must retrieve an informative token from a context containing many noisy tokens via attention, then map it to the correct label via factual recall.
To mitigate the complex optimization dynamics arising from non-orthogonal embeddings, we follow \cite{bietti2023birth,oymak2023role} and consider a simplified training regime involving only a few gradient steps with finite samples on the attention and value matrices. This perspective  effectively zooms in the “early phase’’ of the training as commonly studied in the 
feature-learning literature \cite{ba2022high,damian2022neural,dandi2023learning,Vural2024PruningIO,wang2025learning}.

Our analysis provides a fine-grained characterization of how vocabulary size $V$, sample size $N$, embedding dimension $d$, sequence length $L$, and MLP width $m$ interact to permit successful gradient-based learning of the recall mechanism.
Our main result states that
\begin{itemize}[leftmargin=*,topsep=0.5mm,itemsep=0.75mm]
\item The success of learning depends on $(V,N,d,L,m)$ in a \emph{multiplicative} manner: learning becomes easier as $(N,d,m)$ increase, which reflects the benefits of more data, higher-dimensional (and thus more orthogonal) embeddings, and larger MLP width; whereas learning becomes harder as $(V,L)$ increase; that is, the task becomes more difficult with a larger vocabulary or longer sequences. This multiplicative relation is visualized in Figure~\ref{fig:introtradeoff}, where we examine how the parameter size $m\times d$ depends on the vocabulary size $V$ for different sequence lengths $L$. The full phase diagram corresponding to this relation, which formalizes Figure~\ref{fig:introtradeoff}, is shown in Figure~\ref{fig:introphase}.
\item Consequently, while optimal capacity and sample complexity can be achieved jointly for short sequences, successful learning on long sequences requires either a larger embedding dimension (thus sacrificing capacity) or larger sample sizes (worsening statistical complexity).
\end{itemize}

The multiplicative rate above formalizes the ``tradeoff'' intuition that smaller embedding dimension~$d$ — which increases superposition and thereby improves storage capacity — simultaneously yields a harder learning problem, as reflected in the required sample size.
We complement this with a statistical lower bound showing that the trade-off is inherent for any estimator that accesses only gradient information from the initialized transformer.
Finally, although our theory is derived for a specific three-step training algorithm, we empirically observe qualitatively similar multiplicative scaling when the transformer is optimized by gradient descent to low empirical risk.

\begin{figure}[t]
\centering
\vspace{-6mm}
\begin{subfigure}[b]{0.45\textwidth}
    \vspace{0pt}
    \centering
    \includegraphics[width=0.85\linewidth]{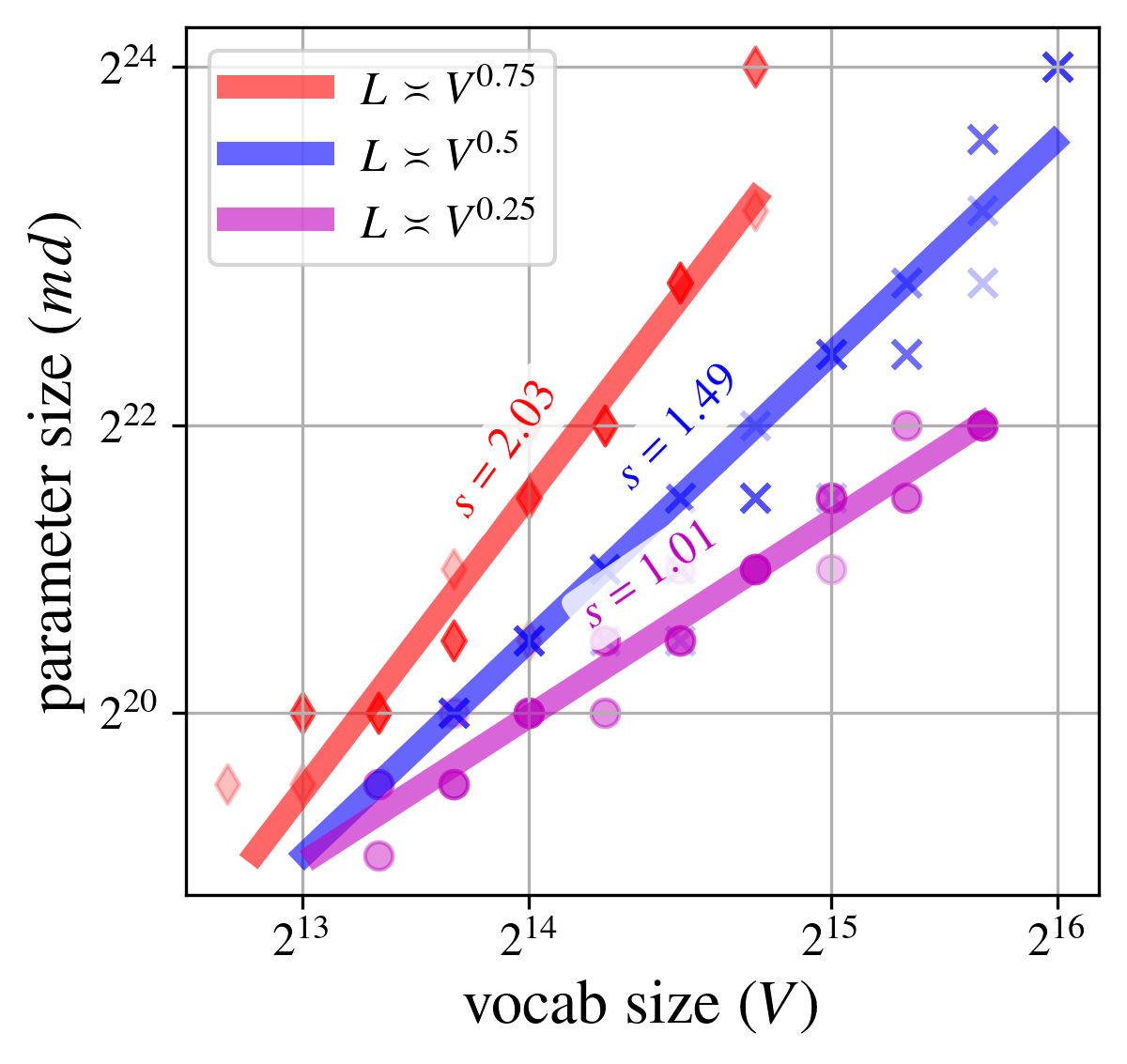}
    \vspace{-2mm}
    \caption{\small}
    \label{fig:introtradeoff}
\end{subfigure}%
\begin{subfigure}[b]{0.48\textwidth}
     \vspace{0pt}
    \centering
    \raisebox{1mm}{
    \begin{tikzpicture}[scale = 0.85, transform shape, font=\normalsize, remember picture] 
             
        \coordinate (O) at (0,0);
        \coordinate (XMax) at (7.5, 0);
        \coordinate (YMax) at (0, 6.0);
         
        \coordinate (Triple) at (2.3, 4.8); 
        \coordinate (CapEnd) at (7.5, 2.55); 

        \fill[gray!20] (O) -- (CapEnd) -- (7.5, 0) -- cycle;
        \node[align=center, gray!60!black, font=\scriptsize] at (5.5, 1.0) {No Learning \\ ($md \ll V$)};

        \fill[red!10] (O) -- (Triple) -- (0, 4.8) -- (0,0) -- cycle;
        \fill[red!10] (0, 4.8) rectangle (2.3, 6.0);
        \fill[green!10] (2.3, 4.8) rectangle (7.5, 6.0);
        \fill[blue!10] (O) -- (Triple) -- (7.5, 4.8) -- (CapEnd) -- cycle;

        \coordinate (MLPSource) at (6.0, 3.0);

        \draw[->, thick] (0,0) -- (XMax) node[below left] {vocab size $(V)$};
        \draw[->, thick] (0,0) -- (YMax) node[above right] {width $(m)$};

        \draw[thick, red!80!black] (O) -- (Triple) 
            node[midway, sloped, above, yshift=1pt, font=\scriptsize] {Slope $\asymp d$};
        \draw[thick, green!60!black] (Triple) -- (7.5, 4.8);
        \draw[thick, dashed, black] (O) -- (CapEnd) 
            node[midway, sloped, above, font=\scriptsize] {Slope $\asymp 1/d$};

        \node[align=center, red!60!black, font=\footnotesize] at (1.15, 5.45) {\textbf{Mean Bias} \\[0.1em] \footnotesize $md \gtrsim \frac{m L^{ \scalebox{0.8}{$\scriptscriptstyle \frac{4}{3}$} }}{V^{ \scalebox{0.8}{$\scriptscriptstyle \frac{2}{3}$} } N^{ \scalebox{0.8}{$\scriptscriptstyle \frac{2}{3}$} } }$};
        \node[align=center, green!40!black, font=\small] at (4.9, 5.35) {\textbf{Gradient Noise} \\[0.1em] \footnotesize $md \gtrsim \tfrac{m L^{ \scalebox{0.8}{$\scriptscriptstyle \frac{1}{4}$} }}{\sqrt{N}}$};
        \node[align=center, blue!60!black, font=\large] at (4.5, 3.15) {
            \textbf{MLP Noise} \\[0.2em]
            \normalsize $md \gtrsim  V   \tfrac{\sqrt{m} L^2}{N}  $
        };
        
        \fill[black] (Triple) circle (2pt);
        \draw[dotted] (Triple) -- (2.3, 0) node[below, font=\small] {$dL$};
        \draw[dotted] (Triple) -- (0, 4.8) node[left, font=\small] {$d^2 L$};

\end{tikzpicture} }
\caption{\small }
\label{fig:introphase}
\end{subfigure}
 \vspace{-2mm}
  \caption{\small (a) Empirical scaling of the parameter size required for a GD-trained one-layer transformer to learn factual recall, where we use $m \asymp d^2$ (see Section~\ref{sec:attnmlp} for details). For small $L$, the trained model achieves the optimal capacity $V \asymp md$ (purple line). As the sequence length $L$ increases, the scaling changes, suggesting a multiplicative rate (blue and red lines). (b) Phase diagram for the theoretical scaling of the parameter count given in Corollary~\ref{cor:attentionmlp}. Each region corresponds to a regime where a particular noise term in Theorem~\ref{thm:main} is dominant. The parameter-size condition ($md$) in each region is given in Corollary~\ref{cor:attentionmlp}. }
\label{fig:intro}
\vspace{-2mm}
\end{figure}
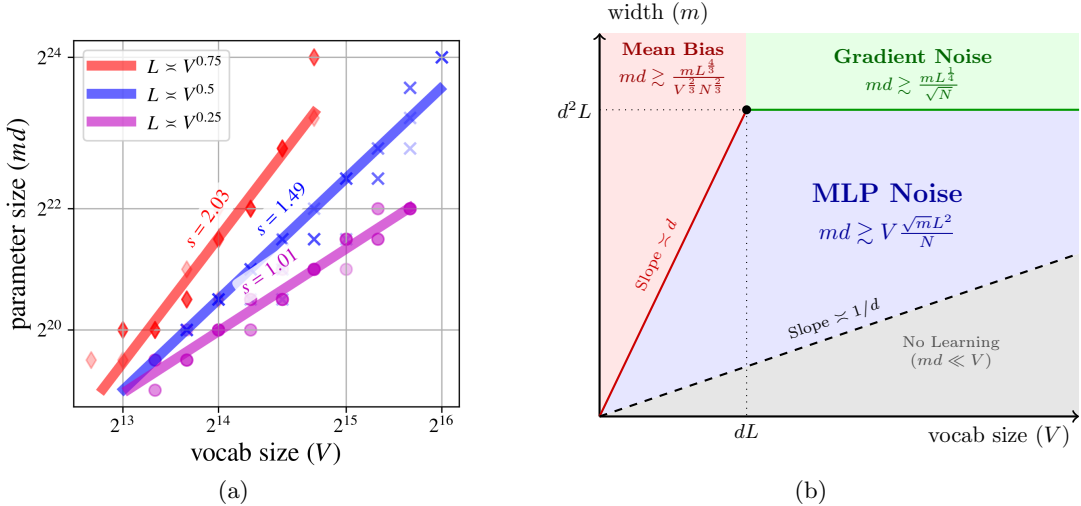

\subsection{Related Work}

\paragraph{Learning dynamics of transformers.}
A growing line of work analyzes how transformers acquire specific behaviors from gradient-based training. Much of this literature imposes population-level assumptions and orthogonal/one-hot embeddings to make gradients tractable, often on discrete synthetic tasks~\cite{li2023transformers,bietti2023birth,tian2023scan,nichani2024transformers,chen2024unveiling,ghosal2024understanding,chen2025distributional,wang2025learning}. Several works study few-step training regimes as a lens on the “early phase’’ of feature learning~\cite{bietti2023birth,wang2025learning}. Beyond discrete settings, related analyses investigate attention learning for continuous inputs and sparse-signal retrieval~\cite{oymak2023role,marion2025attention, duranthon2026singlelocation}. A complementary thread focuses on the emergence of in-context learning and induction mechanisms: single- and two-layer attention trained on linear-regression or Markov data provably implements gradient-descent-like updates and generalized induction heads~\cite{von2023transformers,zhang2024trained,chen2024unveiling,nichani2024transformers}. These results typically rely on simplified settings and do not address storage capacity. In contrast, our work analyzes finite-sample training with  non-orthogonal embeddings in an attention+MLP architecture with a particular focus
on factual recall.

\paragraph{Associative memories and storage capacity.}
Classical associative memories (Hopfield-type models) study recall of vector patterns and established foundational capacity results~\cite{hopfield1982neural,amit1985storing,mceliece1988capacity,krotov2016dense,demircigil2017model,ramsauer2020hopfield,schlag2021linear}. Recent works adapt associative-memory viewpoints to transformers, modeling inner weights as superpositions of outer products and deriving scaling laws and optimization behaviors~\cite{bietti2023birth,cabannes2024scaling,cabannes2024learning}. In factual recall specifically, random (non-orthogonal) embeddings enable near-parameter-count storage, whereas strictly orthogonal embeddings are not capacity-optimal~\cite{nichani2025factual}. Various empirical works have studied the mechanisms and scaling behaviors of LLMs in factual association tasks~\cite{petroni2019language,jiang2020can,geva2020transformer,allen2024physics}. 
We provide a theoretical analysis of such mechanisms and quantify how vocabulary size, sequence length, embedding dimension, and MLP width jointly govern learning efficiency. Our work operates in a setting similar to~\cite{nichani2025factual} but allows finite samples and explicitly considers gradient descent dynamics. Our result is similar to the finite-sample results in~\cite{oymak2023role}, where the required sample size grows with the dimensionality and sparsity level of informative tokens, while we allow non-orthogonal embeddings and show optimal capacity as in~\cite{nichani2025factual} under certain conditions.

\vspace{-2mm} 
\section{Problem Setting}
\vspace{-2mm}

Our goal is to understand the capacity of transformers trained on finite data with non-orthogonal embeddings, in a setting where the relevant information is hidden in a potentially large sequence of non-informative noisy tokens. The attention operation should then identify the relevant token, while the subsequent linear or MLP block can then recall the correct label via an associative memory mechanism. This is similar to the factual recall task studied by~\cite{nichani2025factual}, with simplifications that make the analysis more tractable, as detailed below.

\smallskip
\textbf{Notation. } $\sigma$ denotes the softmax function. $\mathbbm{1}_V\coloneqq(1,\dots,1)^\top\in\R^V$ is the $V$-dimensional all-ones vector; $\bs{e}_i$ is the one-hot vector with a $1$ in the $i$-th position (dimension understood from context). We use $\gtrsim$ (resp.\ $\lesssim$) to mean “$\ge$” (resp.\ “$\le$”) up to polylogarithmic factors in $V$: $f_V\gtrsim g_V \iff f_V\ge \mathrm{poly}(\log V) g_V$ and $f_V\lesssim g_V \iff f_V\le \mathrm{poly}(\log V) g_V$, for some fixed polynomial. Lastly, $\lVert\cdot\rVert_2$ denotes the Euclidean norm for vectors and the operator (spectral) norm for matrices.

\smallskip 
\textbf{Problem setup.} Let the input/output tokens take values from a finite alphabet $[V] \coloneqq \{ 1, \cdots, V \}$. For notational convenience, we represent the alphabet by the one-hot vocabulary $\mathcal{V} = \{ \bs{e}_1, \cdots, \bs{e}_V \}$. Each example in the data consists of a length-$L$ input sequence $\bs{X} = [\bs{x}_1, \dots, \bs{x}_L] \in \mathcal{V}^L$ and a label $\bs{p} \in \mathcal{V}$ generated as follows:
\begin{itemize}[leftmargin = *,itemsep = 0.2em]
\item \emph{Input} tokens are sampled independently and uniformly: $[\bs{x}_1, \dots, \bs{x}_{L}] \sim \mathrm{Unif}(\mathcal{V}^{L})$.
\item \emph{Informative position} is a random index $\ell \sim \mathrm{Unif}([L])$ independent of $\bs{X}$. 
\item \emph{Ground-truth function} is a permutation matrix $\bs{\Pi}_* \in \{0,1\}^{V \times V}$. Labels are generated as the permuted informative token, $\bs{p} = \bs{\Pi}_{*} \bs{x}_{\ell}$, while the remaining tokens are non-informative.
\end{itemize}

The goal is to identify the correct token position $\ell$ and learn the target function (permutation) $\bs{\Pi}_{*}$.

\smallskip
\textbf{Transformer architecture.} We consider a basic transformer block which first maps input tokens into a $d$-dimensional embedding space where $d<V$. The embedding layer is parameterized by $(\Zin, \Zout, \bs{z}_{\mathrm{trig}}, \bs{z}_{\mathrm{EOS}}) \in \R^{d \times V} \times \R^{d \times V} \times \R^{d} \times \R^{d}$, where
\begin{itemize}[leftmargin = *,itemsep = 0.2em]
\item The input tokens are embedded by the columns of the matrix $\Zin \in \R^{d \times {V}}$.
\item Output tokens are associated with unembedding vectors, which are collected in $\Zout \in \R^{d \times V}$.
\item $\bs{z}_{\mathrm{trig}}$ is a trigger vector that marks the informative token. 
\item $\bs{z}_{\mathrm{EOS}}$ is the special embedding vector that marks the end-of-sequence.
\end{itemize}
Given the embedding parameters, we define the self-attention head, parameterized by the key-query matrix $\bs{W}_{\mathrm{KQ}} \in \R^{d \times d}$, which operates on the embedded sequence of inputs $\Zin \bs{X} \in \R^{d \times L}$:
    \eq{
    \mathrm{attn}(\bs{X}; \bs{W}_{\mathrm{KQ}}  ) \coloneqq   \Zin \bs{X}   \sigma \Big( (  \bs{z}_{\mathrm{trig}} \bs{e}_{\ell}^\top +   \Zin \bs{X}  )^\top \bs{W}_{\mathrm{KQ}} \bs{z}_{\mathrm{EOS}} \Big). \label{eq:attention}
    }
The trigger embedding~$\bs{z}_{\mathrm{trig}}$ is used to ``mark'' the informative token with a special direction, mimicking the behavior of previous transformer layers that may learn to flag particular tokens by adding to its residual stream\footnote{The ``trigger'' terminology is borrowed from~\cite{bietti2023birth}, where a special previous token ``triggers'' a retrieval operation in the context of induction heads. Our setup resembles learning only the ``induction head'' layer assuming the first ``previous token head'' layer is already in place. The triggers often appear to be single directions in interpretability literature, see, e.g., the ``X in opposite of X'' feature in~\cite{kamath2025tracing}. \vspace{-3mm}} (note that the number of trainable parameters inside softmax can be reduced to~$d$ by collapsing $\bs{W}_{\mathrm{KQ}} \bs{z}_{\mathrm{EOS}}$ into a vector).  
We consider two different learning models: an \emph{Attention-only} model and a width-$m$, two-layer neural network model \emph{Attention-MLP}, defined as:
\eq{
\hat{\bs{p}}(\bs{X}; \bs{V}, \bs{W}_{\mathrm{KQ}}) = \begin{cases}
\sigma \Big( \Zout^\top \bs{V}    \mathrm{attn}(\bs{X}; \bs{W}_{\mathrm{KQ}}  ) \Big), & \text{Attention only} \\[0.6em]
\sigma\Big( \Zout^\top \bs{V} \phi(\bs{W}_{\mathrm{in}}     \mathrm{attn}(\bs{X}; \bs{W}_{\mathrm{KQ}}  ) \Big), & \text{Attention-MLP}
\end{cases} \label{eq:networkoutput}
}
where $\bs{V} \! \in \!  \R^{d \times d}$ for the \emph{Attention-only} and $\bs{V}\! \in \! \R^{d \times m}$,  $\bs{W}_{\mathrm{in}} \! \in \! \R^{m \times d}$ for the \emph{Attention-MLP} model. 
Note that compared with \emph{Attention-only} model, the \emph{Attention-MLP} model contains an additional set of trainable parameters and nonlinear activation function $\phi$ before the value matrix. 
Similar to in~\cite{nichani2025factual}, the MLP allows using a smaller embedding dimension~$d$ while keeping the capacity large by increasing width~$m$.

For the \emph{Attention-MLP}, we keep $\bs{W}_{\mathrm{in}}$ fixed at its random initialization. The trainable parameters for both of our models are $(\bs{V}, \bs{W}_{\mathrm{KQ}})$. We use cross-entropy loss to train our model:
\eq{
\mathcal{L} \big( (\bs{V}, \bs{W}_{\mathrm{KQ}} ), (\bs{X}, \bs{p}) \big) = \textstyle- \sum_{i = 1}^V p_i \log \hat{p}_i.
} 
\smallskip
\textbf{Training algorithm.}
Following \cite{oymak2023role}, we consider a 3-step gradient-based algorithm with dataset $\{ (\bs{X}_i, \bs{p}_i) \}_{i = 1}^N$ with a sample size of $N$. We initialize our parameters as $\V^{(0)} =0$,  $\bs{W}_{\mathrm{KQ}}^{(0)} = 0$ and use the learning rates $\eta, \gamma > 0$: 
\eq{
 \V^{(1)} & = \V^{(0)} - \textstyle\eta\cdot\frac{1}{N} \sum_{i = 1}^N \grad_{\V} \mathcal{L} \big( (  \V^{(0)} ,   \bs{W}_{\mathrm{KQ}}^{(0)}  );  (\bs{X}_i, \bs{p}_i)  \big)  \label{eq:gdfirststep} \\
\bs{W}_{\mathrm{KQ}}^{(1)}  & = \bs{W}_{\mathrm{KQ}}^{(0)}  -   \textstyle\gamma\cdot\frac{1}{N} \sum_{i = 1}^N \grad_{\bs{W}_{\mathrm{KQ}}} \mathcal{L} \big( (   \V^{(1)} ,    \bs{W}_{\mathrm{KQ}}^{(0)} );  (\bs{X}_i, \bs{p}_i)  \big) \label{eq:gdsecondstep} \\
 \V^{(2)} & = \V^{(1)}  -   \textstyle\gamma\cdot\frac{1}{N} \sum_{i = 1}^N  \grad_{\V} \mathcal{L} \big( (   \V^{(1)} ,   \bs{W}_{\mathrm{KQ}}^{(1)}  );  (\bs{X}_i, \bs{p}_i)  \big) .   \label{eq:gdthirdstep}
}
\smallskip
\textbf{Network prediction and storage.} Given our model and training method, we use argmax decoding at inference and define the test accuracy as
\eq{
\mathrm{Accuracy} \coloneqq \mpr_{(\bs{X}, \bs{p})}\big[ \bs{p} =  \bs{e}_{\mathrm{pred}(\bs{X})}   \big], \quad \text{where}  \quad\mathrm{pred}(\bs{X}) \coloneqq  \argmax_{j \in [V]} \hat{p}_j(\bs{X}; \V^{(2)}, \bs{W}_{\mathrm{KQ}}^{(1)}),
}
where $\hat{\bs{p}}(\bs{X}; \V^{(2)}, \bs{W}_{\mathrm{KQ}}^{(1)})$ is the network output defined in \eqref{eq:networkoutput}. In what follows, we characterize conditions under which the model stores the informative tokens asymptotically, i.e., $\mathrm{Accuracy} \to 1$ as $V \to \infty$, in terms of the relevant parameters $(V, N, d, L, m)$.

\vspace{-2mm}
\section{Main Results}
\label{sec:main-result}
\vspace{-1.5mm}
 
We first present our general theorem on learnability via gradient descent, and then specialize into different regimes to derive more interpretable scaling behaviors in Section~\ref{sec:consequences}. We provide a proof sketch in Section~\ref{sec:overview}, and defer the full proof to Appendix \ref{sec:thmmain}.

\vspace{-1.5mm} 
\subsection{Technical Assumptions} 
 
We first state generic assumptions that apply to both the \emph{Attention-only} and \emph{Attention-MLP} models. 
\begin{ass}~
   \label{ass:condsgeneral}
   \vspace{-1mm}
   \begin{itemize}[leftmargin = *]
    \item \textbf{Parameter range:} Let $L = V^{c}$ for $c \in (0,1)$, $\Omega(V \log V) \leq N = o(VL)$, and $V \geq \Omega(1)$.
     \item \textbf{Learning rate:} We use 
    a sufficiently small learning rate $\eta$ for the initial step \eqref{eq:gdfirststep}, and sufficiently large learning rate $\gamma$ for the remaining steps \eqref{eq:gdsecondstep}-\eqref{eq:gdthirdstep} that satisfy Assumption~\ref{ass:conditions}.
     \item  \textbf{Embeddings:} Let $\Zin ,  \Zout \in \R^{d \times V}$ be independent Gaussian matrices, and let $\bs{z}_{\mathrm{trig}}, \bs{z}_{\mathrm{EOS}} \in \R^d$ be independent Gaussian vectors, all with i.i.d.\ entries distributed as $\cN(0,\nicefrac{1}{d})$.
   \end{itemize}
\end{ass}

We assume $c \in (0,1)$ since  in many practical pretraining setups, the context length is smaller than the vocabulary size, and   the condition $L \ll V$ simplifies several terms in the proofs. The lower bound  $N \gtrsim V \log V$ is required so that each element from the alphabet of size~$V$ is seen at least once with high probability. The learning rates follow  prior analyses~\citep{oymak2023role,nichani2024transformers}: a small $\eta$ ensures that the network’s predictions remain close to uniform after the first step, whereas a large $\gamma$ is needed to push the attention scores and predictions toward one-hot vectors.

In addition to the above assumptions, we require the transformer model to have sufficient capacity to reach perfect test accuracy. Such conditions are characterized by~\cite{nichani2025factual}. For the \emph{Attention-only} model, we have the following condition (see \cite[Theorem 3]{nichani2025factual}). 
\begin{ass}[Attention-only]
\label{ass:condsattentiononly} 
For the \emph{Attention-only} model, we require $d^2 \gtrsim V$.
\end{ass}
With a nonlinear MLP layer, a smaller embedding dimension can suffice if the width is large enough. Hence for \emph{Attention-MLP} we require the following condition.
\begin{ass}[Attention-MLP]
\label{ass:condsattentionmlp} 
For the \emph{Attention-MLP} model, we assume that
\begin{itemize}[leftmargin = *]
\vspace{-1mm}
\item \textbf{Polynomial activation:} $\phi:\R\to\R$ satisfies $\phi(0), \phi^\prime(0), \phi^{\prime \prime}(0) \neq 0$.
\item \textbf{MLP width:} $md \gtrsim V$ and $d \gtrsim V^{\frac{1}{k_\star + 1}}$, where $k_{\star}$ denotes the smallest nonzero Hermite mode of $\phi$, i.e., $k_\star \coloneqq \min \{ k > 0 : \E_{Z \sim \cN(0,1)}[ \phi(Z) h_k(Z) ] \neq 0 \}$ where $h_k$ is the k\textsuperscript{th} Hermite polynomial.
\item \textbf{Initialization:} $\bs{W}_{\mathrm{in}} \in \R^{m \times d}$ are fixed with entries i.i.d.\ distributed as $\cN(0,1)$.
\end{itemize}
\end{ass}
The nonlinear MLP layer allows us to compensate for the embedding dimension and go beyond the $d^2 \gtrsim V$ lower bound required by the \emph{Attention-only} model (Assumption~\ref{ass:condsattentiononly}). Note that $m d \gtrsim V$ is a necessary condition for capacity as shown in \citep{nichani2025factual}. The additional requirements imposed on the polynomial activation function appear to be artifacts of our three-step GD analysis, and we conjecture that they could be relaxed when considering a longer training horizon.

\vspace{-3.5mm}   
\subsection{Learnability Statement}

\vspace{-1.5mm}
Now we are ready to present our main theorem on the complexity of learning the factual recall task. Specifically, the transformer learns the desired mechanism when the signal term dominates the noise and bias terms as stated below. 
\begin{theorem}
\label{thm:main}
Let Assumptions \ref{ass:condsgeneral} and \ref{ass:condsattentionmlp} hold for \emph{Attention-MLP}, and \ref{ass:condsgeneral} and \ref{ass:condsattentiononly} hold for \emph{Attention-only}. The \emph{Attention-MLP} model achieves $\mathrm{Accuracy} = 1 - o_V(1)$ with probability $1 - o_V(1)$ whenever
\eq{
\underbrace{\frac{1}{V L^2}}_{\text{Signal}} \gtrsim   \underbrace{ \frac{1}{N \sqrt{L}  d  (d \wedge L)} }_{\text{Gradient noise}} + \underbrace{ \frac{1}{N \sqrt{V d}  (d \wedge L)} }_{\text{Mean bias}} + \underbrace{ \frac{1}{N d \sqrt{m}} }_{\text{MLP noise}}. \label{eq:mainresult}
}
For the \emph{Attention-only} model, the same holds with the last MLP noise term removed.
\end{theorem}
Theorem~\ref{thm:main} characterizes learnability as a function of $(V, N, d, L, m)$ and identifies the following terms that impact the gradient signal-to-noise ratio:
\begin{enumerate}[leftmargin=*,topsep=0.2em,itemsep=0.2em]
\item \emph{Signal} measures the alignment between the key–query weights $\bWkq^{(1)}$ and the trigger $\bztrig$.
\item \emph{Gradient noise} is due to the concentration error in the update of $\bWkq^{(1)}$.
\item \emph{Mean bias} arises from the nonzero mean of token vectors $\{\bX_i\}_{i=1}^N$.
\item \emph{MLP noise} reflects the randomness in the MLP weight matrix $\bWin$ in \emph{Attention–MLP}.
\end{enumerate}

We make the following observations. 
\begin{itemize}[leftmargin=*,itemsep=0.2em]
    \item \textbf{Multiplicative scaling.} Note that the parameters $(V, N, d, L, m)$ interact in a multiplicative fashion. For example, the noise and bias terms in \eqref{eq:mainresult} all decay with $(N\times d)$, suggesting that increasing the embedding dimensions $d$ can lower the statistical complexity of learning the correct recall mechanism. While the full 5-parameter trade-off can be opaque, in Section~\ref{sec:consequences} we focus on specific regimes that lead to simplification of the scaling relationship and validate the rate empirically. 
    \item \textbf{Optimal storage \& sample complexity.} Recall that the capacity-optimal construction for the factual recall task requires $md\gtrsim V$ parameters (or $d^2 \gtrsim V$ for \emph{Attention–only}); and as discussed earlier, a sample size $N \asymp V\log V$ is necessary to observe all distinct tokens. \eqref{eq:mainresult} implies that in the small-$L$ regime, the optimized transformer achieve optimal capacity and sample complexity simultaneously. 
    For longer sequences, however, these two conditions may not be achieved at the same time, i.e., one must increase either the network width or sample size beyond optimality to learn the task --- this confirms the empirical observation in Figure~\ref{fig:intro}. 
\end{itemize}

\vspace{-3.5mm}
\subsection{Statistical Lower Bound}

\vspace{-1mm}
Theorem~\ref{thm:main} provides an upper bound (i.e., sufficient condition) on the model and sample size for learning factual recall under a 3-gradient-step optimization procedure. 
We complement this sufficient condition with a lower bound indicating that the multiplicative dependence on the problem
parameters is partly statistical; that is, the scaling behavior will be observed in any model satisfying the broader conditions stated below. Our lower bound applies to statistical methods that can query the dataset through the attention outputs at initialization, $\bs{h}_i \coloneqq \mathrm{attn}(\bs{X}_i, \bs{W}_{\mathrm{KQ}}^{(0)})$. In particular, we consider queries of the form $\bs{h}_i$ as the gradient with respect to the key–query matrix $\bs{W}_{\mathrm{KQ}}$ depends on  $\{ \bs{h}_i, \bs{h}_i \bs{h}_i^\top \}_{i = 1}^N$ (see  \eqref{eq:kqgradient}). The statement is given below: \vspace{-3.5mm}
\begin{theorem}[Informal]
\label{thm:lowerbound-informal}
Any method that relies on the noisy version of the queries
$\{ \bs{h}_i, ~ \bs{h}_i \bs{h}_i^\top \}_{i = 1}^N$ fails, i.e., $ \mathrm{Accuracy} \not \to 1$ with finite probability, if  $N \lesssim V \min\{ 1, L/d^2 \}$. 
\end{theorem}
The complete statement of Theorem~\ref{thm:lowerbound-informal} is deferred to Theorem~\ref{thm:lowerbound} in Appendix~\ref{sec:lowerbound}. We observe that the lower bound does not exactly match our upper bound in Theorem~\ref{thm:main}, as $\emph{Signal} \lesssim \emph{Gradient Noise}$ in~\eqref{eq:mainresult} is stronger than the stated lower bound. This being said, Theorem~\ref{thm:lowerbound-informal} also confirms the multiplicative scaling, hence suggesting the trade-off between capacity and sample efficiency is present in a boarder class of learning algorithms. A stronger computational lower bound for transformers and gradient-based optimization is an interesting problem we leave for future work.

\vspace{-2mm}
\section{Implications and Empirical Verifications}
\label{sec:consequences}

\vspace{-1mm}
In this section, we leverage our main theorem to obtain more concrete scalings between parameters, and present empirical evidence on the derived multiplicative rate. 

\begin{figure}[t]
    \vspace{-5mm}
    \centering
    \begin{minipage}[b]{.42\textwidth}
    \vspace{0pt}
    \begin{subfigure}[b]{.49\linewidth}
        \centering
        \includegraphics[width=\linewidth]{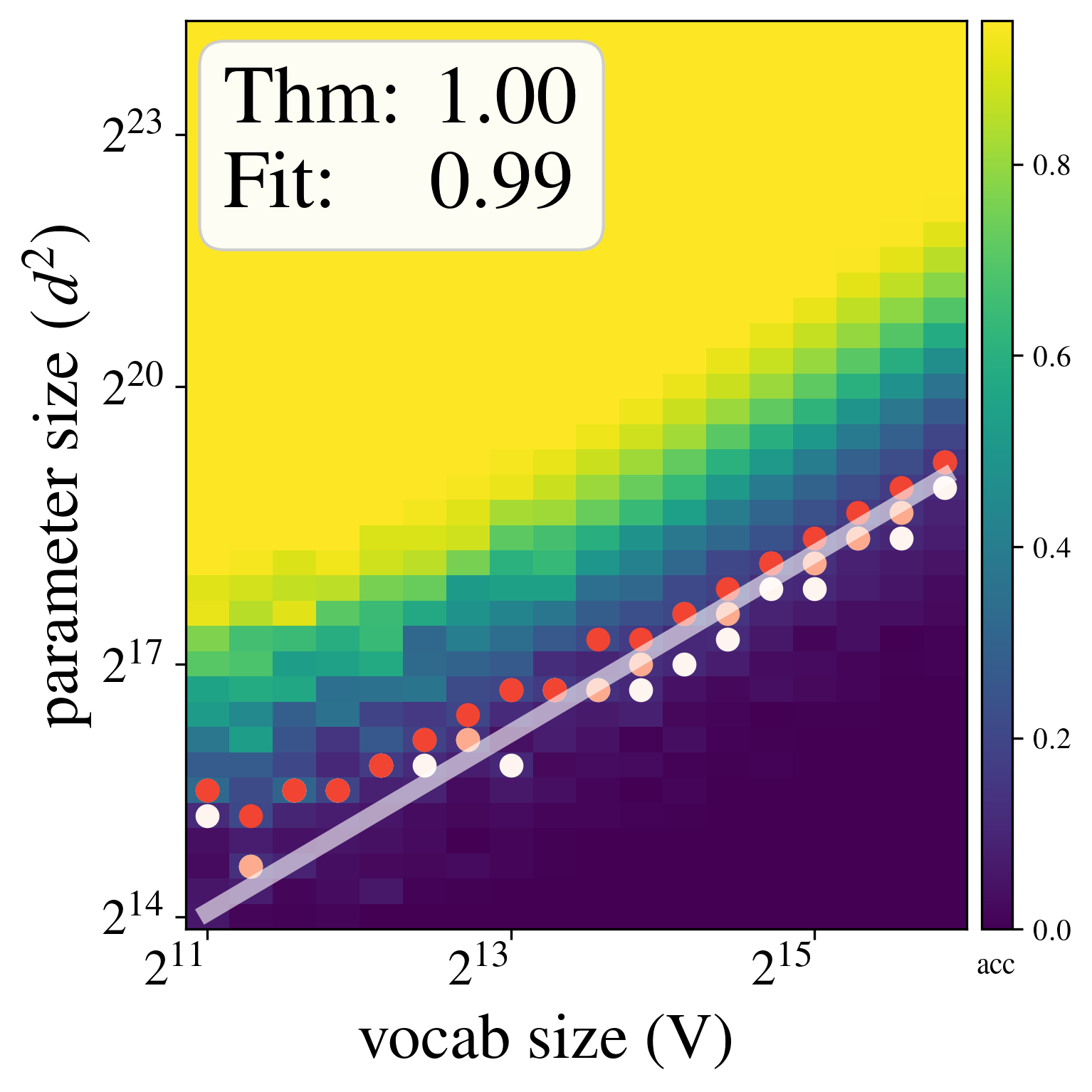}
    \end{subfigure} 
     \begin{subfigure}[b]{.49\linewidth}
        \centering
        \includegraphics[width=\linewidth]{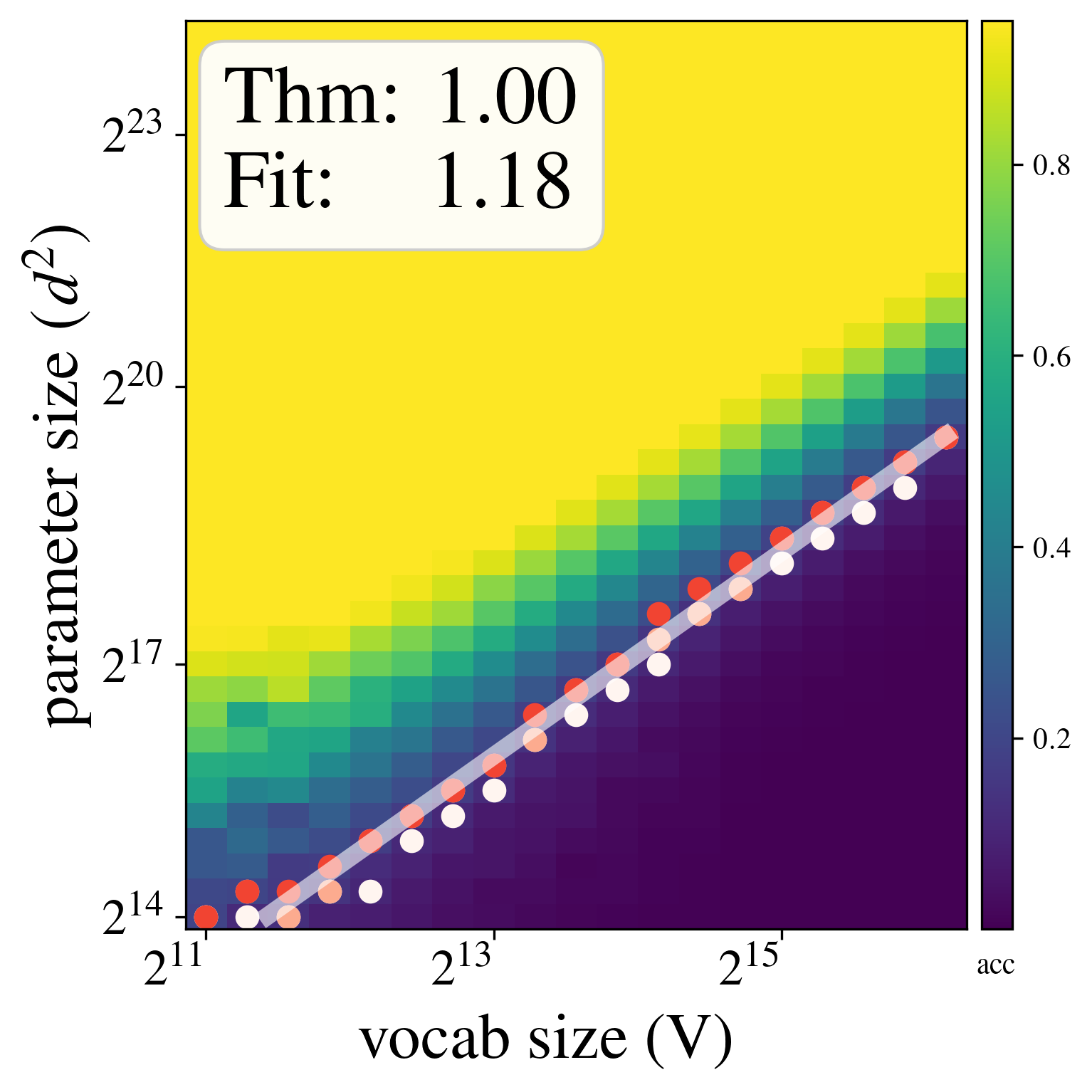}
    \end{subfigure}  \\
      \begin{subfigure}[b]{.49\linewidth}
        \centering
        \includegraphics[width=\linewidth]{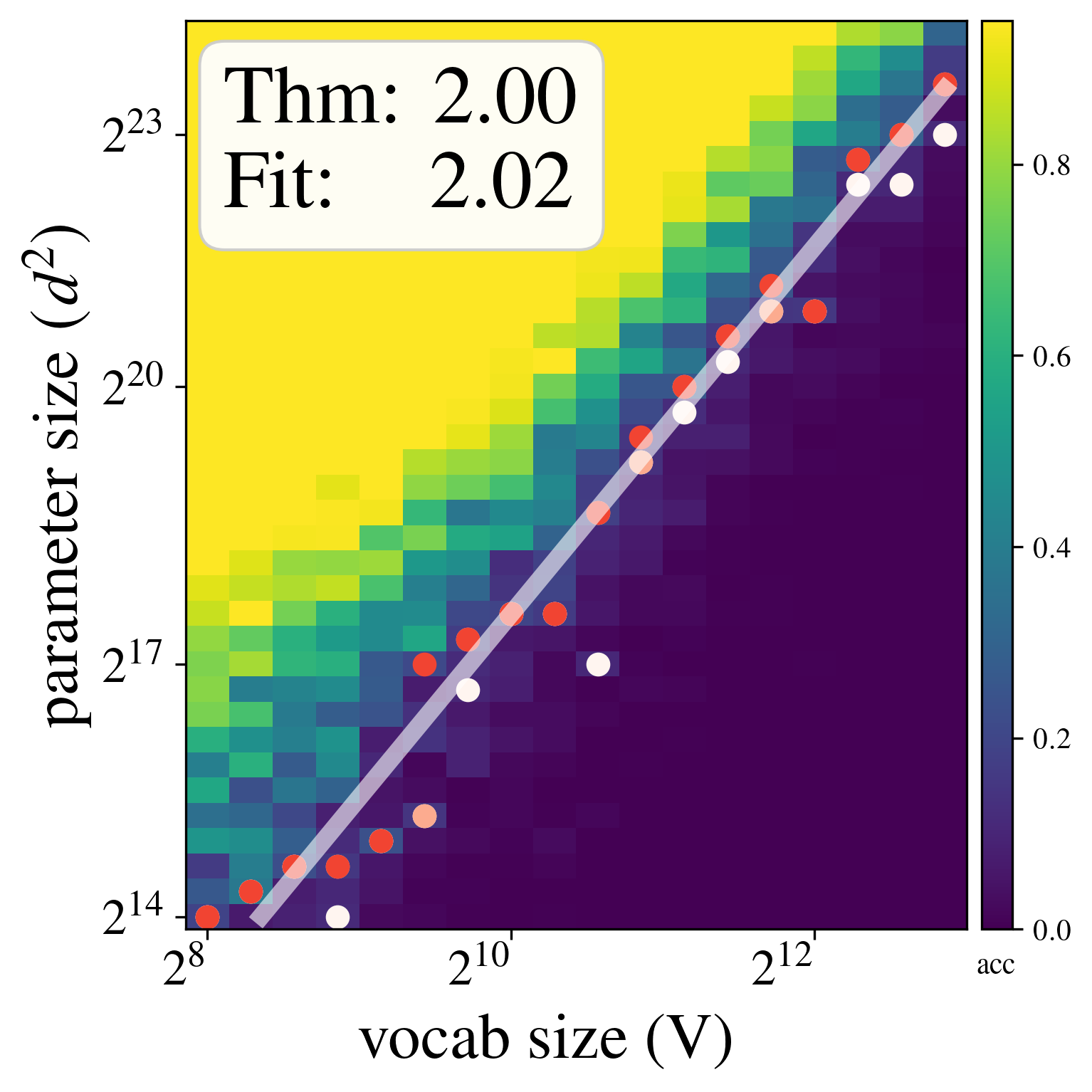}
        \caption{ \small $N \asymp V \log V$ }
         \label{fig:attentiononlyplotsmallsample}
    \end{subfigure}    
    \begin{subfigure}[b]{.49\linewidth}
        \centering
        \includegraphics[width=\linewidth]{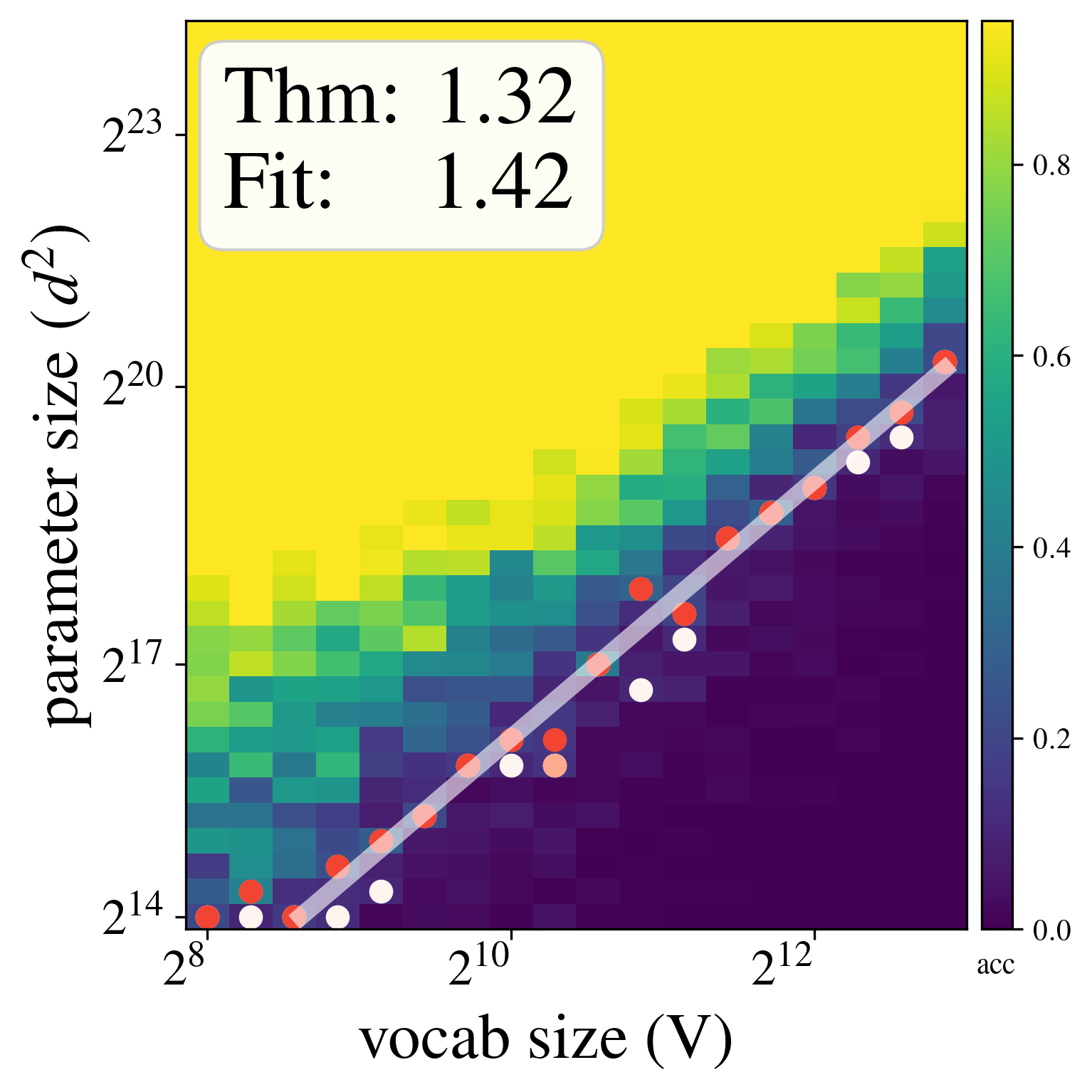}
        \caption{ \small    $N \asymp V^{1.5}$}
         \label{fig:attentiononlyplotlargesample}
    \end{subfigure} 
     \end{minipage}  ~~~~
    \begin{minipage}[b]{.45\textwidth}
    \vspace{0pt}
       \begin{subfigure}[b]{.85\linewidth}
        \centering
        \includegraphics[width=\linewidth]{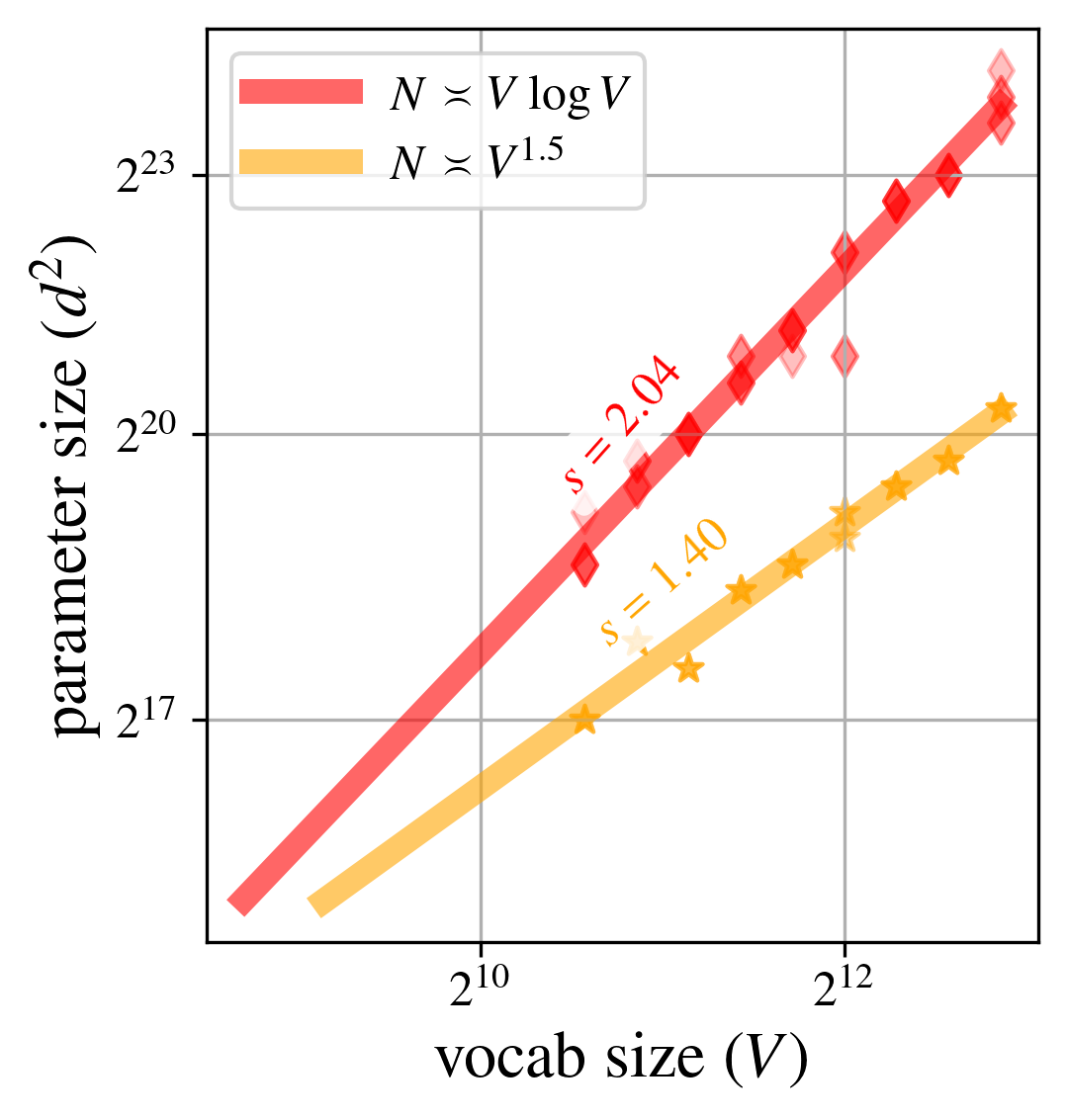}
        \caption{ Attention only, $L \asymp V$}
        \label{fig:attentiononlyplotb}
    \end{subfigure} 
    \end{minipage}
    \vspace{-1.5mm}
\caption{\small Empirical scaling of parameter size via three-step GD for the \emph{Attention-only} model. In (a) and (b), top-left and top-right use $L \asymp V^{0.5}$; bottom-left and bottom-right use $L \asymp V$. In (c), we compare the parameter counts from (a) and (b) for the $L \asymp V$ case under two sample-size regimes, $N \asymp V \log V$ and $N \asymp V^{1.5}$. \emph{Line fitting:} We identify in the heatmaps the smallest embedding dimension that achieves accuracies $\{0.1, 0.125, 0.15\}$ and perform a least squares fit. The slopes of the fitted lines and their theoretical counterparts are reported on the heatmaps.  
}
    \label{fig:attentiononlyplot}
    \vspace{-2mm}
\end{figure}

\vspace{-1.5mm}
\subsection{Attention-only Model}

\vspace{-1.5mm}
We start with the \emph{Attention-only model} which gives a simpler phase diagram. 
\begin{corollary}
\label{cor:attentiononly}
For the \emph{Attention-only} model, the bottleneck term in \eqref{eq:mainresult} is the \emph{Mean bias} term. Therefore,   Theorem \ref{thm:main} is equivalent to the parameter size requirement $d^2 \gtrsim \max \{ V, V^{2/3} L^{8/3} / N^{4/3} \}$.
\end{corollary}
We make the following observations:
\begin{itemize}[leftmargin=*,itemsep=0.2em]
\item The condition in Corollary \ref{cor:attentiononly} is the maximum of two terms, where $d^2 \gtrsim V$ is due to the capacity requirement in Assumption \ref{ass:condsattentiononly}, whereas the second term ensures $\emph{Signal} \gtrsim \emph{Mean bias}$ and implies a multiplicative scaling between the sample size $N$ and embedding dimension $d$ (i.e., increasing one of the parameters can compensate for the other).
\item Note that the \emph{Mean bias} term arises from a nonzero token mean, which can potentially be alleviated by centering the tokens, for instance through an appropriate normalization layer. Exploring the effect of applying normalization in this model is an interesting direction for future work.
\end{itemize}

\paragraph{Empirical Findings.}
We run the three-step gradient descent algorithm  on an \emph{Attention-only} model over varying $V$ and $d$, and report the accuracies in the heatmaps (Figure \ref{fig:attentiononlyplot}). The plots are in log-log scale; therefore, the slopes give the exponent $s$ in $d \asymp V^s$. As shown in the top row of Figures \ref{fig:attentiononlyplotsmallsample}-\ref{fig:attentiononlyplotlargesample}, the slope for relatively small $L$ (where $L \asymp V^{0.5}$) matches the optimal capacity condition $d^2 \asymp V$. 
By contrast, when the context window is larger ($L \asymp V$), the requirement becomes $d\asymp V$, which is also reflected in the experimental results, as observed in the bottom panel of Figure \ref{fig:attentiononlyplotsmallsample}.

In Figure \ref{fig:attentiononlyplotlargesample} we run experiments with increasing sample size to observe the multiplicative trade-off. As seen in the bottom figure of Figure \ref{fig:attentiononlyplotlargesample}, increasing the sample size from $V \log V$ to $V^{1.5}$ reduces the exponent of the parameter size  from $2.02$ to $1.42$ (the theoretical value is $s = 1.32$). 
Finally, the learnability thresholds for $L \asymp V$ in Figures \ref{fig:attentiononlyplotsmallsample} and \ref{fig:attentiononlyplotlargesample} are plotted together in Figure \ref{fig:attentiononlyplotb}, to illustrate that increasing the sample size can compensate for the number of parameters in the network.
 
\vfill
\begin{figure}[!hbt]
    \centering 
        \begin{minipage}[b]{.43\textwidth}
      \vspace{0em}
      \centering
    \begin{subfigure}[b]{.49\linewidth}
        \centering
        \includegraphics[width=\linewidth]{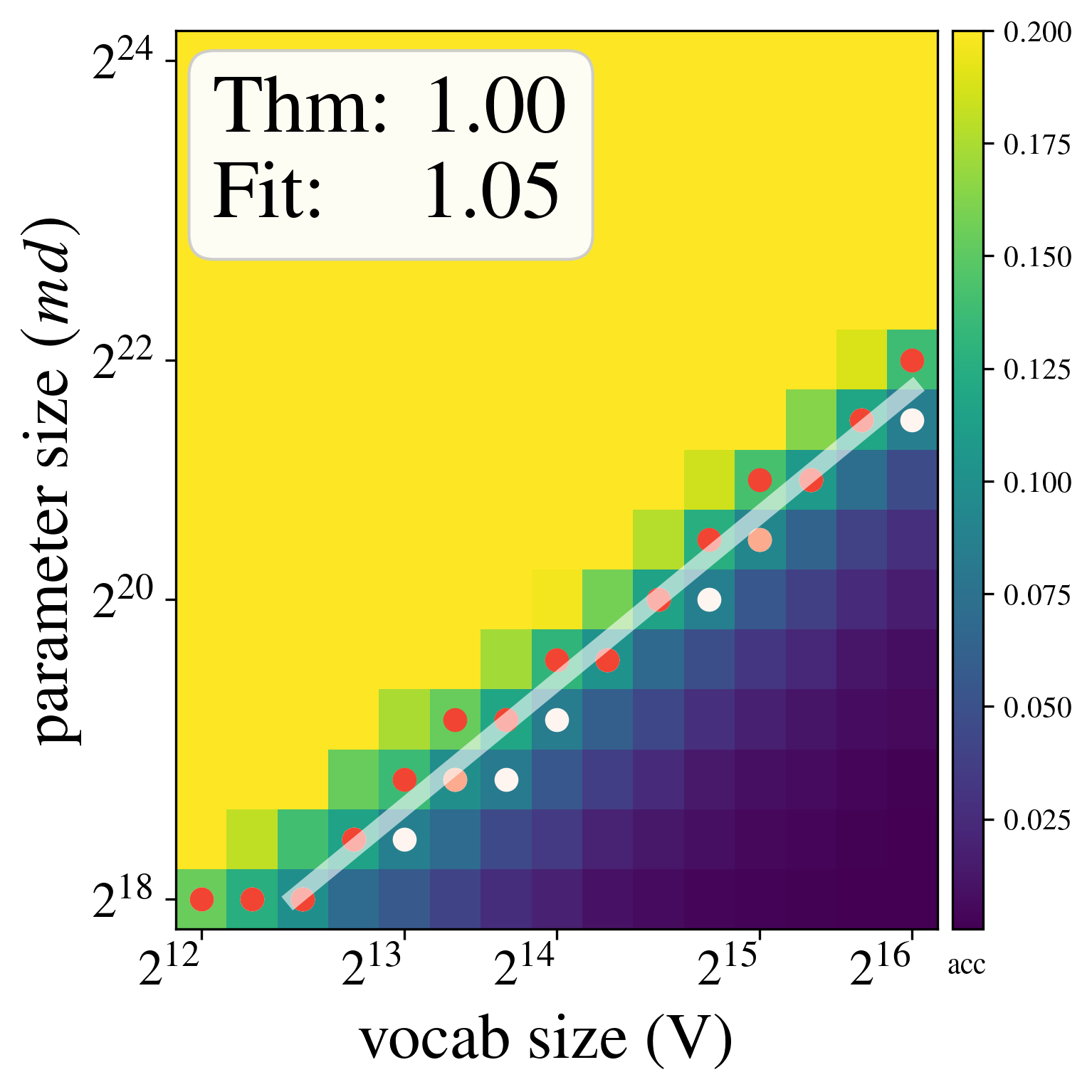}
    \end{subfigure} 
     \begin{subfigure}[b]{.49\linewidth}
        \centering
        \includegraphics[width=\linewidth]{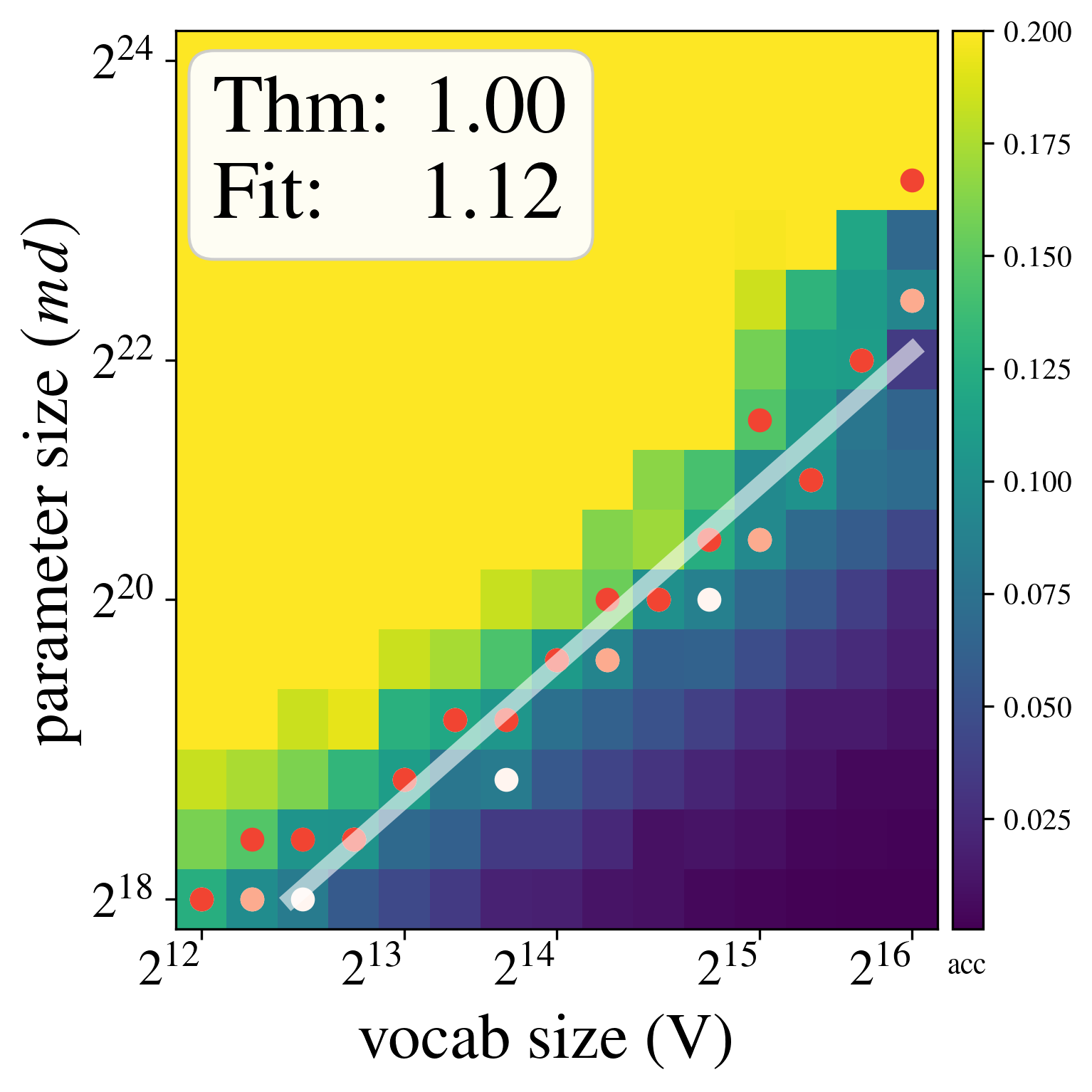}
    \end{subfigure}  \\
      \begin{subfigure}[b]{.49\linewidth}
        \centering
        \includegraphics[width=\linewidth]{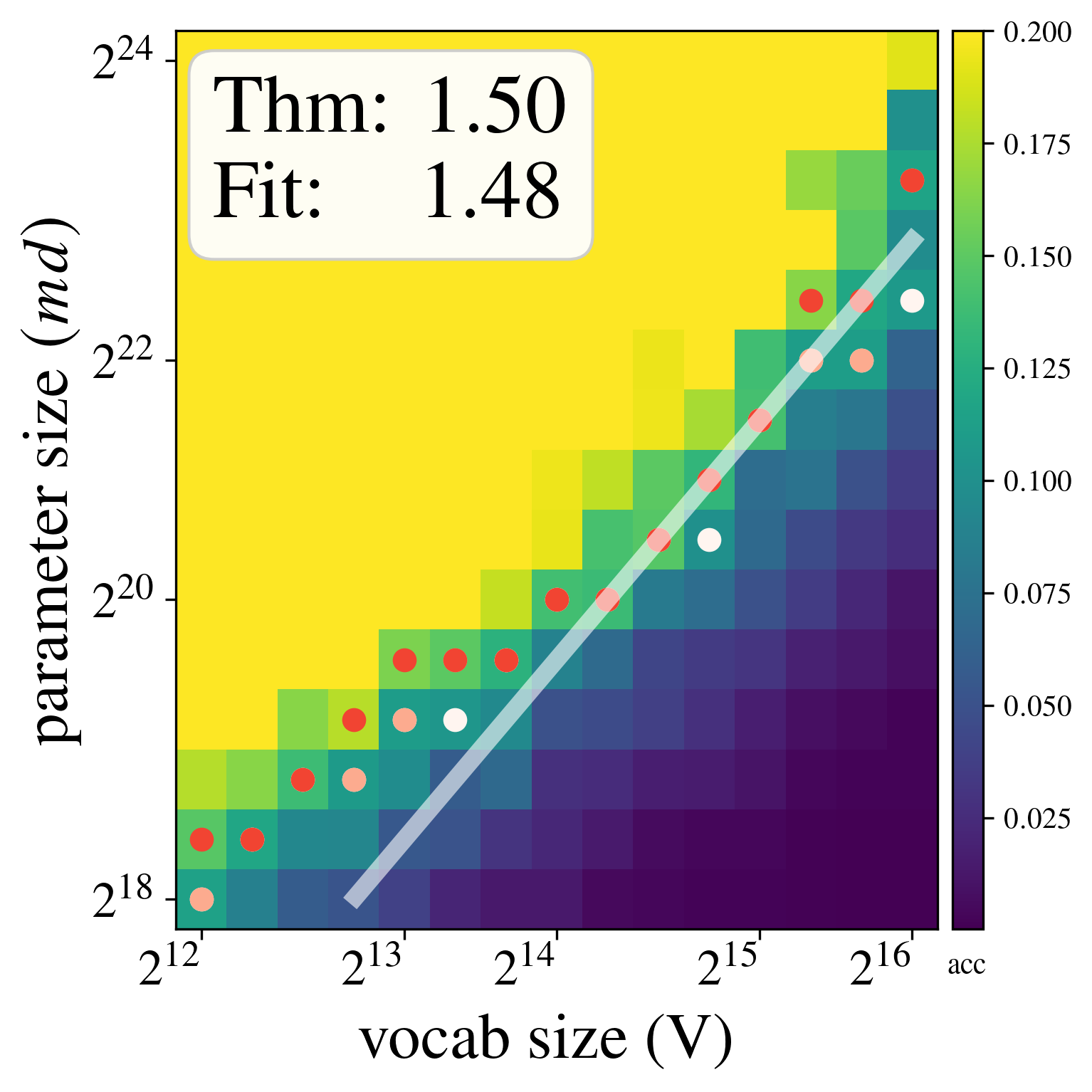}
        \caption{ $N \asymp V \log V$}
        \label{fig:attentionmlps2} 
    \end{subfigure}    
    \begin{subfigure}[b]{.49\linewidth}
        \centering
        \includegraphics[width=\linewidth]{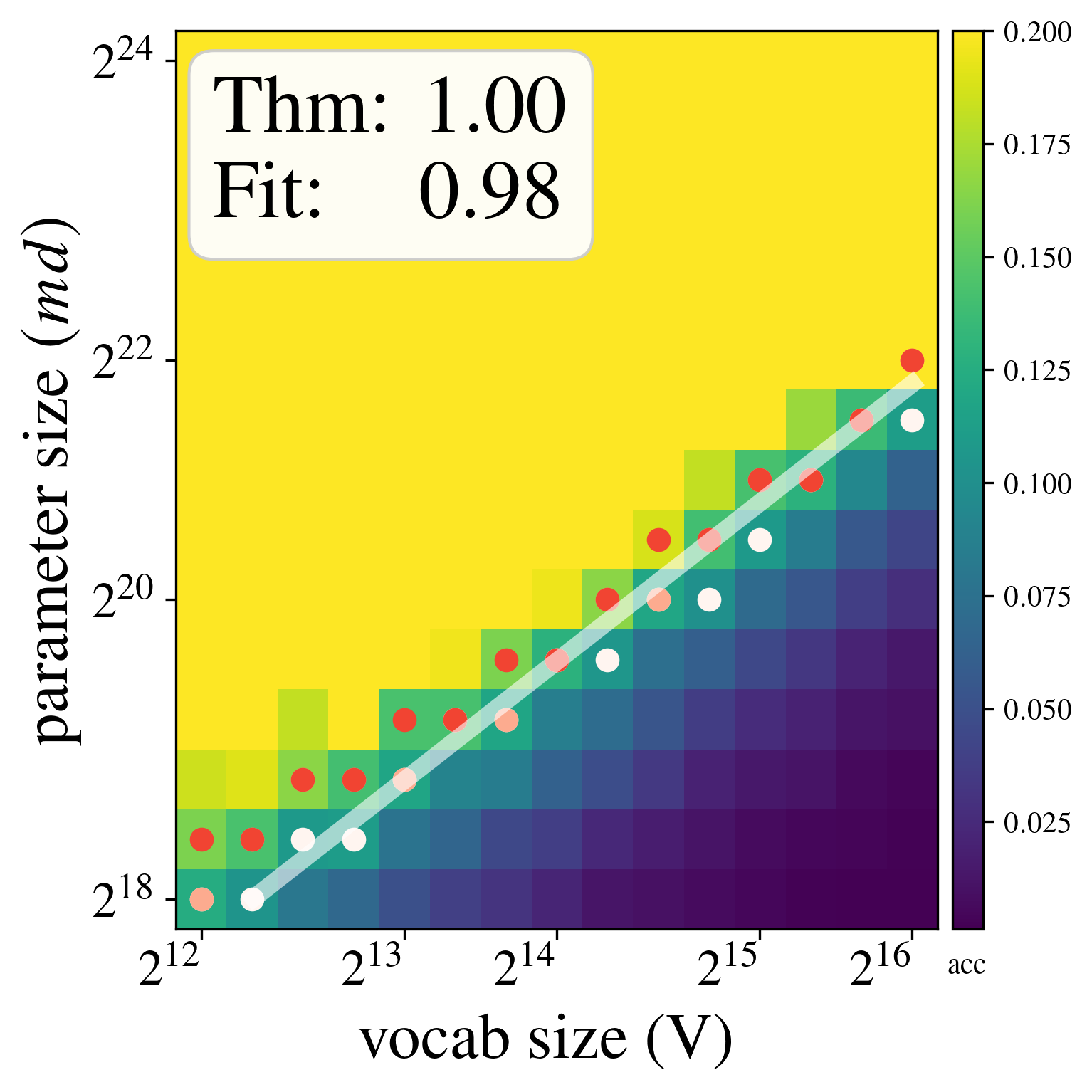}
        \caption{ \small    $N \asymp V^{1.5}$}
        \label{fig:attentionmlps3} 
    \end{subfigure} 
     \end{minipage}%
          \begin{minipage}[b]{.45\textwidth}
    \vspace{0pt}
    \centering
       \begin{subfigure}[b]{.9\linewidth}
        \centering
        \includegraphics[width=\textwidth]{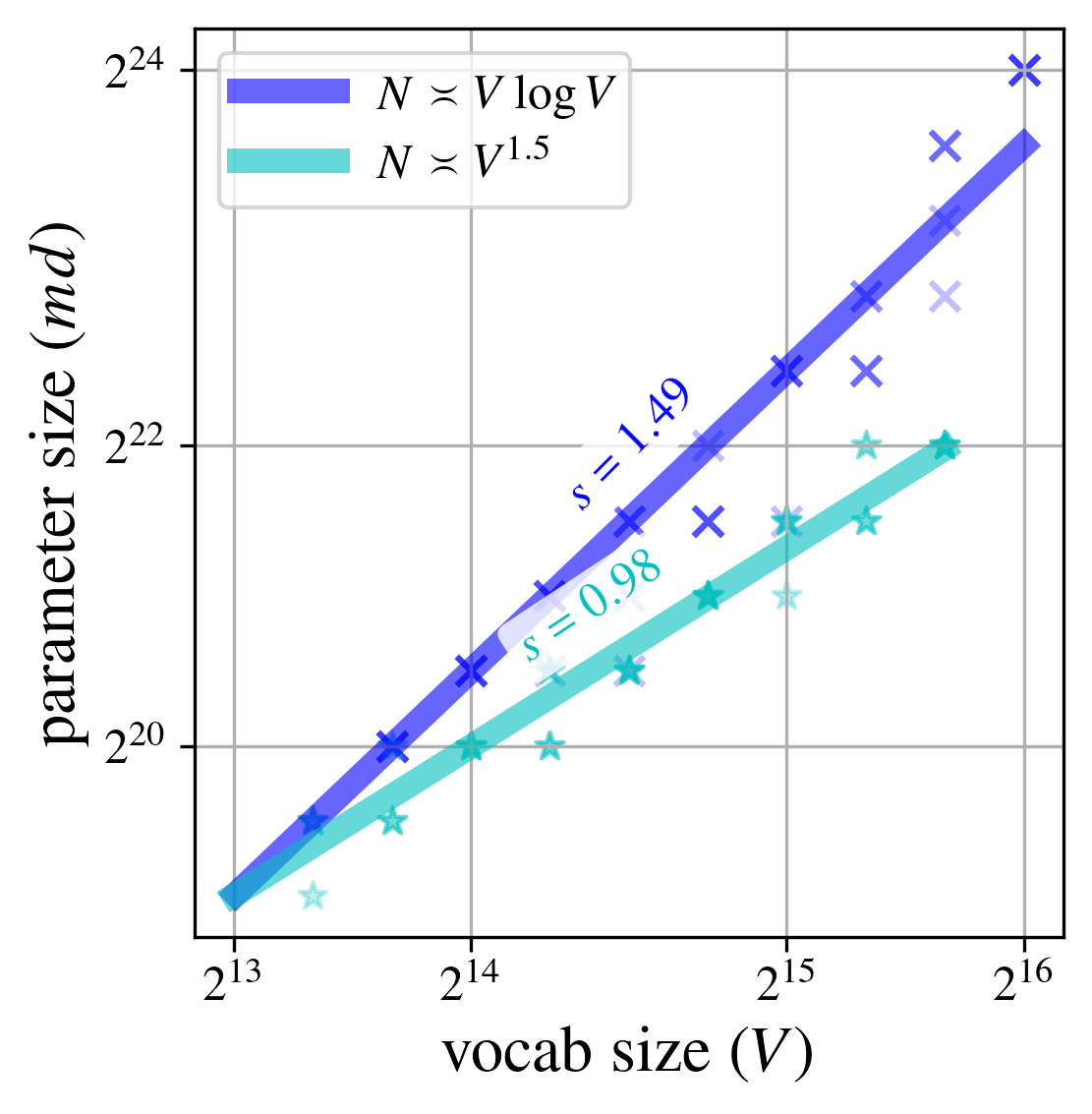}
        \caption{$L \asymp V^{0.5}$ and $m \asymp d^2$}
        \label{fig:attentionmlps1} 
    \end{subfigure} 
    \end{minipage}
        \vspace{-1mm}
  \caption{\small Empirical scaling of  parameter size for the \emph{Attention-MLP} model under two sample size regimes, $N \asymp V \log V$ and $N \asymp V^{1.5}$. In (a) and (b), top-row uses $L \asymp V^{0.25}$; bottom-row uses $L \asymp V^{0.5}$. In (c), we compare the parameter counts from (a) and (b) for the $L \asymp V^{0.5}$ case under both sample-size regimes.}
  \label{fig:attentionmlplargesample}
\end{figure}

\begin{figure}[!hbt]
 \vspace{2mm}
    \centering 
         \begin{minipage}[b]{.43\textwidth}
      \vspace{0pt}
      \centering
    \begin{subfigure}[b]{.49\linewidth}
        \centering
        \includegraphics[width=\linewidth]{_figures/attention_mlp_small_sample_LV025_v2.png}
    \end{subfigure} 
     \begin{subfigure}[b]{.49\linewidth}
        \centering
        \includegraphics[width=\linewidth]{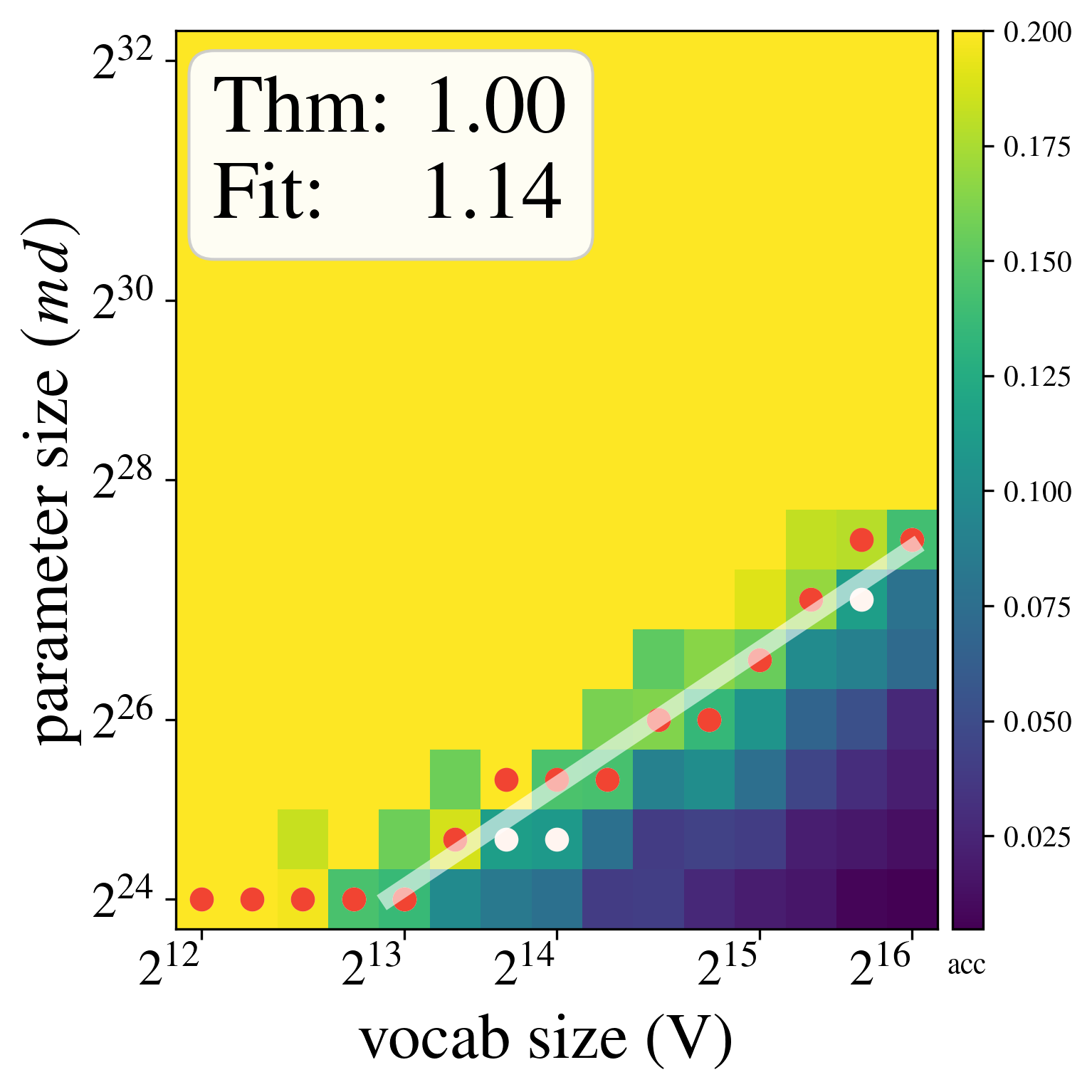}
    \end{subfigure}  \\
      \begin{subfigure}[b]{.48\linewidth}
        \centering
        \includegraphics[width=\linewidth]{_figures/attention_mlp_small_sample_LV05_v2.png}
        \caption{$m \asymp d^2$}
        \label{fig:attentionmlpw2} 
    \end{subfigure}    
    \begin{subfigure}[b]{.49\linewidth}
        \centering
        \includegraphics[width=\linewidth]{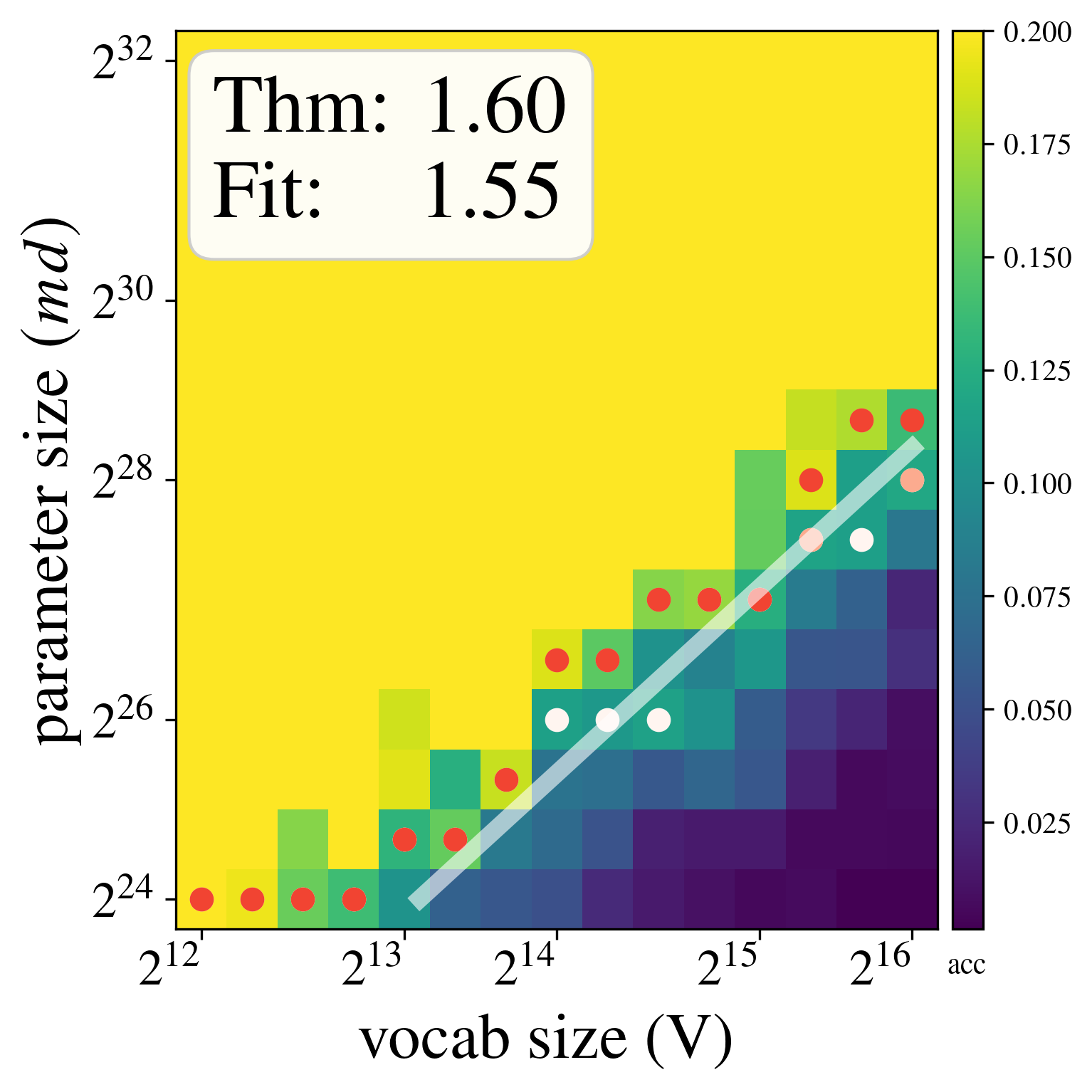}
        \caption{   $m \asymp d^{3}$}
        \label{fig:attentionmlpw3} 
    \end{subfigure} 
     \end{minipage} 
        \begin{minipage}[b]{.45\textwidth}
    \vspace{0pt}
    \centering
       \begin{subfigure}[b]{.9\linewidth}
        \centering
        \includegraphics[width=\textwidth]{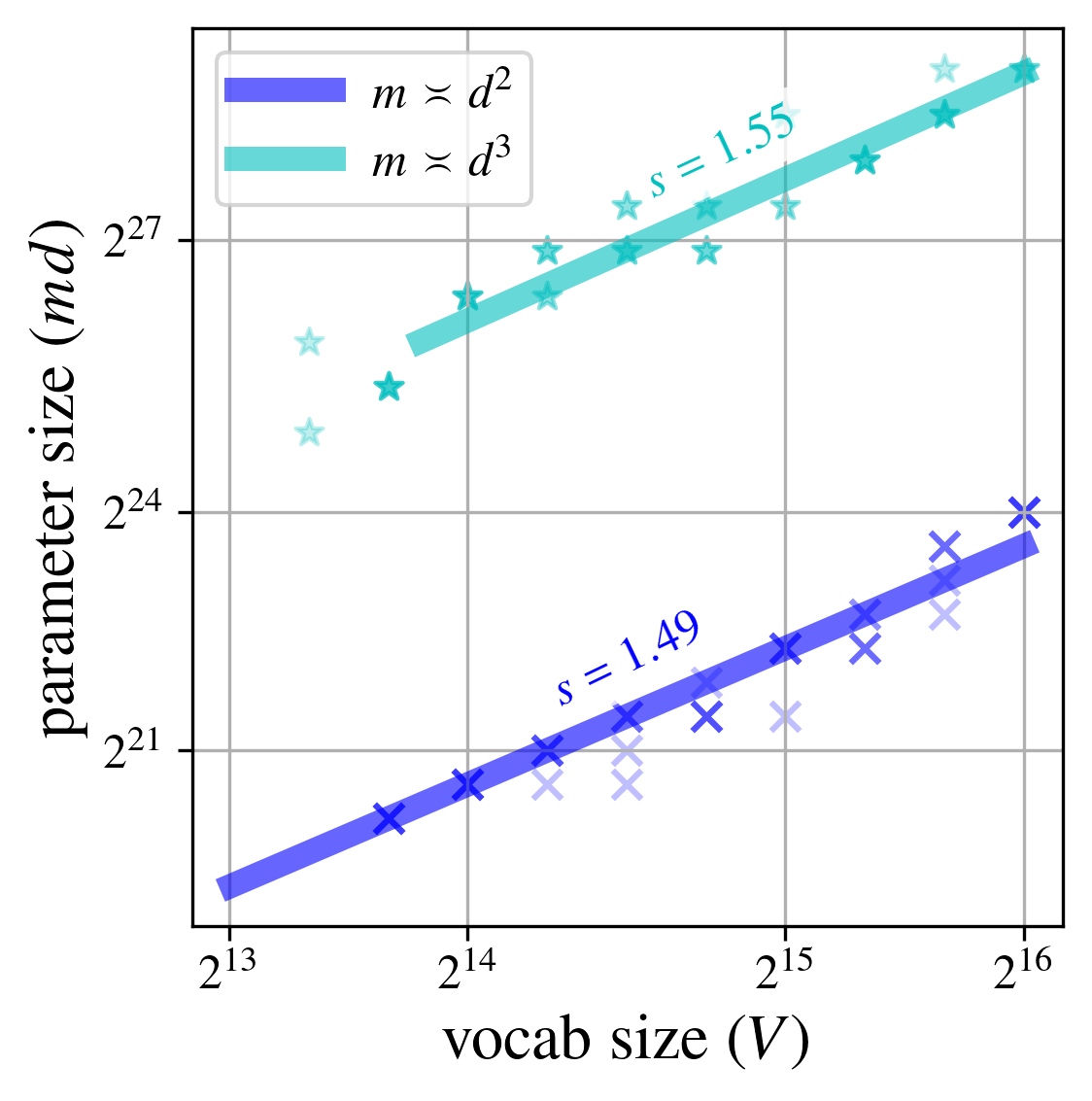}
        \caption{$L \asymp V^{0.5}$ and $N \asymp V \log V$}
        \label{fig:attentionmlpw1} 
    \end{subfigure} 
    \end{minipage}
     \vspace{-1mm}
      \caption{\small Empirical scaling of embedding  parameter size for the \emph{Attention-MLP} model under two width regimes, $m \asymp d^2$ and $m \asymp d^3$. In (a) and (b), top-row uses $L \asymp V^{0.25}$; bottom-row uses $L \asymp V^{0.5}$. In (c), we compare the parameter counts from (a) and (b) for the $L \asymp V^{0.5}$ case under both width regimes. }
    \label{fig:attentionmlplargewidth} 
\end{figure}

\newpage
\vspace{-3mm} 
\subsection{Attention-MLP Model} 
\label{sec:attnmlp}
\vspace{-1mm} 
For the attention-MLP model, the nonlinear MLP layer introduces additional phases as stated below.    

\begin{corollary}
\label{cor:attentionmlp}
For the \emph{Attention-MLP} model, 
 Theorem 1 translates to $m d \gtrsim V$ and
\eq{
\emph{Signal} \gtrsim 
\begin{cases}
\emph{MLP noise}, &    m = o(d^2 L) ~~ \text{and} ~~ m = o(dV)  \\
\emph{Gradient noise}, & V \gtrsim dL ~~ \text{and} ~~ m \gtrsim d^2 L \\
\emph{Mean bias}, &  V = o(dL) ~~ \text{and} ~~ m \gtrsim dV,
\end{cases}
}
where
\begin{itemize}[leftmargin =*,topsep=0.15mm,itemsep=0.05mm]
    \item  $\emph{Signal} \gtrsim  \emph{MLP noise}$ is equivalent to  $m d \gtrsim  V \frac{\sqrt{m} L^2}{N}$.
    \item  $\emph{Signal} \gtrsim  \emph{Gradient noise}$  is equivalent to   $m d \gtrsim  V  \frac{m L^{\frac{1}{4}}}{\sqrt{N} }$
    \item  $\emph{Signal} \gtrsim  \emph{Mean bias}$  is equivalent to  $m d  \gtrsim \frac{m L^{\frac{4}{3}} V^{\frac{1}{3}} }{N^{\frac{2}{3}}}$.
\end{itemize}
\end{corollary}

The phase diagram for the Attention-MLP model is visualized in Figure~\ref{fig:introphase}. Compared to the \emph{Attention-only} case, it exhibits additional regimes because we can trade off $m$ and $d$ and thus use a smaller embedding dimension; this can lead to different dominant terms in the gradient. In particular, since large $L$ and $d$ entail a larger magnitude of the \emph{Mean bias} (as in the \emph{Attention-only} setting), increasing the MLP width $m$ and thereby reducing the required embedding dimension $d$ may suppress this bias term.

\paragraph{Empirical Findings.} We run the 3-step gradient descent algorithm on an Attention-MLP network over varying $V$ and $d$ and plot the accuracies in Figures \ref{fig:attentionmlplargesample} and \ref{fig:attentionmlplargewidth}. We take the nonlinearity to be the mixture of two Hermite polynomials $\phi = 0.7 h_2 + 0.3 h_3$, satisfying the conditions in Assumption \ref{ass:condsattentionmlp}. We run experiments with width $m \asymp d^2$ and $m \asymp d^3$. Due to the prohibitive cost of increasing the width further, we restrict ourselves to the \emph{MLP noise}-dominated region.

In Figure \ref{fig:introtradeoff}, we plot the scaling of the number of parameters ($md$) as a function of vocabulary size $V$ for different sequence-length regimes in $L$. We observe that $L \asymp V^{0.25}$ requires $m d \asymp V$, which is the optimal capacity, as predicted by our theory. As $L$ increases, we need more parameters to achieve the same capacity, as observed in the $L \asymp V^{0.5}$ and $L \asymp V^{0.75}$ cases in Figure \ref{fig:intro}, where the slopes agree with our theoretical predictions as well (see also Figures \ref{fig:attentionmlps2} and \ref{fig:attentionmlps3}). 

We further test the effect of sample size in Figure \ref{fig:attentionmlplargesample}, where we use $L \asymp V^{0.5}$ and $m \asymp d^2$. We plot both heat maps in Figures \ref{fig:attentionmlps2} and \ref{fig:attentionmlps3}, and the fitted lines for $L \asymp V^{0.5}$ together in Figure \ref{fig:attentionmlps1}.  We observe that increasing $N$ from $N \asymp V \log V$ to $N \asymp V^{1.5}$ reduces the network size to the optimal level, aligning with our theoretical prediction. The heatmap versions of these experiments are shown in Figures \ref{fig:attentionmlps2} and \ref{fig:attentionmlps3}.

Lastly, we probe the width scaling by keeping the sample size $N \asymp V \log V$ and $L \asymp V^{0.5}$ fixed in Figure \ref{fig:attentionmlplargewidth}. Here, we observe that we can reduce the embedding-dimension requirement by increasing $m$ in Theorem \ref{thm:main}, although it increases the total parameter count overall, as seen in Figures \ref{fig:attentionmlpw3} and \ref{fig:attentionmlpw1}, since width must grow proportionally more than $d$ to achieve the same accuracy. This is also consistent with our result.

\vspace{-2mm}
\subsection{Beyond Early Phase of Training}
\label{sec:exps}

\vspace{-1mm} 
While our theoretical analysis focuses on a particular three-gradient-step training procedure, we empirically observe qualitatively similar multiplicative scalings when the transformer model is optimized beyond the ``early phase''. Specifically, we train our \emph{Attention-only} model using Adam~\citep{KingmaB14} with mini-batch gradients. In the experiments, we use layer normalization in both the attention and output layers and set the learning rate to $0.005$. We use a batch size of $\floor{N/2}$ (except in the last experiment, where we use $\floor{N/16}$), and run the training for $16$ epochs.  We highlight the following observations:

\vspace{-1mm} 
\begin{itemize}[leftmargin=*, itemsep=0.25em]
\item \emph{Capacity improvement with multi-pass training.}  
In the top row of Figure~\ref{fig:multistepadam}, we plot the heatmaps for $L \asymp V$ and $N \asymp V \log V$. In early training the slope is suboptimal;  notably, by the end of Epoch~1 it closely aligns with our theoretical prediction. Moreover, training the network additional epochs improves the capacity condition to a near-optimal level, as shown in Figures~\ref{fig:multistepadamc} and~\ref{fig:multistepadamd}.

\item  \emph{Effect of sample size.}  
In the bottom row of Figure~\ref{fig:multistepadam}, we plot the heatmaps for $L \asymp V$ and $N \asymp V^{1.5}$. We observe a similar trajectory in capacity, while the overall capacities improve compared to the small-sample regime,  showing the multiplicative dependence on  sample size $N$.

\item   \emph{Effect of sequence length.}  
In Figure~\ref{fig:multistepadamLV085}, we plot the heatmaps for $L \asymp V^{0.85}$ and $N \asymp V \log V$. We observe improvements in capacity over multiple epochs, while the capacity is larger than in the $L \asymp V$ setting at every stage of training, which shows the effect of the sequence length $L$.

\item   \emph{Effect of batch size.}  
In Figure~\ref{fig:multistepadamlb}, we repeat the experiments from this section using the same learning rate and architecture but with a smaller batch size $\floor{N/16}$. As before, we consider $L \asymp V$ in two sample-size regimes, $N \asymp V \log V$ and $N \asymp V^{1.5}$. We observe behavior similar to the larger batch  size setting, but with improved slopes in Figure~\ref{fig:multistepadamlb}. This suggests that smaller batch sizes may improve capacity in practice.

\end{itemize}

\vspace{-1mm}
Overall, these experiments suggest that the multiplicative relation between the hyperparameters remains throughout training. However, the  exponents depend on the iteration number and batch size. 
Understanding how capacity evolves during training remains an interesting open question.

\vfill
\begin{figure}[!hbt]
    \centering 
      \begin{subfigure}[b]{.24\linewidth}
        \centering
        \includegraphics[width=0.85\textwidth]{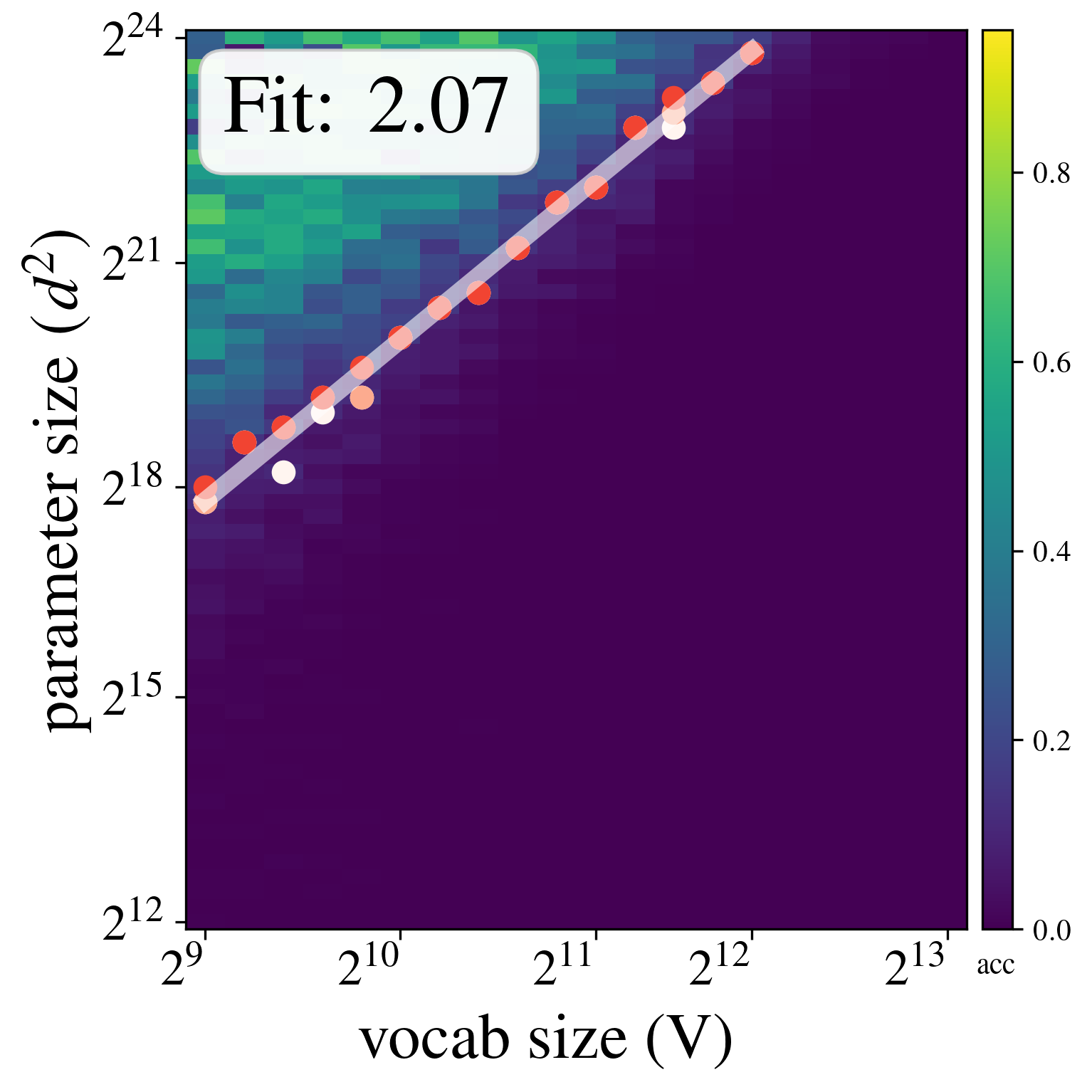}
        \caption{\small Epoch 1}
        \label{fig:multistepadama} 
    \end{subfigure} 
     \begin{subfigure}[b]{.24\linewidth}
        \centering
        \includegraphics[width=0.85\textwidth]{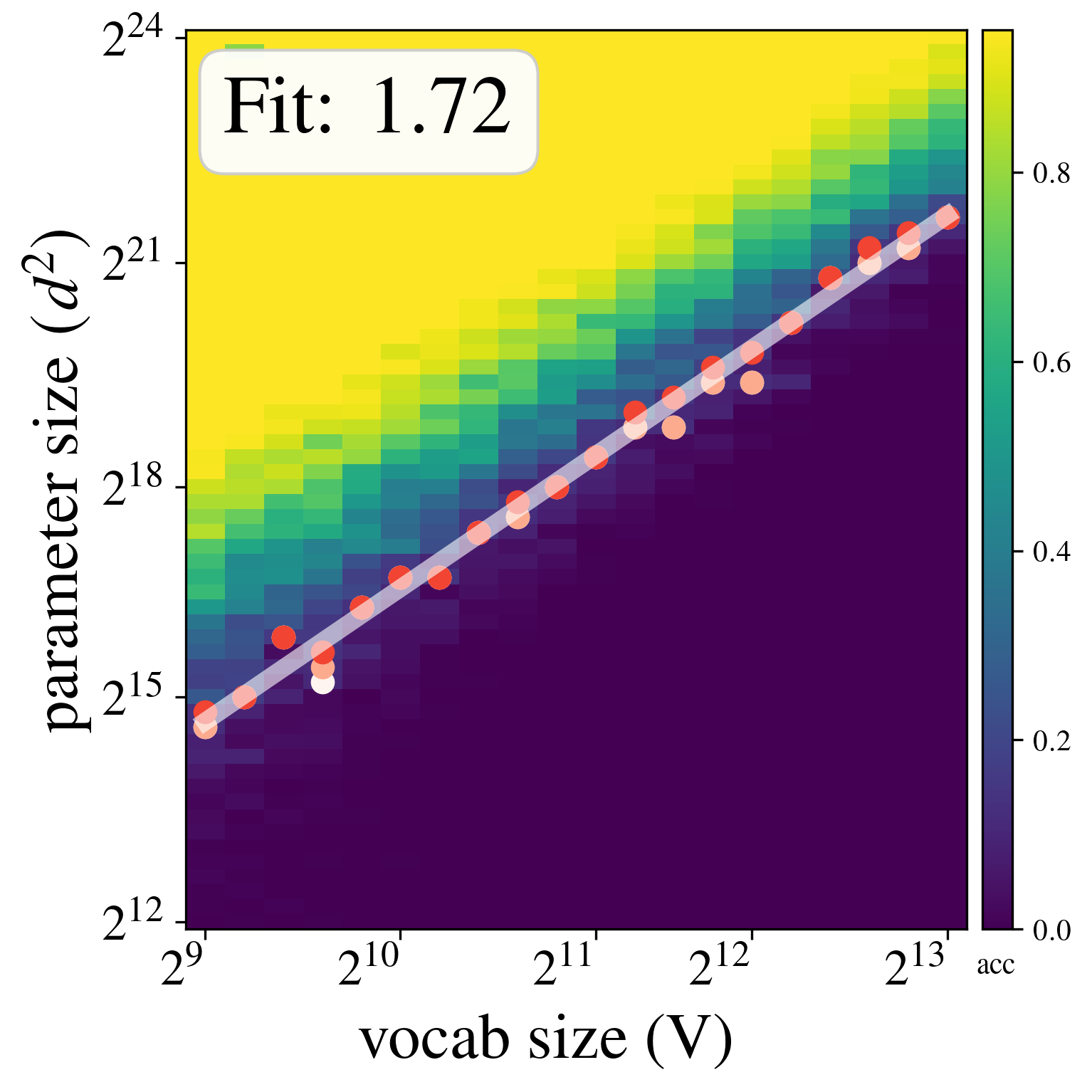}
        \caption{\small Epoch 2}
        \label{fig:multistepadamb}  
    \end{subfigure} 
     \begin{subfigure}[b]{.24\linewidth}
        \centering
        \includegraphics[width=0.85\textwidth]{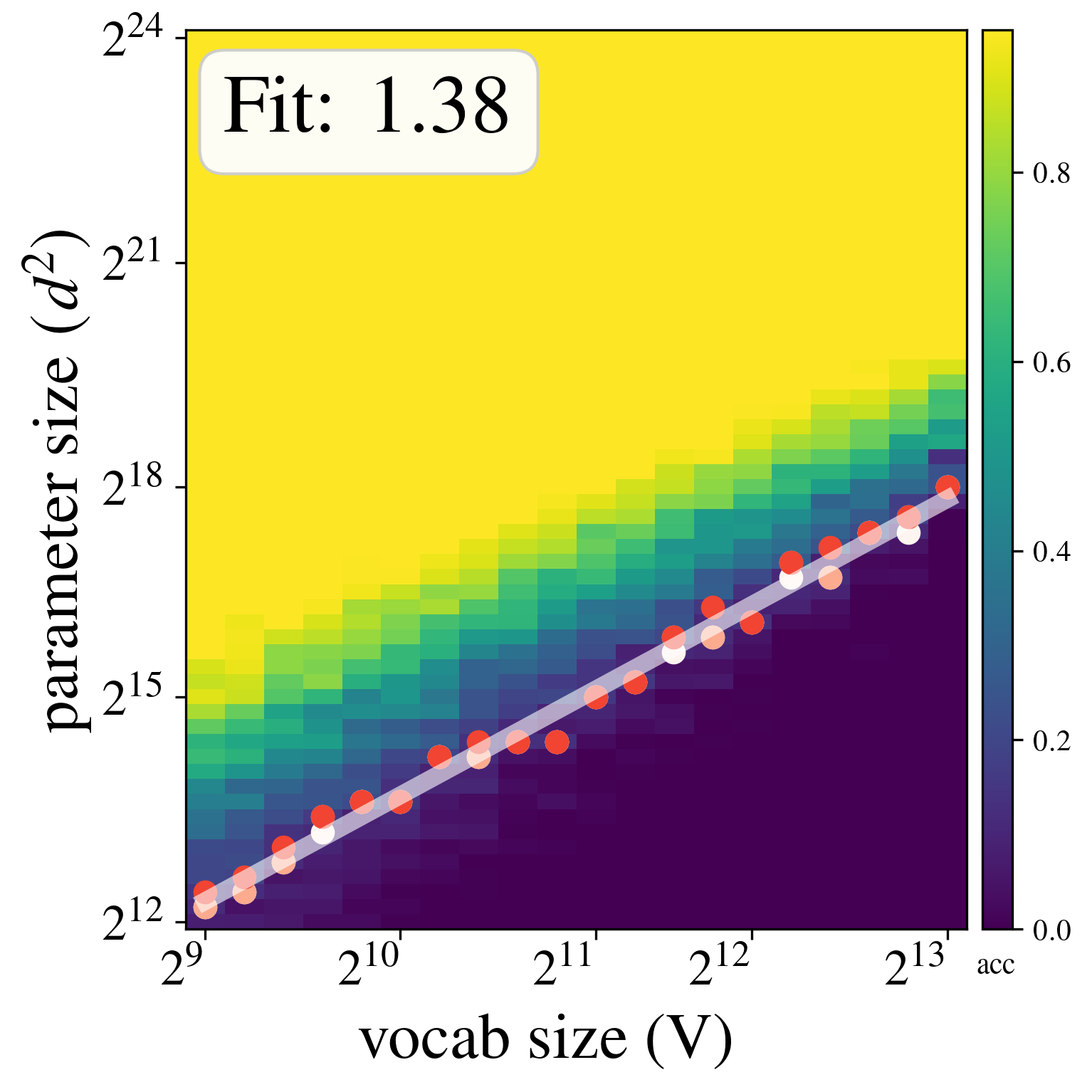}
        \caption{\small Epoch 8}
        \label{fig:multistepadamc} 
    \end{subfigure} 
    \begin{subfigure}[b]{.24\linewidth}
        \centering
        \includegraphics[width=0.85\textwidth]{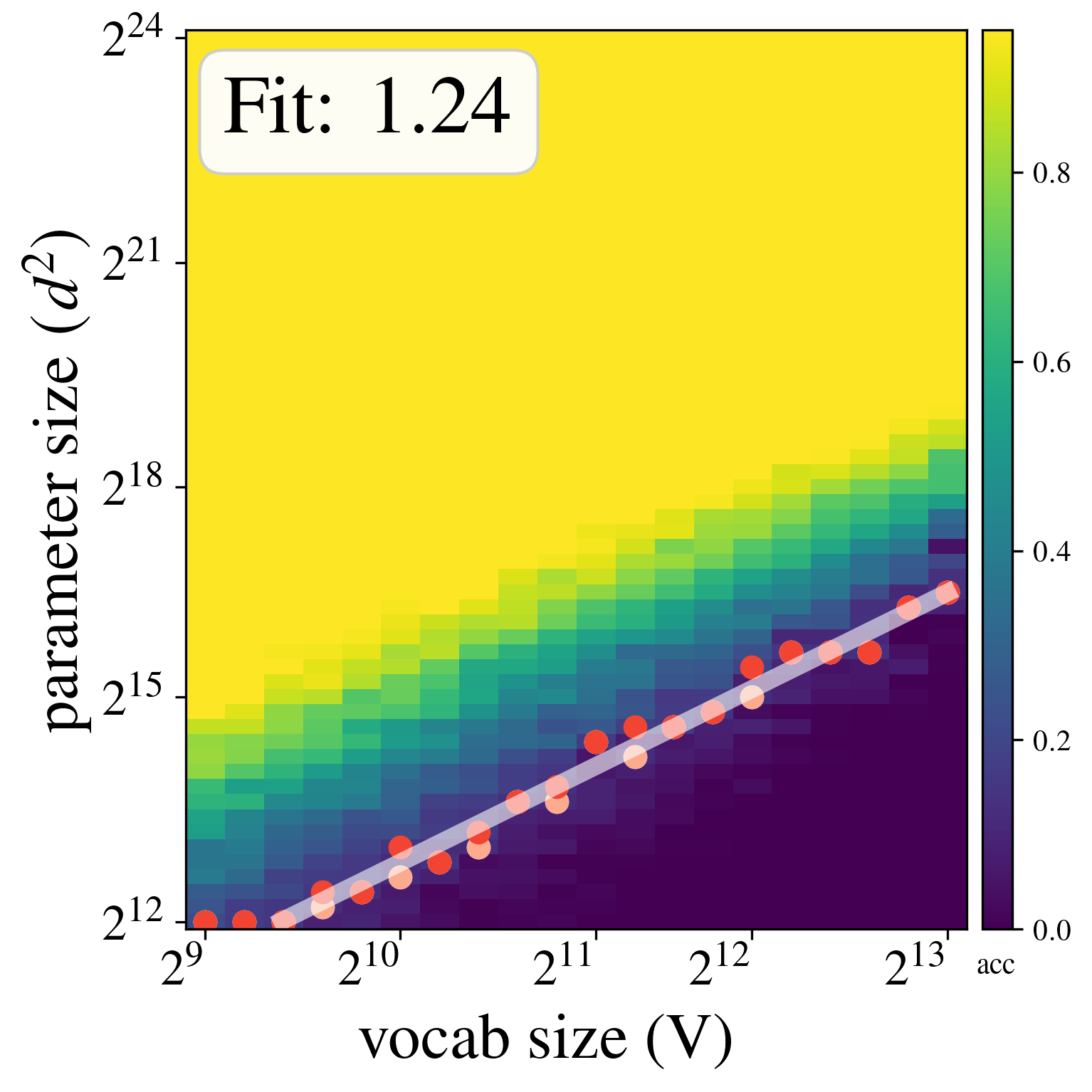}
        \caption{\small Epoch 16}
        \label{fig:multistepadamd} 
    \end{subfigure}  \\[0.25em]
    \begin{subfigure}[b]{.24\linewidth}
        \centering
        \includegraphics[width=0.85\textwidth]{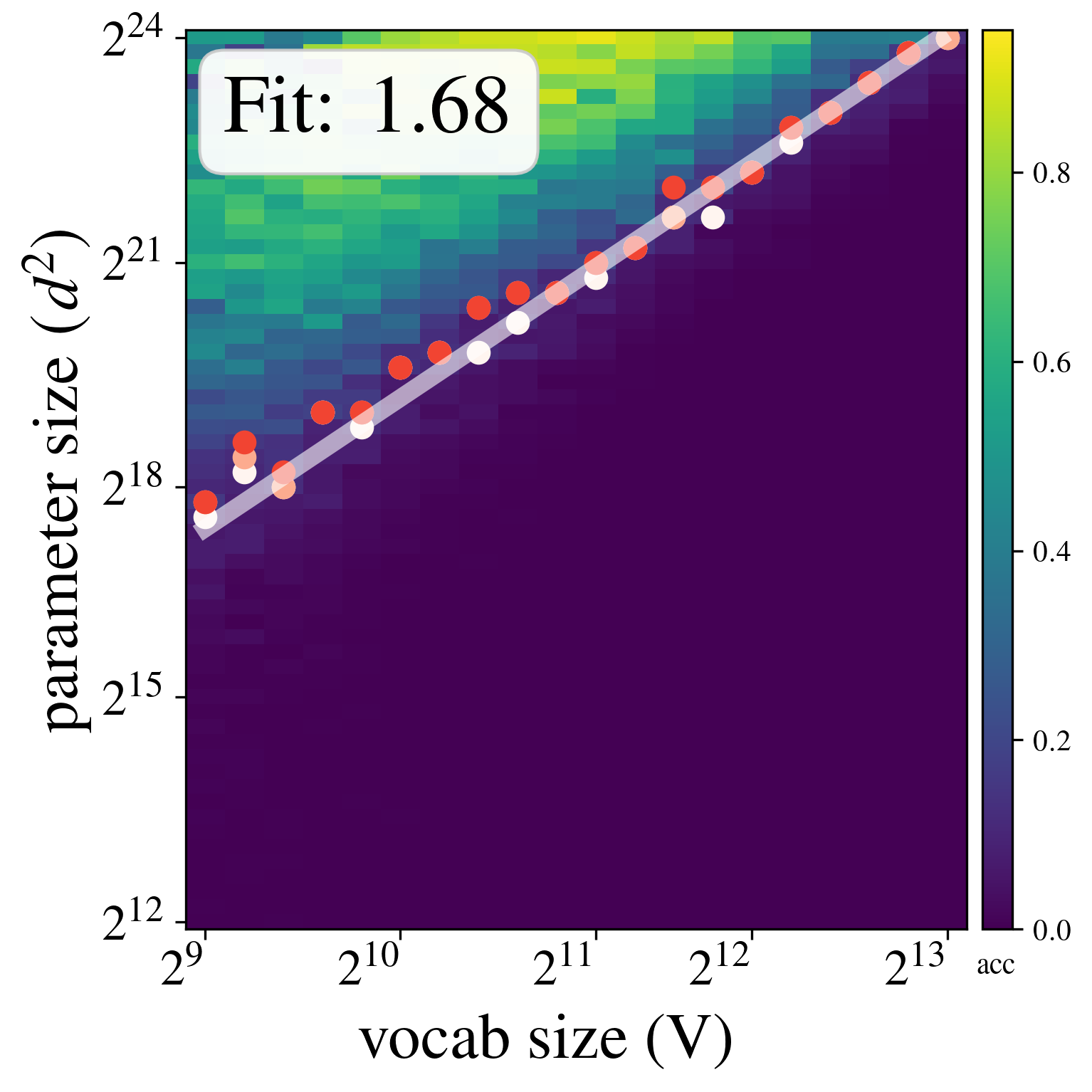}
        \caption{\small Epoch 1}
        \label{fig:multistepadame} 
    \end{subfigure} 
     \begin{subfigure}[b]{.24\linewidth}
        \centering
        \includegraphics[width=0.85\textwidth]{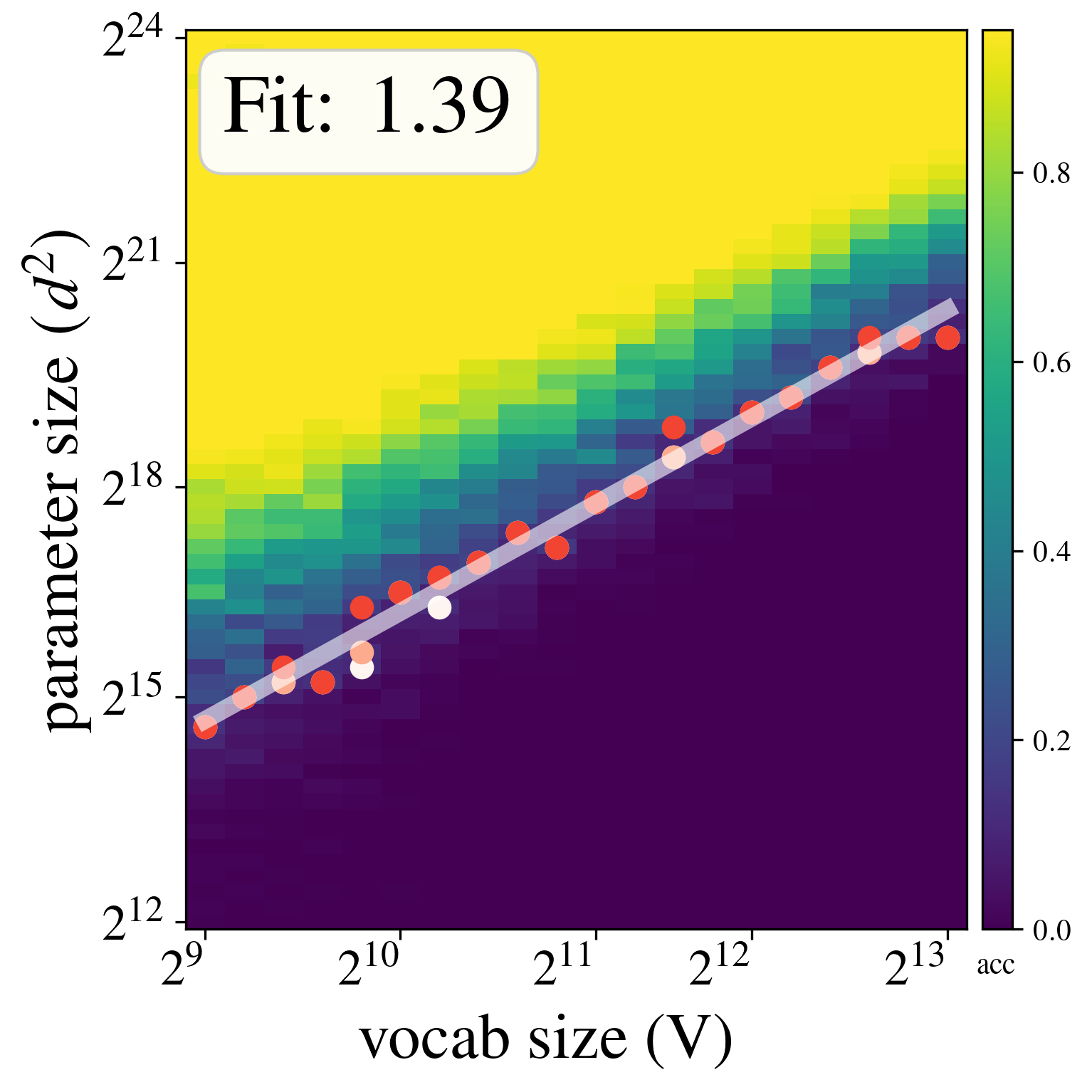}
        \caption{\small  Epoch 2}
        \label{fig:multistepadamf} 
    \end{subfigure} 
     \begin{subfigure}[b]{.24\linewidth}
        \centering
        \includegraphics[width=0.85\textwidth]{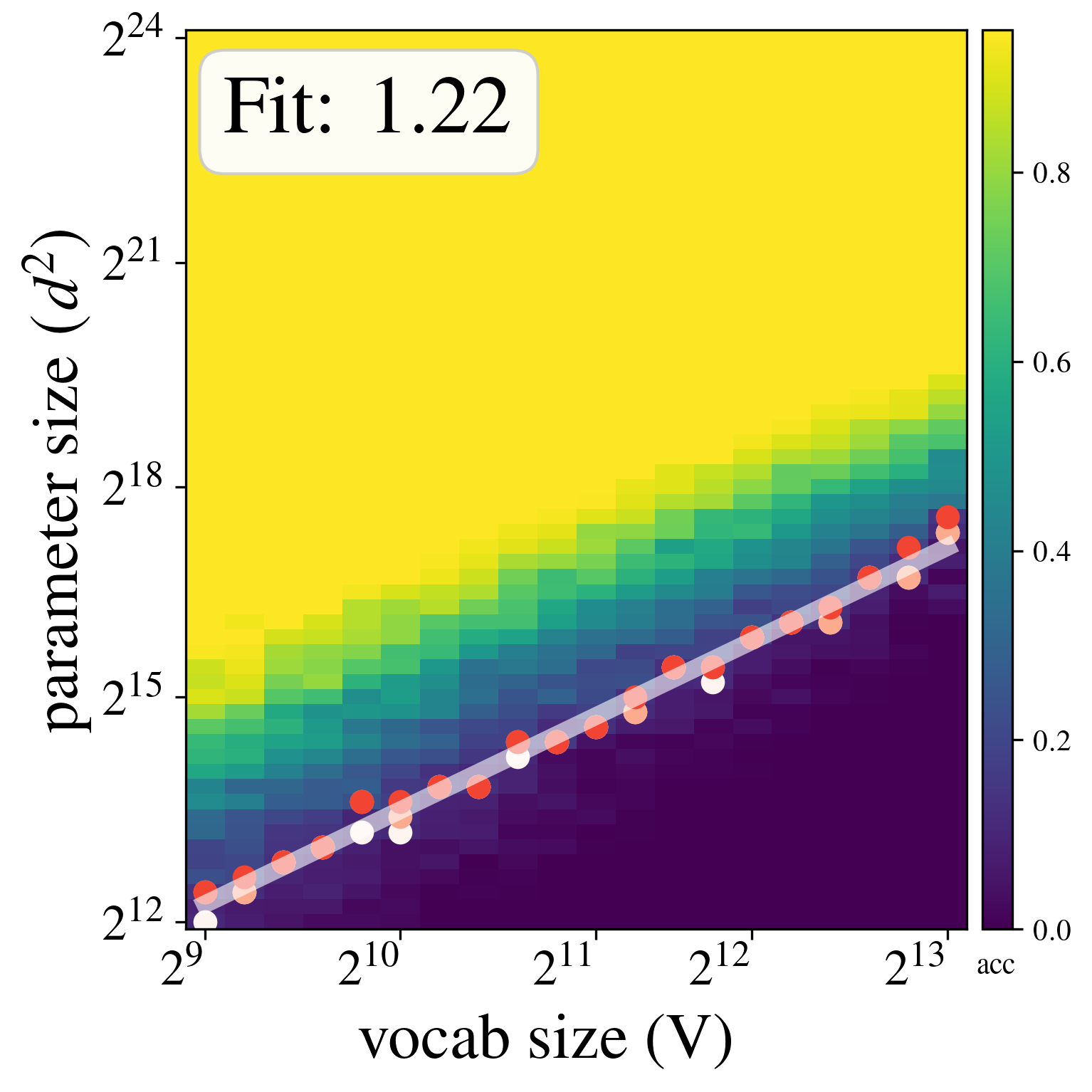}
        \caption{\small Epoch 8}
        \label{fig:multistepadamg} 
    \end{subfigure} 
    \begin{subfigure}[b]{.24\linewidth}
        \centering
        \includegraphics[width=0.85\textwidth]{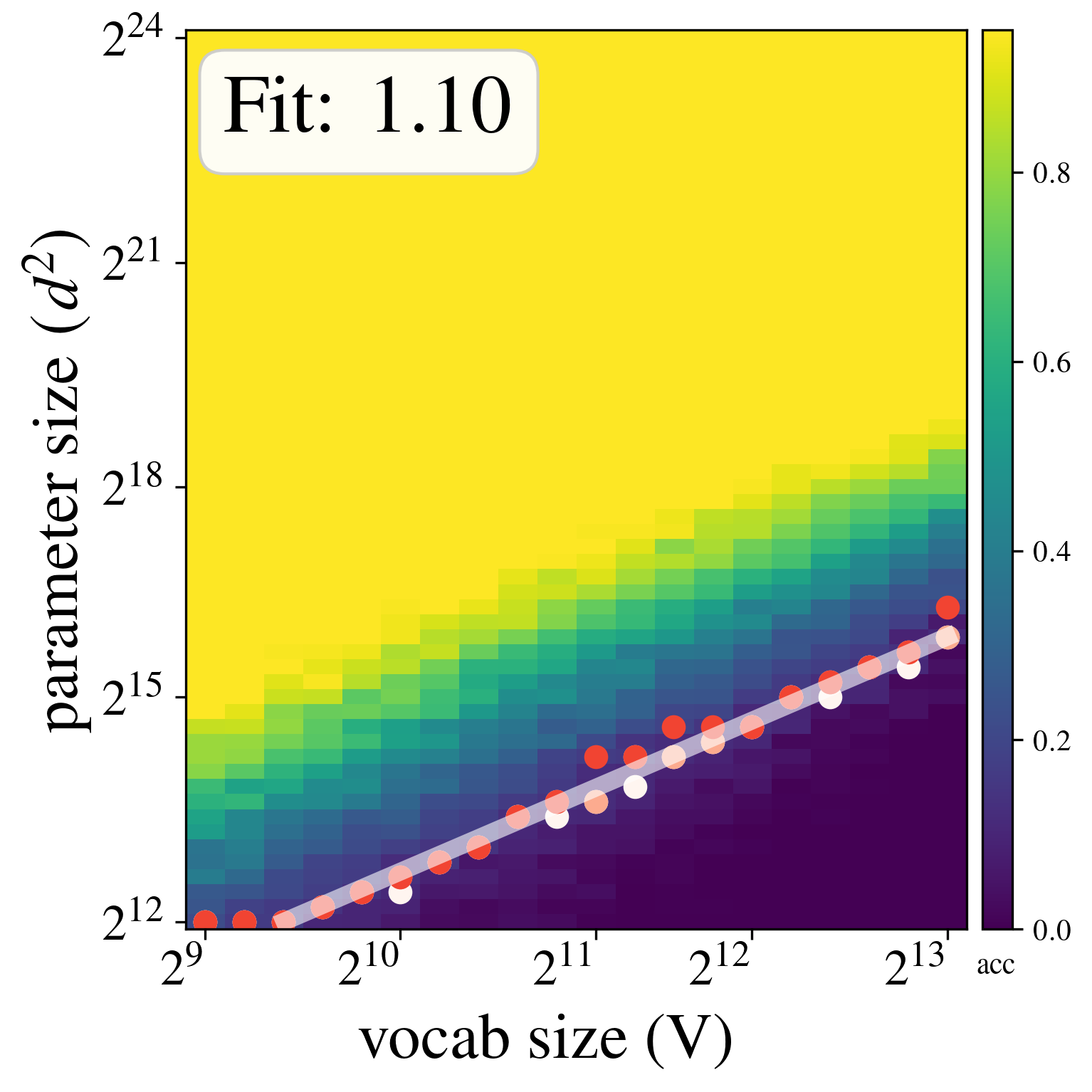}
        \caption{\small Epoch 16}
        \label{fig:multistepadamh} 
    \end{subfigure} 
    \vspace{-2mm}
     \caption{\small Empirical scaling of the parameter size for the \emph{Attention-only} model under two sample size regimes. \textit{Top row (a--d):} $N \asymp V \log V$. \textit{Bottom row (e--h):} $N \asymp V^{1.5}$. The model uses $L \asymp V$ and is trained using Adam over $16$ epochs. We observe that  the capacity improves as the number of epochs increases.}
    \label{fig:multistepadam} 
    
\end{figure}
\begin{figure}[!hbt]
 \vspace{3mm}
    \centering 
      \begin{subfigure}[b]{.24\linewidth}
        \centering
        \includegraphics[width=0.85\textwidth]{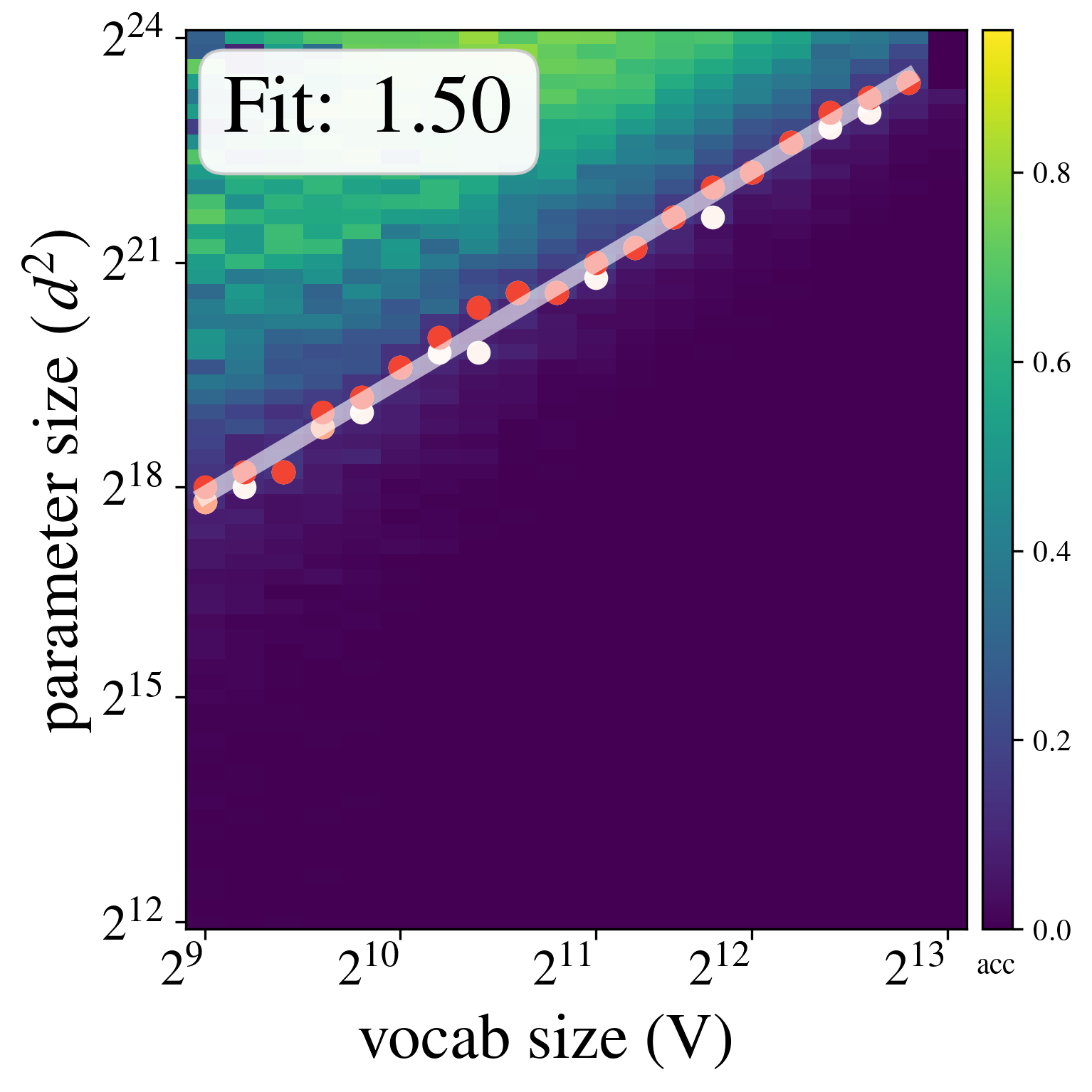}
        \caption{ $\mathrm{Epoch} = 1$}
        \label{fig:multiadam3} 
    \end{subfigure} 
     \begin{subfigure}[b]{.24\linewidth}
        \centering
        \includegraphics[width=0.85\textwidth]{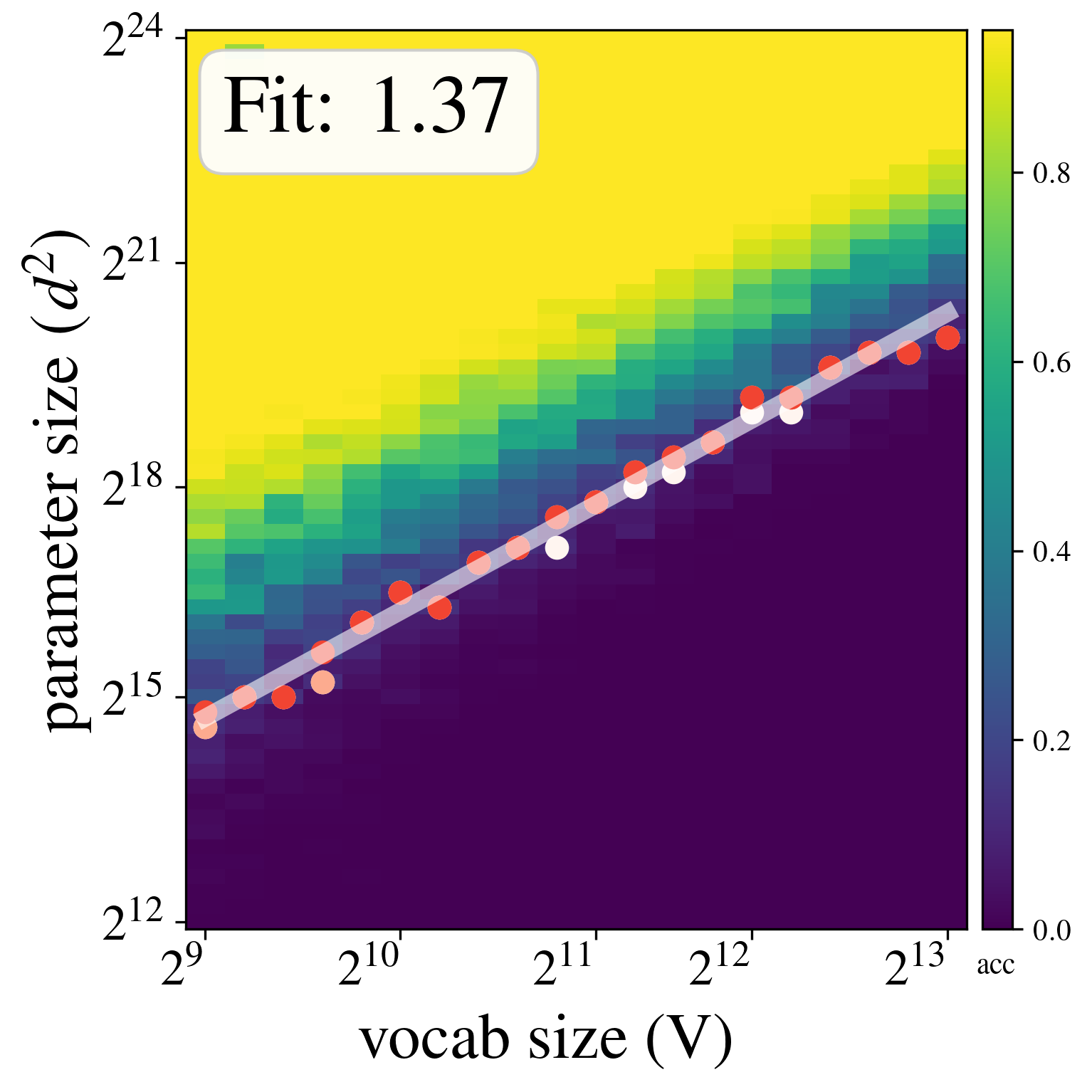}
        \caption{  $\mathrm{Epoch} = 2$}
        \label{fig:multiadam4} 
    \end{subfigure} 
     \begin{subfigure}[b]{.24\linewidth}
        \centering
        \includegraphics[width=0.85\textwidth]{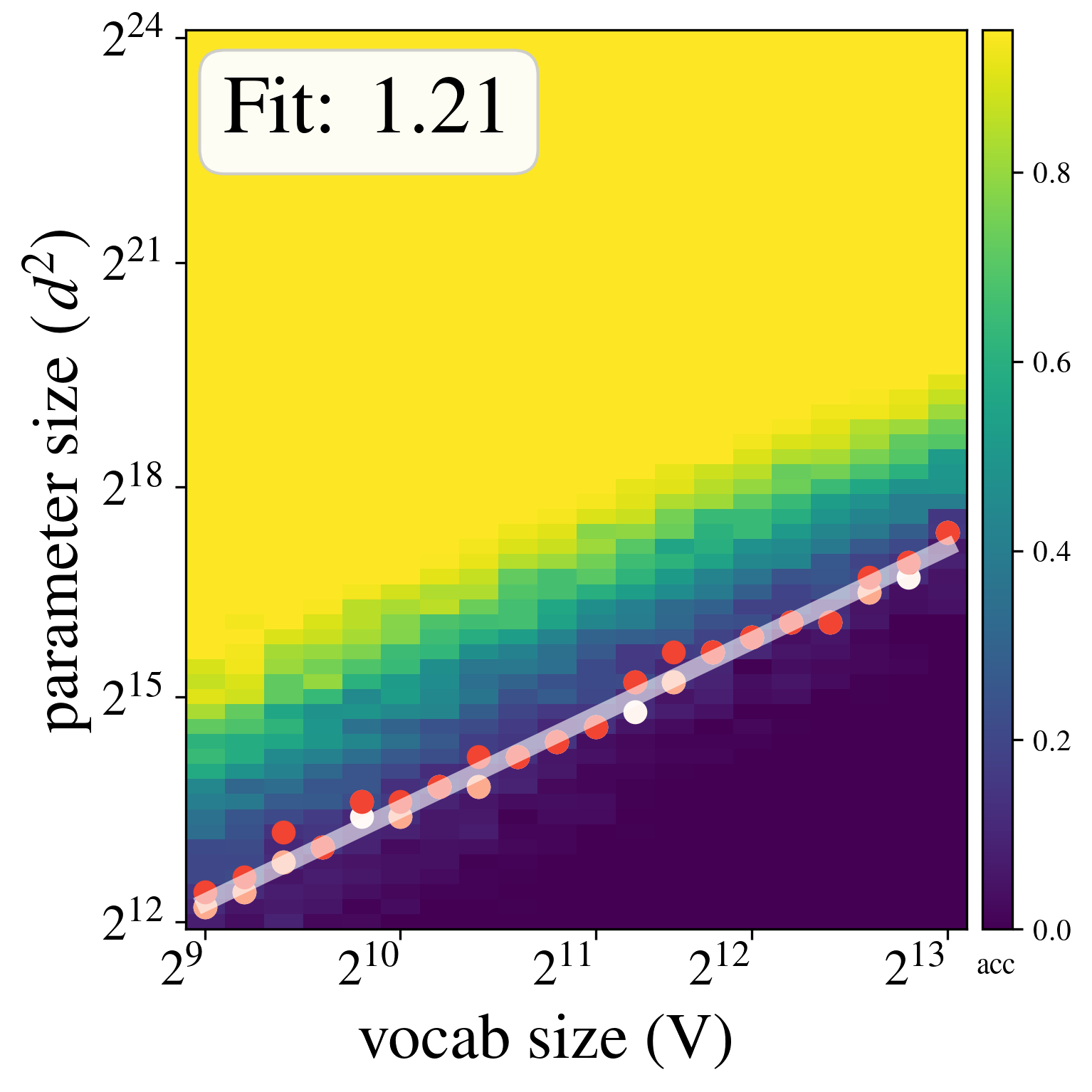}
        \caption{$\mathrm{Epoch} = 8$}
        \label{fig:multiadam1} 
    \end{subfigure} 
    \begin{subfigure}[b]{.24\linewidth}
        \centering
        \includegraphics[width=0.85\textwidth]{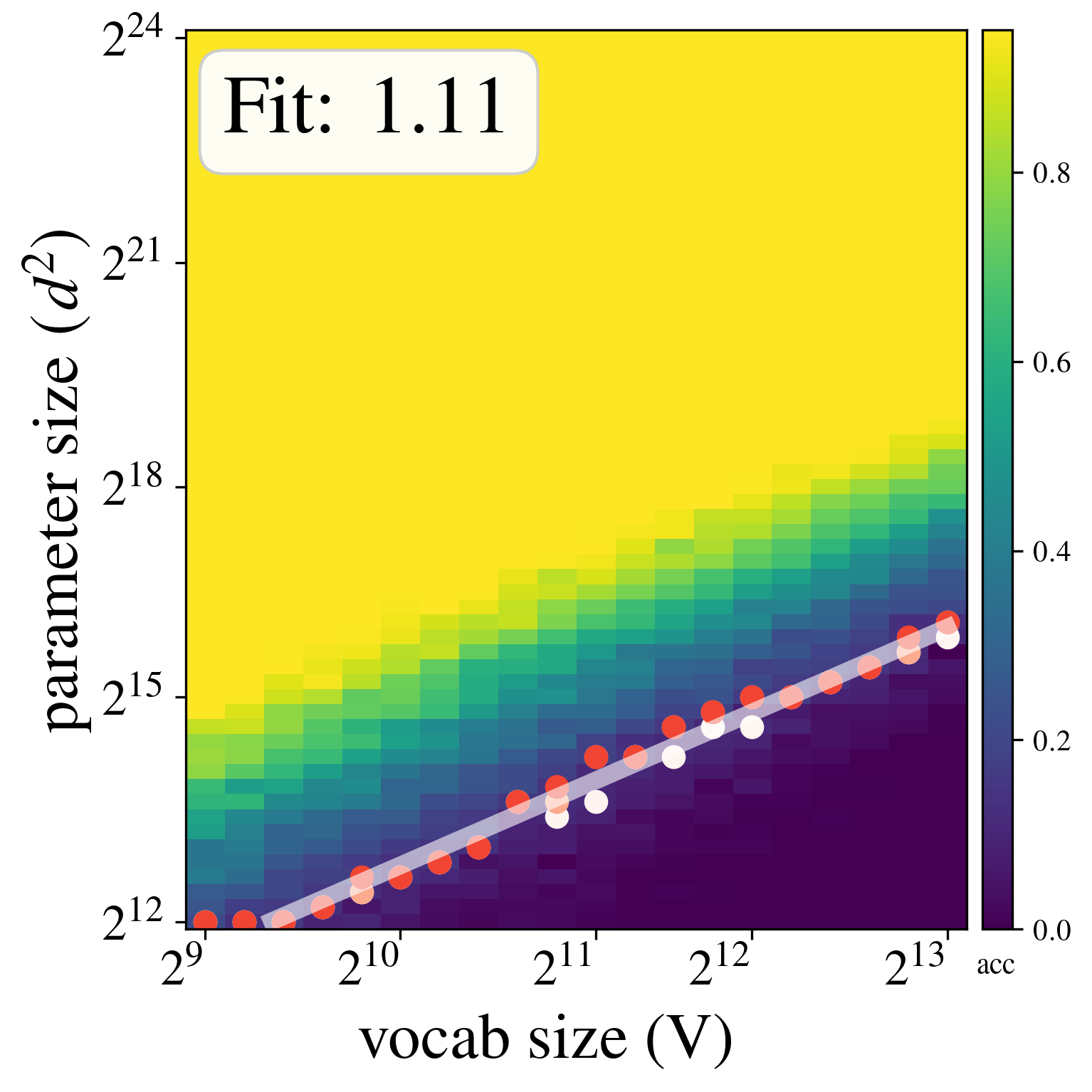}
        \caption{$\mathrm{Epoch} = 16$}
        \label{fig:multiadam2} 
    \end{subfigure} 
    \vspace{-2mm}
          \caption{\small Empirical scaling of the parameter size for the \emph{Attention-only} model with $N \asymp V \log V$ and $L \asymp V^{0.85}$. The model is trained using Adam over $16$ epochs. The slopes are smaller than in Figure \ref{fig:multistepadam}, which is consistent with the shorter sequence length.}
    \label{fig:multistepadamLV085} 
\end{figure}

\begin{figure}[!hbt]
    \centering 
      \begin{subfigure}[t]{.24\linewidth}
        \centering
        \includegraphics[width=0.85\textwidth]{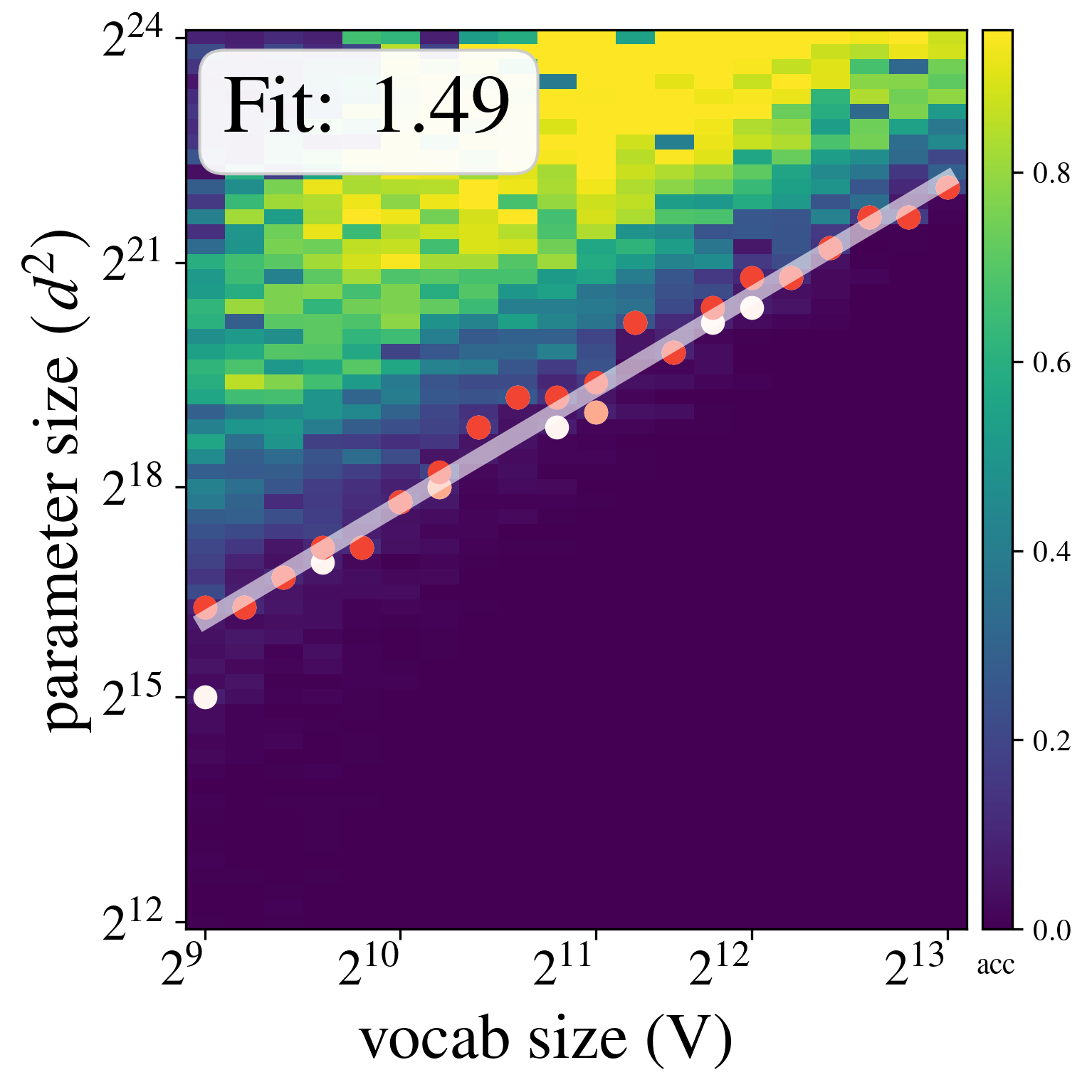}
        \caption{\small Epoch 1}
        \label{fig:multistepadamlba} 
    \end{subfigure} 
     \begin{subfigure}[t]{.24\linewidth}
        \centering
        \includegraphics[width=0.85\textwidth]{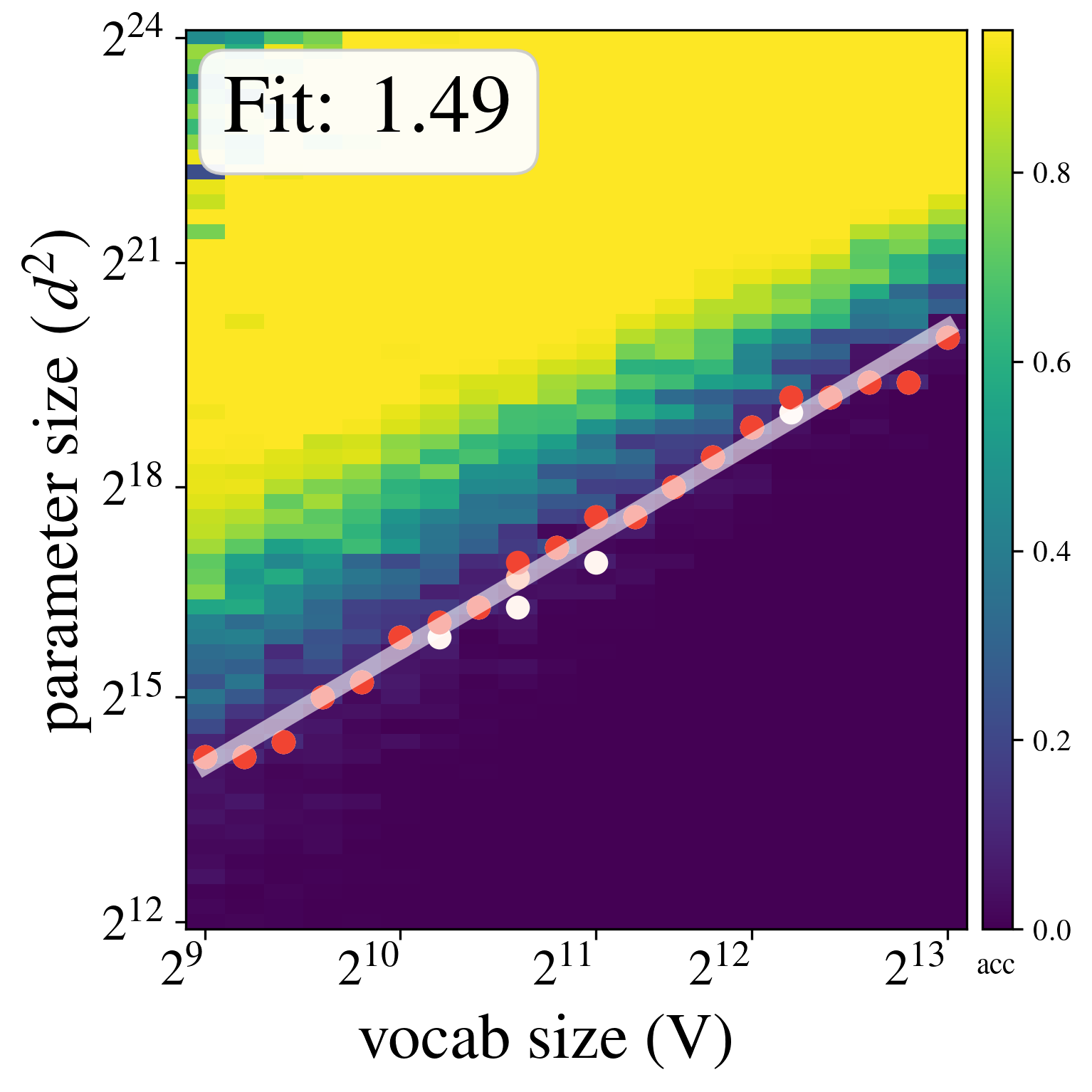}
        \caption{\small Epoch 2}
        \label{fig:multistepadamlbb}  
    \end{subfigure} 
     \begin{subfigure}[t]{.24\linewidth}
        \centering
        \includegraphics[width=0.85\textwidth]{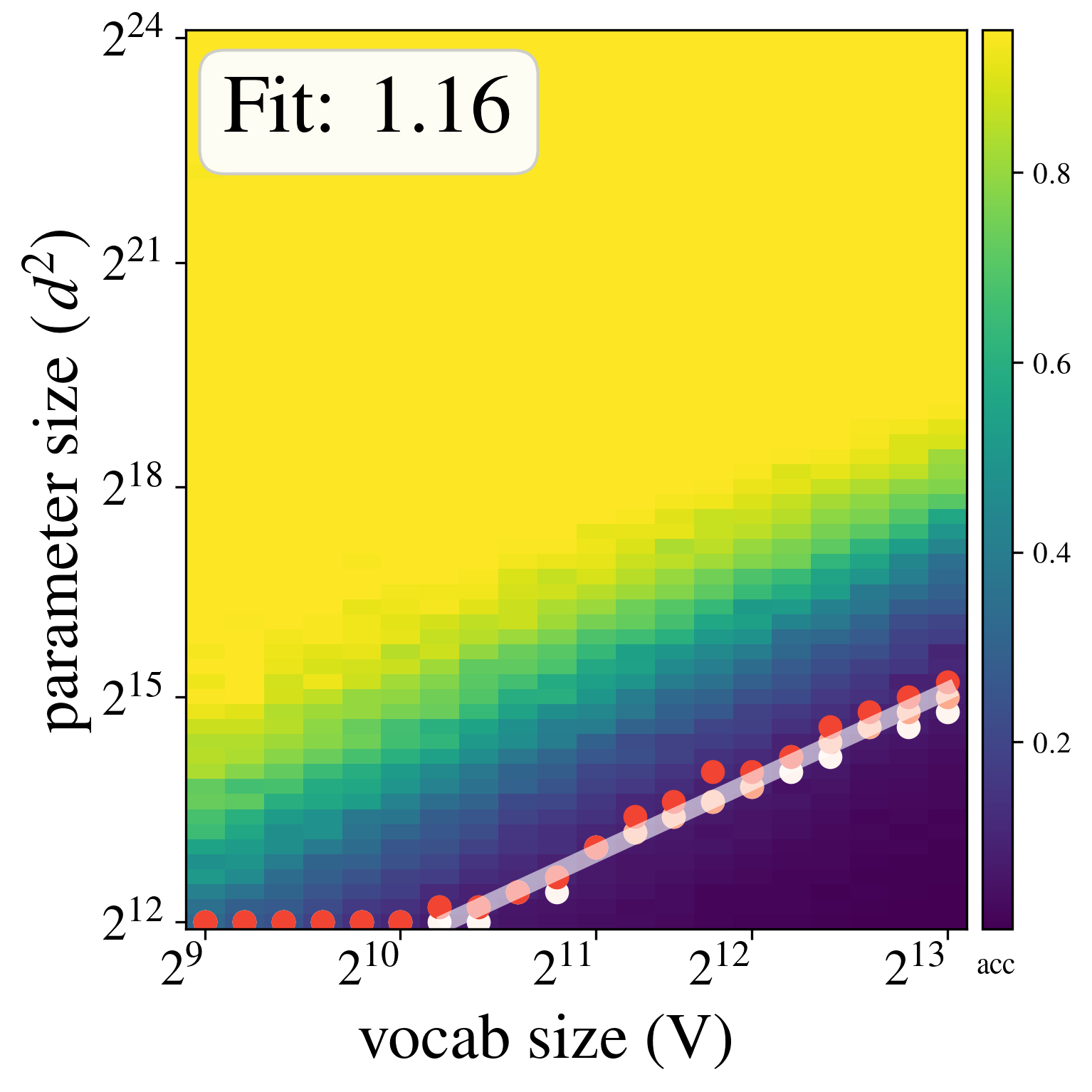}
        \caption{\small Epoch 8}
        \label{fig:multistepadamlbc} 
    \end{subfigure} 
    \begin{subfigure}[t]{.24\linewidth}
        \centering
        \includegraphics[width=0.85\textwidth]{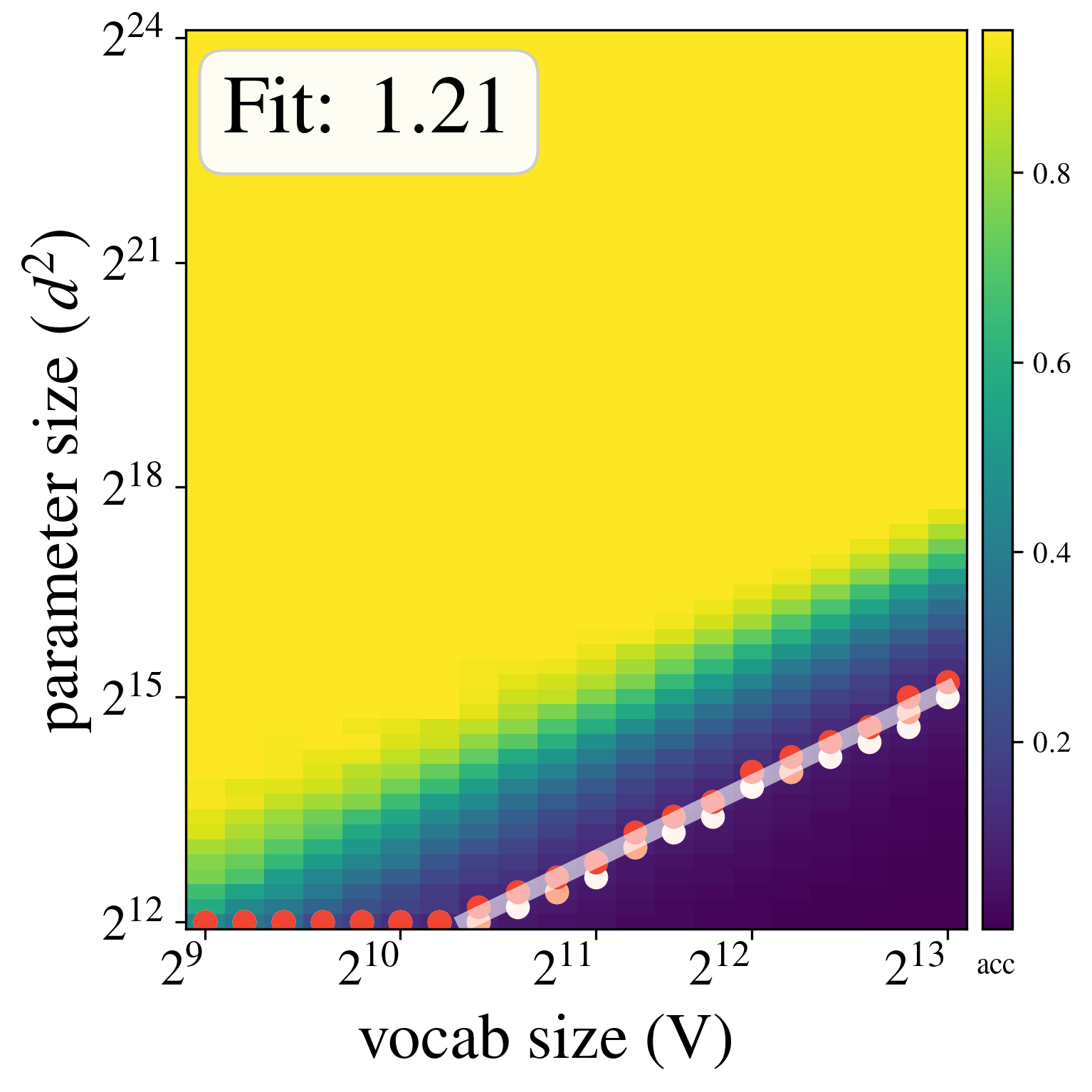}
        \caption{\small Epoch 16}
        \label{fig:multistepadamlbd} 
    \end{subfigure}  \\[0.5em]
    \begin{subfigure}[t]{.24\linewidth}
        \centering
        \includegraphics[width=0.85\textwidth]{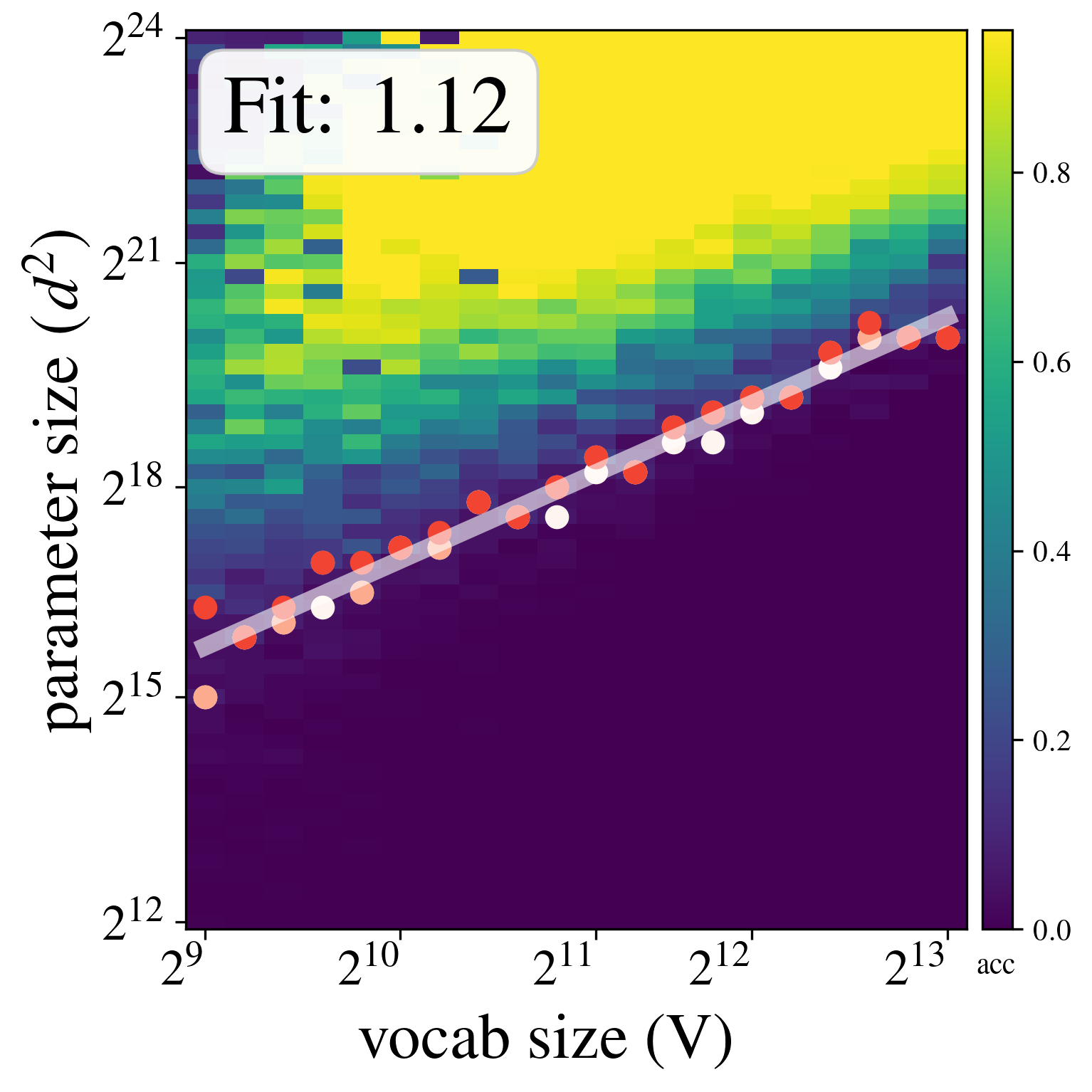}
        \caption{\small Epoch 1}
        \label{fig:multistepadamlbe} 
    \end{subfigure} 
     \begin{subfigure}[t]{.24\linewidth}
        \centering
        \includegraphics[width=0.85\textwidth]{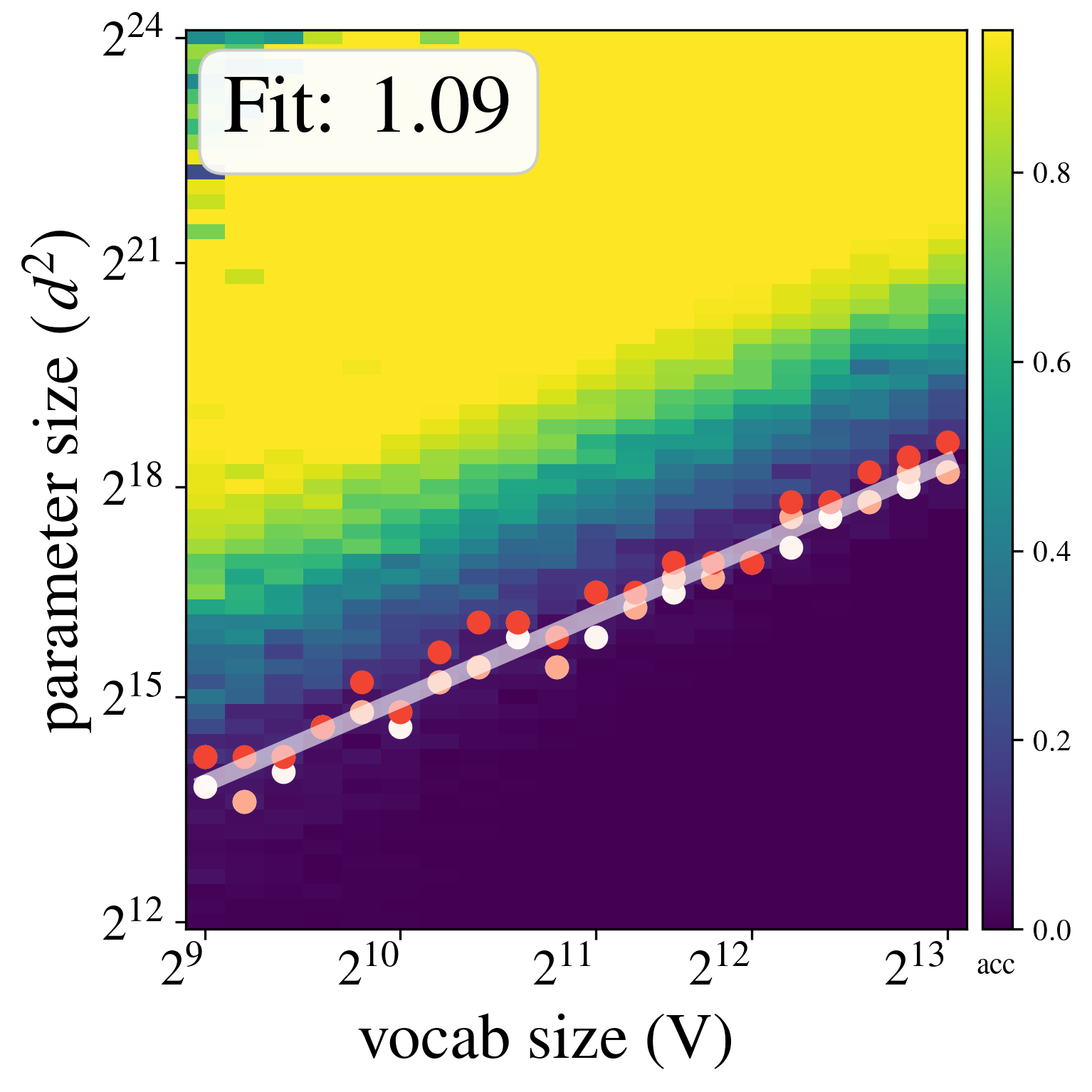}
        \caption{\small  Epoch 2}
        \label{fig:multistepadamlbf} 
    \end{subfigure} 
     \begin{subfigure}[t]{.24\linewidth}
        \centering
        \includegraphics[width=0.85\textwidth]{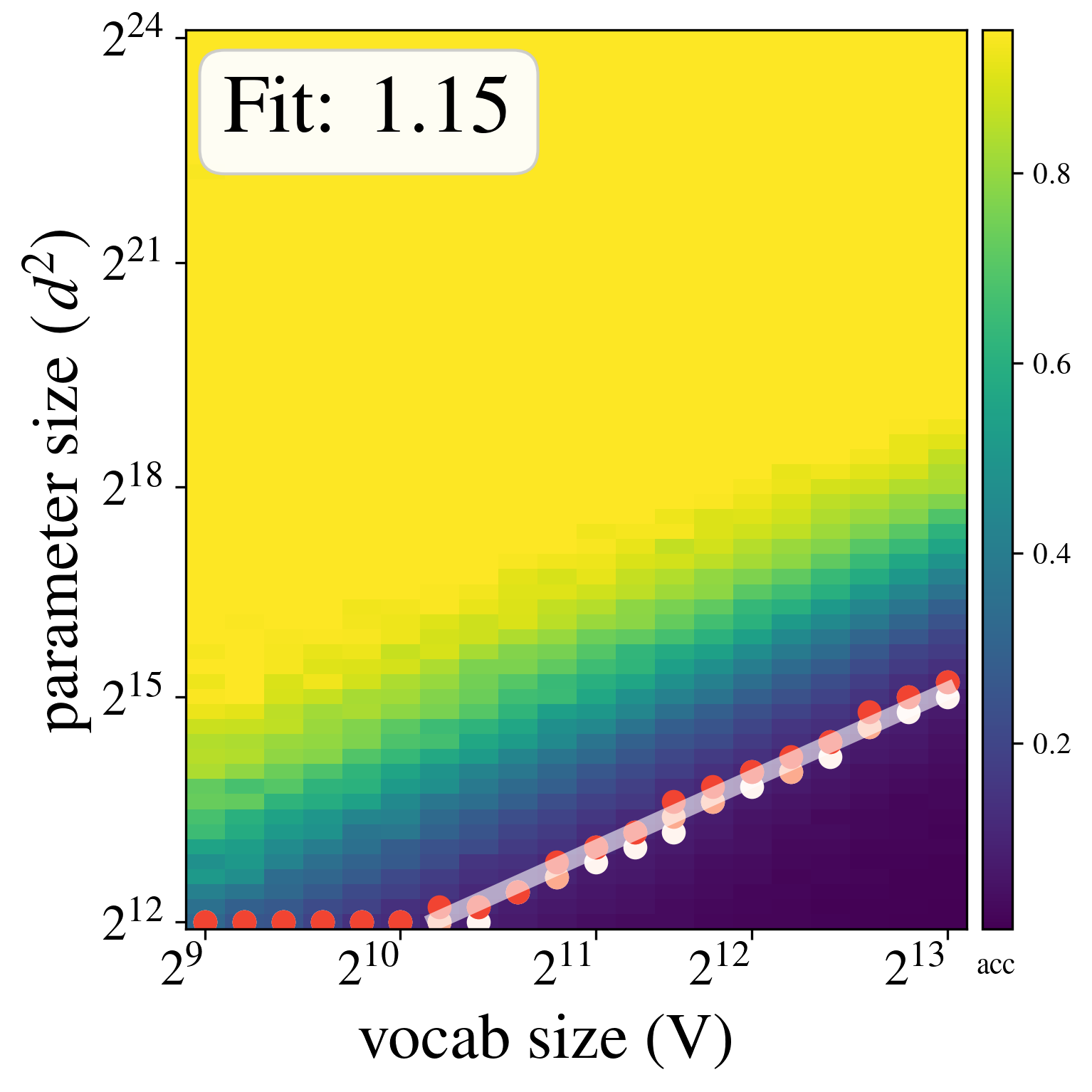}
        \caption{\small Epoch 8}
        \label{fig:multistepadamlbg} 
    \end{subfigure} 
    \begin{subfigure}[t]{.24\linewidth}
        \centering
        \includegraphics[width=0.85\textwidth]{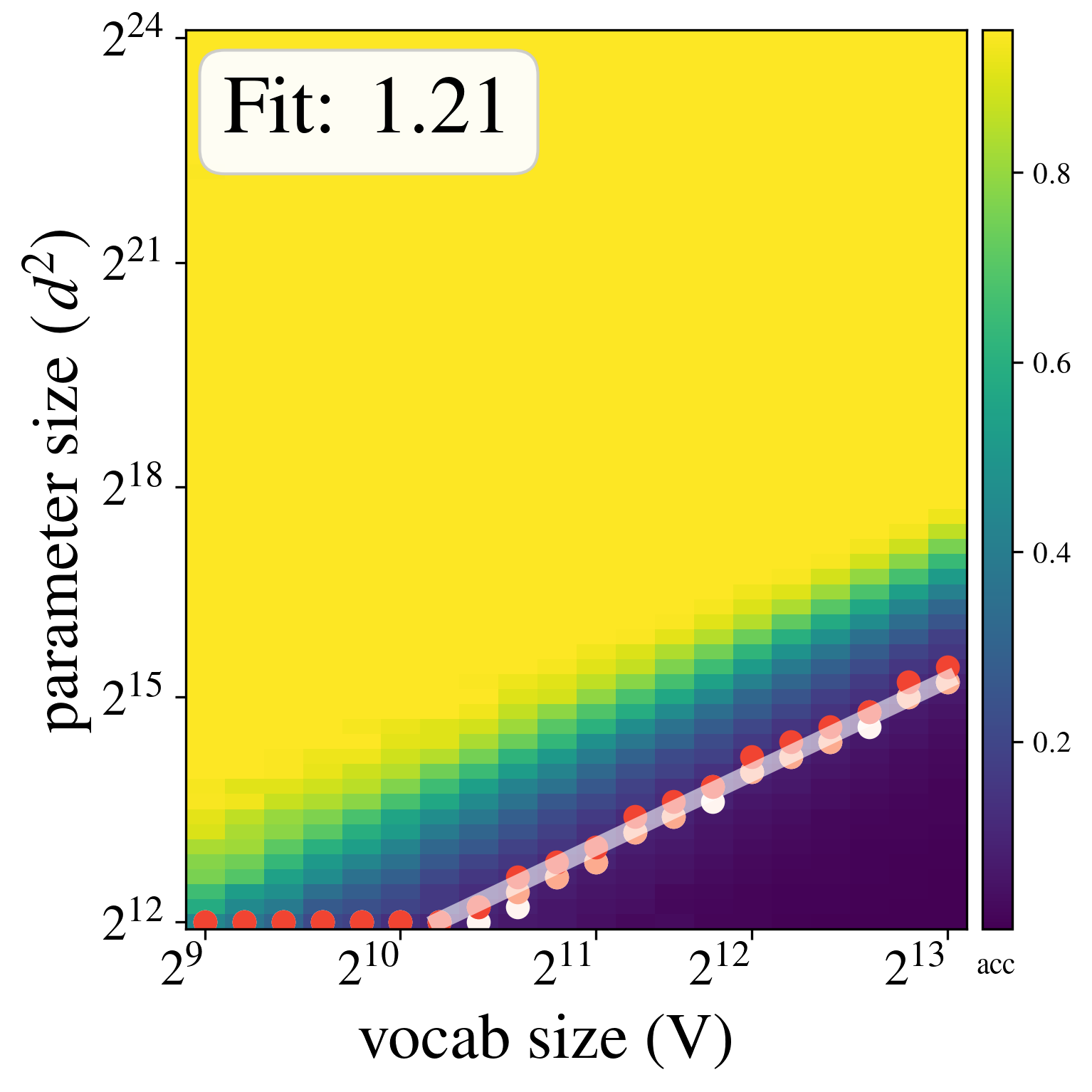}
        \caption{\small Epoch 16}
        \label{fig:multistepadamlbh} 
    \end{subfigure} 
    \vspace{-2mm}
    \caption{\small Empirical scaling of the parameter size for the \emph{Attention-only} model ($L \asymp V$), trained with Adam  using a batch size of $\lfloor N/16 \rfloor$.  \textit{Top row (a--d):} $N \asymp V \log V$. \textit{Bottom row (e--h):} $N \asymp V^{1.5}$.}
    \label{fig:multistepadamlb} 
\end{figure}

\section{Proof Overview}
\label{sec:overview}

In this section, we outline the main ideas behind the proof of Theorem~\ref{thm:main}. The key observation is that the recall task is achieved with near-perfect accuracy if and only if the attention mechanism can distinguish informative tokens. Once this occurs, the remaining task reduces to learning a linearly separable problem, which is well understood. Therefore, the proof focuses on the attention scores in \eqref{eq:attention} and characterizes the conditions under which the mechanism selects the informative tokens.

The pre-softmax scores evaluated on a fresh sequence $\bs{X}_{\mathrm{in}}$, with the key-query matrix given by the first gradient-descent iterate $\bs{W}^{(1)}_{\mathrm{KQ}}$, are given by
\eq{
\mathrm{scores} & \coloneqq \big( \bs{z}_{\mathrm{trig}} \bs{e}_{\ell}^\top + \Zin \bs{X}_{\mathrm{in}} \big)^\top \bs{W}^{(1)}_{\mathrm{KQ}} \bs{z}_{\mathrm{EOS}} . \label{eq:scoresdef}
}
By substituting the exact expression for $\bs{W}_{\mathrm{KQ}}^{(1)}$ into \eqref{eq:scoresdef}, we analyze $\mathrm{scores}$. For intuition, we present the simplified expression below (see \eqref{eq:kqgradient} for the full expression):
\eq{
\mathrm{scores} & \approx \gamma \lVert \bztrig \rVert_2^2 \bs{e}_{\ell}
\underbrace{\Big( \frac{1}{N L} \sum_{i=1}^N \bs{x}_{i,\ell}^\top \Zin^\top (\V^{(1)})^\top \Zout (\bs{p}_{i} - \tfrac{1}{V}\mathbbm{1}_V) \Big)}_{\text{Informative}} \label{eq:informative} \\[-0.5em]
&\quad + \gamma \bs{X}_{\mathrm{in}}^\top \Zin^\top
\underbrace{\Big( \frac{1}{N L} \sum_{i=1}^N \Zin \bs{X}_i \bs{X}_i^\top \Zin^\top (\V^{(1)})^\top \Zout (\bs{p}_{i} - \tfrac{1}{V}\mathbbm{1}_V) \Big)}_{\text{Non-informative}} . \label{eq:noninformative}
}
Here $\V^{(1)}$ denotes the first iterate of the value matrix defined in \eqref{eq:gdfirststep}. The \emph{informative term} in \eqref{eq:informative} captures the alignment between the trigger vector in the fresh input and the one encoded in the learned weights $\bs{W}^{(1)}_{\mathrm{KQ}}$, and therefore contains position information about the informative token. By contrast, the \emph{non-informative term} in \eqref{eq:noninformative} reflects correlations between tokens and does not contain information about the token position.

Thus, the proof reduces to characterizing the conditions under which the informative term in \eqref{eq:informative} dominates the non-informative term in \eqref{eq:noninformative}. Under these conditions, the attention mechanism correctly identifies the informative token, and the remaining prediction problem becomes linearly separable.

\subsection{Empirical Dynamics with Non-Orthogonal Embeddings}  
\label{sec:empdynamics}

We now provide a proof sketch for the finite-sample setting with non-orthogonal embeddings and explain how each noise term in Theorem~\ref{thm:main} arises. In particular, we consider \eqref{eq:informative}–\eqref{eq:noninformative} and, without loss of generality, assume $\ell = 1$ and $\bs{\Pi}_* = \bs{I}_V$ (accordingly, $\bs{p}_i = \bs{x}_{i,1}$). Our goal is to show how the \emph{Signal}, \emph{Gradient noise}, \emph{Mean bias}, and \emph{MLP noise} terms arise from the dynamics of the first gradient step.

The analysis proceeds in two steps. First, we show that the first iterate of the value matrix $\V^{(1)}$ admits a natural decomposition into mean, bias, and noise components. We then show how this decomposition gives rise to the terms appearing in Theorem~\ref{thm:main}.

\subsubsection{Decomposition of the value matrix} 
Both the informative and non-informative terms depend on $\V^{(1)}$. We show that it can be decomposed as
\eq{
\V^{(1)} & = \Zout \Big( \frac{1}{NL}  \sum_{i = 1}^N (\bs{x}_{i,1} -  \tfrac{1}{V} \mathbbm{1}_V ) ( \bs{X}_i \mathbbm{1}_L )^\top  \Big) \Zin^\top  \label{eq:outputiterate} \\
& = \Zout \Big( \underbrace{  \frac{1}{V L}  (\bs{I}_V - \tfrac{1}{V} \mathbbm{1}_V \mathbbm{1}_V^\top)  }_{\text{Mean}}  +  \underbrace{ \frac{1}{V N}  \sum_{i = 1}^N  (\bs{x}_{i,1} -  \tfrac{1}{V} \mathbbm{1}_V ) \mathbbm{1}_V^\top }_{\text{Bias}}  +  \underbrace{  \frac{1}{\sqrt{ L V N }}   \bs{\Xi}  }_{\text{Noise}} \Big) \Zin^\top \label{eq:outputiteratedecomp}
}
where the noise term is given by
\eq{
\bs{\Xi} \coloneqq \sqrt{ \frac{V}{L N} } \Big( \sum_{i = 1}^N  (\bs{x}_{i,1} -  \frac{1}{V} \mathbbm{1}_V ) ( \bs{X}_i \mathbbm{1}_L - \tfrac{L}{V} \mathbbm{1}_V )^\top  -    \frac{1}{V} (\bs{I}_V - \frac{1}{V} \mathbbm{1}_V \mathbbm{1}_V^\top)   \Big). \label{eq:noiseterms}
}
Here, the bias term arises from aggregating tokens at initialization; specifically, the aggregate-token averages $\tfrac{1}{L}\bs{X}_i \mathbbm{1}_L$ in \eqref{eq:outputiterate} concentrate around their mean $\tfrac{1}{V}\mathbbm{1}_V$ as $L$ grows, and this effect appears as the bias term. The noise term captures finite-sample fluctuations of tokens around this mean. In \eqref{eq:outputiteratedecomp}, we explicitly factor out the typical operator-norm scaling $1/\sqrt{VLN}$ from the noise term so that the remaining matrix $\bs{\Xi}$ has constant norm on average, i.e., $\E [ \lVert \bs{\Xi} \rVert_2^2 ] \asymp 1$.

\subsubsection{Characterization of noise terms} 

\textbf{Signal.} Using the mean component in \eqref{eq:outputiteratedecomp}, the informative term in \eqref{eq:informative} can be written as
\eq{
\text{Informative} & =  \gamma \lVert \bztrig \rVert_2^2 \bs{e}_{1}
 \Big( \frac{1}{N L} \sum_{i=1}^N \bs{x}_{i,1}^\top \Zin^\top (\V^{(1)})^\top \Zout (\bs{p}_{i} - \tfrac{1}{V}\mathbbm{1}_V) \Big) \\
& =  \underbrace{  \frac{\gamma \lVert \bztrig \rVert_2^2}{V L^2} \frac{1}{N} \sum_{i = 1}^N  \bs{x}_{i,1}^\top   \Zin^\top  \Zin (\bs{I}_V - \frac{1}{V} \mathbbm{1}_V \mathbbm{1}_V^\top)  \Zout^\top  \Zout       (\bs{x}_{i,1}  -\frac{1}{V} \mathbbm{1}_V)  }_{ \text{Signal} \ \asymp \ \tfrac{1}{VL^2}}  +   \underbrace{ \text{negligible terms} }_\text{$= o( \tfrac{1}{VL^2})$}.
}
The first term is due to the mean component; the negligible terms are due to the bias and noise in \eqref{eq:outputiteratedecomp}.  Standard concentration arguments for Gaussian matrices can be used to show that the leading term scales as $\tfrac{1}{V L^2}$, which gives us the \emph{Signal} term in \eqref{eq:mainresult}. The detailed derivations are provided in Section~\ref{sec:concentrations11}.

\medskip

\textbf{Gradient Noise and Mean Bias.} For ease of presentation, we focus on the large-$L$ regime where we can use the following approximation due to concentration
\eq{
\frac{1}{L} \Zin \bs{X}_i \bs{X}_i^\top   \Zin^\top \approx \frac{1}{d} \bs{I}_d. \label{eq:largeLapproximation}
}
Let $\bs{x}_{\mathrm{in}}$ denote an arbitrary row of $\bs{X}_{\mathrm{in}}$. Using \eqref{eq:largeLapproximation}, we can approximate the non-informative with
\eq{
\text{Non-informative}  & \approx \underbrace{  \frac{1}{d \sqrt{LNV}} \bs{x}_{\mathrm{in}}^\top \Zin^\top   \Zin    \bs{\Xi}  \Zout^\top \Zout   \frac{1}{N} \sum_{i = 1}^N   (\bs{x}_{i,1}  -\frac{1}{V} \mathbbm{1}_V  ) }_{\text{Gradient noise}}  \\
& + \underbrace{  \frac{1}{V d}  
\bs{x}_{\mathrm{in}}^\top \Zin^\top   \Zin   \mathbbm{1}_V  \Big  \lVert  \Zout \frac{1}{N} \sum_{i = 1}^N  (\bs{x}_{i,1}  -\frac{1}{V} \mathbbm{1}_V  )  \Big \rVert_2^2}_{\text{Mean bias}} +  \text{ negligible terms}.
}
The first term arises from the noise component $\bs{\Xi}$ and determines the scaling of the \emph{Gradient noise} term. The second term comes from the bias component and yields the \emph{Mean bias} term in \eqref{eq:mainresult}. We hide the contributions from the mean component in the negligible terms, since they are smaller in magnitude.  The fluctuations of each term can be bounded as stated in Theorem~\ref{thm:main} using standard concentration arguments. The detailed derivations are provided in Section~\ref{sec:concentrations12}.

\medskip

\textbf{MLP noise.} In this part, we consider the \emph{Attention-MLP} model. The scores in \eqref{eq:scoresdef} can be defined in the same way for this case as well. Let $\{ \bs{w}_k \}_{k = 1}^m$
denote the rows of $\bs{W}_{\mathrm{in}}$, where $\bs{w}_k \sim \cN(0, \bs{I}_d)$. For illustration, we work in the large-$L$ regime and adopt the approximation in \eqref{eq:largeLapproximation}.

We define the MLP-noise term as the deviation of the scores from their expectation with respect to the randomness in $\bs{W}_{\mathrm{in}}$:
\eq{
\text{MLP-noise} \coloneqq \mathrm{scores} - \E_{\bs{W}_{\mathrm{in}}}[\mathrm{scores}].
}
Under the large-$L$ assumption in \eqref{eq:largeLapproximation}, the scores admit the approximation (see \eqref{eq:s3full} for the full form)
\eq{
&\text{MLP-noise} \approx  \bs{x}_{\mathrm{in}}^\top \Zin^\top \frac{1}{N^2 d} \sum_{i,j = 1}^N   \FW(\bWin; \bZin,    \bX_i,    \bX_j) \big(\bs{x}_{i, 1} - \tfrac{1}{V} \mathbbm{1}_V \big) \Zout^\top \Zout \big(\bs{x}_{j, 1} - \tfrac{1}{V} \mathbbm{1}_V \big).   
}
Here $\FW(\bWin; \bZin, \bX_i, \bX_j)$ denotes the noise induced by the finite width of $\bWin$, defined as
\eq{
\FW(\bWin; \bZin,    \bX_i,    \bX_j) & \coloneqq    \frac{1}{m} \sum_{k = 1}^m \bs{w}_k  \phi^\prime \Big(\tfrac{1}{L} \bs{w}_k^\top \Zin \bs{X}_i \mathbbm{1}_L \Big)  \phi \Big(\tfrac{1}{L} \bs{w}_k^\top \Zin \bs{X}_j \mathbbm{1}_L \Big) \\
& \qquad  \qquad  \qquad - \E \Big[\bs{w}_k  \phi^\prime \Big(\tfrac{1}{L} \bs{w}_k^\top \Zin \bs{X}_i \mathbbm{1}_L \Big)  \phi \Big(\tfrac{1}{L} \bs{w}_k^\top \Zin \bs{X}_j \mathbbm{1}_L \Big)  \Big]. \label{def:fwdefmaintext}
}
For large $L$, standard concentration arguments imply  that $\big \lVert \tfrac{1}{L} \bs{w}_k^\top \Zin \bs{X}_i \mathbbm{1}_L \big \rVert_2 \approx L^{-1/2} \to 0$. Hence
\eq{
\phi^\prime \Big(\tfrac{1}{L} \bs{w}_k^\top \Zin \bs{X}_i \mathbbm{1}_L \Big) \phi \Big(\tfrac{1}{L} \bs{w}_k^\top \Zin \bs{X}_j \mathbbm{1}_L \Big) \to \underbrace{\phi(0) \phi^\prime(0)}_{\text{nonzero constant}}, 
}
where Assumption~\ref{ass:conditions} ensures $\phi(0)\phi^\prime(0)\neq 0$. Since $\E[\bs{w}_k] = 0$, replacing the $\phi$-dependent factors by this constant yields, we have
\eq{
\FW(\bWin; \bZin,    \bX_i,    \bX_j) \to 
  \frac{\phi(0) \phi^\prime(0)}{m} \sum_{k = 1}^m \bs{w}_k.
}
Substituting this into the expression above gives
\eq{
\text{MLP noise}  \approx \underbrace{ \frac{ \phi(0) \phi^\prime(0) }{d m} 
 \sum_{k = 1}^m \bs{x}_{\mathrm{input}}^\top \Zin^\top \bs{w}_k }_{\asymp \ \tilde O \Big( \tfrac{1}{d\sqrt{m}} \Big)} \ \underbrace{ \Big \lVert \Zout \frac{1}{N} \sum_{i = 1}^N \big(\bs{x}_{i, 1} - \tfrac{1}{V} \mathbbm{1}_V \big) \Big \rVert_2^2 }_{\asymp \ \tilde O \Big( \tfrac{1}{N} \Big)}.
}
Here, the terms can be bounded as in the displayed equation using standard concentration arguments, which yield the scaling of the \emph{MLP noise} term in \eqref{eq:mainresult}. The detailed derivations are provided in Section~\ref{sec:concentrationmlp}.

\vspace{-3mm}
\section{Conclusion}

\vspace{-2mm}
In this paper, we derived precise asymptotic rates for learning with gradient descent on transformers trained on a simple recall task with random embeddings and finite samples. Our analysis and experiments reveal a rich picture of multiplicative scalings between various problem parameters, showing that parameter count is not the only important factor controlling capacity when learning with finite samples on large noisy sequences.
Our results suggest that finer control of the data distribution may be necessary for learning efficiently at optimal capacity, for instance by ensuring sequences are less noisy and more informative, hoping that the discovered mechanisms are robust to harder settings. This is reminiscent of the procedures used for long context extension in LLMs, where most of training happens on shorter sequences, but the final models are extended to work with very long sequences, and empirically do well on retrieval tasks such as ``needle-in-a-haystack''~\citep[e.g.,][]{gemini2024multimodal}, which resembles our theoretical setup. Analyzing similar scalings in more structured data distributions and architectures is thus an interesting avenue for future work.

\bigskip

\subsection*{Acknowledgment}
The work of M. Soltanolkotabi was partially supported by AWS credits through an Amazon Faculty Research Award, a NAIRR Pilot Award, and generous funding by Coefficient Giving, and the USC-Capital One Center for Responsible AI and Decision Making in Finance (CREDIF) Fellowship. M. Soltanolkotabi is also supported by the Packard Fellowship in Science and Engineering, a Sloan Research Fellowship in Mathematics, NSF CAREER Award \#1846369, DARPA FastNICS program, NSF CIF Awards \#1813877 and \#2008443, and NIH Award DP2LM014564-01.

\bigskip

{

\bibliography{full,bibli}
\bibliographystyle{alpha}

\newpage
}

\appendix

\tableofcontents

\newpage

\allowdisplaybreaks

\section{Additional Experiments}

\section{Preliminaries for Appendix}

\textbf{Additional Notation. } For a vector $\bs{x} \in \R^{V},$  we use $\diag(\bs{x}) \in \R^{V \times V}$  denotes the diagonal matrix which has the same diagonal entries with $\bs{x}$, while for a matrix $\bs{A}$,   $\diag(\bs{A}) \in \R^{V}$ denotes the column vector whose elements coincide with the diagonal entries of $\bs{A}$.  For a random variable $\bw$,  $\E_{\bw}[\cdot]$ denotes taking expectation with respect to $\bw$ and keeping the remaining independent terms  fixed. Similarly, we use  $\E[\cdot \vert  \bw]$ for conditional expectation, conditioned on $\bw$. We use $\indic{\mathrm{Event}}$ as an indicator function, which takes values $\{ 0,1 \}$ depending on the event holds or not.  We use $C$ to denote any constant in the upper-bound, which might depend on $\phi$. We use $\mathrm{poly}_{p, q}(N,d,V,L)$ denotes a polynomial function of $(N, d,V,L)$ whose degree depends on  $(p, q)$  polynomially. For vectors $\bs{w}, \hat{\bs{w}}$ and a scaler variable $\eta > 0$, we use $\hat{\bs{w}} = \hat{\bs{w}} + O(\eta)$ to denote  $\lVert \hat{\bs{w}} - \bs{w} \rVert_{\infty} = O(\eta)$.

\noindent
Since we do not use positional encoding in the model,  without loss of generality we can fix the informative index $\ell = 1$.  We define the sequence of non-informative tokens as $\bN_i \coloneqq   [\bx_{i,2}, \cdots,  \bx_{i,L}]^\top$.   We will denote the rows of $\bWin$ with $\{ \bw_{k} \}_{k = 1}^m$.   For compact representation the attention with the trigger, we define 
\eq{
\sZin \eqqcolon \begin{bmatrix}
\bZin & \bztrig
\end{bmatrix}   ~~ \text{and} ~~  \sX_i  \eqqcolon \begin{bmatrix}
 \bx_{i,1}^\top &  1  \\[0.1em]
\bN_i & 0
\end{bmatrix}   \in \R^{L \times (V+1)}   \label{eq:data}
}

With this notation,  we can write the iterates in  three-step GD. Let 
\eq{
\hat{\bs{p}}_{t,i}  \coloneqq \hat{\bs{p}}(\bs{X}_i; \bs{V}^{(t)}, \bs{W}_{\mathrm{KQ}}^{(0)}  ), ~~ \text{and} ~~ \bs{\alpha}_{0,i} \coloneqq   \sigma \Big( \sX_i \sZin^\top \bs{W}_{\mathrm{KQ}}^{(0)} \bs{z}_{\mathrm{EOS}} \Big). 
}
We have
\eq{
& \V^{(1)} \!  = \!    \Zout \Big( \frac{ \eta}{N} \sum_{i = 1}^N (\bs{p}_i  - \hat{\bs{p}}_{0,i}) \phi \big(    \bs{\alpha}_{0,i}^\top    \bX_i  \bZin^\top \bWin^\top  \big)   \Big)   \\
& \bWkq^{(1)}  \! =  \!\sZin \frac{\gamma}{N}  \sum_{i = 1}^N   \sX_i^\top \! \big( \diag(\ba_{0,i} ) \! - \! \ba_{0,i}  \ba_{0,i}^\top \big)    \bX_i  \bZin^\top  \bWin^\top  \mathrm{diag} \Big( \!  \phi^\prime  \big( \bWin  \bZin    \bX_i^\top  \ba_{0,i}   \big)  \! \Big)    ( \V^{(1)} )^\top \!   \Zout  (\bs{p}_i \! - \! \hat{\bs{p}}_{i,1} )  \bs{z}_{\text{EOS}}^\top. ~~~ ~~~~ \label{eq:kqgradient}
}
For notational convenience,  we define the noise due to finite width as (which we defined equivalently in \eqref{def:fwdefmaintext}) 
\eq{
\FW(\bWin; \bZin,    \bX_i,    \bX_j) & \coloneqq   \frac{1}{m} \Big(   \bWin^\top    \mathrm{diag} \Big(  \phi^\prime  \big(  \tfrac{1}{L} \bWin  \bZin    \bX_i^\top \mathbbm{1}_L    \big)  \Big)    \phi \Big( \tfrac{1}{L}   \bWin   \bZin    \bX_j^\top \mathbbm{1}_L  \Big) \\
&\hspace{1em}  - \E_{\bWin} \Big[ \bWin^\top    \mathrm{diag} \Big(  \phi^\prime  \big(  \tfrac{1}{L} \bWin  \bZin    \bX_i^\top \mathbbm{1}_L    \big)  \Big)    \phi \Big( \tfrac{1}{L}   \bWin   \bZin    \bX_j^\top \mathbbm{1}_L  \Big) \Big]  \Big).   ~~~~~~ \label{def:expectedtermsfw}
}
For the terms arising in the expected value term in \eqref{def:expectedtermsfw}, we define
\begin{itemize}
    \item $\alpha_{ij} \! \coloneqq    \!   \E_{\bw} \Big[ \phi^{\prime}  \big(  \tfrac{1}{L} \bw^\top  \bZin    \bX_i^\top \mathbbm{1}_L    \big)   \phi^\prime  \big(  \tfrac{1}{L} \bw^\top  \bZin    \bX_j^\top \mathbbm{1}_L    \big)  \Big],$
    \item $\beta_{ij} \!  \coloneqq \!   \E_{\bw} \Big[ \phi^{\prime \prime}   \big(  \tfrac{1}{L} \bw^\top  \bZin    \bX_i^\top \mathbbm{1}_L    \big)   \phi  \big(  \tfrac{1}{L} \bw^\top  \bZin    \bX_j^\top \mathbbm{1}_L    \big)  \Big].$
\end{itemize}
Moreover, we make the following definitions to simplify the notation in the following:
\eq{
& \bs{A}_{1, ir} \coloneqq  \bZin \Big( \frac{1}{L N} \sum_{j = 1}^N  \alpha_{ij}  (\bx_j  - \tfrac{1}{V} \mathbbm{1}_V )     ( \bx_j -   \tfrac{1}{V} \mathbbm{1}_{V} )^\top     \Big)  \Big( \frac{1}{L N} \sum_{j = 1}^N  \alpha_{rj}  (\bx_j  - \tfrac{1}{V} \mathbbm{1}_V )       ( \bx_j -   \tfrac{1}{V} \mathbbm{1}_{V} )^\top     \Big)  \bZin^\top \label{def:aone}  \\[0.3em]
& \bs{A}_{2, ir} \coloneqq \bZin   \Big(    \frac{1}{L N} \sum_{j = 1}^N \! \alpha_{ij}   (\bN_j^\top  - \tfrac{1}{V} \mathbbm{1}_V  \mathbbm{1}_{L -1}^\top  )   \mathbbm{1}_{L -1}    ( \bx_j -   \tfrac{1}{V} \mathbbm{1}_{V} )^\top     \Big)\\[-0.2em]
& \hspace{12em} \times    \Big(    \frac{1}{L N}  \sum_{j = 1}^N  \! \alpha_{rj}   ( \bx_j -   \tfrac{1}{V} \mathbbm{1}_{V} )   \mathbbm{1}_{L -1}^\top (\bN_j^\top  - \tfrac{1}{V} \mathbbm{1}_V  \mathbbm{1}_{L -1}^\top  ) ^\top     \Big)     \bZin^\top ~~~~~~~  \label{def:atwo} \\
&  \bs{A}_{3, ir} \coloneqq  \frac{1}{L^2 V^2}  \Big( \frac{1}{N} \sum_{j = 1}^N  \alpha_{ij}    ( \bx_j -   \tfrac{1}{V} \mathbbm{1}_{V}    )   \Big) ^\top   \Big( \frac{1}{N} \sum_{j = 1}^N  \alpha_{rj}  (\bx_j  - \tfrac{1}{V} \mathbbm{1}_V )          \Big)   \bZin   \mathbbm{1}_V \mathbbm{1}_V^\top   \bZin ^\top  \label{def:athree}
}
and
\eq{
& \bs{S}_1 \! \coloneqq  \!  \Big( \frac{1}{ L N} \sum_{j = 1}^N    (\bx_j  - \tfrac{1}{V} \mathbbm{1}_V )     ( \bx_j -   \tfrac{1}{V} \mathbbm{1}_{V} )^\top     \Big)   \Big( \frac{1}{ L N} \sum_{j = 1}^N    (\bx_j  - \tfrac{1}{V} \mathbbm{1}_V )       ( \bx_j -   \tfrac{1}{V} \mathbbm{1}_{V} )^\top     \Big)  \label{def:sone} \\
& \bs{S}_2 \! \coloneqq  \!    \Big(  \frac{1}{LN}  \sum_{j = 1}^N   (\bN_j^\top  - \tfrac{1}{V} \mathbbm{1}_V  \mathbbm{1}_{L -1}^\top  )   \mathbbm{1}_{L -1}    ( \bx_j -   \tfrac{1}{V} \mathbbm{1}_{V} )^\top     \Big) \Big( \frac{1}{LN}    \sum_{j = 1}^N  ( \bx_j -   \tfrac{1}{V} \mathbbm{1}_{V} )   \mathbbm{1}_{L -1}^\top (\bN_j^\top  - \tfrac{1}{V} \mathbbm{1}_V  \mathbbm{1}_{L -1}^\top  ) ^\top     \Big) ~~~~~~~     \label{def:stwo}  \\
& \bs{S}_3 \! \coloneqq  \!   \frac{1}{L^2 V^2}  \Big( \frac{1}{N} \sum_{j = 1}^N    ( \bx_j -   \tfrac{1}{V} \mathbbm{1}_{V}    )   \Big) ^\top   \Big( \frac{1}{N} \sum_{j = 1}^N  (\bx_j  - \tfrac{1}{V} \mathbbm{1}_V )          \Big)     \mathbbm{1}_V \mathbbm{1}_V^\top.   \label{def:sthree}
}

\subsection{Preliminary Results: Characterization of Good Events}

We start with characterizing ``good events'' which we will use in the proof of  Theorem   \ref{thm:main}. 
\begin{lemma}
\label{lem:niceevent}
We consider $l \in \N$, and $V^3 \gg N \gg V \gg L$ and $L \asymp V^{\epsilon_1}$, and $d \asymp   V^{\epsilon_2}$ for some  $\epsilon_1, \epsilon_2 \in (0,1)$.
For the following we define,  $m_{ij} \coloneqq (1 - 1/V) \delta_{ij} + \frac{L}{V}$. We define the following events: 
\begin{enumerate}[leftmargin = *, label= (E\arabic*)]
\item  \label{event:boundZin} Let   $\bs{z}_k \coloneqq  \bZin \be_k$ and  $\mathsf{z}_{k} \coloneqq    (  \bs{z}_k +   \indic{l = 1}   \ztrig ).$  We have
\begin{enumerate}[leftmargin=1em, label= (E1.\arabic*)]
\item \label{event:boundZin1} $\frac{1}{V} \norm{ \bZin \bZin^\top }_2 \leq  \frac{2}{d}$  and $\max_{k \leq V} \norm{\bs{z}_k}_2   \vee  \norm{\ztrig}_2   \leq 2$ and $\max_{j \neq k} \lvert \inner{\bs{z}_j}{\mathsf{z}_k} \rvert \leq \frac{\log V}{\sqrt{d}}$.
\item \label{event:boundZin2}   $\frac{1}{\sqrt{V}} \norm{\bZin \mathbbm{1}_V}_2 \leq 2$  and    $\frac{1}{\sqrt{V}} \norm{  \bZin^\top \bZin \mathbbm{1}_V }_{\infty} \leq \frac{\log V}{\sqrt{d}}$
\item   \label{event:boundZin3}   $ \big \lvert     \mathsf{z}_{k} ^\top  \bZin \mathbbm{1}_V \big \rvert \leq  2 \log V \sqrt{\frac{V}{d}}$    and     $\big \lvert   \mathsf{z}_{k} ^\top  \bZin \bZin^\top  \bZin   \mathbbm{1}_V   \big \rvert \leq  C_K \log V \big( \frac{ V}{d} \big)^{\frac{3}{2}}  $ and $ \Big   \lvert \mathsf{z}_{k}^\top  \bZin \mathrm{diag} \big( \bZin^\top \bZin   \big) \Big \rvert \leq  C_K \log V \sqrt{\frac{V}{d}}$
\item   \label{event:boundZin4}   For all $i \in [N]$,  $ \lvert  \mathsf{z}_{k}^\top  \bZin \bX_i^\top \mathbbm{1}_L \rvert \leq  \be_{k}^\top  \bX_i^\top \mathbbm{1}_L     +  C_K \log V      \frac{ \lVert \bX_i^\top \mathbbm{1}_L \rVert_2}{\sqrt{d}}$
\item    \label{event:boundZin5}  For all $i \in [N]$,    $  \lvert \mathbbm{1}_V^\top  \bZin^\top  \bZin \bX_i^\top \mathbbm{1}_L \rvert \leq L + C_K  \log V  \lVert \bX_i^\top \mathbbm{1}_L \rVert_2  \sqrt{ \frac{V}{d}}$. 
\item    \label{event:boundZin6}   For all $i \in [N]$,  $\big \lvert   \mathsf{z}_{k}^\top  \bZin \bZin^\top  \bZin    \bX_i^\top \mathbbm{1}_L  \big \rvert \leq   \frac{2 V}{d}  \big(    \be_{k}^\top  \bX_i^\top \mathbbm{1}_L      +  C_K  \log V    \frac{ \lVert \bX_i^\top \mathbbm{1}_L \rVert_2}{\sqrt{d}} \big) $.   
\item  \label{event:boundZin7}  For all $i, j \in [N]$,   $ \lvert  \frac{1}{L} \mathbbm{1}_L^\top       \bX_j  \bZin^\top   \bZin     \bX_i^\top \mathbbm{1}_L  - m_{ij}   \rvert \leq    \lvert \frac{1}{L} \mathbbm{1}_L       \bX_j^\top    \bX_i^\top \mathbbm{1}_L  - m_{ij} \rvert   +   C_K \frac{ \lVert \bX_i^\top \mathbbm{1}_L \rVert_2 \lVert \bX_j^\top \mathbbm{1}_L \rVert_2 }{L}  \frac{  \log V}{\sqrt{d}}$
\item  \label{event:boundZin8}  For all $i \in [N]$,   $\lVert \bZin      \bN_i^\top   \bN_i       \bZin^\top  \mathsf{z}_{k} \rVert_2 \leq  C_K \big(  \be_{k}^\top  \bN_i^\top  \mathbbm{1}_{L - 1} + \frac{L}{d}  + \log^6 V  \frac{ \lVert  \bN_i^\top  \mathbbm{1}_{L - 1} \rVert_2  }{\sqrt{d} } \big)$.
\end{enumerate}
\item  \label{event:discrete}    We have
\begin{enumerate}[leftmargin=1em, label= (E2.\arabic*)]
\item    \label{event:discrete1}   For all $i, j \in [N]$,  $ \lvert \frac{1}{L} \mathbbm{1}_L^\top       \bX_j   \bX_i^\top \mathbbm{1}_L  - m_{ij} \rvert \leq C_K  \frac{\log^2 V}{\sqrt{V} \wedge L}$,
\item  \label{event:discrete2}    For all  $i \in [N]$,    $  \lVert    \bX_i^\top \mathbbm{1}_L   \rVert_{\infty} \leq \log L$ and  $  \lVert    \bX_i^\top \mathbbm{1}_L   \rVert_{0} \geq \frac{L}{2}$ 
\item   \label{event:discrete3}    $\Big \lvert \lVert \frac{1}{N} \sum_{i = 1}^N \bx_i \rVert_2 - \frac{1}{N}  - \frac{1}{V} \Big \rvert \leq C_K  \frac{\log^2 N}{N \sqrt{V}} $ and     $\Big \lvert \lVert \frac{1}{N} \sum_{i = 1}^N \bx_i - \frac{1}{V} \mathbbm{1}_V \rVert_2  - \frac{1}{N}    \Big \rvert  \leq   C_K  \frac{\log^2 N}{N \sqrt{V}}$ and $  \lVert \frac{1}{N} \sum_{i = 1}^N    \bx_i - \frac{1}{V} \mathbbm{1}_V   \rVert_{\infty} \leq  \frac{(e+1)}{V}$ 
\item   \label{event:discrete4}   $\sum_{i, j = 1}^N \lvert \indic{\bx_i =  \bx_j} - \frac{1}{V} \rvert \leq \frac{4 N^2}{V}$   and $\sum_{i, j = 1}^N ( \indic{\bx_i =  \bx_j} - \frac{1}{V} ) \leq \frac{4 N^2}{V}$  and for any $k \in [V]$, $\Big \lvert   \sum_{i, j = 1}^N   \vert \indic{\bx_j = \be_k }   -    \tfrac{1}{V}   \rvert  (\indic{\bx_i = \be_k }   -  \tfrac{1}{V} )   \Big \rvert \leq \frac{CN^2}{V^2}$. 
\item   \label{event:discrete5}  $\lVert  \bs{S}_1  \rVert_2 \leq  \frac{e}{L^2 V^2}$   and   $\lvert   \tr(\bs{S}_1)  -   (1- \frac{1}{V}) \frac{1}{L^2} \big(   \frac{1}{N} +  \frac{1}{V} \big)  \rvert \leq \frac{C_K \log^2 V}{L^2 N \sqrt{V}}  $
\item   \label{event:discrete6}   $\lVert  \bs{S}_2  \rVert_2 \leq     \frac{C_K \log^2 V}{NLV}$   and   $\lvert   \tr(\bs{S}_2)  -  (1 - \frac{1}{V})^2 \frac{L - 1}{L^2 N} \rvert \leq   \frac{C_K\log^3 V}{N \sqrt{LV}} $
\item   \label{event:discrete7}   $ \frac{-  C_K \log^2 V}{N \sqrt{V}}   \frac{1}{V^2 L^2} \mathbbm{1}_V \mathbbm{1}_V^\top  \preceq   \bs{S}_3 - \frac{1}{N} \frac{1}{V^2 L^2} \mathbbm{1}_V \mathbbm{1}_V^\top  \preceq  \frac{ C_K  \log^2 V}{N \sqrt{V}}    \frac{1}{V^2 L^2} \mathbbm{1}_V \mathbbm{1}_V^\top$
\item   \label{event:discrete8}
$\Big \lVert   \frac{1}{NL } \sum_{  j = 1 }^N         ( \bN_j^\top -   \tfrac{1}{V} \mathbbm{1}_{L-1} )   \mathbbm{1}_{L-1} \mathbbm{1}_{L-1}^\top  ( \bN_j^\top -   \tfrac{1}{V} \mathbbm{1}_{L-1} )^\top  \Big \rVert_2 = \frac{1}{V} \pm+ \frac{C_K \log^2 V}{\sqrt{NV}}$.
\end{enumerate}
\end{enumerate}
For any $K > 0,$ there exists a universal constant $C_K > 0$ depending only on  $K$ such that 
\eq{
\mpr[\ref{event:boundZin} \vert \{ \bs{X}_i\}_{i = 1}^N ]  \geq 1 - \frac{1}{V^K} ~~ \text{and} ~~  \mpr[\ref{event:discrete}] \geq 1 - \frac{1}{V^K}.
}
\end{lemma}

\begin{proof}
For \ref{event:boundZin}:
\begin{itemize}[leftmargin=*]
\item By Proposition \ref{prop:gaussquadraticform},  we have  $ \norm{ \frac{1}{V} \bZin \bZin^\top - \frac{1}{d} \bs{I}_d }_2 \leq \frac{2 \log V}{\sqrt{V d}}$  and by Proposition \ref{prop:gausssquare}, we have  $\max_{k \leq V} \norm{\bs{z}_k}_2   \vee  \norm{\ztrig}_2   \leq 2$   with probability at least $1 - C V d \exp(- c \log^2 V).$ 
\item  By  Proposition \ref{prop:gausssquare},  $\frac{1}{\sqrt{V}} \norm{\bZin \mathbbm{1}_V}_2 \leq 2$  and    $\frac{1}{\sqrt{V}} \norm{  \bZin^\top \bZin \mathbbm{1}_V }_{\infty} \leq \frac{2 \log V}{\sqrt{d}}$  with probability at least $1 - C V d \exp(- c \log^2 V).$   
\item   By  Propositions \ref{prop:gausssquare} and \ref{prop:wishartsquare},  we have $ \frac{1}{\sqrt{V}} \big \lvert \mathsf{z}_{k}^\top    \bZin \mathbbm{1}_V \big \rvert \leq  \frac{2 \log V}{\sqrt{d}}$ and $\big \lvert   \mathsf{z}_{k} ^\top  \bZin \bZin^\top  \bZin   \mathbbm{1}_V   \big \rvert \leq  C_K \log V \big( \frac{ V}{d} \big)^{\frac{3}{2}}  $  with probability at least $1 - C V d \exp(- c \log^2 V).$   Moreover
\eq{
\frac{1}{V} \Big   \lvert \mathsf{z}_{k}^\top  \bZin \mathrm{diag} \big( \bZin^\top \bZin   \big) \Big \rvert 
& = \frac{1}{V} \sum_{\substack{ i = 1 \\ i \neq k}}^V \norm{\bs{z}_i}_2^2 \langle \bs{z}_i,      \bs{z}_k  \rangle 
+  \frac{ \indic{l = 1} }{V} \sum_{\substack{ i = 1 \\ i \neq k}}^V \norm{\bs{z}_i}_2^2 \langle \bs{z}_i,     \ztrig \rangle    + \underbrace{ \frac{1}{V}  \mathsf{z}_{k}^\top  \bs{z}_{k}  }_{\in \frac{1}{V} [-C_K, C_K]},
}
where we used previous items to bound the last term.  For $i \neq k,$ by using Lemma \ref{lem:gaussianvectormoment}, we have  for $p \leq \frac{d}{6}$,
\eq{
 \E[   \norm{\bs{z}_i}_2^{4p} ~  \lvert \langle \bs{z}_i,     \bs{z}_k  \rangle \rvert^{2p}   ] 
& \leq  d^{-p}   \E[   \norm{\bs{z}_i}_2^{6p}    ]  (2p)^{p}  \leq   d^{- p} 2^p p^{p}  \frac{d (d +2) \cdots (d + 6p - 2)}{ d^{ 3p}} \leq  d^{- p} 2^{4p} p^{p}.
}
Therefore,      
\eq{
\E[   \norm{\bs{z}_i}_2^{4p} ~  \lvert \langle \bs{z}_i,     \bs{z}_k  \rangle \rvert^{2p}   ]^{\frac{1}{2p}} \leq 4 d^{-1/2} \sqrt{p}.  
}
By Proposition \ref{prop:rosenthal}, we have for $2 \leq p \leq \frac{d}{6}$ ,
\eq{
\E \Big[  \Big   \lvert      \frac{1}{V}  \mathsf{z}_{k}^\top  \bZin \mathrm{diag} \big( \bZin^\top \bZin   \big) \Big \rvert^{2p} \Big]^{\frac{1}{2p}} \leq C d^{- 1/2} \Big[  \sqrt{ \frac{p}{V} } +  V^{\frac{1}{p}} \frac{p^{3/2}}{V} \Big]
}
By using $p = \log V$, we have   the bound in the statement with probability  $1 - \frac{1}{V^K}$.
\item By Proposition \ref{prop:gausssquare} with probability at least $1 - \frac{1}{V^K}$
\eq{
\lvert  \mathsf{z}_{k}^\top  \bZin \bX_i^\top \mathbbm{1}_L \rvert 
& \leq \lvert \be_k^\top   \bZin^\top   \bZin \bX_i^\top \mathbbm{1}_L \rvert  +   \indic{l = 1}   \lvert  \ztrig ^\top   \bZin \bX_i^\top \mathbbm{1}_L \rvert \\
& \leq   \be_k^\top  \bX_i^\top \mathbbm{1}_L + C_K  \log  V  \frac{\lVert  \bX_i^\top \mathbbm{1}_L \rVert_2}{   \sqrt{d} }.
}
By the union bound, the item follows.
\item By Proposition \ref{prop:gausssquare} with probability at least $1 - \frac{1}{V^K}$,
\eq{
\rvert \mathbbm{1}_V^\top \bZin^\top  \bZin \bX_i^\top \mathbbm{1}_L \rvert 
& \leq  \mathbbm{1}_V^\top \bX_i^\top \mathbbm{1}_L   +  C_K \log V  \lVert  \bX_i^\top \mathbbm{1}_L \rVert_2 \sqrt{ \frac{V }{d} } \\
& =   L  +  C_K \log V  \lVert  \bX_i^\top \mathbbm{1}_L \rVert_2 \sqrt{ \frac{V }{d} }. 
} 
\item  By  Proposition \ref{prop:wishartsquare}, with probability at least $1 - C N \exp(-c \log^2 V)$,  we have $\big \lvert   \mathsf{z}_{k}^\top  \bZin \bZin^\top  \bZin    \bX_i^\top \mathbbm{1}_L  \big \rvert \leq   \frac{2 V}{d}  \big(    \be_{k}^\top  \bX_i^\top \mathbbm{1}_L      +  C_K  \log V    \frac{ \lVert \bX_i^\top \mathbbm{1}_L \rVert_2}{\sqrt{d}} \big) $ for all $i \in [N]$.
\item By  Proposition \ref{prop:gausssquare}, with probability at least $1 - \frac{1}{V^K}$
\eq{
\frac{1}{L} \mathbbm{1}_L^\top       \bX_j  \bZin^\top   \bZin     \bX_i^\top \mathbbm{1}_L  - m_{ij}  
& =  \frac{1}{L} \mathbbm{1}_L^\top       \bX_j  \bX_i^\top \mathbbm{1}_L       - m_{ij}   \pm C_K  \frac{\lVert \bX_j^\top \mathbbm{1}_L \rVert_2 \lVert \bX_i^\top \mathbbm{1}_L \rVert_2}{L}\frac{\log V}{\sqrt{d}}.
}
\item For the last item,  let $n_{k} \coloneqq  \mathbbm{1}_{L - 1}^\top \bN_i  \be_k    $. We have
\eq{
 \bZin      \bN_i^\top   \bN_i       \bZin^\top \mathsf{z}_{k} =   n_{k}  (    \lVert \bs{z}_{k}\rVert^2_2 +  \indic{l = 1} \mathsf{z}_{k}^\top  \ztrig   - \frac{1}{d} ) \bs{z}_{k}    + \frac{L}{d} \mathsf{z}_{k}  + \sum_{\substack{ j = 1 \\ j \neq k } }^V n_j \big( \bs{z}_j   \bs{z}_j^\top - \frac{1}{d} \bs{I}_d \big)  \mathsf{z}_{k}.
}
By Proposition \ref{prop:HCpositivepolynomial}, we have
\eq{
\E \Big[ \Big \lVert \sum_{\substack{j = 1 \\ j \neq k } }^V n_j \big( \bs{z}_j   \bs{z}_j^\top - \frac{1}{d} \bs{I}_d \big)  \mathsf{z}_{k}   \Big \rVert_2^{2p} \Big]^{\frac{1}{p}}   
& \leq  C (p - 1)^6  \E \Big[ \Big \lVert \sum_{\substack{ j = 1 \\ j \neq k } }^V n_j \big( \bs{z}_j   \bs{z}_j^\top - \frac{1}{d} \bs{I}_d \big)  \mathsf{z}_{k}  \Big \rVert_2^{2} \Big]   \\
& \leq  \frac{C }{d}  (p - 1)^6   \lVert \bN_i^\top \mathbbm{1}_{L-1}   \rVert_2^2.
} 
Therefore,  with probability $1 - \frac{1}{V^K}$, we have
\eq{
\lVert   \bZin      \bN_i^\top   \bN_i       \bZin^\top    \mathsf{z}_{k} \rVert_2  \leq    C_K \Big(  n_{k}  + \frac{L}{d} + \log^6 V  \frac{  \lVert \bN_i^\top \mathbbm{1}_{L-1}   \rVert_2}{\sqrt{d}} \Big).
}
\end{itemize}
For  \ref{event:discrete}:
\begin{itemize}[leftmargin = *]
\item  By Proposition \ref{prop:multrandommatrix}, we have  the first item with probability $1 - \frac{N^2}{V^K}$.
\item By Corollary \ref{cor:rosenthalresults}, we have  $ \max_{i \in [N]} \lVert    \bX_i^\top \mathbbm{1}_L   \rVert_{\infty} \leq \log L$ with probability $1 - \frac{N}{V^K}$ for large enough $L$. For the second part, we define $n_k \coloneqq   \mathbbm{1}_L^\top  \bX_i \be_k .$ We observe that
\eq{
\E [ \lVert    \bX_i^\top \mathbbm{1}_L   \rVert_{0} ] =  \sum_{k = 1}^V \mpr[n_k > 0] = V \Big( 1 - (1 - \frac{1}{V})^L \Big) = L \Big(1 - \frac{L}{2V} + o(L/V) \Big).
}
By McDiarmid inequality,  we have 
\eq{
\mpr \Big[  \Big \lvert \lVert    \bX_i^\top \mathbbm{1}_L   \rVert_{0}  -  L \Big(1 - \frac{L}{2V} + o(L/V) \Big) \Big \rvert > \sqrt{L}  \log V \Big] \leq 2 \exp( - 2 \log^2 V ),
}
which gives the result.
\item Let $\bs{n} \coloneqq   \sum_{i = 1}^{N}  \bx_i$, where $\E [ \bs{n}   ] = \frac{N}{V} \mathbbm{1}_V$. By Proposition \ref{prop:multrandommatrix} with probability $1 - \frac{1}{V^K},$  we have
\eq{
\Big \lvert \big \lVert \frac{1}{N} \bs{n} - \frac{1}{V} \mathbbm{1}_V  \big \rVert_2^2   - (1 - \frac{1}{V}) \frac{1}{N} \Big \rvert =  \Big\lvert  \big \lVert  \frac{1}{N} \bs{n}  \big \rVert_2^2 - (1 - \frac{1}{V}) \frac{1}{N} - \frac{1}{V}  \Big \rvert \leq C_K \frac{\log^2 V}{N \sqrt{V}}.
}
 Lastly, by Corollary \ref{cor:rosenthalresults}, we have  $\norm{ \frac{1}{N} \bs{n} - \frac{1}{V} \mathbbm{1}_V  }_{\infty} \leq \frac{(e+1)}{V}$.
\item We have
\eq{
  \sum_{i, j = 1}^N \lvert \indic{\bx_i  =  \bx_j} - \frac{1}{V} \rvert & =  \Big(  \sum_{i, j = 1}^N \lvert \indic{\bx_i =  \bx_j} - \frac{1}{V} \rvert   - \frac{2}{V} (1 - \frac{1}{V}) \Big)   + \frac{2 N^2}{V} (1 - \frac{1}{V})  \\
 & = (1 - \frac{2}{V}) \sum_{i, j = 1}^N  ( \indic{\bx_i =  \bx_j} - \frac{1}{V} ) + \frac{2 N^2}{V} (1 - \frac{1}{V}) \\
&  =   (1 - \frac{2}{V}) \Big \lVert \sum_{i = 1}^{N} ( \bx_i - \frac{1}{V} \mathbbm{1}_V ) \Big \rVert_2^2  + \frac{2 N^2}{V} (1 - \frac{1}{V})
 }
By the previous item,  the statement follows. Moreover,
\eq{
 &  \sum_{i, j = 1}^N  \big(  \lvert \indic{\bx_j = \be_k }   -  \tfrac{1}{V}   \rvert \pm \tfrac{2}{V} (1 - \tfrac{1}{V}) \big)  (\indic{\bx_i = \be_k }   -  \tfrac{1}{V} ) \\
& = ( 1 - \frac{2}{V})  \Big( \sum_{i  = 1}^N (\indic{\bx_i = \be_k }   -  \tfrac{1}{V} ) \Big)^2  +  \frac{2N}{V} ( 1 - \frac{1}{V}) \sum_{i = 1}^N (\indic{\bx_i = \be_k }   -  \tfrac{1}{V} ) \\
& = N^2 ( 1 - \frac{2}{V}) \Big \langle \be_k,  \frac{1}{N}\sum_{i  = 1}^N ( \bx_i     -  \tfrac{1}{V}\mathbbm{1}_V ) \Big \rangle^2 + \frac{2N^2}{V} ( 1 - \frac{1}{V}) \Big \langle \be_k, \frac{1}{N}\sum_{i  = 1}^N ( \bx_i    -  \tfrac{1}{V} \mathbbm{1}_V ) \Big \rangle \\
& \leq \frac{C N^2}{V^2}.
}
\item The events for $\bs{S}_1,$ $\bs{S}_2$ and $\bs{S}_3$ follows Proposition \ref{prop:sbound}.
\item  \ref{event:discrete8} follows the second item in Proposition \ref{prop:multrandommatrix}.
\end{itemize}
\end{proof}

\begin{proposition}
\label{prop:niceeventconsequences}
We consider the parameter regime in  Lemma \ref{lem:niceevent}.  Let  $\bar{\phi} \coloneqq \sup_{k_1, k_2 \geq 1} \lvert  \phi^{(k_1)}(0) \phi^{(k_2)}(0) \rvert $.   The intersection of \ref{event:boundZin} and \ref{event:discrete} implies the following events:
\begin{enumerate}[leftmargin = *, label= (R\arabic*)]
\item     \label{cons:innerprodrestricted}  For all $i, j \in [N]$,   $ \lvert \frac{1}{L} \mathbbm{1}_L       \bX_j^\top \bZin^\top   \bZin     \bX_i^\top \mathbbm{1}_L  - m_{ij} \rvert \leq  C_K \big(  \frac{  \log V}{\sqrt{d}} +    \frac{\log^2 V}{\sqrt{V} \wedge L} \big)$ ,
\item    \label{cons:coeffbound}    $ \sup_{i ,  j} \lvert \alpha_{ij} - \phi^\prime(0)^2 \rvert  \vee  \lvert \beta_{ij} - \phi^{\prime \prime}(0) \phi(0)   \rvert         \leq   \frac{ \bar{\phi} }{L} \big( m_{ij} + C_K    \frac{  \log V}{\sqrt{d}} +  C_K \frac{\log^2 V}{\sqrt{V} \wedge L} \big)$ 
\item   \label{cons:matrixbounds1} Let $\Delta_{*, ir} \coloneqq  \bs{A}_{*, ir}    -  \phi^\prime(0)^4 \bZin   \bs{S}_* \bZin^\top$ for $* \in \{1,2,3\}$.
We have 
\begin{itemize}[leftmargin = -0.5em]
\item[-]  $\sup_{i, r \in [N]} \lVert  \Delta_{1,ir}  \rVert_2  \leq   C_K  \phi^\prime(0)^2 \Big(  \frac{1}{N d  L^3} +   \frac{1}{V d L^2} \frac{1}{V \wedge L^2 \wedge L \sqrt{d} } \Big)$.
\item[-]  $ \sup_{i, r \in [N]}  \lVert  \Delta_{2,ir}  \rVert_2  \leq  \frac{C_K \sqrt{V}}{ d \sqrt{NL}}     \Big(   \frac{1}{NL^{\frac{3}{2}}} + \frac{1}{V \sqrt{L}}  \frac{1}{V \wedge L^2 \wedge L \sqrt{d}} \Big)$.
\item[-]  We have $\Delta_{3, ir} =   \frac{\bar{\Delta}_{3,ir}}{V^2 L^2} \Zin \mathbbm{1}_V\mathbbm{1}_V^\top \Zin^\top$ such that
$$ \sup_{i, r \in [N]}  \lvert  \bar{\Delta}_{3,ir}  \rvert   \leq   \frac{C_K  \phi^\prime(0)^2}{N}  \Big( \frac{1}{N L} + \frac{1}{\sqrt{N}} \frac{1}{V \wedge L^2 \wedge L \sqrt{d}} \Big) +   \Big( \frac{1}{N L} + \frac{1}{\sqrt{N}} \frac{1}{V \wedge L^2 \wedge L \sqrt{d}} \Big)^2.  $$
\end{itemize}
\item   \label{cons:matrixbounds2}    For all $i ,r \in [N]$,
\begin{itemize}[leftmargin = -0.5em]
\item[-]   We have
\eq{
\Big \lVert   \bs{A}_{1, ir}  - \frac{  \phi^\prime(0)^4 }{d} \big(   \frac{1}{N} +  (1- \frac{1}{V}) \frac{1}{V} \big)  \bs{I}_d \Big  \rVert_2 & \leq    C_K  \phi^\prime(0)^2 \Big(  \frac{1}{N d  L^3} +   \frac{1}{V d L^2} \frac{1}{V \wedge L^2 \wedge L \sqrt{d} } \Big)   \\
 & +    C_K  \phi^\prime(0)^4   \Big(   \frac{\log V}{L^2 V^{3/2} \sqrt{d}} +   \frac{ \log^2 V}{L^2 N \sqrt{V} d}    \Big).
 }
\item[-]  We have
\eq{
\Big \lVert   \bs{A}_{2, ir}  -  \frac{   \phi^\prime(0)^4  }{d}   (1 - \frac{1}{V})^2 \frac{L - 1}{L^2 N}   \bs{I}_d \Big  \rVert_2 & \leq      \frac{C_K \sqrt{V}}{ d \sqrt{NL}}     \Big(   \frac{1}{NL^{\frac{3}{2}}} + \frac{1}{V \sqrt{L}}  \frac{1}{V \wedge L^2 \wedge L \sqrt{d}} \Big) \\
 & +    C_K  \phi^\prime(0)^4   \Big(   \frac{\log V}{N L \sqrt{V d}} +   \frac{ \log^3 V}{N \sqrt{LV} d}     \Big).
 }
\item[-]   We have  $\bs{A}_{3, ir}    -  \frac{  \phi^\prime(0)^4 }{N}   \frac{1}{V^2 L^2} \bZin \mathbbm{1}_V \mathbbm{1}_V^\top \bZin^\top \eqqcolon    \frac{\tilde{\Delta}_{3, ir} }{V^2 L^2} \bZin \mathbbm{1}_V \mathbbm{1}_V^\top \bZin^\top$ such that
\eq{
\lvert  \tilde{\Delta} _{3, ir}  \rvert &\leq  \frac{C_K \phi^\prime(0)^4 \log^2 V}{N \sqrt{V}} \\
& +   \frac{C_K  \phi^\prime(0)^2}{N}  \Big( \frac{1}{N L} + \frac{1}{\sqrt{N}} \frac{1}{V \wedge L^2 \wedge L \sqrt{d}} \Big)  +   \Big( \frac{1}{N L} + \frac{1}{\sqrt{N}} \frac{1}{V \wedge L^2 \wedge L \sqrt{d}} \Big)^2.
}
\end{itemize}
\end{enumerate} 
\end{proposition}

\begin{proof}
We have the following arguments.
\begin{itemize}[leftmargin = *]
\item  By \ref{event:boundZin7} and \ref{event:discrete1},  we have    \ref{cons:innerprodrestricted}.
\item  For  \ref{cons:coeffbound},  we assume   \ref{cons:innerprodrestricted} hold. Let $w_i  \coloneqq \frac{1}{L} \bw^\top  \bZin    \bX_i^\top \mathbbm{1}_L/ \norm{  \frac{1}{L}   \bZin    \bX_i^\top \mathbbm{1}_L }_2 .$  We write
\eq{
\Big \lvert \E_{\bs{w}} \big[ \phi^{\prime}  & \big( \norm{  \tfrac{1}{L}   \bZin    \bX_i^\top \mathbbm{1}_L }_2    w_i    \big)      \phi^\prime  \big(  \norm{  \tfrac{1}{L}   \bZin    \bX_i^\top \mathbbm{1}_L }_2 w_j \big) \big] -  \phi^\prime(0)^2 \Big \rvert \\ 
  &  \labelrel={conseq:eqq0}  \Big \lvert  \sum_{u, v = 1}^{p_{\star}}    \norm{  \tfrac{1}{L}   \bZin    \bX_i^\top \mathbbm{1}_L }_2^u   \norm{  \tfrac{1}{L}   \bZin    \bX_j^\top \mathbbm{1}_L }_2^v  \frac{ \E_{\bs{w}} \big[ w_i^u w_j^v \big]}{u! v!} \phi^{(u+1)}(0)  \phi^{(v+1)}(0)  \Big \rvert  \\
&   \labelrel={conseq:eqq1}  \Big \lvert    \tfrac{1}{L^2}    \mathbbm{1}_L^\top \bX_j   \bZin^\top    \bZin    \bX_i^\top \mathbbm{1}_L  \phi^{(2)}(0)  \phi^{(2)}(0) \\
& \qquad +  \sum_{\substack{ u, v = 1 \\ \text{$u + v$ is even} \\ u + v > 2}}^{p_{\star}}    \norm{  \tfrac{1}{L}   \bZin    \bX_i^\top \mathbbm{1}_L }_2^u   \norm{  \tfrac{1}{L}   \bZin    \bX_j^\top \mathbbm{1}_L }_2^v  \frac{\E_{\bs{w}} \big[ w_i^u w_j^v \big]}{u! v!} \phi^{(u+1)}(0)  \phi^{(v+1)}(0) \Big \rvert  \\
& \labelrel\leq{conseq:ineqq2}  \frac{ \bar{\phi} }{L} \big( m_{ij} + C_ K  \frac{ \log V}{\sqrt{d}} +  C_K \frac{\log^2 V}{\sqrt{V} \wedge L} \big) + O \Big( \frac{1}{L^2} \Big),
}
where we used Taylor expansion of $\phi$ and $\E_{\bs{w}} \big[\bw^\top  \bZin    \bX_i^\top \mathbbm{1}_L \big] = 0$ in \eqref{conseq:eqq0},  $\E[Z^{u}_1 Z^v_2] = 0$ if  $u + v$  is odd for jointly Gaussian $(Z_1, Z_2)$ in \eqref{conseq:eqq1}, and \ref{cons:innerprodrestricted} in \eqref{conseq:ineqq2}.
Similarly,
\eq{
  \big \lvert \E \big[ & \phi^{\prime \prime}  \big(   \norm{  \tfrac{1}{L}   \bZin    \bX_i^\top \mathbbm{1}_L }_2 w_i    \big)      \phi   \big(  \norm{  \tfrac{1}{L}   \bZin    \bX_i^\top \mathbbm{1}_L }_2 w_j \big) \big] -  \phi(0) \phi^{(2)}(0) \big \rvert \\ 
  &   =  \Big \lvert  \sum_{u, v = 1}^{k_2}    \norm{  \frac{1}{L}   \bZin    \bX_i^\top \mathbbm{1}_L }_2^u   \norm{  \frac{1}{L}   \bZin    \bX_j^\top \mathbbm{1}_L }_2^v  \frac{ \E \big[ w_i^u w_j^v \big]}{u! v!} \phi^{(u+2)}(0)  \phi^{(v)}(0)  \Big \rvert  \\
  &   =  \Big \lvert    \frac{1}{L^2}    \mathbbm{1}_L^\top \bX_j   \bZin^\top    \bZin    \bX_i^\top \mathbbm{1}_L  \phi(0)  \phi^{(2)}(0) \\
& +  \sum_{\substack{ u, v = 1 \\ \text{$u + v$ is even} \\ u + v > 2}}^{k_2}    \norm{  \frac{1}{L}   \bZin    \bX_i^\top \mathbbm{1}_L }_2^u   \norm{  \frac{1}{L}   \bZin    \bX_j^\top \mathbbm{1}_L }_2^v  \frac{\E \big[ w_i^u w_j^v \big]}{u! v!} \phi^{(u+2)}(0)  \phi^{(v)}(0) \Big \rvert  \\
& \leq  \frac{ \bar{\phi} }{L} \big( m_{ij} + C K  \frac{ \log V}{\sqrt{d}} +  C K \frac{\log^2 V}{\sqrt{V} \wedge L} \big) + O \Big( \frac{1}{L^2} \Big).
}
\item For    \ref{cons:matrixbounds1},  we define
\eq{
& \bar{\Delta}_{1, ir}  \coloneqq  \Big( \frac{1}{L N} \sum_{j = 1}^N  (\alpha_{ij} - \phi^\prime(0)^2 )  (\bx_j \! - \! \tfrac{1}{V} \mathbbm{1}_V )     ( \bx_j \! - \! \tfrac{1}{V} \mathbbm{1}_{V} )^\top     \Big)  \Big( \frac{1}{L N} \sum_{j = 1}^N  \phi^\prime(0)^2 (\bx_j  \! - \! \tfrac{1}{V} \mathbbm{1}_V )       ( \bx_j -   \tfrac{1}{V} \mathbbm{1}_{V} )^\top     \Big)  \\
&  + \!
 \Big( \frac{1}{L N} \sum_{j = 1}^N   \phi^\prime(0)^2 (\bx_j \! - \!\tfrac{1}{V} \mathbbm{1}_V )     ( \bx_j \! - \!  \tfrac{1}{V} \mathbbm{1}_{V} )^\top     \Big)  \Big( \frac{1}{L N} \sum_{j = 1}^N    (\alpha_{rj} - \phi^\prime(0)^2 ) (\bx_j \! - \! \tfrac{1}{V} \mathbbm{1}_V )       ( \bx_j \! - \!  \tfrac{1}{V} \mathbbm{1}_{V} )^\top     \Big)  \\
& + \!
 \Big( \frac{1}{L N} \sum_{j = 1}^N    (\alpha_{ij} - \phi^\prime(0)^2 ) (\bx_j  \! - \! \tfrac{1}{V} \mathbbm{1}_V )     ( \bx_j \! - \!  \tfrac{1}{V} \mathbbm{1}_{V} )^\top     \Big) \Big( \frac{1}{L N} \sum_{j = 1}^N    (\alpha_{rj} - \phi^\prime(0)^2 ) (\bx_j \! - \! \tfrac{1}{V} \mathbbm{1}_V )       ( \bx_j \! - \!   \tfrac{1}{V} \mathbbm{1}_{V} )^\top     \Big). 
}
We have
\eq{
\norm{  \bar{\Delta}_{1, ir}   }_2 & \leq \frac{C  \phi^\prime(0)^2  \sup_{i} \abs{ \alpha_{ii} - \phi^\prime(0)^2}}{ L N }    \lVert   \bs{S}_1   \rVert_2^{\frac{1}{2}}      +  \phi^\prime(0)^2  \sup_{i \neq j} \abs{ \alpha_{ij} - \phi^\prime(0)^2}    \lVert   \bs{S}_1   \rVert_2  \\
&  \labelrel\leq{conseq:ineqq3}  C  \phi^\prime(0)^2 \Big(  \frac{1}{N V  L^3} +   \frac{1}{V^2 L^2} \frac{1}{V \wedge L^2 \wedge L \sqrt{d} } \Big),
}
where we used  \ref{cons:coeffbound} and \ref{event:discrete5} in \eqref{conseq:ineqq3}. By \ref{event:boundZin1}, we have
\eq{
\norm{\Delta_{1, ir}}_2  = \norm{ \Zin   \bar{\Delta}_{1, ir} \Zin^\top }_2  \leq C  \phi^\prime(0)^2 \Big(  \frac{1}{N d  L^3} +   \frac{1}{V d L^2} \frac{1}{V \wedge L^2 \wedge L \sqrt{d} } \Big).
}
Moreover, we define
\eq{
  \bar{\Delta}_{2, ir} & \coloneqq   \Big( \frac{1}{NL} \sum_{j = 1}^N  (\alpha_{ij} - \phi^\prime(0)^2 ) (\bN_j^\top  - \tfrac{1}{V} \mathbbm{1}_V  \mathbbm{1}_{L -1}^\top  )   \mathbbm{1}_{L -1}    ( \bx_j -   \tfrac{1}{V} \mathbbm{1}_{V} )^\top     \Big)  \\
& \hspace{8em} \times \Big( \frac{1}{NL} \sum_{j = 1}^N  \phi^\prime(0)^2     ( \bx_j -   \tfrac{1}{V} \mathbbm{1}_{V} )  \mathbbm{1}_{L -1}^\top (\bN_j^\top  - \tfrac{1}{V} \mathbbm{1}_V  \mathbbm{1}_{L -1}^\top  )  ^\top        \Big)  \\
& \quad  + 
 \Big( \frac{1}{NL} \sum_{j = 1}^N   \phi^\prime(0)^2   (\bN_j^\top  - \tfrac{1}{V} \mathbbm{1}_V  \mathbbm{1}_{L -1}^\top  )   \mathbbm{1}_{L -1}    ( \bx_j -   \tfrac{1}{V} \mathbbm{1}_{V} )^\top    \Big)   \\
 & \hspace{8em} \times  \Big( \frac{1}{NL} \sum_{j = 1}^N    (\alpha_{rj} - \phi^\prime(0)^2 )    ( \bx_j -   \tfrac{1}{V} \mathbbm{1}_{V} )  \mathbbm{1}_{L -1}^\top (\bN_j^\top  - \tfrac{1}{V} \mathbbm{1}_V  \mathbbm{1}_{L -1}^\top  )  ^\top        \Big)  \\
& \quad  + 
 \Big( \frac{1}{NL} \sum_{j = 1}^N    (\alpha_{ij} - \phi^\prime(0)^2 )  (\bN_j^\top  - \tfrac{1}{V} \mathbbm{1}_V  \mathbbm{1}_{L -1}^\top  )   \mathbbm{1}_{L -1}    ( \bx_j -   \tfrac{1}{V} \mathbbm{1}_{V} )^\top       \Big)    \\
 & \hspace{8em} \times  \Big( \frac{1}{NL} \sum_{j = 1}^N    (\alpha_{rj} - \phi^\prime(0)^2 )      ( \bx_j -   \tfrac{1}{V} \mathbbm{1}_{V} )  \mathbbm{1}_{L -1}^\top (\bN_j^\top  - \tfrac{1}{V} \mathbbm{1}_V  \mathbbm{1}_{L -1}^\top  )  ^\top       \Big). 
}
We have
\eq{
\norm{   \bar{\Delta}_{2, ir}  }_2  & \leq 
  \phi^\prime(0)^2     \lVert   \bs{S}_2    \rVert_2^{\frac{1}{2}}      \Big \lVert     \frac{1}{NL} \sum_{ j = 1 }^N    (\alpha_{rj} - \phi^\prime(0)^2 )    ( \bx_j -   \tfrac{1}{V} \mathbbm{1}_{V} )  \mathbbm{1}_{L -1}^\top (\bN_j^\top  - \tfrac{1}{V} \mathbbm{1}_V  \mathbbm{1}_{L -1}^\top  )  ^\top          \Big \rVert_2   \\
& +  \phi^\prime(0)^2   \lVert   \bs{S}_2    \rVert_2^{\frac{1}{2}}   \Big \lVert   \frac{1}{NL} \sum_{ j = 1 }^N    (\alpha_{ij} - \phi^\prime(0)^2 )    ( \bx_j -   \tfrac{1}{V} \mathbbm{1}_{V} )  \mathbbm{1}_{L -1}^\top (\bN_j^\top  - \tfrac{1}{V} \mathbbm{1}_V  \mathbbm{1}_{L -1}^\top  )  ^\top         \Big \rVert_2   \\
 & +  \Big \lVert     \frac{1}{NL} \sum_{ j = 1  }^N    (\alpha_{rj} - \phi^\prime(0)^2 )    ( \bx_j -   \tfrac{1}{V} \mathbbm{1}_{V} )  \mathbbm{1}_{L -1}^\top (\bN_j^\top  - \tfrac{1}{V} \mathbbm{1}_V  \mathbbm{1}_{L -1}^\top  )  ^\top         \Big \rVert_2   \\
 & \hspace{3em} \times \Big \lVert     \frac{1}{NL} \sum_{ j = 1 }^N    (\alpha_{ij} - \phi^\prime(0)^2 )    ( \bx_j -   \tfrac{1}{V} \mathbbm{1}_{V} )  \mathbbm{1}_{L -1}^\top (\bN_j^\top  - \tfrac{1}{V} \mathbbm{1}_V  \mathbbm{1}_{L -1}^\top  )  ^\top      \Big \rVert_2.
}
We observe that
\eq{
 \Big \lVert     \frac{1}{NL} \sum_{ j = 1 }^N    (\alpha_{ij} & - \phi^\prime(0)^2 )  ( \bx_j -   \tfrac{1}{V} \mathbbm{1}_{V} )   \mathbbm{1}_{L -1}^\top (\bN_j^\top  - \tfrac{1}{V} \mathbbm{1}_V  \mathbbm{1}_{L -1}^\top  )  ^\top      \Big \rVert_2 
 \\ 
 &  \leq  \Big \lVert     \frac{1}{NL}      (\alpha_{ii} - \phi^\prime(0)^2 )    ( \bx_i -   \tfrac{1}{V} \mathbbm{1}_{V} )  \mathbbm{1}_{L -1}^\top (\bN_i^\top  - \tfrac{1}{V} \mathbbm{1}_V  \mathbbm{1}_{L -1}^\top  )  ^\top      \Big \rVert_2 \\
 & + \Big \lVert     \frac{1}{NL} \sum_{ \substack{j = 1 \\ j \neq i } }^N    (\alpha_{ij} - \phi^\prime(0)^2 )    ( \bx_j -   \tfrac{1}{V} \mathbbm{1}_{V} )  \mathbbm{1}_{L -1}^\top (\bN_j^\top  - \tfrac{1}{V} \mathbbm{1}_V  \mathbbm{1}_{L -1}^\top  )  ^\top      \Big \rVert_2 \\ 
 & \labelrel\leq{conseq:ineqq4} \frac{C}{NL\sqrt{L}} +
 \frac{\sup_{i \neq j} \lvert \alpha_{ij} - \phi^\prime(0)^2 \rvert}{\sqrt{L}} \Big \lVert   \frac{1}{N } \sum_{ \substack{j = 1 \\ j \neq i } }^N        ( \bx_j -   \tfrac{1}{V} \mathbbm{1}_{V} )     ( \bx_j -   \tfrac{1}{V} \mathbbm{1}_{V} )^\top   \Big \rVert_2 \\
  & +
 \frac{\sup_{i \neq j} \lvert \alpha_{ij} - \phi^\prime(0)^2 \rvert}{\sqrt{L}} \Big \lVert   \frac{1}{NL } \sum_{ \substack{j = 1 \\ j \neq i } }^N         ( \bN_j^\top -   \tfrac{1}{V} \mathbbm{1}_{L-1} )   \mathbbm{1}_{L-1} \mathbbm{1}_{L-1}^\top  ( \bN_j^\top -   \tfrac{1}{V} \mathbbm{1}_{L-1} )^\top  \Big \rVert_2 \\
 & \labelrel\leq{conseq:ineqq5}   \frac{C}{NL\sqrt{L}} +   \frac{C}{V\sqrt{L}}  \frac{1}{V \wedge L^2 \wedge L\sqrt{d}}.
}
where we used \ref{cons:coeffbound} and \ref{event:discrete1} in \eqref{conseq:ineqq4}, and 
\ref{cons:coeffbound}, \ref{event:discrete5} and \ref{event:discrete8} in \eqref{conseq:ineqq5}. Then, by \ref{event:discrete6}, we have
\eq{
  \norm{   \bar{\Delta}_{2, ir}  }_2 \leq   \frac{C}{\sqrt{NVL}}     \Big(   \frac{1}{NL^{\frac{3}{2}}} + \frac{1}{V \sqrt{L}}  \frac{1}{V \wedge L^2 \wedge L \sqrt{d}} \Big) +  C^2  \Big(   \frac{1}{NL^{\frac{3}{2}}} + \frac{1}{V \sqrt{L}}  \frac{1}{V \wedge L^2 \wedge L \sqrt{d}} \Big)^2.
}
Therefore, by \ref{event:boundZin1}
\eq{
\norm{\Delta_{2, ir}}_2  = \norm{ \Zin   \bar{\Delta}_{2, ir} \Zin^\top }_2  \leq \frac{C \sqrt{V}}{ d \sqrt{NL}}     \Big(   \frac{1}{NL^{\frac{3}{2}}} + \frac{1}{V \sqrt{L}}  \frac{1}{V \wedge L^2 \wedge L \sqrt{d}} \Big) .
}
Lastly, we define
\eq{
  \bar{ \Delta }_{3,ir} & \coloneqq       \Big( \frac{1}{N} \sum_{j = 1}^N (  \alpha_{ij} - \phi^\prime(0)^2   )   ( \bx_j -   \tfrac{1}{V} \mathbbm{1}_{V}    )   \Big) ^\top   \Big( \frac{1}{N} \sum_{j = 1}^N  \phi^\prime(0)^2 (\bx_j  - \tfrac{1}{V} \mathbbm{1}_V )          \Big)    \\
 &+     \Big( \frac{1}{N} \sum_{j = 1}^N  \phi^\prime(0)^2  ( \bx_j -   \tfrac{1}{V} \mathbbm{1}_{V}    )   \Big) ^\top   \Big( \frac{1}{N} \sum_{j = 1}^N  (\alpha_{rj}  -  \phi^\prime(0)^2) (\bx_j  - \tfrac{1}{V} \mathbbm{1}_V )          \Big)     \\
& +     \Big( \frac{1}{N} \sum_{j = 1}^N  (  \alpha_{ij} - \phi^\prime(0)^2   )   ( \bx_j -   \tfrac{1}{V} \mathbbm{1}_{V}    )   \Big) ^\top   \Big( \frac{1}{N} \sum_{j = 1}^N   (  \alpha_{rj} - \phi^\prime(0)^2   )   (\bx_j  - \tfrac{1}{V} \mathbbm{1}_V )          \Big)   
}
We have
\eq{
\lvert  \bar{\Delta}_{3,ir}  \rvert &\leq   \phi^\prime(0)^2 \Big \lVert   \frac{1}{N} \sum_{j = 1}^N  ( \bx_j -   \tfrac{1}{V} \mathbbm{1}_{V}    )   \Big \rVert_2
  \Big \lVert    \frac{1}{N} \sum_{ j = 1 }^N (  \alpha_{ij} - \phi^\prime(0)^2   )   ( \bx_j -   \tfrac{1}{V} \mathbbm{1}_{V}    )     \Big \rVert_2 \\
 & +  \phi^\prime(0)^2 \Big \lVert   \frac{1}{N} \sum_{j = 1}^N  ( \bx_j -   \tfrac{1}{V} \mathbbm{1}_{V}    )   \Big \rVert_2
  \Big \lVert    \frac{1}{N} \sum_{ j = 1  }^N (  \alpha_{rj} - \phi^\prime(0)^2   )   ( \bx_j -   \tfrac{1}{V} \mathbbm{1}_{V}    )     \Big \rVert_2  \\
& +  \Big \lVert    \frac{1}{N} \sum_{ j = 1 }^N (  \alpha_{ij} - \phi^\prime(0)^2   )   ( \bx_j -   \tfrac{1}{V} \mathbbm{1}_{V}    )   \underbrace{   \Big \rVert_2     \Big \lVert    \frac{1}{N} \sum_{ j = 1 }^N (  \alpha_{rj} - \phi^\prime(0)^2   )   ( \bx_j -   \tfrac{1}{V} \mathbbm{1}_{V}    )     \Big \rVert_2 }_{\leq \frac{C}{N L} + \frac{C}{\sqrt{N}} \frac{1}{V \wedge L^2 \wedge L \sqrt{d}}} \\
& \leq \frac{C  \phi^\prime(0)^2}{N}  \Big( \frac{1}{N L} + \frac{1}{\sqrt{N}} \frac{1}{V \wedge L^2 \wedge L \sqrt{d}} \Big) +   \Big( \frac{1}{N L} + \frac{1}{\sqrt{N}} \frac{1}{V \wedge L^2 \wedge L \sqrt{d}} \Big)^2. \label{eq:deltabar3bound}
}
We observe that $\Delta_{3, ir} =  \frac{\bar{\Delta}_{3,ir}}{V^2 L^2} \Zin \mathbbm{1}_V\mathbbm{1}_V^\top \Zin^\top$, and by \eqref{eq:deltabar3bound}, the last result follows.
\item For    \ref{cons:matrixbounds2},  we assume  \ref{event:boundZin1},  \ref{event:boundZin2},  \ref{event:discrete3},  and \ref{event:discrete5}-\ref{event:discrete7}.  We write
\eq{
 \bs{A}_{1, ir} & -  \frac{  \phi^\prime(0)^4 }{d}    \frac{(1- \frac{1}{V})}{L^2} \big(   \frac{1}{N} +   \frac{1}{V} \big)  \bs{I}_d  \\
 & = \Delta_{1, ir} + \phi^\prime(0)^4  \Big( \bZin \bs{S}_1 \bZin  \pm  \frac{\tr(\bs{S}_1)}{d} \bs{I}_d -  \frac{  1}{d} \frac{(1- \frac{1}{V})}{L^2}\big(     \frac{1}{N} +   \frac{1}{V} \big)  \bs{I}_d    \Big).
}
We have
\eq{
& \Big \lVert  \bs{A}_{1, ir}   -  \frac{  \phi^\prime(0)^4 }{d} \frac{ (1- \frac{1}{V}) }{L^2} \big(   \frac{1}{N} +   \frac{1}{V} \big)  \bs{I}_d \Big \rVert_2 \\
& \leq \norm{ \Delta_{1, ir} }_2 + 2   \phi^\prime(0)^4 \log V  \frac{\lVert \bs{S}_1 \rVert_F}{\sqrt{d}} 
 +  \frac{\lvert \tr(\bs{S}_1)  -    \frac{(1- \frac{1}{V}) }{L^2} \big(   \frac{1}{N} +  \frac{1}{V} \big)   \rvert}{d}    \\
& \leq    C_K  \phi^\prime(0)^2 \Big(  \frac{1}{N d  L^3} +   \frac{1}{V d L^2} \frac{1}{V \wedge L^2 \wedge L \sqrt{d} } \Big)  
 +    C_K  \phi^\prime(0)^4   \Big(   \frac{\log V}{L^2 V^{3/2} \sqrt{d}} +   \frac{  \log^2 V}{L^2 N \sqrt{V} d}    \Big).
} 
Similarly,
\eq{
& \Big \lVert  \bs{A}_{2, ir}    -  \frac{  \phi^\prime(0)^4 }{d}   (1 - \frac{1}{V})^2 \frac{L - 1}{L^2 N}   \bs{I}_d \Big \rVert_2 \\
& \leq \norm{ \Delta_{2, ir} }_2 + 2   \phi^\prime(0)^4 \log V  \frac{\lVert \bs{S}_2 \rVert_F}{\sqrt{d}} 
 +  \frac{\lvert \tr(\bs{S}_2)  -     (1 - \frac{1}{V})^2 \frac{L - 1}{L^2 N} \rvert}{d}    \\
& \leq     \frac{C_K \sqrt{V}}{ d \sqrt{NL}}     \Big(   \frac{1}{NL^{\frac{3}{2}}} + \frac{1}{V \sqrt{L}}  \frac{1}{V \wedge L^2 \wedge L \sqrt{d}} \Big)
 +    C_K  \phi^\prime(0)^4   \Big(   \frac{\log V}{N L \sqrt{V d}} +   \frac{ \log^3 V}{N \sqrt{LV} d}     \Big).
}
Lastly,
\eq{
& \bs{A}_{3, ir}    -  \frac{  \phi^\prime(0)^4 }{N}   \frac{1}{V^2 L^2} \bZin \mathbbm{1}_V \mathbbm{1}_V^\top \bZin^\top   =    \Delta_{3, ir} +     \phi^\prime(0)^4   \Big( \Big \lVert \frac{1}{N} \sum_{j = 1}^N    ( \bx_j -   \frac{1}{V} \mathbbm{1}_{V}    )   \Big \rVert_2^2   - \frac{1}{N}     \Big)  \frac{1}{V^2 L^2}  \bZin \mathbbm{1}_V \mathbbm{1}_V^\top \bZin^\top.
}
By \ref{event:discrete3}, we have
\eq{
 \frac{C K^2 \log^2 V}{N \sqrt{V}}   \frac{1}{V^2 L^2}  \bZin \mathbbm{1}_V \mathbbm{1}_V^\top \bZin^\top & \preceq \Big( \Big \lVert \frac{1}{N} \sum_{j = 1}^N    ( \bx_j -   \frac{1}{V} \mathbbm{1}_{V}    )   \Big \rVert_2^2   - \frac{1}{N}     \Big)  \frac{1}{V^2 L^2}  \bZin \mathbbm{1}_V \mathbbm{1}_V^\top \bZin^\top \\
 & \preceq \frac{C K^2 \log^2 V}{N \sqrt{V}}   \frac{1}{V^2 L^2}  \bZin \mathbbm{1}_V \mathbbm{1}_V^\top \bZin^\top.
}
 By \ref{cons:matrixbounds1}, the result follows.
\end{itemize}
\end{proof}

\begin{proposition}
\label{prop:intermediateresults1}
We recall that  $ \mathsf{z}_k  = \bs{z}_k  +   \indic{k = 1}   \ztrig$.    Given that   \ref{event:boundZin} holds,  the following statements hold:
\begin{enumerate}[leftmargin=*, label= (P\arabic*)]
\item \label{event:prop8result1}
We have for $i \neq j$ and any $k \in [V]$,
\eq{
\Big \lvert  \E \big[ (\indic{\bx_i = \bx_j} - \tfrac{1}{V} )   \mathsf{z}_k^\top    \bZin  \bX_i^\top      \bX_i  \bZin^\top    \bZin    \bX_j^\top   \bX_j  \bZin^\top \mathsf{z}_{k}  \vert \sZin  \big]  \Big \rvert   \leq   \frac{C}{V d}. 
 }
\item 
\label{event:prop8result2} For any  $k \in [V]$,
\eq{
 \E[    \norm{   \bZin    \bX_i^\top   \bX_i \bZin^\top    \mathsf{z}_{k}  }_2^2 \vert \sZin  \big] \leq C \Big( \frac{L}{d} + \frac{L^2}{d^2} \Big).
}
\item 
\label{event:prop8result3}
We have   for $i \neq j$ and any $k \in [V]$,
\eq{
\Big \lvert  \E \big[  (\indic{\bx_i = \bx_j}  -  \tfrac{1}{V} ) 
 \mathsf{z}_{k}^\top    \bZin  \bX_i^\top      \bX_i  \bZin^\top   \bZin \mathbbm{1}_V  \mathbbm{1}_V^\top  \bZin^\top     \bZin    \bX_r^\top   \bX_r \bZin^\top  \mathsf{z}_{k} \vert \sZin \Big]  \Big \rvert \leq    \frac{ C \log^2 V}{d^2}. 
  }
  \item
  \label{event:prop8result4}
  For any $k \in [V]$,
\eq{
 \E \big[     \big( \mathbbm{1}_V^\top \bZin ^\top \bZin    \bX_i^\top   \bX_i \bZin^\top  \mathsf{z}_{k} \big )^2  \vert \sZin \Big]  \leq C   \frac{  V \log^2 V}{d} \Big(  \frac{L  }{d } + \frac{L^2}{d^2}  \Big).
  }
\item 
\label{event:prop8result5}
For notational convenience,  let 
\eq{
\varsigma \coloneqq \mathsf{z}_{k}^\top    \bZin    \Big(     \bX_i^\top     \bX_i  - \frac{L}{V} \bs{I}_V \Big)   \bZin^\top      \bZin  \big(  \bX_i^\top  - \frac{1}{V} \mathbbm{1}_V \mathbbm{1}_L^\top \big)  \mathbbm{1}_L.
}
For any $i, k \in [V]$,
\eq{
  \lvert   \E  [  \varsigma \vert \sZin  ]   \rvert   \leq \frac{C L \log V}{\sqrt{V d}} ~~ \text{and} ~~ \E  [   \varsigma^2    \vert   \sZin ]  \leq   C \log^2 V  \Big( \frac{L}{d} + \frac{L^2}{d^2} \Big).
}
\end{enumerate}
\end{proposition}

\begin{proof}
For the first item,   we have
\eq{
 \E \big[ (\indic{\bx_i = \bx_j} - \tfrac{1}{V} )  & \mathsf{z}_{k}^\top   \bZin  \bX_i^\top       \bX_i  \bZin^\top    \bZin    \bX_j^\top   \bX_j \bZin^\top  \mathsf{z}_{k} \vert \sZin  \big] \\
 & \labelrel={intm:eqq0} \E \big[ (\indic{\bx_i = \bx_j} - \tfrac{1}{V} )   \mathsf{z}_{k}^\top \bZin  \bx_i      \bx_i^\top  \bZin^\top    \bZin    \bx_j \bx_j^\top     \bZin^\top   \mathsf{z}_{k} \vert \sZin  \big]  \\
 & =    \frac{1}{V} \E \big[    \mathsf{z}_{k}^\top   \bZin  \bx_i     \bx_i^\top   \bZin^\top    \bZin    \bx_i \bx_i^\top   \bZin^\top    \mathsf{z}_{k} \vert \sZin   \big]   - \frac{1}{V^3}     \mathsf{z}_{k}^\top  \bZin   \bZin^\top    \bZin    \bZin^\top    \mathsf{z}_{k}  \\
& \leq  \frac{C}{V d},
}
where we used the independence of the rows of $\bs{X}$ in \eqref{intm:eqq0}. 
For the second item, we write
\eq{
 \E[   \mathsf{z}_{k}^\top  \bZin   \bX_i^\top   \bX_i   \bZin^\top \bZin    \bX_i^\top   \bX_i \bZin^\top  \mathsf{z}_{k} \vert \sZin    ]   
& = \frac{L}{V}   \mathsf{z}_{k}^\top   \bZin  \mathrm{diag}( \bZin^\top \bZin ) \bZin^\top   \mathsf{z}_{k} \\
& +  \frac{L (L-1)}{V^2}    \mathsf{z}_{k}^\top  \bZin    \bZin^\top \bZin   \bZin^\top    \mathsf{z}_{k}^\top  \\   
& \leq C \Big(  \frac{L}{d} + \frac{L^2}{d^2}   \Big).
}
For the third item, we have
\eq{
 \E \big[  (\indic{\bx_i = \bx_j} - \tfrac{1}{V} ) 
 \mathsf{z}_{k}^\top  & \bZin  \bX_i^\top      \bX_i  \bZin^\top   \bZin \mathbbm{1}_V  \mathbbm{1}_V^\top  \bZin^\top      \bZin    \bX_j^\top   \bX_j \bZin^\top  \mathsf{z}_{k}  \vert \sZin \big]    \\
 & =  \E \big[  (\indic{\bx_i = \bx_j} - \tfrac{1}{V} ) 
 \mathsf{z}_{k}^\top  \bZin  \bx_i      \bx_i^\top  \bZin^\top   \bZin \mathbbm{1}_V  \mathbbm{1}_V^\top  \bZin^\top    \bZin  \bx_j \bx_j^\top  \bZin^\top  \mathsf{z}_{k}  \vert \sZin  \big]  \\
  & =  \frac{1}{V} \E \big[   
 \mathsf{z}_{k}^\top  \bZin  \bx_i      \bx_i^\top  \bZin^\top   \bZin \mathbbm{1}_V  \mathbbm{1}_V^\top  \bZin^\top    \bZin  \bx_i \bx_i^\top  \bZin^\top  \mathsf{z}_{k}  \vert \sZin \big]   \\  
&  -  \frac{1}{V^3}  
 \mathsf{z}_{k}^\top  \bZin  \bZin^\top   \bZin \mathbbm{1}_V  \mathbbm{1}_V^\top  \bZin^\top    \bZin    \bZin^\top  \mathsf{z}_{k}   \\
&  \leq  C   \frac{\log^2 V}{d^2}. 
}
For the fourth item, we have
\eq{
\E \big[     \mathsf{z}_{k}^\top   \bZin      \bX_i^\top   \bX_i  \bZin ^\top \bZin       \mathbbm{1}_V   & \mathbbm{1}_V^\top \bZin^\top \bZin    \bX_i^\top   \bX_i \bZin^\top   \mathsf{z}_{k} \vert \sZin  \Big] \\
 &= L   \E \big[   \mathsf{z}_{k}^\top   \bZin      \bx_i   \bx_i^\top  \bZin ^\top \bZin       \mathbbm{1}_V   \mathbbm{1}_V^\top \bZin^\top    \bZin     \bx_i   \bx_i^\top \bZin^\top  \mathsf{z}_{k}  \vert \sZin     \Big]    \\
 & + \frac{L (L - 1)}{V^2}     \mathsf{z}_{k}^\top  \bZin      \bZin ^\top  \bZin     \mathbbm{1}_V   \mathbbm{1}_V^\top   \bZin^\top\bZin   \bZin^\top     \mathsf{z}_{k}   \\
 & \leq C\Big(  \frac{L V \log^2 V}{d^2} + \frac{L^2 V \log^2 V}{d^3}  \Big).
  }
For the fifth item, we have
\eq{
  \E \Big[   \mathsf{z}_{k}^\top     \bZin    \Big(     \bX_i^\top     \bX_i  - & \tfrac{L}{V} \bs{I}_V \Big)   \bZin^\top       \bZin    \bX_i^\top \mathbbm{1}_L \vert \sZin  \Big]   \\   
 & =  \E \Big[    \mathsf{z}_{k}^\top   \bZin       \bX_i^\top     \bX_i     \bZin^\top      \bZin    \bX_i^\top     \bX_i  \vert \sZin  \Big]  \mathbbm{1}_V - \frac{L^2}{V^2} \mathsf{z}_{k}^\top    \bZin        \bZin^\top      \bZin      \mathbbm{1}_V  \\
  & = L   \mathsf{z}_{k}^\top   \bZin   \E \Big[       \bx_i     \bx_i^\top     \bZin^\top      \bZin \bx_i     \bx_i^\top \vert \sZin  \Big]  \mathbbm{1}_V 
  - \frac{L^2}{V^2}     \mathsf{z}_{k}^\top    \bZin        \bZin^\top      \bZin      \mathbbm{1}_V.   \label{eq:fifitemexp} 
}
By \ref{event:boundZin}, we have
\eq{
\abs{\eqref{eq:fifitemexp}} \leq \frac{C L \log V}{\sqrt{V d}}.
}
For the second part,  let $n_{j}   \coloneqq   \bs{e}_j^\top   \bX_i  \mathbbm{1}_L$. We have  
\eq{
 \mathsf{z}_{k}^\top    \bZin        \big(     \bX_i^\top     \bX_i  - \frac{L}{V} \bs{I}_V \big) &    \bZin^\top      \bZin  \big(  \bX_i^\top - \frac{1}{V}   \mathbbm{1}_V  \mathbbm{1}_L^\top \big) \mathbbm{1}_L   \\
& = \sum_{j = 1}^V     (n_j - \tfrac{L}{V})     \mathsf{z}_{k}^\top   \bs{z}_j   \bs{z}_j^\top \Big( \sum_{l = 1}^V       (n_l  - \tfrac{L}{V})    \bs{z}_l ) \Big) \\
& =    \sum_{j = 1}^V  \sum_{l = 1}^V      (n_{j} - \tfrac{L}{V})     (n_l  - \tfrac{L}{V})       \mathsf{z}_{k}^\top   \bs{z}_i   \bs{z}_i^\top \bs{z}_j.  \label{eq:quadraticS}
}
Let $\bs{S} = (s_{jl})_{jl \in [V]}$ such that     $s_{jl} \coloneqq \frac{1}{2} \big( \mathsf{z}_{k}^\top   \bs{z}_j   \bs{z}_j^\top    \bs{z}_l +   \mathsf{z}_{k}^\top   \bs{z}_l  \bs{z}_l^\top    \bs{z}_j  \big)$. We will use the third item in Proposition \ref{prop:multimatrixexpectation} to bound second moment of \eqref{eq:quadraticS}. We bound each term separately below.
\begin{itemize}[leftmargin=*]
\item  We have
\eq{
\Big \lvert \tr \big( \big(\bs{I}_V - \frac{1}{V} \mathbbm{1}_V \mathbbm{1}_V^\top \big) \bs{S} \big) \Big \rvert  & =  \Big \lvert  \tr ( \bs{S}  )  -   \frac{1}{V}  \mathbbm{1}_V^\top   \bs{S} \mathbbm{1}_V \Big \rvert  \\
&  = V  \Big \lvert    \mathsf{z}_{k}^\top     \bZin      \E[ \bx_1  \bx_1^\top       \bZin ^\top           \bZin    \bx_1   ]   -   \frac{1}{V}   \mathsf{z}_{k}^\top     \bZin      \bZin ^\top           \bZin  \mathbbm{1}_V  \Big \rvert \\
& \leq  \frac{C\log V \sqrt{V}}{\sqrt{d}}.
}
\item  Moreover,
\eq{
& \tr  \big( \big(\bs{I}_V - \frac{1}{V} \mathbbm{1}_V \mathbbm{1}_V^\top \big) \bs{S}  \big(\bs{I}_V - \frac{1}{V} \mathbbm{1}_V \mathbbm{1}_V^\top \big) \bs{S} \big)  =  \tr  (  \bs{S}^2  ) - \frac{2}{V} \lVert \bs{S }  \mathbbm{1}_V  \rVert_2^2 +  \frac{1}{V^2}\big( \mathbbm{1}_V^\top  \bs{S}   \mathbbm{1}_V \big)^2.
}
We have $\tr (  \bs{S}^2   ) \leq  \frac{C V^2 \log^2 V }{d^2}$ and  
\eq{
\bs{e}_{i} ^\top \bs{S }  \mathbbm{1}_V =   \frac{1}{2} \sum_{l = 1}^V \mathsf{z}_{k}^\top   \bs{z}_j   \bs{z}_j^\top    \bs{z}_l +     \frac{1}{2} \sum_{l = 1}^V  \mathsf{z}_{k}^\top   \bs{z}_l   \bs{z}_l^\top    \bs{z}_j   =    \mathsf{z}_{k}^\top   \bs{z}_j   \bs{z}_j^\top \bZin   \mathbbm{1}_V  +    \mathsf{z}_{k}^\top   \bZin    \bZin  ^\top  \bs{z}_j.
}
Therefore,
\eq{
\Big \lvert  \tr \big( \big(\bs{I}_V - \frac{1}{V} \mathbbm{1}_V \mathbbm{1}_V^\top \big) \bs{S}  \big(\bs{I}_V - \frac{1}{V} \mathbbm{1}_V \mathbbm{1}_V^\top \big) \bs{S} \big)  \Big \rvert   \leq \frac{C V^2 \log^2 V }{d^2}.
}
\item Moreover,  $\lVert \mathrm{diag}(\bs{S}) \rVert^2_2  \leq \frac{C V \log^2 V }{d}.$ 
\end{itemize}
Therefore, by Proposition \ref{prop:multimatrixexpectation}, we have
\eq{
\E \Big[  \Big(  \mathsf{z}_{k}^\top     \bZin        \big(     \bX_i^\top     \bX_i  - \frac{L}{V} \bs{I}_V \big)   \bZin^\top      \bZin  \big(  \bX_i^\top - \frac{1}{V}   \mathbbm{1}_V  \mathbbm{1}_L^\top \big) \mathbbm{1}_L  \Big)^2   \Big \vert  \sZin \Big] \leq C \log^2 V  \Big( \frac{L}{d} + \frac{L^2}{d^2} \Big).
}
 
\end{proof}

\section{Proof of Theorem \ref{thm:main}}
\label{sec:thmmain}

We consider the following technical assumptions for the subsequent proof.

\begin{ass}[Technical conditions]
\label{ass:conditions}
We work under the following conditions:
\begin{itemize}[leftmargin = *]
\item \textbf{Permutation.} Without loss of generality, assume $\bs{\Pi}=\bs{I}_V$.
\item \textbf{Learning rates.} Take $\eta=o_V(1)$, chosen sufficiently small so that we can write $\hat{\bs{p}}_1=\frac{1}{V}\mathbbm{1}_V+ O(\eta)$.
\item \textbf{Activation.}  We consider a polynomial activation $\phi$ with a degree of $p_{\star}$  satisfying:
\begin{itemize}[leftmargin=*]
\item $\phi(0), \phi^\prime(0), \phi^{\prime\prime}(0) \neq 0$
\item The smallest non-zero Hermite component of $\phi$ has index $q_\star$, i.e,  $q^\star \coloneqq \min \{ k > 0 \vert \E[ \phi(Z) H_{e_{k}}(Z) ] \neq 0 \}$, for $Z \sim N(0,1)$.
\end{itemize}
\end{itemize}
\end{ass} 

Since the learning algorithm does not assume any structure in the ground-truth permutation, we may, without loss of generality, take it to be the identity. This simplifies the notation in the analysis below. The learning rate $\eta$ is chosen sufficiently small so that the network output remains close to its initialization $\frac{1}{V}\mathbbm{1}_V$, which simplifies the analysis of the three-step gradient descent algorithm. The assumption on the activation function is technical and is needed for the analysis of the three-step gradient descent dynamics; however, we believe that such an assumption would not be necessary for general multi-step training.

\subsection{Attention scores and their asymptotic scaling}

Let $\sX$ be an independent copy of input sequence. 
By using the technical condition above,  we decompose the attention scores in to three terms   $\bs{s}_1 ,  \bs{s}_2,  \bs{s}_3 \in \R^L$:
\eq{
 &   \sX \sZin^\top \bWkq^{(1)}   \\
&  = \frac{\eta \gamma }{N^2 L^2}  \sX \sZin^\top  \sZin  \sum_{i, j = 1}^N    \alpha_{ij} \sX_i^\top \big(  \bs{I}_L - \tfrac{1}{L}   \mathbbm{1}_L \mathbbm{1}_L^\top  \big)    \bX_i  \bZin^\top   \bZin    \bX_j^\top \mathbbm{1}_L    ( \bx_j -   \tfrac{1}{V} \mathbbm{1}_{V} )^\top            \Zout^\top    \Zout    ( \bx_i -   \tfrac{1}{V} \mathbbm{1}_{V} ) \label{eq:s1full} \\[0.5em]
&  + \frac{\eta \gamma }{N^2 L^2}  \sX \sZin^\top  \sZin  \sum_{i, j = 1}^N  \beta_{ij}  \sX_i^\top \big(  \bs{I}_L - \tfrac{1}{L}   \mathbbm{1}_L \mathbbm{1}_L^\top  \big)    \bX_i  \bZin^\top     \bZin    \bX_i^\top \mathbbm{1}_L      ( \bx_j -   \tfrac{1}{V} \mathbbm{1}_{V} )^\top          \Zout^\top    \Zout    ( \bx_i -   \tfrac{1}{V} \mathbbm{1}_{V} ) \label{eq:s2full} \\[0.5em]
&   + \!\frac{\eta \gamma}{N^2 L}  \sX \sZin^\top   \sZin  \!\! \sum_{i, j = 1}^N    \!\!  \sX_i^\top \big(  \bs{I}_L \! - \! \tfrac{1}{L}   \mathbbm{1}_L \mathbbm{1}_L^\top  \big)    \bX_i  \bZin^\top  \FW(\bWin; \bZin,    \bX_i,    \bX_j)     ( \bx_j \! - \!   \tfrac{1}{V} \mathbbm{1}_{V} )^\top  \!          \Zout^\top    \Zout   ( \bx_i \! - \!  \tfrac{1}{V} \mathbbm{1}_{V} )  ~~~~~ ~~\label{eq:s3full} \\
& + O(\eta^2 \gamma \mathrm{poly}(d, N)) \\[0.75em]
&   \eqqcolon  \eta \gamma    \big( \bs{s}_1 +     \bs{s}_2 +   \bs{s}_3 \big) + O(\eta^2 \gamma \mathrm{poly}(d, N)).  \label{eq:scoresexplicit}
}
The following theorem characterizes the scaling of each term. We recall that $\{\bs{e}_1, \cdots, \bs{e}_L\}$ denotes the standard basis vectors in $\R^L$.
 
\begin{theorem}
\label{thm:appendixmain}
With probability at least $1 - o_V(1)$, we have the following:
\eq{
& \hspace{-9.5em} \bullet \text{For all $l \in [L]$, } \Big \lvert  \inner{\bs{e}_l}{\bs{s}_1}  - \frac{\indic{l = 1}}{VL^2}  \Big \rvert  \lesssim  \frac{\indic{l = 1}}{\sqrt{N V} L^{3/2}	d }   +  \frac{1}{N \sqrt{L} d (d \wedge L^2)^{1/2} (d \wedge L)^{1/2}},   \\
& \hspace{-9.5em} \bullet   \lVert  \bs{s}_2   \rVert_{\infty} \lesssim    \frac{1}{N  \sqrt{L} d  (L \wedge d)} +   \frac{1}{N L d (L \wedge d)^{1/2}}   \\ 
& \hspace{-9.5em} \bullet   \lVert \bs{s}_3  \rVert_{\infty}  \lesssim \frac{   1}{N d \sqrt{m}}.
}
\end{theorem}

We first make an observation that we will frequently rely on in the following:
\begin{proposition}
\label{prop:polynomialbound}
For any $p  \in \N$, we have
\eq{
\E[ \norm{ \bs{A}_{1, ir} }^p_2 ]  \vee  \E[ \norm{ \bs{A}_{2, ir} }_2^p ] \vee \E[ \norm{ \bs{A}_{3, ir} }^p_2 ] \leq \mathrm{poly}_{p, p_\star}(d,V,L).
}
\end{proposition}

\begin{proof}
By Proposition \ref{prop:scaledhermiteexpansion}, we observe that $\alpha_{ij} \leq   \mathrm{poly}_{p_\star}(d,V,L)$. Therefore,   we have
\eq{
 \norm{ \bs{A}_{1, ir} }_2 \vee  \norm{ \bs{A}_{2, ir} }_2 \vee \norm{ \bs{A}_{3, ir} }_2 \leq   \mathrm{poly}_{p_\star}(d,V,L) \norm{\sZin \sZin^\top}_2,
}
from which the result follows.
\end{proof}

\section{Proof of Theorem \ref{thm:appendixmain}}
We observe that
\eq{
  \sX \sZin^\top  \sZin     \sX_i^\top & =  \Big(      \bZin \bX^\top +   \sZin \be_{V+1} \be_1^\top \Big)^\top \Big(     \bZin \bX_i^\top +  \sZin \be_{V+1} \be_1^\top \Big) \\
   &  =  \bX    \bZin^\top     \bZin \bX_i^\top +  \be_1 \ztrig^\top   \bZin \bX_i^\top  +   \bX  \bZin^\top \ztrig  \be_1^\top   +   \norm{\ztrig}_2^2  \be_1 \be_1^\top. \label{eq:triggerexpansion}
}
In the following,  we will consider  $ \bx_l = \be_{k}$,  for a fixed $k \in [V]$. We will write
\eq{
 \big(   \ztrig  \be_1^\top  +    \bZin \bX^\top  \big)  \bs{e}_l      = \bs{z}_{k} +   \indic{l = 1}    \ztrig =   \mathsf{z}_{k}, \label{def:triggerexpansionnoise}
}
and
\eq{
\big( \be_1   \ztrig^\top  \bZin      \bX^\top + \norm{\ztrig}_2^2  \be_1 \be_1^\top \big) \be_l = \underbrace{ \inner{\mathsf{z}_{k}}{ \ztrig  } }_{\eqqcolon \mu_{kl}} \be_1  =    \mu_{kl}     \be_1.  \label{def:triggerexpansionspike}
}
In the following, we will consider the event.
\eq{
\mathtt{Event} \coloneqq  \ref{event:boundZin} \cap \ref{event:discrete}.
}
\subsection{Concentration bound for $\bf{s}_1$  }
By \eqref{eq:triggerexpansion}-\eqref{def:triggerexpansionnoise}-\eqref{def:triggerexpansionspike}, we can write that
\eq{
  \inner{\bs{e}_l}{\bs{s}_1}  
& =  
 \frac{1 }{N^2 L^2}   \sum_{i, j = 1}^N \alpha_{ij}  \mathsf{z}_{k}^\top    \bZin       \bX_i^\top \big(  \bs{I}_L - \frac{1}{L}   \mathbbm{1}_L \mathbbm{1}_L^\top  \big)    \bX_i  \bZin^\top   \bZin    \bX_j^\top \mathbbm{1}_L     ( \bx_j -   \tfrac{1}{V} \mathbbm{1}_{V} )^\top             \Zout^\top    \Zout ( \bx_i -   \tfrac{1}{V} \mathbbm{1}_{V} )      \\     
  & +       \frac{  \mu_{kl}  }{N^2 L^2}   \sum_{i, j = 1}^N  \alpha_{ij}  \big(  \bs{e}_1 - \frac{1}{L}  \mathbbm{1}_L   \big)^\top    \bX_i  \bZin^\top       \bZin    \bX_j^\top \mathbbm{1}_L              ( \bx_j -   \tfrac{1}{V} \mathbbm{1}_{V} )^\top             \Zout^\top    \Zout ( \bx_i -   \tfrac{1}{V} \mathbbm{1}_{V} )  \\
&  \eqqcolon \vartheta  +  \varphi. 
}
We will analyze  $\vartheta$ and   $\varphi$ separately.   
We define 
\eq{
\bs{B}_i  &  \coloneqq    \bs{B}_{i,1}  + \bs{B}_{i,2} + \bs{B}_{i,3},  
}
where 
\eq{
\bs{B}_{i,1} &  \coloneqq      \Big( \frac{1}{NL} \sum_{j = 1}^N  \alpha_{ij}  (\bx_j  - \tfrac{1}{V} \mathbbm{1}_V )     ( \bx_j -   \tfrac{1}{V} \mathbbm{1}_{V} )^\top     \Big)  \\
\bs{B}_{i, 2}  &  \coloneqq      \Big( \frac{1}{NL} \sum_{j = 1}^N   \alpha_{ij} \big( \bN_j^\top - \tfrac{1}{V} \mathbbm{1}_V \mathbbm{1}_{L - 1}^\top \big) \mathbbm{1}_{L-1}   ( \bx_j -   \tfrac{1}{V} \mathbbm{1}_{V} )^\top     \Big)  \\
\bs{B}_{i, 3}  &  \coloneqq    \frac{1}{V} \mathbbm{1}_V    \Big(  \frac{1}{N L } \sum_{j = 1}^N   \alpha_{ij}   (  \bx_j -   \tfrac{1}{V} \mathbbm{1}_{V} )^\top     \Big)  . 
}

\subsubsection{Concentration bound for $\vartheta$}  
\label{sec:concentrations11}
We define
\eq{
\bs{C}_i  &  \coloneqq  \frac{1}{L}  ( \bx_i -   \tfrac{1}{V} \mathbbm{1}_{V} )  \mathsf{z}_{k}^\top   \bZin \bX_i^\top   \big(\bs{I}_L - \frac{1}{L} \mathbbm{1}_L  \mathbbm{1}_L^\top \big)  \bX_i  \bZin^\top      \bZin \\
& = \underbrace{  \frac{1}{L}  ( \bx_i -   \tfrac{1}{V} \mathbbm{1}_{V} )     \mathsf{z}_{k}^\top  \bZin \bX_i^\top    \bX_i  \bZin^\top      \bZin  }_{\coloneqq \bs{C}_{i,1}} -   \underbrace{   \frac{1}{L^2}  ( \bx_i -   \tfrac{1}{V} \mathbbm{1}_{V} )    \mathsf{z}_{k}^\top   \bZin \bX_i^\top   \mathbbm{1}_L  \mathbbm{1}_L^\top   \bX_i  \bZin^\top      \bZin   }_{\coloneqq \bs{C}_{i,2}}
}
By Chebyshev's inequality,  with  probability $1 - o_V(1)$,
\eq{
\vartheta    = \frac{1}{N}  \sum_{i = 1}^N \tr( \bs{B}_i   \Zout^\top \Zout \bs{C}_i  )
& =    \tr \Big( \Zout  \frac{1}{N}  \sum_{i = 1}^N  \bs{C}_i   \bs{B}_i   \Zout^\top \Big)\\
& =  \underbrace{ \tr \Big(   \frac{1}{N}  \sum_{i = 1}^N  \bs{C}_i   \bs{B}_i   \Big) }_{\vartheta_{1}} \pm \underbrace{ \frac{1}{\sqrt{d}} \Big \lVert  \frac{\log V}{N}  \sum_{i = 1}^N \bs{C}_i   \bs{B}_i  \Big \rVert_F }_{\vartheta_{2}}.
}
\paragraph{Bounding $\vartheta_2$:}
We start with bounding $\vartheta_{2}$ term. We have
\eq{
\vartheta_{2} \leq   \frac{1}{\sqrt{d}}   \Big \lVert  \frac{1}{N}  \sum_{i = 1}^N \bs{C}_i  \bs{B}_{i,1} \Big \rVert_F +  \frac{1}{\sqrt{d}} \Big \lVert  \frac{1}{N}  \sum_{i = 1}^N \bs{C}_i  \bs{B}_{i,2} \Big \rVert_F +   \frac{1}{\sqrt{d}} \Big \lVert  \frac{1}{N}  \sum_{i = 1}^N \bs{C}_i  \bs{B}_{i,3}  \Big \rVert_F. 
}
We have  
\begin{itemize}
    \item $\Big \lVert  \frac{1}{N}  \sum_{i = 1}^N \bs{C}_i  \bs{B}_{i,1} \Big \rVert_F^2  \leq   \frac{2}{N^2} \Big(  \sum_{i,  r = 1}^N   \tr(  \bs{C}_{i,1}  \bs{B}_{i,1}  \bs{B}_{r,1}^\top     \bs{C}_{r,1}^\top ) +   \tr(  \bs{C}_{i,2}  \bs{B}_{i,1}  \bs{B}_{r,1}^\top     \bs{C}_{r,2}^\top ) \Big)$.
    \item $\Big \lVert  \frac{1}{N}  \sum_{i = 1}^N \bs{C}_i  \bs{B}_{i,2} \Big \rVert_F^2  \leq   \frac{2}{N^2} \Big(  \sum_{i,  r = 1}^N   \tr(  \bs{C}_{i,2}  \bs{B}_{i,2}  \bs{B}_{r,2}^\top     \bs{C}_{r,2}^\top ) +   \tr(  \bs{C}_{i,2}  \bs{B}_{i,2}  \bs{B}_{r,2}^\top     \bs{C}_{r,2}^\top ) \Big)$.
    \item $\Big \lVert  \frac{1}{N}  \sum_{i = 1}^N \bs{C}_i  \bs{B}_{i,3} \Big \rVert_F^2  \leq   \frac{2}{N^2} \Big(  \sum_{i,  r = 1}^N   \tr(  \bs{C}_{i,1}  \bs{B}_{i,3}  \bs{V}_{r,3}^\top     \bs{C}_{r,1}^\top ) +   \tr(  \bs{C}_{i,2}  \bs{B}_{i,3}  \bs{B}_{r,3}^\top     \bs{C}_{r,2}^\top ) \Big)$.
\end{itemize}
We define the scalars
\eq{
t_1 \coloneqq  \frac{\phi^\prime(0)^4}{d } \frac{ (1- \frac{1}{V})}{L^2}  \big(   \frac{1}{N} +  \frac{1}{V} \big), \qquad t_2 \coloneqq  \frac{   \phi^\prime(0)^4  }{d}   (1 - \frac{1}{V})^2 \frac{L - 1}{L^2 N}, \qquad t_3 \coloneqq  \frac{\phi^\prime(0)^4}{N V^2 L^2}.
}
First, we will bound the first two terms. Let $* \in \{1,2\}$.

\textit{Bounding first two terms.} For $i \neq r$, by using the definition in $\bs{A}_{1,ir}$ and  $\bs{A}_{2,ir}$  in  \eqref{def:aone}-\eqref{def:atwo}, we have  
\eq{
\tr(  \bs{C}_{i,1}  \bs{B}_{i,*}  & \bs{B}_{r,*}^\top \bs{C}_{r,1}^\top )   
  =    \frac{1}{L^2}   (\indic{\bx_i = \bx_r} - \tfrac{1}{V} )   \mathsf{z}_{k}^\top   \bZin  \bX_i^\top       \bX_i  \bZin^\top \bs{A}_{*,ir}   \bZin    \bX_r^\top    \bX_r \bZin^\top   \mathsf{z}_{k}   \\[0.25em]
 &  =    \frac{t_* }{  L^2}    (\indic{\bx_i = \bx_r} - \tfrac{1}{V} )   \mathsf{z}_{k}^\top \bZin  \bX_i^\top      \bX_i  \bZin^\top  \bZin    \bX_r^\top  \bX_r \bZin^\top   \mathsf{z}_{k}   \\
 &  +    \frac{1}{L^2}   (\indic{\bx_i = \bx_r} - \tfrac{1}{V} )    \mathsf{z}_{k}^\top   \bZin  \bX_i^\top    \bX_i  \bZin^\top (   \bs{A}_{*,ir} - t_*  \bs{I}_d  )   \bZin    \bX_r^\top   \bX_r \bZin^\top  \mathsf{z}_{k}  \\[0.25em]
 & \leq   \frac{t_*}{  L^2}    (\indic{\bx_i = \bx_r} - \tfrac{1}{V} )    \mathsf{z}_{k}^\top   \bZin  \bX_i^\top     \bX_i  \bZin^\top  \bZin    \bX_r^\top    \bX_r \bZin^\top    \mathsf{z}_{k}  \\
&  +    \frac{1}{L^2}   (\indic{\bx_i = \bx_r} - \tfrac{1}{V} ) \norm{    \bs{A}_{*, ir} - t_*  \bs{I}_d  }_2 \norm{   \bZin    \bX_r^\top   \bX_r \bZin^\top   \mathsf{z}_{k} }_2   \norm{   \bZin    \bX_i^\top   \bX_i \bZin^\top   \mathsf{z}_{k}  }_2. 
}  
By \ref{event:prop8result1} in Proposition \ref{prop:intermediateresults1}, \ref{event:boundZin} implies
\eq{
 \frac{t_*}{  L^2} \Big \lvert \E \Big[   (\indic{\bx_i = \bx_r} - \tfrac{1}{V} )    \mathsf{z}_{k}^\top  \bZin  \bX_i^\top     \bX_i  \bZin^\top  \bZin    \bX_r^\top    \bX_r \bZin^\top   \mathsf{z}_{k} \Big \vert \sZin \Big] \Big \rvert \leq \frac{C t_*}{  L^2} \frac{1}{Vd}. \label{eq:boundvar1term1}
}
Moreover,  by using   \ref{cons:matrixbounds2} and \ref{event:prop8result2}, we have  
\eq{
 &  \frac{1}{L^2}  \E \Big[ (\indic{\bx_i = \bx_r} - \tfrac{1}{V} ) \norm{    \bs{A}_{*,ir} - t_* \bs{I}_d  }_2 \norm{   \bZin    \bX_r^\top   \bX_r \bZin^\top   \mathsf{z}_{k} }_2   \norm{   \bZin    \bX_i^\top   \bX_i \bZin^\top      \mathsf{z}_{k}  }_2   \ \Big \vert \ \sZin \Big] \\[1em]
&  \leq    \frac{C}{V d (L\wedge d)}  
\begin{cases}
   \phi^\prime(0)^2 \Big(  \frac{1}{N d  L^3} +   \frac{1}{V d L^2} \frac{1}{V \wedge L^2 \wedge L \sqrt{d} } \Big)   +     \phi^\prime(0)^4   \Big(   \frac{\log V}{L^2 V^{3/2} \sqrt{d}} +   \frac{ \log^2 V}{L^2 N \sqrt{V} d}    \Big),  & * = 1 \\[0.7em]
    \frac{  \sqrt{V}}{ d \sqrt{NL}}     \Big(   \frac{1}{NL^{\frac{3}{2}}} + \frac{1}{V \sqrt{L}}  \frac{1}{V \wedge L^2 \wedge L \sqrt{d}} \Big)  +     \phi^\prime(0)^4   \Big(   \frac{\log V}{N L \sqrt{V d}} +   \frac{ \log^3 V}{N \sqrt{LV} d}     \Big),  & * = 2
\end{cases} \\[1em]
& \leq    \frac{C}{ N^{3/2} \sqrt{V}  d^2 L^{2} } \frac{1}{L \wedge d}  +    \frac{C}{V^{3/2} \sqrt{N} L  d^2}   \frac{1}{L \wedge d}   \frac{1}{V \wedge L^2 \wedge L \sqrt{d}} +  \frac{C \log^3 V}{N V^{3/2} L^{1/2} d^{3/2} }  \frac{1}{(L \wedge d)^{3/2}} .  \label{eq:boundvar1term2}
}
On the other hand, 
\eq{ 
 \tr(  & \bs{C}_{i,2}  \bs{B}_{i,*} \bs{B}_{r,*}^\top \bs{C}_{r,2}^\top )  \\
&  =    \frac{1}{L^4}   (\indic{\bx_i = \bx_r} - \tfrac{1}{V} )     \mathsf{z}_{k}^\top   \bZin  \bX_i^\top   \mathbbm{1}_L  \mathbbm{1}_L^\top    \bX_i  \bZin^\top \bs{A}_{*,ir}   \bZin    \mathsf{X}_r^\top   \mathbbm{1}_L  \mathbbm{1}_L^\top     \bX_r \bZin^\top    \mathsf{z}_{k}  \\[0.25em]
 &  =    \frac{t_* }{ L^4}    (\indic{\bx_i = \bx_r} - \tfrac{1}{V} )     \mathsf{z}_{k}^\top   \bZin  \bX_i^\top      \mathbbm{1}_L  \mathbbm{1}_L^\top   \bX_i  \bZin^\top  \bZin    \bX_r^\top  \mathbbm{1}_L  \mathbbm{1}_L^\top   \bX_r \bZin^\top    \mathsf{z}_{k}  \\
 &  +    \frac{1}{L^4}   (\indic{\bx_i = \bx_r} - \tfrac{1}{V} )    \mathsf{z}_{k}^\top   \bZin  \bX_i^\top    \mathbbm{1}_L  \mathbbm{1}_L^\top   \bX_i  \bZin^\top (   \bs{A}_{*,ir} - t_*  \bs{I}_d  )   \bZin    \bX_r^\top   \mathbbm{1}_L  \mathbbm{1}_L^\top   \bX_r \bZin^\top     \mathsf{z}_{k}  \\[0.25em]
 & \leq   \frac{t_* }{ L^4}    (\indic{\bx_i = \bx_r} - \tfrac{1}{V} )     \mathsf{z}_{k}^\top   \bZin  \bX_i^\top   \mathbbm{1}_L  \mathbbm{1}_L^\top     \bX_i  \bZin^\top  \bZin    \bX_r^\top   \mathbbm{1}_L  \mathbbm{1}_L^\top    \bX_r \bZin^\top     \mathsf{z}_{k} \\
&  +    \frac{1}{L^4}   (\indic{\bx_i = \bx_r} - \tfrac{1}{V} ) \norm{    \bs{A}_{*,ir} - t_*  \bs{I}_d  }_2 \norm{   \bZin    \bX_r^\top   \mathbbm{1}_L  \mathbbm{1}_L^\top   \bX_r \bZin^\top     \mathsf{z}_{k}  }_2   \norm{   \bZin    \bX_i^\top  \mathbbm{1}_L  \mathbbm{1}_L^\top    \bX_i \bZin^\top   \mathsf{z}_{k} }_2. 
}  
By using \ref{event:boundZin4}, \ref{event:discrete2}, and \ref{cons:innerprodrestricted}, we have
\eq{
& \frac{t_*}{L^4}   \E \Big[ (\indic{\bx_i = \bx_r} \! - \! \tfrac{1}{V} )     \mathsf{z}_{k}^\top   \bZin  \bX_i^\top   \mathbbm{1}_L  \mathbbm{1}_L^\top     \bX_i  \bZin^\top  \bZin    \bX_r^\top   \mathbbm{1}_L  \mathbbm{1}_L^\top    \bX_r \bZin^\top    \mathsf{z}_{k}  \vert \sZin  \Big]  \\ 
& \leq  \frac{C t_* }{V  L } \frac{\log^2 V}{V \wedge L^2 \wedge L\sqrt{d}} \frac{1}{L \wedge d}.
}
Moreover,  by using \ref{event:boundZin4}, \ref{event:discrete2},  \ref{cons:innerprodrestricted}, and \ref{cons:matrixbounds2}
\eq{
 &  \frac{1}{L^4} \E \Big[     (\indic{\bx_i   =  \bx_r} \! - \!   \tfrac{1}{V} ) \norm{    \bs{A}_{*,ir}  \! - \! t_*  \bs{I}_d  }_2 \norm{   \bZin    \bX_r^\top \!  \mathbbm{1}_L \!  \mathbbm{1}_L^\top  \! \bX_r \bZin^\top    \mathsf{z}_{k}  }_2    \norm{   \bZin    \bX_i^\top \!  \mathbbm{1}_L \!  \mathbbm{1}_L^\top  \!   \bX_i \bZin^\top     \mathsf{z}_{k} }_2   \ \Big \vert \ \sZin  \Big] \\[1em] 
& \leq      \frac{C}{V L^2 (L \wedge d)}    \begin{cases}
   \phi^\prime(0)^2 \Big(  \frac{1}{N d  L^3} +   \frac{1}{V d L^2} \frac{1}{V \wedge L^2 \wedge L \sqrt{d} } \Big) \!  + \!    \phi^\prime(0)^4   \Big(   \frac{\log V}{L^2 V^{3/2} \sqrt{d}} +   \frac{ \log^2 V}{L^2 N \sqrt{V} d}    \Big),  \hspace{-1mm}  &* = 1 \\
    \frac{  \sqrt{V}}{ d \sqrt{NL}}     \Big(   \frac{1}{NL^{\frac{3}{2}}} + \frac{1}{V \sqrt{L}}  \frac{1}{V \wedge L^2 \wedge L \sqrt{d}} \Big) \!  + \!      \phi^\prime(0)^4   \Big(   \frac{\log V}{N L \sqrt{V d}} +   \frac{ \log^3 V}{N \sqrt{LV} d}     \Big),  & * = 2
\end{cases}  \\[1em]
& \leq    \frac{C}{ N^{3/2} \sqrt{V}  d L^{4} }    \frac{1}{L \wedge d}  +    \frac{C}{V^{3/2} \sqrt{N} L^3 d}  \frac{1}{L \wedge d}   \frac{1}{V \wedge L^2 \wedge L \sqrt{d}} +  \frac{C \log^3 V}{N V^{3/2} L^{5/2} \sqrt{d}}  \frac{1}{(L \wedge d)^{3/2}}.   ~~~~ \label{eq:boundvar1term3}
}
\smallskip

On the other hand,  for $i = r$, by \ref{cons:matrixbounds2},
\eq{
\tr(  \bs{C}_{i,1}  \bs{B}_{i,*} \bs{B}_{i,*}^\top \bs{C}_{i,1}^\top ) &  +   \tr(  \bs{C}_{i,2}  \bs{B}_{i,*} \bs{B}_{i,*}^\top \bs{C}_{i,2}^\top )    \\
&  =    \frac{1}{L^2}   (1 -  \frac{1}{V} )     \mathsf{z}_{k}^\top   \bZin  \bX_i^\top       \bX_i  \bZin^\top \bs{A}_{*, ii}   \bZin    \bX_i^\top    \bX_i \bZin^\top    \mathsf{z}_{k}  \\
& + \frac{ (1 -  \frac{1}{V} )  }{L^4}        \mathsf{z}_{k}^\top   \bZin  \bX_i^\top \mathbbm{1}_L  \mathbbm{1}_L^\top        \bX_i  \bZin^\top \bs{A}_{*, ii}   \bZin    \bX_i^\top   \mathbbm{1}_L  \mathbbm{1}_L^\top   \bX_i \bZin^\top      \mathsf{z}_{k}  \\
& \leq    \frac{t_*}{L^2}   (1 -  \frac{1}{V} )  \lVert    \bZin   \bX_i^\top    \bX_i \bZin^\top     \mathsf{z}_{k}  \rVert_2^2  +   \frac{t_*}{L^4}   \lVert    \bZin   \bX_i^\top  \mathbbm{1}_L  \mathbbm{1}_L^\top    \bX_i \bZin^\top    \mathsf{z}_{k} \rVert_2^2.
}
By   using  \ref{event:boundZin4}, \ref{event:discrete2},  \ref{cons:innerprodrestricted}, and \ref{event:prop8result2}
\eq{
& \frac{t_*}{L^2}   \E \! \Big[    \lVert    \bZin   \bX_i^\top    \bX_i \bZin^\top     \mathsf{z}_{k}   \rVert_2^2   \Big \vert   \sZin  \Big]  \! + \! \frac{t_*}{L^4}   \E \! \Big[    \lVert    \bZin   \bX_i^\top  \mathbbm{1}_L  \mathbbm{1}_L^\top    \bX_i \bZin^\top     \mathsf{z}_{k}  \rVert_2^2   \Big \vert   \sZin \Big]   \leq   \frac{C t_*}{L^2}   \Big(  \frac{L  }{d } \! + \! \frac{L^2}{d^2}  \Big) \! + \! \frac{C t_*}{L^2} \frac{1}{L \wedge d}.  ~~~ \label{eq:boundvar1term4}
}
Therefore, we have by \eqref{eq:boundvar1term1},\eqref{eq:boundvar1term2},\eqref{eq:boundvar1term3},\eqref{eq:boundvar1term4} and using $N \ll VL$ and $L \ll V$, we have
\eq{
 \E \Big[      \Big \lVert  \frac{1}{N}  \sum_{i = 1}^N \bs{C}_i  \bs{B}_{i,1} \Big \rVert_F^2   \Big  \vert \ \sZin   \ \Big]   &   +   \E \Big[      \Big \lVert  \frac{1}{N}  \sum_{i = 1}^N \bs{C}_i  \bs{B}_{i,2} \Big \rVert_F^2  \Big \vert \ \sZin \ \Big]  \\
& \leq   \frac{C}{N^2 d L (d \wedge L^2) (d \wedge L)} +      \frac{C}{ N^{3/2} \sqrt{V}  d L^{2} (d \wedge L^2) (L \wedge d) }    \\
& +     \frac{C}{V^{3/2} \sqrt{N} L  d (d \wedge L^2) (L \wedge d)  }    \frac{1}{V \wedge L^2 \wedge L \sqrt{d}}  +  \frac{C \log^3 V}{N V^{3/2} \sqrt{L d}  (d \wedge L^2) (L \wedge d)^{3/2}}  \\
& \leq   \frac{C}{N^2 d L (d \wedge L^2) (d \wedge L)}   +  \frac{C \log^3 V}{N V^{3/2} \sqrt{L d}  (d \wedge L^2) (L \wedge d)^{3/2}} . \label{eq:boundvar1}
}
\smallskip

\textit{Bounding the third term.} We have
\eq{
  \Big \lVert  \frac{1}{N}  \sum_{i = 1}^N \bs{C}_i  \bs{B}_{i,3} \Big \rVert_F^2  \leq   \frac{2}{N^2} \Big(  \sum_{i,  r = 1}^N   \tr(  \bs{C}_{i,1}  \bs{B}_{i,3}  \bs{B}_{r,3}^\top     \bs{C}_{r,1}^\top ) +   \tr(  \bs{C}_{i,2}  \bs{B}_{i,3}  \bs{B}_{r,3}^\top     \bs{C}_{r,2}^\top ) \Big).
}
We recall the definition $\tilde{\Delta}_{3,ir}$ in \ref{cons:matrixbounds2}:
\eq{
\tilde{\Delta}_{3,ir} =    \Big( \frac{1}{N} \sum_{j = 1}^N  \alpha_{ij}    ( \bx_j -   \tfrac{1}{V} \mathbbm{1}_{V}    )   \Big) ^\top   \Big( \frac{1}{N} \sum_{j = 1}^N  \alpha_{rj}  (\bx_j  - \tfrac{1}{V} \mathbbm{1}_V )          \Big)   -  \frac{\phi^\prime(0)^4}{N }. 
}
We have for $i \neq r$,
\eq{
 \tr(  \bs{C}_{i,1}  \bs{B}_{i,3} \bs{B}_{r,3}^\top \bs{C}_{r,1}^\top )     +  \tr(  & \bs{C}_{i,2}   \bs{B}_{i,3} \bs{B}_{r,3}^\top \bs{C}_{r,2}^\top )    \\
 & =    \frac{1}{L^2}   (\indic{\bx_i = \bx_r} - \tfrac{1}{V} )   \mathsf{z}_{k}^\top   \bZin  \bX_i^\top       \bX_i  \bZin^\top \bs{A}_{3,ir}   \bZin    \bX_r^\top    \bX_r \bZin^\top     \mathsf{z}_{k}  \\
&  +    \frac{1}{L^4}   (\indic{\bx_i = \bx_r} - \tfrac{1}{V} )     \mathsf{z}_{k}^\top   \bZin  \bX_i^\top   \mathbbm{1}_L  \mathbbm{1}_L^\top    \bX_i  \bZin^\top \bs{A}_{3,ir}   \bZin    \bX_r^\top   \mathbbm{1}_L  \mathbbm{1}_L^\top     \bX_r \bZin^\top    \mathsf{z}_{k}  \\[0.5em]
 & \leq     \frac{t_3  }{  L^2}    (\indic{\bx_i = \bx_r} - \tfrac{1}{V} )    \mathsf{z}_{k}^\top   \bZin  \bX_i^\top      \bX_i  \bZin^\top \bZin   \mathbbm{1}_V \mathbbm{1}_V^\top \bZin^\top  \bZin    \bX_r^\top  \bX_r \bZin^\top    \mathsf{z}_{k}   \\
 & +   \frac{t_3  }{  L^4}    (\indic{\bx_i = \bx_r} - \tfrac{1}{V} )    \mathsf{z}_{k}^\top   \bZin  \bX_i^\top      \mathbbm{1}_L  \mathbbm{1}_L^\top  \bX_i  \bZin^\top   \bZin    \mathbbm{1}_V \mathbbm{1}_V^\top   \bZin^\top \bZin    \bX_r^\top  \mathbbm{1}_L  \mathbbm{1}_L^\top  \bX_r \bZin^\top     \mathsf{z}_{k}  \\
 & +  \frac{\tilde{\Delta}_{3,ir}}{V^2 L^4}   (\indic{\bx_i = \bx_r} - \tfrac{1}{V} )    \mathsf{z}_{k}^\top   \bZin  \bX_i^\top       \bX_i  \bZin^\top  \mathbbm{1}_V \mathbbm{1}_V^\top   \bZin    \bX_r^\top    \bX_r \bZin^\top    \mathsf{z}_{k}  \\
 & +    \frac{\tilde{\Delta}_{3,ir}}{V^2 L^6}   (\indic{\bx_i = \bx_r} - \tfrac{1}{V} )     \mathsf{z}_{k}^\top   \bZin  \bX_i^\top   \mathbbm{1}_L  \mathbbm{1}_L^\top    \bX_i  \bZin^\top  \mathbbm{1}_V \mathbbm{1}_V^\top \bZin    \bX_r^\top   \mathbbm{1}_L  \mathbbm{1}_L^\top     \bX_r \bZin^\top    \mathsf{z}_{k}
}  
For the first term,  by \ref{event:prop8result3},
\eq{
& \frac{t_3  }{  L^2}   \E \Big[  (\indic{\bx_i = \bx_r} - \tfrac{1}{V} ) \mathsf{z}_{k}^\top   \bZin  \bX_i^\top      \bX_i  \bZin^\top  \bZin   \mathbbm{1}_V \mathbbm{1}_V^\top   \bZin^\top \bZin    \bX_r^\top  \bX_r \bZin^\top  \mathsf{z}_{k} \vert \sZin \Big]  \leq    \frac{C \phi^\prime(0)^4}{N V^2 L^4}   \frac{\log^2 V}{d^2}.
}
For the second term, by using \ref{event:boundZin4}, \ref{event:boundZin5} and \ref{event:discrete2}
\eq{
&  \frac{t_3  }{  L^4}   \E \! \Big[   (\indic{\bx_i = \bx_r} \! - \! \tfrac{1}{V} )  \mathsf{z}_{k}^\top   \bZin  \bX_i^\top      \mathbbm{1}_L  \mathbbm{1}_L^\top  \bX_i  \bZin^\top  \bZin   \mathbbm{1}_V  \mathbbm{1}_V^\top    \bZin^\top   \bZin    \bX_r^\top  \mathbbm{1}_L  \mathbbm{1}_L^\top  \bX_r \bZin^\top  \mathsf{z}_{k}  \vert \sZin \Big]   \!\leq   \! \frac{\phi^\prime(0)^4}{N V^3 L^4} \frac{1}{L \wedge d} \Big(L \vee \frac{V}{d} \Big).
}
For the last two terms,   by using \ref{event:boundZin1},  \ref{event:boundZin4}, \ref{event:discrete2}, \ref{cons:innerprodrestricted},
\ref{cons:matrixbounds2},  and \ref{event:prop8result2},
\eq{
 & \frac{1 }{V^2  L^4}     \E \! \Big[   (\indic{\bx_i = \bx_r} - \tfrac{1}{V} )   \lvert \tilde{\Delta}_{3,ir} \rvert    \lVert  \bZin    \bX_i^\top    \bX_i \bZin^\top  \mathsf{z}_{k} \rVert_2   \lVert  \bZin    \bX_r^\top    \bX_r \bZin^\top    \mathsf{z}_{k} \rVert_2     \vert \sZin \Big] + \\
 &  \frac{1}{ V^2 L^6}  \E \! \Big[  (\indic{\bx_i = \bx_r}\! - \! \tfrac{1}{V} )  \lvert \tilde{\Delta}_{3,ir} \rvert    \lVert  \bZin^\top \bZin    \bX_i^\top  \mathbbm{1}_L  \mathbbm{1}_L^\top  \bX_i \bZin^\top   \mathsf{z}_{k}   \rVert_2    \lVert  \bZin^\top \bZin    \bX_r^\top  \mathbbm{1}_L  \mathbbm{1}_L^\top  \bX_r \bZin^\top   \mathsf{z}_{k}   \rVert_2   \vert \sZin \Big]  \\
& \leq \frac{C}{V^3 L^4} \Big( \frac{L}{d} + \frac{L^2}{d^2}  +\frac{V}{d^2}   +\frac{V}{Ld}  \Big)  \Bigg( \frac{  \phi^\prime(0)^4 \log^2 V}{N \sqrt{V}}   +   \frac{  \phi^\prime(0)^2}{N}  \Big( \frac{1}{N L} + \frac{1}{\sqrt{N}} \frac{1}{V \wedge L^2 \wedge L \sqrt{d}} \Big)   \Bigg)  \\
&  +    \frac{C}{V^3 L^4} \Big( \frac{L}{d} + \frac{L^2}{d^2}  +\frac{V}{d^2}   +\frac{V}{Ld}  \Big)    \Big( \frac{1}{N L} + \frac{1}{\sqrt{N}} \frac{1}{V \wedge L^2 \wedge L \sqrt{d}} \Big)^2 \\
& \leq  \frac{C}{N V^3 L^2 d (L\wedge d)}   \frac{\phi^\prime(0)^4 \log^2 V}{\sqrt{V} \wedge L^4 \wedge L^2 d} +  \frac{C}{N V^2 L^4 d^2}  \frac{\phi^\prime(0)^4 \log^2 V}{\sqrt{V} \wedge L^4 \wedge L^2 d}. \label{eq:boundvar2term1}
}
For $i = r$, by using \ref{cons:matrixbounds2},
\eq{
 \tr(  \bs{C}_{i,1}  \bs{B}_{i,3} \bs{B}_{i,3}^\top &  \bs{C}_{i,1}^\top )  +  \tr(  \bs{C}_{i,2}  \bs{B}_{i,3} \bs{B}_{i,3}^\top \bs{C}_{i,2}^\top )    \\
&  =    \frac{1}{L^2}   ( 1 - \frac{1}{V} )  \mathsf{z}_{k}^\top   \bZin  \bX_i^\top       \bX_i  \bZin^\top \bs{A}_{3,ir}   \bZin    \bX_i^\top    \bX_i \bZin^\top    \mathsf{z}_{k}  \\
& +  \frac{1}{L^4}   ( 1 - \frac{1}{V} )    \mathsf{z}_{k}^\top   \bZin  \bX_i^\top   \mathbbm{1}_L  \mathbbm{1}_L^\top    \bX_i  \bZin^\top \bs{A}_{3,ir}   \bZin    \bX_i^\top   \mathbbm{1}_L  \mathbbm{1}_L^\top     \bX_i \bZin^\top  \mathsf{z}_{k}  \\
& \leq  \frac{2 t_3}{L^2}  \lvert   \mathbbm{1}_V^\top \bZin^\top  \bZin    \bX_i^\top    \bX_i \bZin^\top   \mathsf{z}_{k}  \rvert^2  + 
  \frac{2 t_3}{L^4}  \lvert   \mathbbm{1}_V^\top \bZin^\top  \bZin    \bX_i^\top    \mathbbm{1}_L  \mathbbm{1}_L^\top   \bX_i \bZin^\top   \mathsf{z}_{k}  \rvert^2. 
}
Then,    by \ref{event:prop8result4}, \ref{event:boundZin4}, \ref{event:discrete2}, and \ref{event:boundZin5}, we have  
\eq{
& \frac{t_3}{L^2}\E \Big[  (   \mathbbm{1}_V^\top \bZin^\top  \bZin    \bX_i^\top    \bX_i \bZin^\top   \mathsf{z}_{k} )^2 \vert \sZin \Big]  +  \frac{t_3}{L^4}\E \Big[ ( \mathbbm{1}_V^\top \bZin^\top  \bZin    \bX_i^\top    \mathbbm{1}_L  \mathbbm{1}_L^\top   \bX_i \bZin^\top   \mathsf{z}_{k} )^2    \vert \sZin \Big] \\
& \leq    \frac{C \phi^\prime(0)^4  \log^2 V}{N V d^2 L^2}  \frac{1}{L \wedge d} \Big( 1 + \frac{d}{L^2} + \frac{d^2}{V L}  \Big) .   \label{eq:boundvar2term2}
}
Therefore,  by using   \eqref{eq:boundvar2term1}-\eqref{eq:boundvar2term2} and using $L \ll V$ and $N \ll VL$, we have
\eq{
\E \Big[      \Big \lVert  \frac{1}{N}  \sum_{i = 1}^N \bs{C}_i  \bs{B}_{i,3} \Big \rVert_F^2   \vert \sZin   \Big]   \ll \frac{1}{N^2 dL (d \wedge L^2) (d \wedge L)} . \label{eq:boundvar2}
}
Therefore,  by \eqref{eq:boundvar1}-\eqref{eq:boundvar2}, we have
\eq{
\vartheta_2 \leq   \frac{C \log V}{N \sqrt{L} d (d \wedge L^2)^{1/2} (d \wedge L)^{1/2}}  +  \frac{C \log^{5/2} V}{ \sqrt{N} (V d)^{3/4} L^{1/4}  (d \wedge L^2)^{1/2} (L \wedge d)^{3/4}} .
}
\medskip

\paragraph{Bounding $\vartheta_1$:}
We have
\eq{
\vartheta_1 =  \underbrace{  \tr \Big(   \frac{1}{N}  \sum_{i = 1}^N  \bs{C}_i   \bs{B}_{i,1}   \Big) + \tr \Big(   \frac{1}{N}  \sum_{i = 1}^N  \bs{C}_i   \bs{B}_{i,2}   \Big)  }_{\coloneqq \vartheta_{11} } + \underbrace{  \tr \Big(   \frac{1}{N}  \sum_{i = 1}^N  \bs{C}_i   \bs{B}_{i,3}   \Big) }_{\coloneqq  \vartheta_{12} }.
}
We have
\eq{
 \vartheta_{11} 
&  =   \frac{1}{N^2L^2} \sum_{i, j = 1}^N \alpha_{ij} ( \indic{\bx_i = \bx_j }-   \tfrac{1}{V} \mathbbm{1}_{V} )  \mathsf{z}_{k}^\top  \bZin \bX_i^\top   \big(\bs{I}_L - \frac{1}{L} \mathbbm{1}_L  \mathbbm{1}_L^\top \big)  \bX_i  \bZin^\top      \bZin (\bX_j^\top - \tfrac{1}{V} \mathbbm{1}_V  \mathbbm{1}_L^\top ) \mathbbm{1}_L   \\
& =     \frac{ \phi^\prime(0)^2}{N^2 L^2} \sum_{j = 1}^N \mathsf{z}_{k}^\top  \bZin  \Big(    \sum_{ \substack{ i = 1  \\ i \neq j }}^N ( \indic{\bx_i  \! = \! \bx_j } \! - \!   \tfrac{1}{V} \mathbbm{1}_{V} )   \bX_i^\top   \big(\bs{I}_L \! - \! \tfrac{1}{L} \mathbbm{1}_L  \mathbbm{1}_L^\top \big)  \bX_i \Big) \bZin^\top      \bZin (\bX_j^\top - \tfrac{1}{V} \mathbbm{1}_V   \mathbbm{1}_L^\top ) \mathbbm{1}_L  \\
& +     \frac{ (1  - \frac{1}{V} )  }{N^2 L^2} \sum_{j = 1}^N  \alpha_{jj}  \mathsf{z}_{k}^\top  \bZin   \bX_j^\top   \big(\bs{I}_L - \frac{1}{L} \mathbbm{1}_L  \mathbbm{1}_L^\top \big)  \bX_j   \bZin^\top      \bZin  (\bX_j^\top - \tfrac{1}{V} \mathbbm{1}_V \mathbbm{1}_L^\top ) \mathbbm{1}_L   \\
& +     \frac{1}{N^2 L^2} \sum_{j = 1}^N     \sum_{ \substack{ i = 1  \\ i \neq j } }^N  ( \alpha_{ij} -  \phi^\prime(0)^2) ( \indic{\bx_i = \bx_j }  - \tfrac{1}{V} )  \mathsf{z}_{k}^\top  \bZin    \bX_i^\top   \big(\bs{I}_L - \tfrac{1}{L} \mathbbm{1}_L  \mathbbm{1}_L^\top \big)  \bX_i   \bZin^\top      \bZin   (\bX_j^\top - \tfrac{1}{V} \mathbbm{1}_V \mathbbm{1}_L^\top ) \mathbbm{1}_L   \\
& \eqqcolon   \vartheta_{11a} +  \vartheta_{11b} +  \vartheta_{11c}.
}
We start with the last term.  By using H\"{o}lder's inequality,
\eq{
 \lvert   \vartheta_{11c}  \rvert  
&    \leq   \Big(  \frac{1}{N^2 L^2} \sum_{j = 1}^N     \sum_{ \substack{ i = 1  \\ i \neq j } }^N  \lvert \indic{\bx_i = \bx_j }  - \tfrac{1}{V} \rvert \Big)    \sup_{i \neq j \in [N]}   \lvert \alpha_{ij} -  \phi^\prime(0)^2 \rvert   \\
& \hspace{4em} \times \sup_{i \neq j \in [N]}  \lvert  \mathsf{z}_{k}^\top  \bZin    \bX_i^\top   \big(\bs{I}_L - \tfrac{1}{L} \mathbbm{1}_L  \mathbbm{1}_L^\top \big)  \bX_i   \bZin^\top      \bZin   (\bX_j^\top - \tfrac{1}{V} \mathbbm{1}_V \mathbbm{1}_L^\top ) \mathbbm{1}_L   \rvert \\ 
& \leq \frac{C \log V}{V L \sqrt{d} \sqrt{L \wedge d}} \frac{1}{V \wedge L^2 \wedge L \sqrt{d}}. \label{eq:meanbound1stterm}
}
where we used \ref{event:boundZin8},  \ref{event:discrete4}, \ref{event:discrete8}, and \ref{cons:coeffbound}  in \eqref{eq:meanbound1stterm}.
Next, we consider  $\vartheta_{11b}$: 
\eq{
  \lvert   \vartheta_{11b} \rvert 
&  =     \frac{ (1  - \frac{1}{V} )  }{N^2 L^2} \sum_{j = 1}^N \mathsf{z}_{k}^\top  \bZin \Big(   \alpha_{jj}     \bX_j^\top     \bX_j  \bZin^\top      \bZin   \bX^\top_j   \mathbbm{1}_L - \E \big[     \alpha_{jj}     \bX_j^\top     \bX_j  \bZin^\top      \bZin   \bX^\top_j   \mathbbm{1}_L ~ \big \vert \sZin  \big] \Big)     \label{eq:meanparttwo}    \\
 & +      \frac{ (1  - \frac{1}{V} )  }{N  L^2}   \mathsf{z}_{k}^\top  \bZin  \E  \big[     \alpha_{11}    \bX_1^\top     \bX_1 \bZin^\top      \bZin   \bX^\top_1  \mathbbm{1}_L ~ \big \vert \sZin   \big]    \\
 &    - \frac{ (1  - \frac{1}{V} )  }{N^2 L V} \sum_{j = 1}^N  \alpha_{jj}   \mathsf{z}_{k}^\top  \bZin   \bX_j^\top     \bX_j  \bZin^\top      \bZin  \mathbbm{1}_V  \\
 &  -   \frac{ (1  - \frac{1}{V} )  }{N^2 L^3} \sum_{j = 1}^N  \alpha_{jj}  \mathsf{z}_{k}^\top  \bZin   \bX_j^\top    \mathbbm{1}_L  \mathbbm{1}_L^\top  \bX_j  \bZin^\top      \bZin  (\bX^\top_j - \tfrac{1}{V} \mathbbm{1}_V \mathbbm{1}_L^\top ) \mathbbm{1}_L .
}
\begin{itemize}[leftmargin=*]
\item For the first summand,  
\eq{
 &  \E \Bigg[  \Bigg \lbrace \frac{ (1  - \frac{1}{V} )  }{N^2 L^2}  \sum_{j = 1}^N  \mathsf{z}_{k}^\top  \bZin \Big(   \alpha_{jj}     \bX_j^\top     \bX_j  \bZin^\top      \bZin   \bX^\top_j   \mathbbm{1}_L   - \E  \big[     \alpha_{jj}     \bX_j^\top     \bX_j  \bZin^\top      \bZin   \bX^\top_j   \mathbbm{1}_L  \big \vert \sZin  \big] \Big)     \Bigg \rbrace^2   \big \vert \sZin   \Bigg] \\
& \leq  \frac{ (1  - \frac{1}{V} )^2  }{N^3 L^4}    \E  \Big[  \alpha_{jj}^2 \Big (   \mathsf{z}_{k}^\top  \bZin  \bX_1^\top     \bX_1  \bZin^\top      \bZin   \bX^\top_1   \mathbbm{1}_L    \Big)^2 ~ \big \vert \sZin   \Big]   \\
& \leq  \frac{ C \phi^\prime(0)^4  }{N^3 L^3}    \E  \Big[ \big  \lVert  \mathsf{z}_{k}^\top  \bZin  \bX_1^\top     \bX_1  \bZin^\top \big  \rVert_2^2 ~ \big \vert \sZin  \Big] ~~~ ~~ \label{eq:meanbound22term1}
} 
where we used  \ref{cons:innerprodrestricted} in \eqref{eq:meanbound22term1}.

By Chebyshev's inequality and  \ref{event:prop8result2}, with probability $1 - o_V(1)$, we have
\eq{
 & \mathsf{z}_{k}^\top  \bZin  \frac{ (1  - \frac{1}{V} )  }{N^2 L^2} \sum_{j = 1}^N \Big(   \alpha_{jj}     \bX_j^\top     \bX_j  \bZin^\top      \bZin   \bX^\top_j   \mathbbm{1}_L - \E  \big[     \alpha_{jj}     \bX_j^\top     \bX_j    \bZin^\top      \bZin   \bX^\top_j    \mathbbm{1}_L  \vert \sZin  \big] \Big)  \leq \frac{C \phi^\prime(0)^2 \log V}{N^{\frac{3}{2}} \sqrt{L d} \sqrt{L \wedge d}} .
}
\item For the second summand,  
\eq{
  \frac{ (1  - \frac{1}{V} )  }{N  L^2}    &  \mathsf{z}_{k}^\top  \bZin  \E  \big[  \alpha_{11}    \bX_1^\top     \bX_1 \bZin^\top      \bZin   \bX^\top_1  \mathbbm{1}_L ~ \big \vert \sZin   \big]    \\
&  =   \frac{ (1  - \frac{1}{V} )   \phi^\prime(0)^2 }{N  L^2}   \mathsf{z}_{k}^\top  \bZin  \E  \big[      \bX_1^\top     \bX_1 \bZin^\top      \bZin   \bX^\top_1    \bX_1  ~ \big \vert \sZin    \big]  \mathbbm{1}_V  \\
& +  \frac{ (1  - \frac{1}{V} )  }{N  L^2}    \mathsf{z}_{k}^\top  \bZin  \E  \big[ \big( \alpha_{11}  -  \phi^\prime(0)^2\big)    \bX_1^\top     \bX_1 \bZin^\top      \bZin   \bX^\top_1    \bX_1  ~ \big \vert \sZin    \big]  \mathbbm{1}_V   \\
& =   \frac{ (1  - \frac{1}{V} )     \phi^\prime(0)^2  }{N  L }   \mathsf{z}_{k}^\top  \bZin   \E  \big[      \bx_1   \bx_1^\top   \bZin^\top      \bZin   \bx_1    \bx_1^\top ~ \big \vert \sZin     \big]  \mathbbm{1}_V   +   \frac{ (1  - \frac{1}{V} )    \phi^\prime(0)^2   }{N V^2 }   \mathsf{z}_{k}^\top  \bZin   \bZin^\top      \bZin    \mathbbm{1}_V  ~~~~~ \label{eq:s20}  \\
&  +  \frac{ (1  - \frac{1}{V} )  }{N  L V}   \mathsf{z}_{k}^\top  \bZin  \E  \big[ \big( \alpha_{11}  -  \phi^\prime(0)^2\big)      \bZin^\top      \bZin   \bX^\top_1     \mathbbm{1}_{L}   \big \vert \sZin    \big]  \label{eq:s21}  \\
&  +  \frac{ (1  - \frac{1}{V} )  }{N  L^2}    \mathsf{z}_{k}^\top  \bZin  \E  \big[ \big( \alpha_{11}  -  \phi^\prime(0)^2\big)     \big( \bX_1^\top     \bX_1 - \tfrac{L}{V} \bs{I}_V \big)  \bZin^\top      \bZin  (  \bX^\top_1       - \frac{1}{V} \mathbbm{1}_{V} \mathbbm{1}_{L}^\top )  \mathbbm{1}_{L} \big \vert \sZin    \big]   \label{eq:s22}   \\
&  +  \frac{ (1  - \frac{1}{V} )  }{N  L V}    \mathsf{z}_{k}^\top  \bZin  \E  \big[ \big( \alpha_{11}  -  \phi^\prime(0)^2\big)     \big( \bX_1^\top     \bX_1 - \tfrac{L}{V} \bs{I}_V \big)  \bZin^\top      \bZin     \mathbbm{1}_{V}    \big \vert \sZin    \big]   \label{eq:s23}  \\
& \leq C \log V \Big(    \frac{1}{N   \sqrt{V d} (L \wedge d)}   + \frac{1}{N L^{3/2} \sqrt{d} (L \wedge d) }   \Big). \label{eq:s24}
}
where we use \ref{event:boundZin3} to bound \eqref{eq:s20};  \ref{event:boundZin6}, \ref{cons:coeffbound} for \eqref{eq:s21};  \ref{event:prop8result5}, \ref{cons:coeffbound} for \eqref{eq:s22}; and \ref{event:boundZin8}, \ref{cons:coeffbound} for \eqref{eq:s23}.
\item For the third summand,  
\eq{
 \frac{ 1  }{N^2 L V} \sum_{j = 1}^N  \alpha_{jj}   \mathsf{z}_{k}^\top  \bZin    \bX_j^\top       \bX_j  \bZin^\top      \bZin  \mathbbm{1}_V  
&  =   \frac{ \phi^\prime(0)^2   }{N   V^2}  \mathsf{z}_{k}^\top        \bZin     \bZin     \bZin^\top  \mathbbm{1}_V  \\
 &  + \frac{ \phi^\prime(0)^2 }{N^2 L V} \sum_{j = 1}^N    \mathsf{z}_{k}^\top  \bZin \Big(   \bX_j^\top     \bX_j  - \tfrac{L}{V} \bs{I}_V \Big) \bZin^\top      \bZin  \mathbbm{1}_V \\
& +  \frac{1}{N^2 L V} \sum_{j = 1}^N   (   \alpha_{jj} -  \phi^\prime(0)^2)   \mathsf{z}_{k}^\top  \bZin   \bX_j^\top     \bX_j  \bZin^\top      \bZin  \mathbbm{1}_V.   
}

The first term:
\eq{
\Big \lvert \frac{ \phi^\prime(0)^2   }{N   V^2}   \mathsf{z}_{k}^\top     \bZin     \bZin     \bZin^\top   \mathbbm{1}_V \Big \rvert \leq  \frac{ C \log V}{N \sqrt{V} d^{\frac{3}{2}}}.
}
The second term: By using 
\eq{
   \E \Big[   \Big( \sum_{j = 1}^N  \mathsf{z}_{k}^\top  \bZin  \Big(   \bX_j^\top     \bX_j  -  \tfrac{L}{V} \bs{I}_V \Big)   \bZin^\top      \bZin  \mathbbm{1}_V \Big)^2 \vert \sZin \Big]   &   =   \sum_{j = 1}^N   \E \Big[  \Big(   \mathsf{z}_{k}^\top  \bZin \Big(   \bX_j^\top     \bX_j  - \tfrac{L}{V} \bs{I}_V \Big) \bZin^\top      \bZin  \mathbbm{1}_V \Big)^2 \vert \sZin \Big]  \\
   & =\frac{L  V N }{d^2}.
}
Therefore, by Chebyshev's inequality, we have  
\eq{
\Big \lvert   \frac{ \phi^\prime(0)^2 }{N^2 L V} \sum_{j = 1}^N    \mathsf{z}_{k}^\top  \bZin \Big(   \bX_j^\top     \bX_j  - \tfrac{L}{V} \bs{I}_V \Big) \bZin^\top      \bZin  \mathbbm{1}_V   \Big \rvert \leq    \frac{ \phi^\prime(0)^2 }{N^{3/2}  \sqrt{V L} d }.
}
Finally,
\eq{
 \Big \lvert \frac{1}{N^2 L V} \sum_{j = 1}^N   (   \alpha_{jj} -  \phi^\prime(0)^2)     \mathsf{z}_{k}^\top   \bZin   \bX_j^\top     \bX_j  \bZin^\top      \bZin  \mathbbm{1}_V  \Big \rvert 
 & \leq \frac{C}{N \sqrt{V}} \Big \lVert  \frac{1}{NL}  \sum_{j = 1}^N   (   \alpha_{jj} -  \phi^\prime(0)^2)     \bZin   \bX_j^\top     \bX_j  \bZin^\top   \Big \rVert_2     \\
 & \leq  \frac{C}{N \sqrt{V} L d},
}
where we use \ref{event:discrete8} and \ref{event:boundZin1} and \ref{cons:coeffbound}.

Therefore,
\eq{
\Big \lvert  \frac{ 1  }{N^2 L V} \sum_{j = 1}^N  \alpha_{jj}    \mathsf{z}_{k}^\top  \bZin    \bX_j^\top       \bX_j  \bZin^\top      \bZin  \mathbbm{1}_V  \Big \rvert \leq  C \log V \Big(   \frac{ 1}{N \sqrt{V} d^{\frac{3}{2}}}   + \frac{1}{N \sqrt{V} L d} \Big).
}
\item For the last summand,
\eq{
&  \frac{ (1  - \frac{1}{V} )  }{N^2 L^3} \sum_{j = 1}^N  \alpha_{jj}     \mathsf{z}_{k}^\top  \bZin   \bX_j^\top  \mathbbm{1}_L  \mathbbm{1}_L^\top  \bX_j  \bZin^\top      \bZin  (\bX^\top_j - \tfrac{1}{V} \mathbbm{1}_V \mathbbm{1}_L^\top ) \mathbbm{1}_L  \\
&  =      \frac{ (1  - \frac{1}{V} )  }{N^2 L^3} \sum_{j = 1}^N    \mathsf{z}_{k}^\top  \bZin  \Big(   \alpha_{jj}   \bX_j^\top  \mathbbm{1}_L  \mathbbm{1}_L^\top  \bX_j  \bZin^\top      \bZin  (\bX^\top_j - \tfrac{1}{V} \mathbbm{1}_V \mathbbm{1}_L^\top ) \mathbbm{1}_L  \\
& \hspace{12em}- \E  \big[   \alpha_{jj}   \bX_j^\top  \mathbbm{1}_L  \mathbbm{1}_L^\top  \bX_j  \bZin^\top      \bZin  (\bX^\top_j - \tfrac{1}{V} \mathbbm{1}_V \mathbbm{1}_L^\top ) \mathbbm{1}_L  ~ \big \vert \sZin  \big] \Big) \\
 & \quad + \frac{ (1  - \frac{1}{V} ) \phi^\prime(0)^2 }{N  L^3}     \mathsf{z}_{k}^\top  \bZin  \E  \big[   \bX_1^\top  \mathbbm{1}_L  \mathbbm{1}_L^\top  \bX_1  \bZin^\top      \bZin  (\bX^\top_1 - \frac{1}{V} \mathbbm{1}_V \mathbbm{1}_L^\top ) \mathbbm{1}_L  ~ \big \vert \sZin  \big] \\
 & \quad +  \frac{ (1  - \frac{1}{V} )  }{N  L^3}   \mathsf{z}_{k}^\top  \bZin  \E  \big[  ( \alpha_{11} - \phi^\prime(0)^2 )  \bX_1^\top  \mathbbm{1}_L  \mathbbm{1}_L^\top  \bX_1  \bZin^\top      \bZin  (\bX^\top_1 - \frac{1}{V} \mathbbm{1}_V \mathbbm{1}_L^\top ) \mathbbm{1}_L  ~ \big \vert \sZin  \big]. ~~~~~~~~ \label{eq:meantermtsecondandfourth}
}
We have 
\eq{
 \mathsf{z}_{k}^\top  \bZin    \E  \Big[ \Big(    \alpha_{jj}   \bX_1^\top  \mathbbm{1}_L  \mathbbm{1}_L^\top  \bX_1  \bZin^\top      \bZin  (\bX^\top_1 \! - \! \frac{1}{V} \mathbbm{1}_V \mathbbm{1}_L^\top ) \mathbbm{1}_L \Big)^2  ~ \big \vert \sZin\Big] \bZin ^\top    \mathsf{z}_{k}   
& \leq C  L^2  \mathsf{z}_{k}^\top  \bZin     \E  \Big[   \bX_1^\top  \mathbbm{1}_L  \mathbbm{1}_L^\top  \bX_1   \Big] \bZin ^\top   \mathsf{z}_{k} \\
& \leq C L^2 \Big( \frac{L}{d} + \frac{L^2}{V d} \Big).
}
Moreover, by using Proposition \ref{prop:multimatrixexpectation}
\eq{
    & \E  \big[   \bX_1^\top  \mathbbm{1}_L  \mathbbm{1}_L^\top  \bX_1  \bZin^\top      \bZin  (\bX^\top_1  \! - \! \frac{1}{V} \mathbbm{1}_V \mathbbm{1}_L^\top ) \mathbbm{1}_L ~ \big \vert \sZin   \big]    \\
& =  \E  \big[   \bX_1^\top  \mathbbm{1}_L  \mathbbm{1}_L^\top  \bX_1  \bZin^\top      \bZin   \bX^\top_1   \mathbbm{1}_L    \big \vert \sZin  \big]  \! - \! \frac{L}{V}  \E  \big[   \bX_1^\top  \mathbbm{1}_L  \mathbbm{1}_L^\top  \bX_1  \bZin^\top      \bZin     \mathbbm{1}_V ~ \big \vert \sZin     \big]  \\
& =    L   \E  \big[ \bx_1   \bx_1^\top  \bZin^\top      \bZin  \bx_1   \big \vert \sZin \big] \!  + \! \Big(  \frac{L(L-1)}{V^2} \tr(\bZin^\top      \bZin  ) \! - \!    \frac{2 L(L-1) }{V^3} \mathbbm{1}_V^\top    \bZin^\top      \bZin    \mathbbm{1}_V   \Big)  \mathbbm{1}_V  \! + \!  \frac{  L (L- 2)}{V^2}    \bZin^\top      \bZin  \mathbbm{1}_V.
    }    
Lastly,
\eq{
 &  \Big \lvert  \frac{1}{N  L^3}  \mathsf{z}_{k}^\top  \bZin  \E  \big[  ( \alpha_{11} - \phi^\prime(0)^2 )  \bX_1^\top  \mathbbm{1}_L  \mathbbm{1}_L^\top  \bX_1  \bZin^\top      \bZin  (\bX^\top_1 - \frac{1}{V} \mathbbm{1}_V \mathbbm{1}_L^\top ) \mathbbm{1}_L  ~ \big \vert \sZin  \big]  \Big \rvert \\
& \leq  \frac{1}{N  L^3}    \E  \Big[  \lvert  \alpha_{11} - \phi^\prime(0)^2  \rvert \lvert   \mathsf{z}_{k}^\top  \bZin   \bX_1^\top  \mathbbm{1}_L  \rvert ~ \lvert      \mathbbm{1}_L^\top  \bX_1  \bZin^\top      \bZin  (\bX^\top_1 - \frac{1}{V} \mathbbm{1}_V \mathbbm{1}_L^\top ) \mathbbm{1}_L \rvert   ~ \Big \vert \sZin  \Big]   \leq \frac{C \log V}{N L^{5/2} \sqrt{d}},
}
where we used \ref{cons:innerprodrestricted},\ref{cons:coeffbound}, \ref{event:boundZin4}, \ref{event:boundZin5} for the last inequality.

Therefore,  by Chebyshev's inequality,  with probability $1 - o_V(1)$, we have 
\eq{
\eqref{eq:meantermtsecondandfourth} 
& \leq  C \log V \Big(     \frac{1}{N  L^2 \sqrt{L \wedge d}} +    \frac{1}{N  L \sqrt{Vd} }     \Big)
}
\end{itemize}
Therefore, we have
\eq{
 \lvert   \vartheta_{11b}  \rvert  \leq C \log V  \Big( \frac{1 }{N \sqrt{V} ( L \wedge d) \sqrt{d} }     +      \frac{1}{N  L^2 \sqrt{L \wedge d}}       \Big).  \label{eq:meanbound2ndterm}
}
Finally,  we consider $ \vartheta_{11a}$:
\eq{
 & \vartheta_{11a}  =  (1 - \tfrac{1}{L}) \frac{ \phi^\prime(0)^2}{N L} \sum_{j = 1}^N   \mathsf{z}_{k}^\top  \bZin  \Big(  \frac{1}{NL} \sum_{ \substack{ i = 1  \\ i \neq j }}^N ( \indic{\bx_i = \bx_j }-   \tfrac{1}{V}  )   \bx_i   \bx_i^\top  \Big) \bZin^\top      \bZin (\bX_j^\top - \tfrac{1}{V} \mathbbm{1}_V   \mathbbm{1}_L^\top ) \mathbbm{1}_L  \\ 
 & -   \frac{ \phi^\prime(0)^2}{N L^2} \sum_{j = 1}^N  \mathsf{z}_{k}^\top  \bZin  \Big(  \frac{1}{NL} \sum_{ \substack{ i = 1  \\ i \neq j }}^N ( \indic{\bx_i = \bx_j }  \! - \!   \tfrac{1}{V} ) (  \bx_i   \mathbbm{1}_{L}^\top   \bX_i  \!  + \!   \bX_i^\top   \mathbbm{1}_{L}   \bx_i^\top \Big) \bZin^\top      \bZin (\bX_j^\top \! -\! \tfrac{1}{V} \mathbbm{1}_V   \mathbbm{1}_L^\top ) \mathbbm{1}_L  \\
 & +  \frac{ \phi^\prime(0)^2}{N L} \sum_{j = 1}^N   \mathsf{z}_{k}^\top  \bZin  \Big(  \frac{1}{NL} \sum_{ \substack{ i = 1  \\ i \neq j }}^N ( \indic{\bx_i = \bx_j }-   \tfrac{1}{V}  )   \bN_i^\top   \bN_i  \Big) \bZin^\top      \bZin (\bx_j  - \tfrac{1}{V} \mathbbm{1}_V  )   \\ 
 & +  \frac{ \phi^\prime(0)^2}{N L} \sum_{j = 1}^N  \mathsf{z}_{k}^\top  \bZin  \Big(  \frac{1}{NL} \sum_{ \substack{ i = 1  \\ i \neq j }}^N ( \indic{\bx_i = \bx_j }-   \tfrac{1}{V} )   \bN_i^\top   \bN_i  \Big) \bZin^\top      \bZin (\bN_j^\top - \tfrac{1}{V} \mathbbm{1}_V   \mathbbm{1}_{L-1}^\top ) \mathbbm{1}_{L-1}  \\ 
& -   \frac{ \phi^\prime(0)^2}{N L^2} \sum_{j = 1}^N  \mathsf{z}_{k}^\top  \bZin  \Big(  \frac{1}{NL} \sum_{ \substack{ i = 1  \\ i \neq j }}^N ( \indic{\bx_i = \bx_j }-   \tfrac{1}{V} )   \bN_i^\top   \mathbbm{1}_L  \mathbbm{1}_L^\top   \bN_i \Big) \bZin^\top      \bZin (\bN_j^\top - \tfrac{1}{V} \mathbbm{1}_V   \mathbbm{1}_{L-1}^\top ) \mathbbm{1}_{L-1}  \\
& \eqqcolon   \vartheta_{aa} +   \vartheta_{ab} +   \vartheta_{ac} +   \vartheta_{ad} +   \vartheta_{ae}.
}
For the first summand, we write
\eq{
&  \vartheta_{aa}
 \coloneqq  \underbrace{ (1 - \frac{1}{L}) \frac{ \phi^\prime(0)^2}{N L} \sum_{j = 1}^N   \mathsf{z}_{k}^\top  \bZin  \Big(  \frac{1}{NL} \sum_{ \substack{ i = 1  \\ i \neq j }}^N ( \indic{\bx_i = \bx_j }-   \tfrac{1}{V}  )    \bx_i   \bx_i^\top  \Big) \bZin^\top      \bZin (\bx_j - \tfrac{1}{V} \mathbbm{1}_V) }_{\coloneqq  \vartheta_{aa1} } \\
& + \underbrace{ (1 - \frac{1}{L}) \frac{ \phi^\prime(0)^2}{N L} \sum_{j = 1}^N   \mathsf{z}_{k}^\top  \bZin  \Big(  \frac{1}{NL} \sum_{ \substack{ i = 1  \\ i \neq j }}^N ( \indic{\bx_i = \bx_j }-   \frac{1}{V} )   \bx_i   \bx_i^\top  \Big) \bZin^\top      \bZin (\bN_j^\top - \frac{1}{V} \mathbbm{1}_V   \mathbbm{1}_{L-1}^\top ) \mathbbm{1}_{L-1} }_{\coloneqq  \vartheta_{aa2} }.
}
We have  
\eq{
 \lvert \vartheta_{aa1}  \rvert  
& \leq  \Big(   \frac{ \phi^\prime(0)^2}{N^2 L^2} \sum_{j = 1}^N  \sum_{ \substack{ i = 1  \\ i \neq j }}^N \vert \indic{\bx_i = \bx_j }-   \tfrac{1}{V}   \rvert \Big)  \sup_{i \neq j } \Big \lvert   \indic{\bx_i \neq \be_{k} }     \mathsf{z}_{k}^\top  \bZin \bx_i \bx_i^\top  \bZin^\top      \bZin (\bx_j - \frac{1}{V} \mathbbm{1}_V) \Big   \rvert \\
& +  \Big( \frac{ \phi^\prime(0)^2}{N^2 L^2} \sum_{j = 1}^N  \sum_{ \substack{ i = 1  \\ i \neq j }}^N  \lvert \indic{\bx_j = \be_k }-   \tfrac{1}{V} \rvert ( \indic{\bx_i  =  \be_{k} } - \tfrac{1}{V} )   \Big)  \sup_{j } \lvert  \mathsf{z}_{k}^\top  \bZin  \be_k    \be_k^\top  \bZin^\top      \bZin (\bx_j - \tfrac{1}{V} \mathbbm{1}_V)  \Big  \rvert  \\
& +  \Big( \frac{ \phi^\prime(0)^2}{N^2 V L^2} \sum_{j = 1}^N  \sum_{ \substack{ i = 1  \\ i \neq j }}^N  \lvert \indic{\bx_j = \be_k }-   \tfrac{1}{V} \rvert \Big)  \sup_{j } \lvert   \mathsf{z}_{k}^\top  \bZin  \be_k    \be_k^\top  \bZin^\top      \bZin (\bx_j - \tfrac{1}{V} \mathbbm{1}_V)  \Big  \rvert  \\
& \leq \frac{C \log V}{V L^2 \sqrt{d}}.
}
where we use \ref{event:boundZin1} and \ref{event:discrete4}. 

Moreover, let
\eq{
 \vartheta_{aa2}  \eqqcolon  (1 - \frac{1}{L}) \frac{ \phi^\prime(0)^2}{N L} \sum_{j = 1}^N   \vartheta_{aa2,j}.
}
We have $\E [ \vartheta_{aa2,j} \vert \sZin  ] = 0$ and    $\E [ \vartheta_{aa2,j}  \vartheta_{aa2,j^\prime} \vert \sZin  ] = 0$  for $j \neq j^\prime$, and
\eq{
\E [   \vartheta_{aa2,j}^2 \vert \sZin ]  &    \leq  \frac{ CL }{d}    \E \Big[   \mathsf{z}_{k}^\top  \bZin  \Big(  \frac{1}{NL} \sum_{ \substack{ i = 1  \\ i \neq j }}^N  \indic{\bx_i = \bx_j }   \bx_i   \bx_i^\top  \Big) \bZin^\top \bZin  \Big(  \frac{1}{NL} \sum_{ \substack{ i = 1  \\ i \neq j }}^N   \indic{\bx_i = \bx_j }  \bx_i   \bx_i^\top  \Big)  \bZin^\top  \mathsf{z}_{k}   \vert \sZin \Big ] \\
& +  \frac{ CL }{ d V^2}  \E \Big[ \mathsf{z}_{k}^\top  \bZin  \Big(  \frac{1}{NL} \sum_{ \substack{ i = 1  \\ i \neq j }}^N    \bx_i   \bx_i^\top  \Big) \bZin^\top \bZin  \Big(  \frac{1}{NL} \sum_{ \substack{ i = 1  \\ i \neq j }}^N    \bx_i   \bx_i^\top  \Big)  \bZin^\top \mathsf{z}_{k} \vert \sZin \Big]   \leq \frac{ C  }{V^2 L d^2},
}
where we use  \ref{event:discrete8} and \ref{event:boundZin1}.

Therefore,  by Chebyshev's inequality with probability $1  - o_V(1)$, we have
\eq{
\lvert \vartheta_{aa2}   \rvert \leq  \frac{ C  \log V}{\sqrt{N} V L^{3/2} d} .
}
Therefore,
\eq{
\lvert \vartheta_{aa}  \rvert \leq   \frac{C \log V}{V L^2 \sqrt{d}} +\frac{ C  \log V}{\sqrt{N} V L^{3/2} d}.
}
Moreover,  for the second term, we write
\eq{
& \lvert \vartheta_{ab}  \rvert  
 \leq  \Big(  \frac{ \phi^\prime(0)^2}{N^2 L^3} \sum_{j = 1}^N  \sum_{ \substack{ i = 1  \\ i \neq j }}^N \vert \indic{\bx_i = \bx_j }-   \tfrac{1}{V}  \rvert \Big)    \sup_{i \neq j } \lvert \indic{\bx_i \neq \be_k } \mathsf{z}_{k}^\top  \bZin  (  \bx_i   \mathbbm{1}_{L}^\top   \bX_i   +   \bX_i^\top   \mathbbm{1}_{L}   \bx_i^\top  ) \bZin^\top      \bZin (\bX_j^\top - \tfrac{1}{V} \mathbbm{1}_V   \mathbbm{1}_L^\top ) \mathbbm{1}_L  \rvert \\
& + \Big(  \frac{ \phi^\prime(0)^2}{N^2 L^3} \sum_{j = 1}^N  \sum_{ \substack{ i = 1  \\ i \neq j }}^N \vert \indic{\bx_j = \be_k } \! - \!   \tfrac{1}{V}   \rvert  (\indic{\bx_i = \be_k } \! - \! \tfrac{1}{V} )   \Big)     \sup_{ j } \lvert  \mathsf{z}_{k}^\top  \bZin   (  \be_k  \mathbbm{1}_{L}^\top   \bX_i   +   \bX_i^\top   \mathbbm{1}_{L}   \be_k^\top \Big) \bZin^\top      \bZin (\bX_j^\top \! - \! \tfrac{1}{V} \mathbbm{1}_V   \mathbbm{1}_L^\top ) \mathbbm{1}_L  \rvert \\
& + \Big(  \frac{ \phi^\prime(0)^2}{N^2 V L^3} \sum_{j = 1}^N  \sum_{ \substack{ i = 1  \\ i \neq j }}^N \vert \indic{\bx_j = \be_k }-   \tfrac{1}{V}   \rvert      \Big)    \sup_{ j } \lvert  \mathsf{z}_{k}^\top  \bZin   (  \be_k  \mathbbm{1}_{L}^\top   \bX_i   +   \bX_i^\top   \mathbbm{1}_{L}   \be_k^\top \Big) \bZin^\top      \bZin (\bX_j^\top - \tfrac{1}{V} \mathbbm{1}_V   \mathbbm{1}_L^\top ) \mathbbm{1}_L  \rvert \\
& \leq \frac{C \log V}{V L^2 d }.
}
where we used \ref{event:boundZin1}, \ref{event:boundZin3}, \ref{event:discrete4} and \ref{cons:innerprodrestricted}.

For the third term, we write
\eq{
  \vartheta_{ac} & =  \frac{ \phi^\prime(0)^2}{N L^2} \sum_{i = 1}^N  \mathsf{z}_{k}^\top  \bZin  \Big(  \bN_i^\top   \bN_i - \frac{L-1}{V} \bs{I}_V \Big) \bZin^\top      \bZin   \Big( \frac{1}{N } \sum_{ \substack{ j = 1  \\ j \neq i }}^N ( \indic{\bx_i = \bx_j }-   \tfrac{1}{V}  )  (\bx_j  - \tfrac{1}{V} \mathbbm{1}_V  )    \Big)  \\
& + \frac{ \phi^\prime(0)^2 (L-1) }{N L^2 V}  \mathsf{z}_{k}^\top  \bZin  \bZin^\top      \bZin    \Big( \frac{1}{N }  \sum_{   j = 1   }^N   (\bx_j  - \tfrac{1}{V} \mathbbm{1}_V  )    (\bx_j  - \tfrac{1}{V} \mathbbm{1}_V  )^\top   \Big) \sum_{i = 1}^N  (\bx_i  - \tfrac{1}{V} \mathbbm{1}_V  )   \\
& - \frac{ \phi^\prime(0)^2 (1 - \frac{1}{V}) (L-1) }{N^2 L^2 V}  \mathsf{z}_{k}^\top  \bZin  \bZin^\top      \bZin     \sum_{i = 1}^N  (\bx_i  - \tfrac{1}{V} \mathbbm{1}_V  )   \\
& \eqqcolon    \vartheta_{ac1}  +   \vartheta_{ac2}  +   \vartheta_{ac3}. 
} 
By using  \ref{event:boundZin1}, \ref{event:discrete3}, \ref{event:discrete5},
\eq{
\E[    \vartheta_{ac2}^2 \vert \sZin ] \leq \frac{C}{N^2 L^2 V^3 d}    \mathsf{z}_{k}^\top  \bZin  \bZin^\top       \bZin    \bZin^\top   \mathsf{z}_{k}
 \leq \frac{C}{N^2 L^2 V d^3}.
}
Moreover,  
\eq{
\E[    \vartheta_{ac3}^2 \vert \sZin ] \leq    \frac{C}{N^3 L^2 V^2 d}      \mathsf{z}_{k}^\top  \bZin  \bZin^\top       \bZin    \bZin^\top   \mathsf{z}_{k} \leq  \frac{C}{N^3 L^2   d^3}  
}
Therefore,
\eq{
\lvert   \vartheta_{ac2}  \rvert \leq \frac{C \log V}{N \sqrt{V} L d^{\frac{3}{2}}}, \qquad  \lvert   \vartheta_{ac3}  \rvert \leq \frac{C \log V}{N \sqrt{N} L d^{\frac{3}{2}}}.
}

Moreover,  we have
\eq{
 \E[    \vartheta_{ac1}^2 ]   
&\leq \frac{C}{N^2 V^2 L^4 d} \sum_{i = 1}^N \E \Big[   \mathsf{z}_{k}^\top  \bZin  \Big(  \bN_i^\top   \bN_i - \frac{L-1}{V} \bs{I}_V \Big) \bZin^\top        \bZin  \Big(  \bN_i^\top   \bN_i - \frac{L-1}{V} \bs{I}_V \Big)   \bZin ^\top  \mathsf{z}_{k} \vert \sZin \Big] \\
&  \leq \frac{C}{N^2 V^2 L^4 d} \frac{L- 1}{V} \sum_{i = 1}^N \E \Big[  \mathsf{z}_{k}^\top  \bZin     \bZin ^\top  \mathsf{z}_{k} \vert \sZin \Big]  
  -  \frac{C}{N^2 V^2 L^4 d} \frac{L - 1}{V^2} \sum_{i = 1}^N \E \Big[   \mathsf{z}_{k}^\top  \bZin    \bZin^\top        \bZin    \bZin ^\top   \mathsf{z}_{k} \vert \sZin \Big] \\
& \leq  \frac{C}{N V^2 L^3 d^2}.
}
Therefore, by Chebyshev's inequality, we have 
\eq{
\lvert \vartheta_{ac1} \rvert \leq \frac{C \log V}{\sqrt{N} V L^{\frac{3}{2}} d}
}
For the fourth term, we have $\E[    \vartheta_{ad}  \vert \sZin ] = 0$ and
\eq{
 \E[  \vartheta_{ad}^2 \vert \sZin ] 
& =    \frac{C}{N^4 L^4}  \E  \! \Big[ \Big( \sum_{i, j = 1}^N \! \mathbbm{1}_{i \neq j} (\mathbbm{1}_{\bx_i = \bx_j} \! - \! \tfrac{1}{V} \mathbbm{1}_V) \mathsf{z}_{k}^\top  \bZin  \bN_i^\top   \bN_i    \bZin^\top      \bZin (\bN_j^\top \! - \! \frac{1}{V} \mathbbm{1}_V   \mathbbm{1}_{L-1}^\top ) \mathbbm{1}_{L-1} \Big)^2 \vert \sZin \Big] \\
& =    \frac{C}{V N^4 L^4}  \sum_{i, j = 1}^N  \E \Big[ \Big( \mathsf{z}_{k}^\top  \bZin      \bN_i^\top   \bN_i    \bZin^\top      \bZin (\bN_j^\top - \frac{1}{V} \mathbbm{1}_V   \mathbbm{1}_{L-1}^\top ) \mathbbm{1}_{L-1} \Big)^2 \vert \sZin \Big] \\
& \leq    \frac{C}{V N^4 L^3 d}  \sum_{i, j = 1}^N \E \Big[  \mathsf{z}_{k}^\top  \bZin      \bN_i^\top   \bN_i    \bZin^\top   \bZin  \bN_i^\top   \bN_i   \bZin^\top  \mathsf{z}_{k}  \vert \sZin \Big] \\
&  \leq    \frac{C}{V N^2 L  d^2 (L \wedge d)}  
}
where we used \ref{event:prop8result2} in the last step.

 Therefore,  by Chebyshev's inequality with probability $1  - o_V(1)$, we have
\eq{
\lvert  \vartheta_{ad}  \rvert \leq    \frac{C \log V}{N \sqrt{V L }  d \sqrt{L \wedge d}} 
}

For the last term, we have $\E[  \vartheta_{ae} \vert \sZin   ] = 0$ and
\eq{
 \E[ \vartheta_{ae}^2  \vert \sZin  ] 
& \leq      \frac{ C}{N^4 L^6} \E \Big[ \Big( \sum_{i, j = 1}^N   \mathbbm{1}_{i \neq j} ( \mathbbm{1}_{\bx_i = \bx_j} - \tfrac{1}{V}) \mathsf{z}_{k}^\top  \bZin    \bN_i^\top   \mathbbm{1}_L  \mathbbm{1}_L^\top   \bN_i   \bZin^\top      \bZin (\bN_j^\top - \frac{1}{V} \mathbbm{1}_V   \mathbbm{1}_{L-1}^\top ) \mathbbm{1}_{L-1}  \Big)^2  \vert \sZin   \Big] \\
& \leq      \frac{ C}{V N^4 L^6} \sum_{i, j = 1}^N    \E \Big[ \Big( \mathsf{z}_{k}^\top  \bZin   \bN_i^\top   \mathbbm{1}_L  \mathbbm{1}_L^\top   \bN_i   \bZin^\top      \bZin (\bN_j^\top - \frac{1}{V} \mathbbm{1}_V   \mathbbm{1}_{L-1}^\top ) \mathbbm{1}_{L-1}  \Big)^2 \vert \sZin   \Big] \\
&  \leq      \frac{ C}{N^4 L^4 d V}  \sum_{i, j = 1}^N \E \Big[  \mathsf{z}_{k}^\top  \bZin   \bN_i^\top   \mathbbm{1}_{L-1}  \mathbbm{1}_{L-1}^\top   \bN_i    \bZin  ^\top   \mathsf{z}_{k}    \vert \sZin  \Big] \\
&  \leq      \frac{ C}{N^2 L^4 d V} \Big(1 + \frac{L}{d}  \Big)
}
where we use \ref{event:boundZin4} in the last step.

 Therefore,  by Chebyshev's inequality with probability $1  - o_V(1)$, we have
\eq{
\lvert \vartheta_{ae} \rvert \leq    \frac{C \log V}{N  \sqrt{V}  L^{3/2}  \sqrt{d}(L \wedge d)^{1/2}} . 
} 
Overall, we have
\eq{
\lvert  \vartheta_{11a} \rvert   \leq  C \log V \Big(   \frac{1}{V L^2 \sqrt{d}} +  \frac{1}{\sqrt{N} V \sqrt{L d} (L \wedge d)} \Big).   \label{eq:meanbound3rdterm}
}
Therefore, by \eqref{eq:meanbound1stterm}-\eqref{eq:meanbound2ndterm}-\eqref{eq:meanbound3rdterm} and using $N \ll VL$ and $L \ll V$, we have 
\eq{
   \lvert   \vartheta_{11} \rvert   &  \leq 
 C \log V \Big(   \frac{1}{V L^2 \sqrt{d}} +  \frac{1}{\sqrt{N} V \sqrt{L d} (L \wedge d)} \Big) \\
& +   C \log V  \Big( \frac{1 }{N \sqrt{V} ( L \wedge d) \sqrt{d} }     +      \frac{1}{N  L^2 \sqrt{L \wedge d}}       \Big)  
 + \frac{C \log V}{V L \sqrt{d} \sqrt{L \wedge d} } \frac{1}{V \wedge L^2 \wedge L \sqrt{d}} \\
 & \leq  C \log V  \Big( \frac{1 }{N \sqrt{V} ( L \wedge d) \sqrt{d} }   + \frac{1}{V L^2 (L \wedge d)^{1/2}}  \Big)  .   \label{eq:meanfirsttwoterms}
}
Finally,
\eq{
 \vartheta_{12}  
 & =   \frac{1}{N^2 L^2 V}  \sum_{i, j = 1}^N    \alpha_{ij}  ( \indic{\bx_i = \bx_j}  -   \tfrac{1}{V}  )   \mathsf{z}_{k}^\top  \bZin \bX_i^\top   \big(\bs{I}_L - \tfrac{1}{L} \mathbbm{1}_L  \mathbbm{1}_L^\top \big)  \bX_i  \bZin^\top      \bZin       \mathbbm{1}_V  \\
 & =   \frac{1}{N^2 L^2 V}  \sum_{i, j = 1}^N    \alpha_{ij}  ( \indic{\bx_i = \bx_j}  -   \tfrac{1}{V}  )   \mathsf{z}_{k}^\top  \bZin \Big( \bX_i^\top   \bX_i  - \tfrac{L}{V} \bs{I}_V \Big) \bZin^\top      \bZin       \mathbbm{1}_V  \\
& +  \frac{1}{N^2 L  V^2} \sum_{i, j = 1}^N    \alpha_{ij}  ( \indic{\bx_i = \bx_j}  -   \tfrac{1}{V}  )    \mathsf{z}_{k}^\top  \bZin  \bZin^\top      \bZin       \mathbbm{1}_V  \\ 
&  -  \frac{1}{N^2 L^3 V}  \sum_{i, j = 1}^N    \alpha_{ij}  ( \indic{\bx_i = \bx_j}  -   \tfrac{1}{V}  )   \mathsf{z}_{k}^\top  \bZin \bX_i^\top    \mathbbm{1}_L  \mathbbm{1}_L^\top \bX_i  \bZin^\top      \bZin       \mathbbm{1}_V 
}
\begin{itemize}[leftmargin=*]
\item For the first term,
\eq{
& \frac{1}{N^2 L^2 V}  \sum_{i, j = 1}^N    \alpha_{ij}  ( \indic{\bx_i = \bx_j}  -   \tfrac{1}{V}  )   \mathsf{z}_{k}^\top  \bZin \big( \bX_i^\top   \bX_i  - \tfrac{L}{V} \bs{I}_V \big) \bZin^\top      \bZin       \mathbbm{1}_V \\
&\leq  \frac{1}{N^2 L^2 V}  \Big( \sum_{i, j = 1}^N      \lvert \indic{\bx_i = \bx_j}  -   \tfrac{1}{V}  \rvert  \Big) \sup_{i,j} \Big \lvert (\alpha_{ij} - \phi^\prime(0)^2) \mathsf{z}_{k}^\top  \bZin \big( \bX_i^\top   \bX_i  - \tfrac{L}{V} \bs{I}_V \big) \bZin^\top      \bZin       \mathbbm{1}_V \Big \rvert  \\
& + \frac{\phi^\prime(0)^2}{N^2 L^2 V}  \Big( \sum_{i = 1}^N    \Big     \lvert  \sum_{j = 1}^N   ( \indic{\bx_i = \bx_j}  -   \tfrac{1}{V} ) \Big  \rvert  \Big) \sup_{i } \Big \lvert (  \mathsf{z}_{k}^\top  \bZin \big( \bX_i^\top   \bX_i  - \tfrac{L}{V} \bs{I}_V \big) \bZin^\top      \bZin       \mathbbm{1}_V \Big \rvert  \\
&\leq  \frac{1}{N^2 L^2 V}  \Big( \sum_{i, j = 1}^N      \lvert \indic{\bx_i = \bx_j}  -   \tfrac{1}{V}  \rvert  \Big) \sup_{i,j} \Big \lvert (\alpha_{ij} - \phi^\prime(0)^2) \mathsf{z}_{k}^\top  \bZin \big( \bX_i^\top   \bX_i  - \tfrac{L}{V} \bs{I}_V \big) \bZin^\top      \bZin       \mathbbm{1}_V \Big \rvert  \\
& + \frac{\phi^\prime(0)^2}{N  L^2 V}     \Big     \lVert  \sum_{j = 1}^N   (  \bx_j   -   \tfrac{1}{V} \mathbbm{1}_V) \Big  \rVert_{\infty} \sup_{i } \Big \lvert   \mathsf{z}_{k}^\top  \bZin \big( \bX_i^\top   \bX_i  - \tfrac{L}{V} \bs{I}_V \big) \bZin^\top      \bZin       \mathbbm{1}_V \Big \rvert  \\
& \lesssim \frac{1}{VL^2 \sqrt{V} (L \wedge d)},
}
where we \ref{event:boundZin2}, \ref{event:boundZin8},   \ref{event:discrete3}, and \ref{cons:coeffbound}.
\item For the second term
\eq{
 & \frac{1}{N^2 L  V^2}  \sum_{i = 1}^N  \sum_{j = 1}^N  \alpha_{ij}  ( \indic{\bx_i = \bx_j}  -   \tfrac{1}{V}  )    \mathsf{z}_{k}^\top  \bZin  \bZin^\top      \bZin       \mathbbm{1}_V  \\
 & \leq     \frac{1}{N^2 L  V^2}  \Big( \sum_{i = 1}^N  \sum_{j = 1}^N \lvert \indic{\bx_i = \bx_j}  -   \tfrac{1}{V}  \rvert \Big) \sup_{i,j} \lvert   \alpha_{ij}     \mathsf{z}_{k}^\top  \bZin  \bZin^\top      \bZin       \mathbbm{1}_V  \rvert   \leq    \frac{C \log V}{V^{3/2} L   d^{3/2}},
}
where we used \ref{event:boundZin3}, \ref{event:discrete4}, and \ref{cons:coeffbound}.
 \item For the third term,
\eq{
& \frac{1}{N^2 L^3 V}  \sum_{i, j = 1}^N    \alpha_{ij}  ( \indic{\bx_i = \bx_j}  -   \tfrac{1}{V}  )   \mathsf{z}_{k}^\top  \bZin \bX_i^\top    \mathbbm{1}_L  \mathbbm{1}_L^\top \bX_i  \bZin^\top      \bZin       \mathbbm{1}_V \\
& \leq \frac{1}{N^2 L^3 V}  \Big( \sum_{i, j = 1}^N \lvert \mathbbm{1}_{\bx_i = \bx_j} - \tfrac{1}{V} \mathbbm{1}_V  \rvert \Big) \sup_{i,j} \lvert \alpha_{ij}   \mathsf{z}_{k}^\top  \bZin \bX_i^\top    \mathbbm{1}_L  \mathbbm{1}_L^\top \bX_i  \bZin^\top      \bZin       \mathbbm{1}_V  \rvert  \\
& \leq \frac{1}{V^2 L^2} \frac{\Big( \sqrt{L} + \sqrt{\frac{V}{d}} \Big)}{\sqrt{L \wedge d}} 
}
where we used \ref{event:boundZin4}, \ref{event:boundZin5}, \ref{event:discrete4}, and \ref{cons:coeffbound}.
\end{itemize}

Therefore, by Chebyshev's inequality with probability $1  - o_V(1)$, we have
\eq{
\lvert  \vartheta_{12}  \rvert \leq   \frac{C  \log^3 V}{N  L^3 V d^2}   +   \frac{C \log^2 V}{V L  \sqrt{V} d^{3/2}}. \label{eq:meanthirdterm}
} 
By   \eqref{eq:meanfirsttwoterms}-\eqref{eq:meanthirdterm}, overall we have
\eq{
 \vartheta_{1} & \leq 
 C \log V \Big(   \frac{1}{V L^2 \sqrt{d}} +  \frac{1}{\sqrt{NV} L^{3/2} d} +  \frac{1}{N \sqrt{V L}  d \sqrt{L \wedge d}} +    \frac{1}{N L  d \sqrt{V} (L \wedge \sqrt{V})} \Big) \\
& +   C \log V  \Big( \frac{1 }{N \sqrt{V} ( L \wedge d) \sqrt{d} }   +   \frac{1 }{N  \sqrt{V} d^{\frac{3}{2}}}  +      \frac{1}{N  L^2 \sqrt{L \wedge d}}       \Big)  
 + \frac{C \log V}{V \sqrt{L d} (L \wedge d)} \frac{1}{V \wedge L^2 \wedge L \sqrt{d}} \\
 & \leq  C \log V  \Big( \frac{1 }{N \sqrt{V} ( L \wedge d) \sqrt{d} }   +  \frac{1}{\sqrt{NV} L^{3/2}  d}  + \frac{1}{V L^{3/2} d (L \wedge d)}  + \frac{1}{V^2 \sqrt{ Ld} (L \wedge d) } \Big) \\
 & + \frac{C \log V}{V L^2 (L \wedge d)^{1/2}}. 
}
 
\subsubsection{Concentration bound for $ \varphi$}
\label{sec:concentrations12}

We recall that
\eq{
   \varphi & =     \frac{  \mu_{kl}  }{N^2 L^2}   \sum_{i, j = 1}^N  \alpha_{ij}  \big(  \bs{e}_1 - \tfrac{1}{L}  \mathbbm{1}_L   \big)^\top    \bX_i  \bZin^\top       \bZin    \bX_j^\top \mathbbm{1}_L              ( \bx_j -   \tfrac{1}{V} \mathbbm{1}_{V} )^\top             \Zout^\top    \Zout ( \bx_i -   \tfrac{1}{V} \mathbbm{1}_{V} )        
}
In this part, we will focus on the term
\eq{
&  \frac{1}{N^2 L^2}   \sum_{i, j = 1}^N   \alpha_{ij}  \big(  \be_1 - \tfrac{1}{L}  \mathbbm{1}_L   \big)^\top     \bX_i  \bZin^\top       \bZin    \bX_j^\top \mathbbm{1}_L              ( \bx_j -   \tfrac{1}{V} \mathbbm{1}_{V} )^\top             \Zout^\top    \Zout ( \bx_i -   \tfrac{1}{V} \mathbbm{1}_{V} ) \\
& =    \frac{1}{N^2 L^2}   \sum_{i, j = 1}^N   \alpha_{ij} \tr \Big(  \bZin^\top       \bZin       \bX_j^\top \mathbbm{1}_L              ( \bx_j -   \tfrac{1}{V} \mathbbm{1}_{V} )^\top             \Zout^\top    \Zout         ( \bx_i -   \tfrac{1}{V} \mathbbm{1}_{V} )    \bx_i^\top \Big) \\
& -    \frac{1}{N^2 L^3}   \sum_{i, j = 1}^N    \alpha_{ij}   \tr \Big(    \bZin^\top       \bZin    \bX_j^\top \mathbbm{1}_L              ( \bx_j -   \tfrac{1}{V} \mathbbm{1}_{V} )^\top             \Zout^\top    \Zout ( \bx_i -   \tfrac{1}{V} \mathbbm{1}_{V} )     \mathbbm{1}_L^\top    \bX_i  \Big) \\
& \eqqcolon   \varphi_1  +    \varphi_2.
}
For the first term,  we write
\eq{
  \varphi_1 
  &  =  \frac{  \phi^\prime(0)^2  }{N^2 L^2}   \sum_{i, j = 1}^N  \tr \Big(   \Zout  ( \bx_i -   \tfrac{1}{V} \mathbbm{1}_{V} )    \bx_i^\top  \bZin^\top       \bZin       \bX_j^\top \mathbbm{1}_L              ( \bx_j -   \tfrac{1}{V} \mathbbm{1}_{V} )^\top             \Zout^\top    \Big)    \\
&  +     \frac{1}{N^2 L^2}   \sum_{i, j = 1}^N  \tr \Big(   \Zout      ( \alpha_{ij} -   \phi^\prime(0)^2 )     ( \bx_i -   \frac{1}{V} \mathbbm{1}_{V} )   \bx_i^\top \bZin^\top       \bZin       \bX_j^\top \mathbbm{1}_L              ( \bx_j -   \frac{1}{V} \mathbbm{1}_{V} )^\top             \Zout^\top    \Big)  \\
& \eqqcolon     \varphi_{11}   +    \varphi_{12} 
}
We start with the second term. By Proposition \ref{prop:gaussquadraticformtrace}, we have 
\eq{
\varphi_{12} 
& = \frac{1}{N^2 L^2}     \sum_{i, j = 1}^N  ( \alpha_{ij} -   \phi^\prime(0)^2 )     ( \indic{\bx_i = \bx_j}-   \frac{1}{V} )   \bx_i^\top   \bZin^\top       \bZin       \bX_j^\top \mathbbm{1}_L             \\
& \pm \frac{\log^2 V}{N^2 L^2 \sqrt{d}} \Big \lVert \sum_{i, j = 1}^N    ( \alpha_{ij} -   \phi^\prime(0)^2 )     ( \bx_i -   \frac{1}{V} \mathbbm{1}_{V} )   \bx_i^\top   \bZin^\top       \bZin       \bX_j^\top \mathbbm{1}_L              ( \bx_j -   \frac{1}{V} \mathbbm{1}_{V} )^\top   \Big \rVert_F            
}
Let $\bs{M}  \coloneqq \big \{  ( \alpha_{ij} -   \phi^\prime(0)^2) \bx_i^\top   \bZin^\top       \bZin       \bX_j^\top \mathbbm{1}_L    )   \big\}_{i,j \in [N]}$. We have
\eq{
\Big \lVert \sum_{i, j = 1}^N    ( \alpha_{ij} -   \phi^\prime(0)^2 )     ( \bx_i -   \tfrac{1}{V} \mathbbm{1}_{V} )   \bx_i^\top   \bZin^\top    &    \bZin       \bX_j^\top \mathbbm{1}_L              ( \bx_j -   \tfrac{1}{V} \mathbbm{1}_{V} )^\top   \Big \rVert_F  \\
& \leq    \lVert \bs{M}    \rVert_F  \Big  \lVert   \sum_{i = 1}^N      ( \bx_i -   \frac{1}{V} \mathbbm{1}_{V} )      ( \bx_i -   \frac{1}{V} \mathbbm{1}_{V} )  ^\top  \Big \rVert_2  \leq \frac{N}{V}  \lVert \bs{M}    \rVert_F, 
}
where we used \ref{event:discrete5}. 
Moroever, 
\eq{
\lVert \bs{S}    \rVert_F^2 & = \sum_{i, j = 1}^N  ( \alpha_{ij} -   \phi^\prime(0)^2 ) ^2  (   \bx_i^\top   \bZin^\top       \bZin       \bX_j^\top \mathbbm{1}_L   )^2 \\
& \leq  \Big( \sum_{i \neq j = 1}^N  \lvert \alpha_{ij} -   \phi^\prime(0)^2 \rvert^2 +  \sum_{i  = 1}^N  \lvert \alpha_{ii} -   \phi^\prime(0)^2 \rvert^2 \Big) \sup_{i,j } \lvert \bx_i^\top   \bZin^\top       \bZin       \bX_j^\top \mathbbm{1}_L \rvert \\
& \lesssim \Big(  \frac{N^2}{(V \wedge L^2 \wedge L\sqrt{d})^2} + \frac{N}{L^2} \Big) \Big(1 + \frac{L}{d} \Big),
}
where we used  \ref{event:boundZin4} and \ref{event:discrete4}.
Therefore,
\eq{
&  \frac{1}{N^2 L^2 \sqrt{d}} \Big \lVert \sum_{i, j = 1}^N     ( \alpha_{ij} -   \phi^\prime(0)^2 )      ( \bx_i -   \frac{1}{V} \mathbbm{1}_{V} )   \bx_i^\top   \bZin^\top       \bZin       \bX_j^\top \mathbbm{1}_L              ( \bx_j -   \frac{1}{V} \mathbbm{1}_{V} )^\top   \Big \rVert_F   \\
& \leq    \frac{1}{N V L^{3/2} \sqrt{d} (L \wedge d)^{1/2}}   \Big(  \frac{N}{ V \wedge L^2 \wedge L \sqrt{d}}  + \frac{\sqrt{N}}{L}  \Big)  \lesssim  \frac{1}{V L^2 \sqrt{L \wedge d} }.
}
Moreover,
\eq{
&  \frac{1}{N^2 L^2}  \Big \lvert  \sum_{i, j = 1}^N  ( \alpha_{ij} -   \phi^\prime(0)^2 )     ( \indic{\bx_i = \bx_j}-   \tfrac{1}{V} )   \bx_i^\top   \bZin^\top       \bZin       \bX_j^\top \mathbbm{1}_L   \Big \rvert   \\
&  \leq  \frac{1}{N^2 L^2}  \Big(    \sum_{i, j = 1}^N \lvert \indic{\bx_i = \bx_j}-   \tfrac{1}{V} \rvert \Big)  \sup_{i,  j \in [N]}  \lvert (\alpha_{ij} -   \phi^\prime(0)^2 )      \bx_i^\top   \bZin^\top       \bZin       \bX_j^\top \mathbbm{1}_L    \rvert       \lesssim   \frac{1}{V  L^2 \sqrt{L \wedge d}},
}
where we used \ref{event:boundZin4}, \ref{event:discrete4}, \ref{cons:coeffbound}.
Therefore,  $ \lvert  \varphi_{12}  \rvert   \lesssim   \frac{1}{V  L^2 \sqrt{L \wedge d}}$.    

Next, we consider $\lvert  \varphi_{2}  \rvert $. By Proposition \ref{prop:gaussquadraticformtrace}, 
\eq{
 \varphi_{2}
& =   \frac{\phi^\prime(0)^2}{N^2 L^3}   \sum_{i, j = 1}^N      ( \indic{\bx_i = \bx_j}-   \tfrac{1}{V}  )      \mathbbm{1}_L^\top    \bX_i     \bZin^\top       \bZin    \bX_j^\top \mathbbm{1}_L  \\
&    \pm \frac{\phi^\prime(0)^2}{\sqrt{d}}  \frac{1}{N^2 L^3}   \Big   \lVert  \sum_{i, j = 1}^N     ( \bx_i -   \tfrac{1}{V} \mathbbm{1}_{V} )     \mathbbm{1}_L^\top    \bX_i     \bZin^\top       \bZin    \bX_j^\top \mathbbm{1}_L   ( \bx_j -   \tfrac{1}{V} \mathbbm{1}_{V} )^\top  \Big  \rVert_F   \\
& +  \frac{1}{N^2 L^3}   \sum_{i, j = 1}^N   ( \alpha_{ij} - \phi^\prime(0)^2 )   ( \indic{\bx_i = \bx_j}-   \tfrac{1}{V}  )      \mathbbm{1}_L^\top    \bX_i     \bZin^\top       \bZin    \bX_j^\top \mathbbm{1}_L \\
&  \pm \frac{1}{\sqrt{d}} \frac{1}{N^2 L^3}    \Big   \lVert  \sum_{i, j = 1}^N     ( \alpha_{ij} - \phi^\prime(0)^2 )      ( \bx_i -   \tfrac{1}{V} \mathbbm{1}_{V} )     \mathbbm{1}_L^\top    \bX_i     \bZin^\top       \bZin    \bX_j^\top \mathbbm{1}_L   ( \bx_j -   \tfrac{1}{V} \mathbbm{1}_{V} )^\top  \Big  \rVert_F   \\
& \eqqcolon  \varphi_{21} +   \varphi_{22} +   \varphi_{23} +  \varphi_{24}.
}
For $ \varphi_{24}$, we define $\bs{M} \coloneqq \big \{   ( \alpha_{ij} - \phi^\prime(0)^2 )       \mathbbm{1}_L^\top    \bX_i     \bZin^\top       \bZin    \bX_j^\top \mathbbm{1}_L    \big \}_{i,j\in[N]}$. Similar to above, we have
\eq{
\varphi_{24} \leq \frac{1}{N V L^3 \sqrt{d}} \lVert \bs{M} \rVert_F.
}
We have
\eq{
\lVert \bs{M}    \rVert_F^2 & = \sum_{i, j = 1}^N  ( \alpha_{ij} -   \phi^\prime(0)^2 ) ^2  (   \bX_i^\top   \bZin^\top       \bZin       \bX_j^\top \mathbbm{1}_L   )^2 \\
& \leq  \Big( \sum_{i \neq j = 1}^N  \lvert \alpha_{ij} \! - \!    \phi^\prime(0)^2 \rvert^2 \Big) \sup_{i \neq j } \lvert \mathbbm{1}_L^\top \bX_i^\top   \bZin^\top       \bZin       \bX_j^\top \mathbbm{1}_L \rvert  +  \Big( \sum_{i  = 1}^N  \lvert \alpha_{ii} \! - \!   \phi^\prime(0)^2 \rvert^2 \Big) \sup_{i } \lvert \mathbbm{1}_L^\top \bX_i^\top   \bZin^\top       \bZin       \bX_i^\top \mathbbm{1}_L \rvert \\
& \lesssim L \Big(  \frac{N^2}{(V \wedge L^2 \wedge L\sqrt{d})^3}  + \frac{N}{L^2} \Big),
}
where we used \ref{cons:innerprodrestricted}, \ref{cons:coeffbound}.
Therefore,
\eq{
\lvert  \varphi_{24}   \rvert \leq   \frac{1}{N V L^2 \sqrt{d}} \Big( \frac{N}{( V \wedge L^2 \wedge L \sqrt{d})^{3/2}}    + \frac{\sqrt{N}}{L}    \Big) \lesssim   \frac{1}{V L^2 \sqrt{L \wedge d}}.
}
For $\varphi_{23}$, we have
\eq{
\lvert  \varphi_{23} \rvert & \leq \frac{1}{N^2 L^3} \Big(  \sum_{i \neq j = 1}^N \lvert  \indic{\bx_i = \bx_j}-   \frac{1}{V} \rvert \Big) \rvert \sup_{i \neq j} \lvert ( \alpha_{ij} - \phi^\prime(0)^2 )      \mathbbm{1}_L^\top    \bX_i     \bZin^\top       \bZin    \bX_j^\top \mathbbm{1}_L \rvert \\
& +  \frac{1}{N L^3}  \sup_{i} \lvert   ( \alpha_{ii} - \phi^\prime(0)^2 )    \mathbbm{1}_L^\top    \bX_i     \bZin^\top       \bZin    \bX_i^\top \mathbbm{1}_L \rvert \\
& \lesssim   \frac{1}{V L^2 \sqrt{L \wedge d}},
}
where we used \ref{event:discrete4}, \ref{cons:innerprodrestricted}, \ref{cons:coeffbound}.

For the first two terms,  we define
\begin{alignat}{2}
& \bs{V}_0 \coloneqq     \frac{1}{N L}   \sum_{  j = 1}^N     \bX_j^\top \mathbbm{1}_L   ( \bx_j -   \frac{1}{V} \mathbbm{1}_{V} )^\top, ~~ &&  \bs{V}_{0,1} \coloneqq   \frac{1}{NL} \sum_{i = 1}^N   \bx_i   ( \bx_i -   \frac{1}{V}  \mathbbm{1}_{V} ) ^\top  \\
& \bs{V}_{0,2} \coloneqq    \frac{1}{NL} \sum_{i = 1}^N \big ( \bN_j^\top - \frac{1}{V}  \mathbbm{1}_{L-1} ^\top \big) \mathbbm{1}_{L-1}     ( \bx_i -   \frac{1}{V}  \mathbbm{1}_{V} ) ^\top, ~~ && \bs{V}_{0,3} \coloneqq    \frac{1}{V} \mathbbm{1}_{V}      \frac{1}{N L} \sum_{i = 1}^N    ( \bx_i -   \frac{1}{V}  \mathbbm{1}_{V} ) ^\top.
\end{alignat}
We have by  \ref{event:discrete5}-\ref{event:discrete7},
\eq{
\lvert  \varphi_{22}  \rvert  \leq   \frac{\phi^\prime(0)^2}{\sqrt{d}}  \frac{1}{L}     \Big   \lVert     \bZin    \bs{V}_0  \bs{V}_0^\top  \bZin^\top    \Big  \rVert_F  \leq  \frac{\phi^\prime(0)^2}{NVL^2  \sqrt{d} }    \Big   \lVert     \bZin   \bZin^\top    \Big  \rVert_F  &  \leq   \frac{2 \phi^\prime(0)^2}{L d  LN}   \lesssim   \frac{1}{V L^2 d}.
}
Lastly,
\eq{
\lvert  \varphi_{21}  \rvert  =    \frac{\phi^\prime(0)^2}{L} \tr \Big(   \bZin    \bs{V}_0  \bs{V}_0^\top  \bZin^\top  \Big)  \leq    \frac{\phi^\prime(0)^2}{NV L^2 \sqrt{d}} \tr \Big(   \bZin    \bZin^\top  \Big)   \leq    \frac{2 \phi^\prime(0)^2}{N L^2 \sqrt{d}}     \lesssim   \frac{1}{V L^2 \sqrt{d}}.
}
Therefore,  $ \lvert      \varphi_{21}    \rvert \lesssim   \frac{1}{V L^2 \sqrt{L \wedge d}}$.    

Lastly, we consider $ \varphi_{11}$. By Proposition \ref{prop:gaussquadraticformtrace}, we have  
\eq{
& \lvert   \varphi_{11}  \rvert \\
& =   \phi^\prime(0)^2    \tr(      \bs{V}_{0,1}^\top   \bZin^\top  \bZin  \bs{V}_{0,1}  )  +    \phi^\prime(0)^2    \tr(    \bs{V}_{0,1}^\top  \bZin^\top  \bZin      \bs{V}_{0,2}   )  +  \phi^\prime(0)^2 L   \tr(    \bs{V}_{0,1}^\top  \bZin^\top  \bZin      \bs{V}_{0,3}         )  \\
& \pm \frac{1}{\sqrt{d}} \Big \lVert     \bs{V}_{0,1}^\top \bZin^\top  \bZin     \bs{V}_{0,1} \Big     \rVert_F  \pm  \frac{1}{ \sqrt{d}} \Big \lVert  \bs{V}_{0,1}^\top \bZin^\top  \bZin      \bs{V}_{0,2}  \Big     \rVert_F  \pm \frac{L}{\sqrt{d}} \Big \lVert      \bs{V}_{0,1}^\top   \bZin^\top  \bZin    \bs{V}_{0,3}  \Big     \rVert_F.
}
\begin{itemize}[leftmargin=*]
\item For the first term, by \ref{event:discrete5}, we have
\eq{
   \tr(      \bs{V}_{0,1}^\top   \bZin^\top  \bZin  \bs{V}_{0,1}  )  =    \tr(     \bZin  \bs{V}_{0,1}  \bs{V}_{0,1}^\top   \bZin^\top    )  \asymp \frac{1}{VL^2}
}
\item  For the second term, 
\eq{
 \tr(    \bs{V}_{0,1}^\top  \bZin^\top  \bZin      \bs{V}_{0,2}   ) =  \frac{1}{N  L } \sum_{j = 1}^N (  \bx_j - \tfrac{1}{V} \mathbbm{1}_V )^\top \bs{V}_{0,1}^\top     \bZin^\top  \bZin    ( \bN_j^\top - \tfrac{1}{V}  \mathbbm{1}_{L-1} ^\top \big) \mathbbm{1}_{L-1}
}
We have
\eq{
\E \Big[ \big( (  \bx_j - \tfrac{1}{V} \mathbbm{1}_V )^\top \bs{V}_{0,1}^\top     \bZin^\top  \bZin    ( \bN_j^\top - \tfrac{1}{V}  \mathbbm{1}_{L-1} ^\top \big) \mathbbm{1}_{L-1} \big)^2 \Big \vert \sZin \Big] \lesssim \frac{1}{V^2 L d}
}
By Chebyshev's inequality,
\eq{
\Big \lvert \tr(    \bs{V}_{0,1}^\top  \bZin^\top  \bZin      \bs{V}_{0,2}   ) \Big \rvert \lesssim \frac{1}{\sqrt{N} V L^{3/2} \sqrt{d}}.
}
\item  The third summand:  We have
\eq{
 (L-1) &  \tr(    \bs{V}_{0,1}^\top  \bZin^\top  \bZin      \bs{V}_{0,3}   ) \leq L \lVert     \bZin  \bs{V}_{0,1} \rVert_2    \lVert    \bZin      \bs{V}_{0,3} \rVert_2  
}
where we used  that $\bZin      \bs{V}_{0,3}$ is 1-rank. By \ref{event:discrete5}-\ref{event:discrete7}
\eq{
L \lVert  \bs{V}_{0,1}^\top  \bZin^\top\rVert_2    \lVert    \bZin      \bs{V}_{0,3} \rVert_2   \lesssim \frac{L}{\sqrt{V} L \sqrt{d}} \frac{1}{\sqrt{V} L N} \leq \frac{1}{NVL \sqrt{d}}.
}
\item    The fourth  summand:   We have by \ref{event:discrete5}-\ref{event:discrete7} 
\eq{
\frac{1}{\sqrt{d}} \Big \lVert     \bs{V}_{0,1}^\top \bZin^\top  \bZin     \bs{V}_{0,1} \Big     \rVert_F = \frac{1}{\sqrt{d}} \Big \lVert     \bZin     \bs{V}_{0,1}   \bs{V}_{0,1}^\top \bZin^\top   \Big     \rVert_F \leq \frac{C}{V L^2 d}
}
\item    The fifth summand:   We have by \ref{event:discrete5}-\ref{event:discrete7}
\eq{
 \Big \lVert     \bs{V}_{0,1}^\top \bZin^\top  \bZin     \bs{V}_{0,2} \Big     \rVert_F^2 &  \leq  \tr \Big(     \bs{V}_{0,1}^\top \bZin^\top  \bZin      \bs{V}_{0,2}      \bs{V}_{0,2}^\top  \bZin^\top  \bZin   \bs{V}_{0,1}  \Big) \\ 
 & \leq     \frac{1}{NLd }  \tr(   \bZin  \bs{V}_{0,1}   \bs{V}_{0,1}^\top     \bZin^\top    )    \leq      \frac{V}{NLd }   \frac{1}{V^2 L^2} =   \frac{1}{N V L^3d }. 
}
Therefore,
\eq{
  \frac{1}{ \sqrt{d}} \Big \lVert  \bs{V}_{0,1}^\top \bZin^\top  \bZin      \bs{V}_{0,2}  \Big     \rVert_F  \leq  \frac{1}{\sqrt{N V} L^{3/2}	d } 
}
\item    The sixth summand:   
\eq{
(L - 1)^2 \Big \lVert      \bs{V}_{0,1}^\top   \bZin^\top  \bZin    \bs{V}_{0,3}  \Big     \rVert_F^2 \leq \frac{1}{V^2 N}    \mathbbm{1}_V   \bZin^\top  \bZin  \bs{V}_{0,1}    \bs{V}_{0,1}^\top   \bZin^\top  \bZin \mathbbm{1}_V \leq \frac{1}{V^2 N L^2 d}
}
Therefore,
\eq{
\frac{L - 1}{\sqrt{d}} \Big \lVert      \bs{V}_{0,1}^\top   \bZin^\top  \bZin    \bs{V}_{0,3}  \Big     \rVert_F \leq \frac{C}{V L \sqrt{N} d}.
}
\end{itemize}
Therefore, we have
\eq{
\varphi  =    \mu_{kl}  \Big(  \frac{1 \pm o_V(1)}{VL^2} \pm  \frac{\tilde O(1)}{\sqrt{N V} L^{3/2}	d } \Big).
}
 
\subsection{Concentration bound for $\bf{s}_2$  }
In this section, we will use 
$\bbt \coloneqq  \phi^{\prime \prime}(0) \phi(0)$.
We have
\eq{
& \be_l^\top \bs{s}_2     =  
 \frac{1 }{N^2 L^2}   \sum_{i, j = 1}^N \beta_{ij}  \mathsf{z}_{k}^\top    \bZin      \bX_i^\top     \bX_i   \bZin^\top      \bZin    \bX_i^\top \mathbbm{1}_L     ( \bx_i -  \tfrac{1}{V} \mathbbm{1}_{V} )^\top             \Zout^\top    \Zout ( \bx_j -  \tfrac{1}{V} \mathbbm{1}_{V} )      \\     
& -  \frac{1 }{N^2 L^3}   \sum_{i, j = 1}^N \beta_{ij}  \mathsf{z}_{k}^\top    \bZin       \bX_i^\top   \mathbbm{1}_L \mathbbm{1}_L^\top      \bX_i  \bZin^\top      \bZin    \bX_i^\top \mathbbm{1}_L     ( \bx_i -  \tfrac{1}{V} \mathbbm{1}_{V} )^\top             \Zout^\top    \Zout ( \bx_j -  \tfrac{1}{V} \mathbbm{1}_{V} )      \\     
  & +   \frac{\mu_{kl}}{N^2 L^2}   \sum_{i, j = 1}^N  \beta_{ij}  \big(  \bs{e}_1 - \tfrac{1}{L}  \mathbbm{1}_L   \big)^\top    \bX_i  \bZin^\top       \bZin    \bX_i^\top \mathbbm{1}_L              ( \bx_i -  \tfrac{1}{V} \mathbbm{1}_{V} )^\top             \Zout^\top    \Zout ( \bx_j -  \tfrac{1}{V} \mathbbm{1}_{V} )  \\
&  \eqqcolon  \kappa + \text{negligible terma}. 
}

\subsubsection{Concentration for $\kappa$} 

We will write  $ \kappa$ as follows:
\eq{
\kappa 
&   =  
 \frac{1 }{N^2 L^2}   \sum_{i, j = 1}^N ( \beta_{ij} - \bbt  ) \mathsf{z}_{k}^\top    \bZin      \bX_i^\top     \bX_i   \bZin^\top      \bZin    \bX_i^\top \mathbbm{1}_L   ( \bx_i -  \tfrac{1}{V} \mathbbm{1}_{V} )^\top             \Zout^\top    \Zout ( \bx_j -  \tfrac{1}{V} \mathbbm{1}_{V} )    \\
& +  \frac{\bbt }{N^2 L^2}   \sum_{i, j = 1}^N  \mathsf{z}_{k}^\top    \bZin    \Big(     \bx_i      \bx_i^\top  -\tfrac{1}{V} \bs{I}_V \Big)   \bZin^\top      \bZin   \big( \bx_i +\tfrac{L-1}{V} \mathbbm{1}_V \big)     ( \bx_i -  \tfrac{1}{V} \mathbbm{1}_{V} )^\top             \Zout^\top    \Zout ( \bx_j -  \tfrac{1}{V} \mathbbm{1}_{V} )    \\
& +  \frac{\bbt }{N^2 L^2}   \sum_{i, j = 1}^N  \mathsf{z}_{k}^\top    \bZin    \Big(     \bN_i^\top      \bN_i   - \tfrac{L -1}{V} \bs{I}_V \Big)   \bZin^\top      \bZin   \big( \bx_i +\tfrac{L-1}{V} \mathbbm{1}_V \big)  ( \bx_i -  \tfrac{1}{V} \mathbbm{1}_{V} )^\top             \Zout^\top    \Zout ( \bx_j -  \tfrac{1}{V} \mathbbm{1}_{V} )    \\
& +  \frac{\bbt }{N^2 L^2}   \sum_{i, j = 1}^N  \mathsf{z}_{k}^\top    \bZin    \Big(    \bx_i      \bx_i^\top    -\tfrac{1}{V} \bs{I}_V \Big)   \bZin^\top      \bZin  \big(  \bN_{i}^\top -\tfrac{1}{V} \mathbbm{1}_V   \mathbbm{1}_{L-1}^\top \big) \mathbbm{1}_{L-1}  ( \bx_i -  \tfrac{1}{V} \mathbbm{1}_{V} )^\top             \Zout^\top    \Zout ( \bx_j -  \tfrac{1}{V} \mathbbm{1}_{V} )    \\
& +  \frac{\bbt }{N^2 L^2}   \sum_{i, j = 1}^N  \mathsf{z}_{k}^\top    \bZin    \Big(      \bN_i      \bN_i^\top  - \tfrac{L -1}{V} \bs{I}_V  \Big)   \bZin^\top      \bZin    \big(  \bN_{i}^\top -\tfrac{1}{V} \mathbbm{1}_V   \mathbbm{1}_{L-1}^\top \big)    ( \bx_i -  \tfrac{1}{V} \mathbbm{1}_{V} )^\top             \Zout^\top    \Zout ( \bx_j -  \tfrac{1}{V} \mathbbm{1}_{V} )    \\
& +  \frac{\bbt }{N^2 L V}   \sum_{i, j = 1}^N   \mathsf{z}_{k}^\top    \bZin      \bZin^\top      \bZin    \bX_i^\top \mathbbm{1}_L ( \bx_i -  \tfrac{1}{V} \mathbbm{1}_{V} )^\top             \Zout^\top    \Zout ( \bx_j -  \tfrac{1}{V} \mathbbm{1}_{V} )    \\
& \eqqcolon    \kappa_1 + \kappa_2  + \kappa_3  + \kappa_4  + \kappa_5 + \kappa_6.  
}
By Proposition \ref{prop:gaussquadraticformtrace}, we have  
\eq{
 \kappa_1 &  =   \frac{1 }{N^2 L^2}   \sum_{i, j = 1}^N ( \beta_{ij} - \bbt  ) ( \indic{ \bx_i = \bx_j }-  \tfrac{1}{V}  )    \mathsf{z}_{k}^\top    \bZin     \bX_i^\top     \bX_i   \bZin^\top      \bZin    \bX_i^\top \mathbbm{1}_L            \\
& \pm  \frac{\log^2 V}{N^2 L^2 \sqrt{d}}  \Big \lVert \sum_{i, j = 1}^N ( \beta_{ij} - \bbt  ) ( \bx_j -  \tfrac{1}{V} \mathbbm{1}_{V} )     \mathsf{z}_{k}^\top    \bZin       \bX_i^\top     \bX_i    \bZin^\top      \bZin    \bX_i^\top \mathbbm{1}_L     ( \bx_i -  \tfrac{1}{V} \mathbbm{1}_{V} )^\top       \Big \rVert_F     \\
& \eqqcolon   \kappa_{11} +      \kappa_{12}.
}
We have
\eq{
 & \lvert   \kappa_{11} \rvert   \leq   \frac{1}{ N^2 L^2}  \Bigg( \sum_{i = 1}^N \Big(  \sum_{j = 1}^N ( \beta_{ij} \! - \! \bbt )(  \indic{ \bx_i = \bx_j }  \! - \!  \tfrac{1}{V}  ) \Big)^2  \Bigg)^{\frac{1}{2}} \!\! \Big( \sum_{i  = 1}^N (  \mathsf{z}_{k}^\top    \bZin     \bX_i^\top     \bX_i   \bZin^\top      \bZin    \bX_i^\top \mathbbm{1}_L )^2    \Big)^{\frac{1}{2}}  \\
&  =     \frac{1}{ N^2 L V}  \Bigg( \sum_{i = 1}^N \Big(  \sum_{j = 1}^N ( \beta_{ij} \! - \! \bbt )(  \indic{ \bx_i = \bx_j }  \! - \!  \tfrac{1}{V}  ) \Big)^2  \Bigg)^{\frac{1}{2}} \!\! \Big( \sum_{i  = 1}^N (  \mathsf{z}_{k}^\top    \bZin     \bX_i^\top     \bX_i   \bZin^\top      \bZin      \mathbbm{1}_V )^2    \Big)^{\frac{1}{2}}  \\
&  +     \frac{1}{ N^2 L^2}  \Bigg( \sum_{i = 1}^N \Big(  \sum_{j = 1}^N ( \beta_{ij} \! - \! \bbt )(  \indic{ \bx_i = \bx_j }  \! - \!  \tfrac{1}{V}  ) \Big)^2  \Bigg)^{\frac{1}{2}} \!\!  \Big( \sum_{i  = 1}^N (  \mathsf{z}_{k}^\top    \bZin     ( \bX_i^\top     \bX_i  \! - \!  \tfrac{L}{V} \bs{I}_V)  \bZin^\top      \bZin    (\bX_i^\top \! - \!  \tfrac{1}{V} \mathbbm{1}_V \mathbbm{1}_L^\top)  \mathbbm{1}_L  )^2    \Big)^{\frac{1}{2}}  \\
&  +     \frac{1}{ N^2 L V}  \Bigg( \sum_{i = 1}^N \Big(  \sum_{j = 1}^N ( \beta_{ij} \! - \! \bbt )(  \indic{ \bx_i = \bx_j }  \! - \!  \tfrac{1}{V}  ) \Big)^2  \Bigg)^{\frac{1}{2}}  \Big( \sum_{i  = 1}^N (  \mathsf{z}_{k}^\top    \bZin     \bZin^\top      \bZin     (\bX_i^\top  - \tfrac{1}{V} \mathbbm{1}_V \mathbbm{1}_L^\top)  \mathbbm{1}_L  )^2    \Big)^{\frac{1}{2}}  \\
& \lesssim   \frac{1}{N^{3/2} L V}  \Big( \frac{\sqrt{N}}{L} + \frac{N^{3/2}}{V} \frac{1}{V \wedge L^2 \wedge L \sqrt{d}} \Big) \Big( \frac{ L}{d} \frac{\sqrt{V}}{\sqrt{L \wedge d}}  + \frac{V\sqrt{L}}{d^{3/2}} \Big)   \\
& +   \frac{1}{N^{3/2} L^2}  \Big( \frac{\sqrt{N}}{L} + \frac{N^{3/2}}{V} \frac{1}{V \wedge L^2 \wedge L \sqrt{d}} \Big)  \frac{L}{\sqrt{d} \sqrt{L \wedge d}}   \\
&\lesssim \frac{1}{V L^2 \sqrt{d} (L \wedge d)^{\frac{1}{2}}} +  \frac{1}{N L^{3/2} d \sqrt{L \wedge d}} + \frac{1}{V L^{1/2} d \sqrt{L \wedge d}}  \frac{1}{V \wedge L^2 \wedge L \sqrt{d}},
}
where we use \ref{event:prop8result4}- \ref{event:prop8result5}, and \ref{event:boundZin3}-\ref{event:boundZin5}.
Moreover,   
\eq{
 \lvert \kappa_{12}  \rvert &  \lesssim       \frac{1 }{N^{3/2} L^2 \sqrt{d}}  \Big \lVert \sum_{i = 1}^N ( \bx_i -  \tfrac{1}{V} \mathbbm{1}_{V} )   ( \bx_i -  \tfrac{1}{V} \mathbbm{1}_{V} )^\top  \Big \rVert_2  \\
 & \hspace{8em} \times \Big(\frac{1}{N} \sum_{i = 1}^N \sum_{j = 1}^N   ( \beta_{ij} - \phi^\prime(0)^2 ) ^2   \lvert   \mathsf{z}_{k}^\top   \bZin        \bX_i^\top     \bX_i      \bZin^\top      \bZin    \bX_i^\top \mathbbm{1}_L       \rvert^2 \Big)^{\frac{1}{2}} \\
&  \lesssim      \frac{1}{N^{3/2} L^2 \sqrt{d}} \frac{N}{V} \Big( \frac{\sqrt{N}}{V \wedge L^2 \wedge L \sqrt{d}}  + \frac{1}{L} \Big) \Big( \frac{ \sqrt{L}}{\sqrt{d}}   +  \frac{L^{3/2}}{d^{3/2}}   \Big) \\
& \lesssim  \frac{1}{V \sqrt{L} d (L \wedge d)} \frac{1}{V \wedge L^2 \wedge L \sqrt{d}}  +  \frac{1}{\sqrt{N} V L^{3/2} d (L \wedge d)}.
}
Therefore,
\eq{
 \lvert \kappa_1   \rvert  
 & \lesssim \frac{1}{V L^2 \sqrt{d} (L \wedge d)^{\frac{1}{2}}} +  \frac{1}{N L^{3/2} d \sqrt{L \wedge d}} + \frac{1}{V L^{1/2} d \sqrt{L \wedge d}}  \frac{1}{V \wedge L^2 \wedge L \sqrt{d}}.
}
By Proposition \ref{prop:gaussquadraticformtrace},
\eq{
  \kappa_2 
& =   \frac{( 1 -  \tfrac{1}{V}  )  }{N^2 L^2}  \Big(    \sum_{j= 1}^N   \bx_j -  \tfrac{1}{V} \mathbbm{1}_{V}    \Big)^\top  \sum_{i = 1}^N   ( \bx_i -  \tfrac{1}{V} \mathbbm{1}_{V} )  \mathsf{z}_{k}^\top    \bZin    \Big(     \bx_i  \bx_i^\top   -\tfrac{1}{V} \bs{I}_V \Big)   \bZin^\top      \bZin    \big( \bx_i +\tfrac{L-1}{V} \mathbbm{1}_V \big)           \\
& \pm  \frac{1 }{N^2 L^2 \sqrt{d}}  \Big \lVert  \sum_{ j = 1}^N  ( \bx_j -  \tfrac{1}{V} \mathbbm{1}_{V} )   \Big \rVert_2    \Big \lVert   \sum_{i = 1}^N     ( \bx_i -  \tfrac{1}{V} \mathbbm{1}_{V} )  \mathsf{z}_{k}^\top   \bZin    \Big(     \bx_i  \bx_i^\top   -\tfrac{1}{V} \bs{I}_V \Big)   \bZin^\top      \bZin    \big( \bx_i +\tfrac{L-1}{V} \mathbbm{1}_V \big)       \Big \rVert_2 .
}
Let $n_i \coloneqq \lvert \{ j \leq N \vert \bx_j = \bs{e}_i \} \rvert$.  We have 
\eq{
& \frac{1 }{N^2 L^2} \Big \lVert   \sum_{j= 1}^N ( \bx_j -  \tfrac{1}{V} \mathbbm{1}_{V} )  \Big \rVert_2 \Big  \lVert   \sum_{i = 1}^N   ( \bx_i -  \tfrac{1}{V} \mathbbm{1}_{V} )    \mathsf{z}_{k}^\top    \bZin    \Big(     \bx_i  \bx_i^\top   -\tfrac{1}{V} \bs{I}_V \Big)   \bZin^\top      \bZin     \bx_i    \Big  \rVert_2 \\
& + \frac{L-1}{L}    \frac{1 }{N^2 L V} \Big \lVert   \sum_{j= 1}^N ( \bx_j -  \tfrac{1}{V} \mathbbm{1}_{V} )  \Big \rVert_2 \Big  \lVert   \sum_{i = 1}^N   ( \bx_i -  \tfrac{1}{V} \mathbbm{1}_{V} )    \mathsf{z}_{k}^\top    \bZin    \big(     \bx_i  \bx_i^\top   -\tfrac{1}{V} \bs{I}_V \big)   \bZin^\top      \bZin    \mathbbm{1}_V   \Big  \rVert_2 \\
&  \lesssim    \frac{1}{N   L^2}   \Big ( \frac{1}{N} \sum_{i = 1}^V   n_i^2  \big \lvert   \mathsf{z}_{k}^\top   \bZin    \big(     \be_i  \be_i^\top   -\tfrac{1}{V} \bs{I}_V \big)   \bZin^\top      \bZin     \be_i     \big \rvert^2 \Big )^{\frac{1}{2}} \\
& +    \frac{1 }{N^2 L V}  \Big  \lVert   \sum_{i = 1}^N   ( \bx_i -  \tfrac{1}{V} \mathbbm{1}_{V} )    \mathsf{z}_{k}^\top    \bZin    \big(     \bx_i  \bx_i^\top   -\tfrac{1}{V} \bs{I}_V \big)   \bZin^\top      \bZin    \mathbbm{1}_V   \Big  \rVert_2    \\
& \lesssim \frac{1}{VL^2\sqrt{d}} +    \frac{1 }{N^2 L V} \Big  \lVert   \sum_{i = 1}^N   ( \bx_i -  \tfrac{1}{V} \mathbbm{1}_{V} )    \mathsf{z}_{k}^\top    \bZin    \big(     \bx_i  \bx_i^\top   -\tfrac{1}{V} \bs{I}_V \big)   \bZin^\top      \bZin    \mathbbm{1}_V   \Big  \rVert_2,
}
where we used \ref{event:boundZin1} and \ref{event:discrete3}.
Moreover,  
\eq{
& \E \Big[  \Big  \lVert  \frac{1}{N} \sum_{i = 1}^N   ( \bx_i -   \tfrac{1}{V} \mathbbm{1}_{V} )    \mathsf{z}_{k}^\top    \bZin    \big(     \bx_i  \bx_i^\top   - \tfrac{1}{V} \bs{I}_V \big)   \bZin^\top      \bZin    \mathbbm{1}_V   \Big  \rVert_2^2  \vert \sZin \Big]\\
&  \leq
\frac{1}{N} \E \Big[   \mathsf{z}_{k}^\top  \!   \bZin    \big(     \bx_i \bx_i^\top   - \tfrac{1}{V} \bs{I}_V \big)   \bZin^\top   \!    \bZin    \mathbbm{1}_V    \mathbbm{1}_V^\top    \bZin^\top    \!   \bZin  \big(     \bx_i  \bx_i^\top   - \tfrac{1}{V} \bs{I}_V \big)    \bZin^\top \!     \mathsf{z}_{k}     \vert \sZin \Big] \\
& +\!   \frac{1}{N^2} \!\!\! \sum_{i \neq j = 1}^N \!\!  \E \!  \Big[    (  \indic{ \bx_i  = \bx_j} \! \!  - \!    \tfrac{1}{V}  )    \mathsf{z}_{k}^\top    \bZin    \big(     \bx_i \bx_i^\top  \!\!   -  \!  \tfrac{1}{V} \bs{I}_V \big)   \bZin^\top      \bZin    \mathbbm{1}_V   \!  \mathbbm{1}_V^\top    \bZin^\top      \bZin  \big(     \bx_j  \bx_j^\top \!\!    - \!  \tfrac{1}{V} \bs{I}_V \big)    \bZin^\top     \mathsf{z}_{k}     \vert \sZin \Big] \\
&  \leq
\big( \frac{1}{N} + \frac{N - 1}{N V}  \big) \E \!  \Big[     \mathsf{z}_{k}^\top    \bZin    \big(     \bx_i \bx_i^\top  \!\!   -  \!  \tfrac{1}{V} \bs{I}_V \big)   \bZin^\top      \bZin    \mathbbm{1}_V   \!  \mathbbm{1}_V^\top    \bZin^\top      \bZin  \big(     \bx_j  \bx_j^\top \!\!    - \!  \tfrac{1}{V} \bs{I}_V \big)    \bZin^\top     \mathsf{z}_{k}     \vert \sZin \Big] \\
&  \lesssim
\big( \frac{1}{N} + \frac{N - 1}{N V}  \big) \frac{ V }{d^2},
}
where we \ref{event:boundZin1}-\ref{event:boundZin3}. 
Therefore,   by Chebyshev's inequality, we have
\eq{
 \lvert   \kappa_{2}  \rvert   \lesssim  \frac{1}{VL^2\sqrt{d}} + \frac{1}{NLVd}
}
Moreover,  by using Chebyshev's inequality
\eq{
 \kappa_3 
&  =    
\frac{1 }{N^2 L^2}     \Big(    \sum_{j= 1}^N ( \bx_j -   \tfrac{1}{V} \mathbbm{1}_{V} )  \Big)^\top \Big(  \sum_{i = 1}^N   ( \bx_i -   \tfrac{1}{V} \mathbbm{1}_{V} ) \mathsf{z}_{k}^\top   \bZin       \Big(     \bN_i^\top      \bN_i   - \tfrac{L -1}{V} \bs{I}_V \Big)    \bZin^\top      \bZin     \big( \bx_i + \tfrac{L-1}{V} \mathbbm{1}_V \big)    \Big)      \\
& \pm  \frac{1 }{N^2 L^2 \sqrt{d}}  \Big  \lVert  \sum_{j= 1}^N ( \bx_j -   \tfrac{1}{V} \mathbbm{1}_{V} )  \Big \rVert_2   \Big  \lVert  \sum_{i = 1}^N   ( \bx_i -   \tfrac{1}{V} \mathbbm{1}_{V} )  \mathsf{z}_{k}^\top     \bZin       \Big(     \bN_i^\top      \bN_i   - \tfrac{L -1}{V} \bs{I}_V \Big)    \bZin^\top      \bZin     \big( \bx_i + \tfrac{L-1}{V} \mathbbm{1}_V \big)      \Big \rVert_2.
}
We have  
\eq{
 & \E \Bigg[  \Big(  \sum_{i,  j = 1}^N (  \indic{ \bx_i = \bx_j }-   \tfrac{1}{V}  )   \mathsf{z}_{k}^\top       \bZin       \Big(     \bN_i^\top      \bN_i   - \tfrac{L -1}{V} \bs{I}_V \Big)    \bZin^\top      \bZin    \big( \bx_i + \tfrac{L-1}{V} \mathbbm{1}_V \big)    \Big)^2  \Big \vert \sZin  \Bigg] \\
& =    \sum_{i = 1}^N \E \Bigg[  \Big(  \sum_{j = 1}^N ( \indic{ \bx_i = \bx_j }-   \tfrac{1}{V}  )  \mathsf{z}_{k}^\top    \bZin       \Big(     \bN_i^\top      \bN_i   - \tfrac{L -1}{V} \bs{I}_V \Big)    \bZin^\top      \bZin  \big( \bx_i + \tfrac{L-1}{V} \mathbbm{1}_V \big)   \Big)^2  \Big \vert \sZin    \Bigg] \\
& \lesssim \frac{N^2}{V}     \E \Bigg[  \Big(    \mathsf{z}_{k}^\top    \bZin       \Big(     \bN_1^\top      \bN_1   - \tfrac{L - 1}{V} \bs{I}_V \Big)    \bZin^\top      \bZin    \big( \bx_1 + \tfrac{L-1}{V} \mathbbm{1}_V \big)    \Big)^2  \Big \vert \sZin    \Bigg] ~~~ \label{eq:s2131secondmoment}
}
We have
\eq{
 & \E \Bigg[  \Big(    \mathsf{z}_{k}^\top      \bZin       \Big(     \bN_1^\top      \bN_1   - \tfrac{L - 1}{V} \bs{I}_V \Big)    \bZin^\top      \bZin    \big( \bx_1 + \tfrac{L-1}{V} \mathbbm{1}_V \big)     \Big)^2   \Big\vert \sZin    \Bigg]  
  \leq  \frac{C L}{d^2} \Big( 1 + \frac{ L^2}{V} \Big).
}
where we used \ref{event:boundZin1}-\ref{event:boundZin3} and \ref{event:boundZin8}. Therefore, we have  $\eqref{eq:s2131secondmoment}   \lesssim \tfrac{N^2 L}{Vd^2} (1 + \tfrac{L^2}{V})$ .  Also,
\eq{
&  \E \Big[    \Big \lVert   \sum_{i = 1}^N     ( \bx_i -   \tfrac{1}{V} \mathbbm{1}_{V} )    \mathsf{z}_{k}^\top      \bZin    \Big(     \bN_i^\top      \bN_i    - \tfrac{L - 1}{V} \bs{I}_V \Big)   \bZin^\top      \bZin        \big( \bx_i + \tfrac{L-1}{V} \mathbbm{1}_V \big)          \Big \rVert_2 ^2  ~ \big\vert \sZin    \Big] \leq  \frac{C L}{d^2} \Big( 1 + \frac{ L^2}{V} \Big).
}
Therefore, we have
\eq{
\lvert   \mathtt{\kappa}_{3} \rvert \lesssim  \Big(  \frac{1}{N \sqrt{V} L^{3/2} d} + \frac{1}{N L^{3/2} d^{3/2}}    + \frac{1}{N  \sqrt{V} \sqrt{L}  d^{3/2}}    \Big).
}
Moreover, by Chebyshev's inequality
\eq{
 \kappa_4  =  
& \frac{1 }{N^2 L^2}     \Big(    \sum_{j= 1}^N ( \bx_j -   \tfrac{1}{V} \mathbbm{1}_{V} )  \Big)^\top   \Big(  \sum_{i = 1}^N   ( \bx_i -   \tfrac{1}{V} \mathbbm{1}_{V} )    \mathsf{z}_{k}^\top    \bZin       \big(     \bx_i      \bx_i^\top   - \tfrac{1}{V} \bs{I}_V \big)    \bZin^\top      \bZin  \big(  \bN_{i} ^\top - \tfrac{1}{V} \mathbbm{1}_{V}   \mathbbm{1}_{L-1}^\top \big)   \mathbbm{1}_{L-1}      \Big)      \\
& \pm  \frac{1 }{N^2 L^2 \sqrt{d}}  \Big  \lVert  \sum_{j= 1}^N ( \bx_j \! - \!  \tfrac{1}{V} \mathbbm{1}_{V} )  \Big \rVert_2    \Big  \lVert  \sum_{i = 1}^N   ( \bx_i \! - \!    \tfrac{1}{V} \mathbbm{1}_{V} )    \mathsf{z}_{k}^\top    \bZin       \big(       \bx_i      \bx_i^\top \! \! - \! \tfrac{1}{V} \bs{I}_V \big)    \bZin^\top      \bZin  \big(  \bN_{i} ^\top \!\! - \!  \tfrac{1}{V} \mathbbm{1}_{V}   \mathbbm{1}_{L-1}^\top \big)   \mathbbm{1}_{L-1}       \Big \rVert_2.
}
We have
\eq{
&  \E \Bigg[  \Big(  \sum_{i,  j = 1}^N  ( \indic{ \bx_i = \bx_j } \! - \!    \tfrac{1}{V}  )   \mathsf{z}_{\nu, \delta}^\top    \bZin       \Big(     \bx_i^\top      \bx_i   - \tfrac{1}{V} \bs{I}_V \Big)    \bZin^\top      \bZin   \big(  \bN_{i} ^\top - \tfrac{1}{V} \mathbbm{1}_{V}   \mathbbm{1}_{L-1}^\top \big)   \mathbbm{1}_{L-1}      \Big)   \Big)^2 ~  \big\vert \sZin    \Bigg] \\ 
& =    \sum_{i = 1}^N \E \Bigg[  \Bigg(  \sum_{j = 1}^N ( \indic{ \bx_i = \bx_j } \! - \!  \tfrac{1}{V}  )     \mathsf{z}_{k}^\top   \bZin       \Big(     \bx_i   \bx_i^\top   \!   - \! \tfrac{1}{V} \bs{I}_V \Big)    \bZin^\top      \bZin     \big(  \bN_{i} ^\top \! - \! \tfrac{1}{V} \mathbbm{1}_{V}   \mathbbm{1}_{L-1}^\top \big)   \mathbbm{1}_{L-1}      \Big)    \Bigg)^2    \big\vert \sZin   \Bigg] \\
& \lesssim \frac{N^2}{V}      \E \Bigg[  \Big(    \mathsf{z}_{k}^\top   \bZin      \Big(     \bx_i  \bx_i^\top   - \tfrac{1}{V} \bs{I}_V \Big)    \bZin^\top      \bZin   \big(  \bN_{i} ^\top - \tfrac{1}{V} \mathbbm{1}_{V}   \mathbbm{1}_{L-1}^\top \big)   \mathbbm{1}_{L-1}      \Big)    \Big)^2  ~  \big\vert \sZin   \Bigg]. \label{eq:s2141secondmoment}
}
We  have
\eq{
 &    \E \Bigg[  \Big(    \mathsf{z}_{k}^\top    \bZin   \Big(     \bx_i  \bx_i^\top   - \tfrac{1}{V} \bs{I}_V \Big)    \bZin^\top      \bZin   \big(  \bN_{i} ^\top - \tfrac{1}{V} \mathbbm{1}_{V}   \mathbbm{1}_{L-1}^\top \big)   \mathbbm{1}_{L-1}      \Big)    \Big)^2    \big\vert \sZin    \Bigg]  \\
& \leq  \frac{C L }{d}   \mathsf{z}_{k}^\top  \Zin   \E  \Big[    \Big(     \bx_i       \bx_i^\top   - \tfrac{1}{V} \bs{I}_V \Big)    \bZin^\top      \bZin    \Big(     \bx_i \bx_i^\top   - \tfrac{1}{V} \bs{I}_V \Big)  ~  \big\vert \sZin   \Big] \bZin^\top    \mathsf{z}_{k}  \leq \frac{C L}{d^2}.
}
where we used Proposition \ref{prop:multimatrixexpectation}.
Therefore,    $\eqref{eq:s2141secondmoment} \lesssim \frac{N^2 L}{Vd^2}$.  Also,
\eq{
&   \E \Bigg[    \Big \lVert   \sum_{i = 1}^N     ( \bx_i -   \tfrac{1}{V} \mathbbm{1}_{V} )  \mathsf{z}_{k} ^\top    \bZin       \Big(     \bx_i  \bx_i^\top   - \tfrac{L - 1}{V} \bs{I}_V \Big)    \bZin^\top      \bZin   \big(  \bN_{i} ^\top - \tfrac{1}{V} \mathbbm{1}_{V}   \mathbbm{1}_{L-1}^\top \big)   \mathbbm{1}_{L-1}        \Big \rVert_2 ^2 ~  \big\vert \sZin   \Bigg] \\
&  = \frac{C L}{d} \sum_{i = 1}^N   \mathsf{z}_{k} ^\top    \bZin      \E \Big[    \Big(    \bx_i  \bx_i^\top   - \tfrac{1}{V} \bs{I}_V \Big)   \bZin^\top    \bZin    \Big(  \bx_i  \bx_i^\top   - \tfrac{1}{V} \bs{I}_V  \Big)  ~  \big\vert \sZin   \Big]     \bZin ^\top  \mathsf{z}_{k}  \leq \frac{CNL}{d^2},
}
where we used Proposition \ref{prop:multimatrixexpectation}.
Therefore,  
\eq{
\lvert   \kappa_{4} \rvert \lesssim\Big(  \frac{1}{N \sqrt{V} L^{3/2} d} + \frac{1}{N L^{3/2} d^{3/2}}   \Big).
}
Moreover,   let  
\eq{
\gamma_i  & \coloneqq     \mathsf{z}_{k}^\top     \bZin      \Big(     \bN_i^\top        \bN_i   - \tfrac{L - 1}{V} \bs{I}_V \Big)    \bZin^\top      \bZin  \big(  \bN_{i} ^\top - \tfrac{1}{V} \mathbbm{1}_{V}   \mathbbm{1}_{L-1}^\top \big)   \mathbbm{1}_{L-1}   \\
&  -   \E \Big[    \mathsf{z}_{k}^\top   \bZin         \Big(     \bN_i^\top      \bN_i   - \tfrac{L - 1}{V} \bs{I}_V \Big)    \bZin^\top      \bZin  \big(  \bN_{i} ^\top - \tfrac{1}{V} \mathbbm{1}_{V}   \mathbbm{1}_{L-1}^\top \big)   \mathbbm{1}_{L-1}  \vert \sZin   \Big] .
}
By Proposition \ref{prop:gaussquadraticformtrace}, we have
\eq{
& \kappa_5  =   \  
 \frac{1 }{N^2 L^2}     \Big(    \sum_{j= 1}^N ( \bx_j -   \tfrac{1}{V} \mathbbm{1}_{V} )  \Big)^\top  \Big(  \sum_{i = 1}^N   ( \bx_i -   \tfrac{1}{V} \mathbbm{1}_{V} ) \gamma_i \Big)  \frac{\log^2 V}{N^2 L^2 \sqrt{d}}  \Big  \lVert  \sum_{j= 1}^N ( \bx_j -   \tfrac{1}{V} \mathbbm{1}_{V} )  \Big \rVert_2   \Big  \lVert  \sum_{i = 1}^N   ( \bx_i -   \tfrac{1}{V} \mathbbm{1}_{V} ) \gamma_i   \Big \rVert_2 \\
&  + \frac{1 }{N^2 L^2}     \Big(    \sum_{j= 1}^N ( \bx_j -   \tfrac{1}{V} \mathbbm{1}_{V} )  \Big)^\top \\
& \times \Big(  \sum_{i = 1}^N   ( \bx_i -   \tfrac{1}{V} \mathbbm{1}_{V} ) \mathsf{z}_{k}^\top    \bZin       \E \Big[     \Big(     \bN_i^\top      \bN_i   - \tfrac{L - 1}{V} \bs{I}_V \Big)    \bZin^\top      \bZin  \big(  \bN_{i} ^\top - \tfrac{1}{V} \mathbbm{1}_{V}   \mathbbm{1}_{L-1}^\top \big)   \mathbbm{1}_{L-1}  \vert \sZin   \Big]   \Big)      \\
& \pm  \frac{\log^2 V}{N^2 L^2 \sqrt{d}}  \Big  \lVert  \sum_{j= 1}^N ( \bx_j -   \tfrac{1}{V} \mathbbm{1}_{V} )  \Big \rVert_2  \\
& \times  \Big  \lVert  \sum_{i = 1}^N   ( \bx_i -   \tfrac{1}{V} \mathbbm{1}_{V} ) \mathsf{z}_{k}^\top    \bZin      \E \Big[     \Big(     \bN_i^\top      \bN_i   - \tfrac{L - 1}{V} \bs{I}_V \Big)    \bZin^\top      \bZin  \big(  \bN_{i} ^\top - \tfrac{1}{V} \mathbbm{1}_{V}   \mathbbm{1}_{L-1}^\top \big)   \mathbbm{1}_{L-1}  \vert \sZin   \Big]      \Big \rVert_2. 
}
By Proposition \ref{prop:intermediateresults1}
\eq{
\E \Bigg[  \Big(  \sum_{i,  j = 1}^N ( \indic{ \bx_i = \bx_j }-   \tfrac{1}{V}  )     \gamma_i   \Big)^2  \Bigg]  
 & =    \sum_{i = 1}^N \E \Bigg[  \Big(  \sum_{j = 1}^N ( \indic{ \bx_i = \bx_j }-   \tfrac{1}{V}  )  \gamma_{i}  \Big)^2  \Bigg]  \\  
& \leq  2  (1 - \frac{1}{V})^2   \sum_{i = 1}^N \E  [   \gamma_i^2 ]  +   \frac{2  (1 - \frac{1}{V})}{V}    \sum_{i = 1}^N \sum_{j \neq i}^N   \E  [  \gamma_i^2   ] \\
& \lesssim  \frac{N^2}{V} \Big( \frac{L}{d} + \frac{L^2}{d^2} \Big). \label{eq:s2151secondmoment}
}
Then,
\eq{
\E \Big[    \Big  \lVert  \sum_{i = 1}^N   ( \bx_i -   \frac{1}{V} \mathbbm{1}_{V} ) \gamma_i   \Big \rVert_2^2 \Big] \leq \sum_{i = 1}^N  \E  [  \gamma_i^2   ] \lesssim N    \Big( \frac{L}{d} + \frac{L^2}{d^2} \Big) .
}
Moreover,  by  Proposition \ref{prop:multimatrixexpectation}, \ref{event:discrete3} and \ref{event:discrete7}, we have
\eq{
 \Big  \lVert  \sum_{i = 1}^N   ( \bx_i \! - \!  \tfrac{1}{V} \mathbbm{1}_{V} ) \mathsf{z}_{k}^\top    \bZin     \!  \E \! \Big[     \Big(     \bN_i^\top      \bN_i  \! - \! \tfrac{L - 1}{V} \bs{I}_V \Big)    \bZin^\top      \bZin  \big(  \bN_{i} ^\top \! - \! \tfrac{1}{V} \mathbbm{1}_{V}   \mathbbm{1}_{L-1}^\top \big)   \mathbbm{1}_{L-1} & \vert \sZin   \Big]      \Big \rVert_2  \lesssim    \frac{L \sqrt{N}}{\sqrt{Vd}}.
} 
Therefore, by Chebyshev's inequality, we have
\eq{
\lvert  \kappa_5 \rvert \lesssim \Big(  \frac{1}{N L  \sqrt{V d} (L \wedge d)^{1/2}} + \frac{1}{N L d (L \wedge d)^{1/2}} + \frac{1}{N L \sqrt{Vd} }  \Big).
}
Lastly, by Proposition \ref{prop:gaussquadraticformtrace},
\eq{
  & \kappa_6   = 
\frac{\bbt }{N^2 L V}   \sum_{i, j = 1}^N      ( \indic{\bx_i = \bx_j} -   \tfrac{1}{V}  )   \mathsf{z}_{k}^\top    \bZin      \bZin^\top      \bZin    \bX_i^\top \mathbbm{1}_L   \\
& \qquad \pm \frac{\bbt }{N^2 L V \sqrt{d}} \Big \lVert  \sum_{i, j = 1}^N         ( \bx_j -   \tfrac{1}{V} \mathbbm{1}_{V} )    \mathsf{z}_{k}^\top    \bZin      \bZin^\top      \bZin    \bX_i^\top \mathbbm{1}_L     ( \bx_i -   \tfrac{1}{V} \mathbbm{1}_{V} )^\top      \Big \rVert_F \\
&  = \frac{\bbt  }{N^2 L V}   \sum_{i, j = 1}^N      ( \indic{\bx_i = \bx_j} -   \tfrac{1}{V}  )   \mathsf{z}_{k}^\top    \bZin      \bZin^\top      \bZin    \bX_i^\top \mathbbm{1}_L   \\
&  \qquad \pm \frac{\bbt}{  V \sqrt{d}} \big \lVert \frac{1}{N} \sum_{j = 1}^N         ( \bx_j -   \tfrac{1}{V} \mathbbm{1}_{V} )        \big \rVert_2  \Big \lVert  \frac{1}{N L} \sum_{i = 1}^N   ( \bx_i -   \tfrac{1}{V} \mathbbm{1}_{V} )    \mathsf{z}_{k}^\top    \bZin      \bZin^\top      \bZin    \bX_i^\top \mathbbm{1}_L    \Big \rVert _2
}
We have
\eq{
 \Big \lVert  \frac{1}{N L}  \sum_{i = 1}^N      \mathsf{z}_{k}^\top    \bZin      \bZin^\top &       \bZin    \bX_i^\top \mathbbm{1}_L  ( \bx_i -   \tfrac{1}{V} \mathbbm{1}_{V} )^\top    \Big \rVert _2 \\
 & \leq   \Big \lVert     \mathsf{z}_{k}^\top    \bZin      \bZin^\top      \bZin   \frac{1}{N L}  \sum_{i = 1}^N    (  \bX_i^\top - \tfrac{1}{V} \mathbbm{1}_V \mathbbm{1}_L^\top) \mathbbm{1}_L  ( \bx_i -   \tfrac{1}{V} \mathbbm{1}_{V} )^\top    \Big \rVert _2 \\
 & +  \frac{1}{  V} \lvert   \mathsf{z}_{k}^\top    \bZin      \bZin^\top      \bZin    \mathbbm{1}_V  \rvert   \Big \lVert \frac{1}{N}  \sum_{i = 1}^N       ( \bx_i -   \tfrac{1}{V} \mathbbm{1}_{V} )^\top    \Big \rVert _2  \\
 & \lesssim    \frac{V }{d}   \Big \lVert      \bZin    \frac{1}{N L}   \sum_{i = 1}^N    (  \bX_i^\top - \frac{1}{V} \mathbbm{1}_V \mathbbm{1}_L^\top) \mathbbm{1}_L  ( \bx_i -   \frac{1}{V} \mathbbm{1}_{V} )^\top    \Big \rVert _2  + \frac{ \sqrt{V}}{ d^{3/2} \sqrt{N}}   \\
 &\lesssim   \frac{C V }{\sqrt{NL} d^{3/2}} . \label{eq:s216normbound}
} 
Moreover,
\eq{
\E \Big[   \Big( \frac{1}{N^2 L V}   \sum_{i, j = 1}^N  &     ( \indic{\bx_i = \bx_j} -   \tfrac{1}{V}  )    \mathsf{z}_{k}^\top    \bZin      \bZin^\top      \bZin    \bX_i^\top \mathbbm{1}_L \Big)^2 \vert \sZin   \Big]\\
&  =   \frac{1}{N^4 L^2 V^2}   \sum_{j = 1}^N \E \Big[   \Big( \sum_{i = 1}^N      ( \indic{\bx_i = \bx_j} -   \tfrac{1}{V}  )   \mathsf{z}_{k}^\top    \bZin      \bZin^\top      \bZin    \bX_i^\top \mathbbm{1}_L \Big)^2  \vert \sZin   \Big] \\
& \leq   \frac{2}{N^4 L^2 V^2}   \sum_{j = 1}^N \E \Big[   \Big(       \mathsf{z}_{k}^\top    \bZin      \bZin^\top      \bZin    \bX_i^\top \mathbbm{1}_L \Big)^2 \vert \sZin   \Big]  \\
& +\frac{2}{N^4 L^2 V^2}   \sum_{j = 1}^N \E \Big[   \Big(    \sum_{\substack{ i = 1 \\ i \neq j }}^N    ( \indic{\bx_i = \bx_j} -   \tfrac{1}{V}  )     \mathsf{z}_{k}^\top    \bZin      \bZin^\top      \bZin    \bX_i^\top \mathbbm{1}_L \Big)^2  \vert \sZin  \Big] \\
& \leq   \frac{2}{N^4 L^2 V^2}   \sum_{j = 1}^N \E \Big[   \Big(       \mathsf{z}_{\nu, \delta}^\top    \bZin      \bZin^\top      \bZin    \bX_i^\top \mathbbm{1}_L \Big)^2 \vert \sZin   \Big]  \\
& +\frac{2}{N^2 V^3}   \sum_{j = 1}^N \E \Big[   \Big \lVert \frac{1}{NL}   \sum_{\substack{ i = 1 \\ i \neq j }}^N  \mathsf{z}_{k}^\top    \bZin      \bZin^\top      \bZin    \bX_i^\top \mathbbm{1}_L  ( \bx_i   -   \tfrac{1}{V} \mathbbm{1}_V )^\top   \Big \rVert_2^2  \vert \sZin  \Big] \\
& \leq   \frac{2}{N^4 L^2 V^2}   \sum_{j = 1}^N \E \Big[   \Big(       \mathsf{z}_{k}^\top    \bZin      \bZin^\top      \bZin    \bX_i^\top \mathbbm{1}_L \Big)^2 \vert \sZin   \Big]  +\frac{C}{N^2 V L d^3},
}
where we used   \ref{eq:s216normbound} in the last step. We have
\eq{
 &  \frac{1}{N^4 L^2 V^2}   \sum_{j = 1}^N \E \Big[   \Big(       \mathsf{z}_{k}^\top    \bZin      \bZin^\top      \bZin    \bX_i^\top \mathbbm{1}_L \Big)^2 \vert \sZin   \Big]  \\
&  \leq    \frac{1}{N^3 L^2 V^2}       \mathsf{z}_{k}^\top    \bZin      \bZin^\top      \bZin    \Big( \tfrac{L^2}{V^2} \mathbbm{1}_V    \mathbbm{1}_V ^\top + \tfrac{L}{V} \bs{I}_V \Big)        \bZin^\top      \bZin      \bZin    \mathsf{z}_{k} 
\leq    \frac{C \log^2 V}{N^3  L  d^{3}}   
}
Therefore, by Chebyshev's inequality, we have
\eq{
\lvert  \kappa_6 \rvert \lesssim  \frac{ 1}{N  \sqrt{L} d^{2}}.
}
Overall, by using $N \ll VL$, 
\eq{
 \lvert  \kappa \rvert &  \leq  \Big(  \frac{1}{N  \sqrt{L} d  (L \wedge d)} +   \frac{1}{N L d (L \wedge d)^{1/2}} + \frac{1}{N L \sqrt{Vd} }     +   \frac{1}{ \sqrt{N}  V  L d} +   \frac{1}{ \sqrt{NV}   L^2 \sqrt{d}}    \Big)   +     \Big(  \frac{1}{V L^2\sqrt{ d} } \frac{1}{V^2  \sqrt{L} d^{3/2}} \Big).
}

\subsection{Concentration bound for $\bf{s}_3$  }
\label{sec:concentrationmlp}

We have 
\eq{
 \be_{l}^\top \bs{s}_3  &  =  
 \frac{1}{N^2 L}    \sum_{i, j = 1}^N    \mathsf{z}_{k}^\top    \bZin       \bX_i^\top   \bX_i  \bZin^\top \\
 &   \times   \Big( \frac{1}{m} \sum_{k = 1}^m \bw_k   \phi^\prime  \big(  \tfrac{1}{L} \bw_k^\top  \bZin    \bX_i^\top \mathbbm{1}_L    \big)     \phi \Big( \tfrac{1}{L}   \bw_k^\top  \bZin    \bX_j^\top \mathbbm{1}_L  \Big)  \\
 & \quad    - \E \Big[  \bw_k    \phi^\prime  \big(  \tfrac{1}{L} \bw_k^\top  \bZin    \bX_i^\top \mathbbm{1}_L    \big)     \phi \Big( \tfrac{1}{L}  \bw_k^\top  \bZin    \bX_j^\top \mathbbm{1}_L  \Big)   \Big] \Big)       ( \bx_j -   \tfrac{1}{V} \mathbbm{1}_{V} )^\top           \Zout^\top     \Zout   ( \bx_i -   \tfrac{1}{V} \mathbbm{1}_{V} )   \Big) 
\\[0.75em]
& - \frac{1}{N^2 L^2}    \sum_{i, j = 1}^N   \mathsf{z}_{k}^\top    \bZin       \bX_i^\top   \mathbbm{1}_L \mathbbm{1}_L^\top    \bX_i  \bZin^\top  \\
 &  \times \Big( \frac{1}{m} \sum_{k = 1}^m \bw_k   \phi^\prime  \big(  \tfrac{1}{L} \bw_k^\top  \bZin    \bX_i^\top \mathbbm{1}_L    \big)     \phi \Big( \tfrac{1}{L}   \bw_k^\top  \bZin    \bX_j^\top \mathbbm{1}_L  \Big)  \\
 & \quad      - \E \Big[  \bw_k    \phi^\prime  \big(  \tfrac{1}{L} \bw_k^\top  \bZin    \bX_i^\top \mathbbm{1}_L    \big)     \phi \Big( \tfrac{1}{L}  \bw_k^\top  \bZin    \bX_j^\top \mathbbm{1}_L  \Big)   \Big] \Big)    ( \bx_j -   \tfrac{1}{V} \mathbbm{1}_{V} )^\top           \Zout^\top     \Zout   ( \bx_i -   \tfrac{1}{V} \mathbbm{1}_{V} )   \Big)  \\[0.75em]
& +          \frac{ \mu_{kl} }{N^2 L}   \sum_{i, j = 1}^N   \big(  \be_1 - \frac{1}{L}  \mathbbm{1}_L   \big)^\top  \bX_i  \bZin^\top \\
 & \quad  \times \Big( \frac{1}{m} \sum_{k = 1}^m \bw_k   \phi^\prime  \big(  \tfrac{1}{L} \bw_k^\top  \bZin    \bX_i^\top \mathbbm{1}_L    \big)     \phi \Big( \tfrac{1}{L}   \bw_k^\top  \bZin    \bX_j^\top \mathbbm{1}_L  \Big)  \\
 & \quad      - \E \Big[  \bw_k    \phi^\prime  \big(  \tfrac{1}{L} \bw_k^\top  \bZin    \bX_i^\top \mathbbm{1}_L    \big)     \phi \Big( \tfrac{1}{L}  \bw_k^\top  \bZin    \bX_j^\top \mathbbm{1}_L  \Big)   \Big] \Big)    ( \bx_j -   \tfrac{1}{V} \mathbbm{1}_{V} )^\top           \Zout^\top     \Zout   ( \bx_i -   \tfrac{1}{V} \mathbbm{1}_{V} )  \Big) \\[0.75em]
& \eqqcolon  \nu + \text{negligible terms}. 
} 
\subsubsection{Concentration bound for $\nu$}  

We define
\eq{
\tilde \nu
  & \coloneqq
    \tr \Big( \frac{1}{N  L}    \sum_{ i = 1}^N  ( \bx_i -   \tfrac{1}{V} \mathbbm{1}_{V} )    \mathsf{z}_{k}^\top    \bZin       \bX_i^\top   \bX_i  \bZin^\top   \bw_k   \phi^\prime  \big(  \tfrac{1}{L} \bw_k^\top  \bZin    \bX_i^\top \mathbbm{1}_L    \big)  \\
 &  \hspace{13em}\times   \frac{1}{N} \sum_{j = 1}^N      \phi \Big( \tfrac{1}{L}   \bw_k^\top  \bZin    \bX_j^\top \mathbbm{1}_L  \Big)       ( \bx_j -   \tfrac{1}{V} \mathbbm{1}_{V} )^\top           \Zout^\top     \Zout   \Big) \\
& =  \tr \Big( \frac{1}{N  L}    \sum_{ i = 1}^N  ( \bx_i -   \tfrac{1}{V} \mathbbm{1}_{V} )    \mathsf{z}_{k}^\top    \bZin       \bX_i^\top   \bX_i  \bZin^\top   \bw_k   \phi^\prime  \big(  \tfrac{1}{L} \bw_k^\top  \bZin    \bX_i^\top \mathbbm{1}_L    \big)  \\
 &  \hspace{13em}\times   \frac{1}{N} \sum_{j = 1}^N      \phi \Big( \tfrac{1}{L}   \bw_k^\top  \bZin    \bX_j^\top \mathbbm{1}_L  \Big)       ( \bx_j -   \tfrac{1}{V} \mathbbm{1}_{V} )^\top              \Big) \\
& \pm  \frac{\log^2 V}{\sqrt{d}} \frac{1}{N  L}  \Big \lVert     \sum_{ i = 1}^N  ( \bx_i -   \tfrac{1}{V} \mathbbm{1}_{V} )   \mathsf{z}_{k}^\top    \bZin       \bX_i^\top   \bX_i  \bZin^\top   \bw_k   \phi^\prime  \big(  \tfrac{1}{L} \bw_k^\top  \bZin    \bX_i^\top \mathbbm{1}_L    \big) \Big \rVert_2 \\
& \hspace{13em} \times  \Big \lVert   \frac{1}{N} \sum_{j = 1}^N      \phi \Big( \tfrac{1}{L}   \bw_k^\top  \bZin    \bX_j^\top \mathbbm{1}_L  \Big)       ( \bx_j -   \tfrac{1}{V} \mathbbm{1}_{V} )   \Big \rVert_2  \\
& \eqqcolon \tilde \nu_1 + \tilde \nu_2,  
}
where we used Proposition \ref{prop:gaussquadraticformtrace} for the second step.
We define
\eq{
\phi(t) \eqqcolon \phi(0) + t \psi(t)  ~~ \text{and} ~~ \phi^\prime(t) \eqqcolon \phi(0) + t \psi_1(t)  ~~ \text{and} ~~ \psi(t) \eqqcolon \psi(0) + t \psi_2(t).
}
and write   
\eq{
 &   \tilde \nu_1 =    \phi  (0  )      \phi^\prime  \big( 0   \big)    ~    \tr \Big( \frac{1}{N  L}    \sum_{ i = 1}^N   \bx_i  \mathsf{z}_{k}^\top    \bZin       \bX_i^\top   \bX_i  \bZin^\top   \bw_k    \frac{1}{N} \sum_{j = 1}^N     ( \bx_j -   \tfrac{1}{V} \mathbbm{1}_{V} )^\top              \Big)  \\
 & +   \phi  (0  )      \tr \Big( \frac{1}{N  L^2}    \sum_{ i = 1}^N      \bx_i   \mathsf{z}_{k}^\top    \bZin     \bx_i \bx_i ^\top     \bZin^\top   \bw_k \bw_k^\top   \bZin    \bX_i^\top \mathbbm{1}_L     \psi_1  \big(  \tfrac{1}{L} \bw_k^\top  \bZin    \bX_i^\top \mathbbm{1}_L    \big)    \frac{1}{N} \sum_{j = 1}^N     ( \bx_j -   \tfrac{1}{V} \mathbbm{1}_{V} )^\top              \Big)  \\
  & +   \phi  (0  )     \tr \Big( \frac{1}{N  L^2} \!   \sum_{ i = 1}^N      \bx_i  \mathsf{z}_{k}^\top    \bZin      \bN_i^\top   \bN_i    \bZin^\top   \bw_k \bw_k^\top  \bZin    \bX_i^\top \mathbbm{1}_L     \psi_1  \big(  \tfrac{1}{L} \bw_k^\top  \bZin    \bX_i^\top \mathbbm{1}_L    \big)    \frac{1}{N} \sum_{j = 1}^N     ( \bx_j -   \tfrac{1}{V} \mathbbm{1}_{V} )^\top              \Big)  \\
 & +   \tr \Big( \frac{1}{N  L}    \sum_{ i = 1}^N  ( \bx_i  -   \tfrac{1}{V} \mathbbm{1}_{V} )   \mathsf{z}_{k}^\top    \bZin       \bX_i^\top   \bX_i  \bZin^\top   \bw_k   \phi^\prime  \big(  \tfrac{1}{L} \bw_k^\top  \bZin    \bX_i^\top \mathbbm{1}_L    \big)  \\
 &  \hspace{10em}\times   \frac{1}{N} \sum_{j = 1}^N          \psi \Big( \tfrac{1}{L}   \bw_k^\top  \bZin    \bX_j^\top \mathbbm{1}_L  \Big)   \frac{1}{L}   \bw_k^\top  \bZin    \bX_j^\top \mathbbm{1}_L       ( \bx_j -   \tfrac{1}{V} \mathbbm{1}_{V} )^\top              \Big) \\
& \eqqcolon  \tilde \nu_{11} +   \tilde \nu_{12}  +   \tilde \nu_{13}  +  \tilde \nu_{14}.
}
In the following, we bound each term separately. Let $n_w \coloneqq \lvert \{ i \in [N] : \bx_i = \be_w \} \rvert$.
\begin{itemize}[leftmargin=*]
\item  We have
\eq{
  \tilde \nu_{11}  & = \frac{1}{L} \sum_{w = 1}^V ( \frac{n_w }{N}- \frac{1}{V})   \frac{n_w}{N}  \mathsf{z}_{k}^\top    \bZin   \Big( \be_w \be_w^\top  + \tfrac{L - 1}{V} \bs{I}_V \Big) \bZin^\top   \bw_k         \\
 & +  \mathsf{z}_{k}^\top    \bZin    \frac{1}{NL}    \sum_{w = 1}^V   ( \frac{n_w}{N} - \frac{1}{V})        \sum_{i \in \{ i_1, \cdots, i_{n_w} \}} \Big( \bN_i^\top   \bN_i  - \tfrac{L - 1}{V} \bs{I}_V \Big)  \bZin^\top   \bw_k   
}
We have by using Lemma \ref{lem:multinomialmgf} and Proposition \ref{prop:multimatrixexpectation},
\eq{
& \E \Big[  \Big(  \mathsf{z}_{k}^\top    \bZin    \frac{1}{NL}    \sum_{w = 1}^V   ( \frac{n_w}{N} - \frac{1}{V})        \sum_{i \in \{ i_1, \cdots, i_{n_w} \}} \Big( \bN_i^\top   \bN_i  - \tfrac{L - 1}{V} \bs{I}_V \Big)  \bZin^\top   \bw_k \Big)^2 \vert \sZin  \Big]  \\
& =  \E \Big[   \Big \lVert   \mathsf{z}_{k}^\top    \bZin    \frac{1}{NL}    \sum_{w = 1}^V   ( \frac{n_w}{N} - \frac{1}{V})        \sum_{i \in \{ i_1, \cdots, i_{n_w} \}} \Big( \bN_i^\top   \bN_i  - \tfrac{L - 1}{V} \bs{I}_V \Big)  \bZin^\top  \Big \rVert_2^2   \vert \sZin  \Big]  \\
& =    \mathsf{z}_{k}^\top    \bZin    \frac{1}{N^2L^2}    \sum_{w = 1}^V  \E \Big[   ( \frac{n_w}{N} - \frac{1}{V})^2      n_w  \Big]  \E \Big[  \Big( \bN_1^\top   \bN_1  - \tfrac{L - 1}{V} \bs{I}_V \Big)  \bZin^\top     \bZin  \Big( \bN_1^\top   \bN_1  - \tfrac{L - 1}{V} \bs{I}_V \Big)    \vert \sZin  \Big]       \bZin^\top    \mathsf{z}_{k}    \\
&  \leq \frac{C}{N^2L^2 V}     \mathsf{z}_{k}^\top    \bZin       \E \Big[  \Big( \bN_1^\top   \bN_1  - \tfrac{L - 1}{V} \bs{I}_V \Big)  \bZin^\top     \bZin  \Big( \bN_1^\top   \bN_1  - \tfrac{L - 1}{V} \bs{I}_V \Big)    \vert \sZin  \Big]       \bZin^\top      \mathsf{z}_{k}    \\
& =  \frac{C}{N^2 V d (L \wedge d)}. 
}
Moreover, by using  \ref{event:boundZin1} and Proposition \ref{prop:multimatrixexpectation},
\eq{
 & \E \Big[ \Big(  \frac{1}{L} \sum_{w = 1}^V ( \frac{n_w }{N}- \frac{1}{V})   \frac{n_w}{N}  \mathsf{z}_{k}^\top    \bZin   \Big( \be_w \be_w^\top  + \tfrac{L - 1}{V} \bs{I}_V \Big) \bZin^\top   \bw_k  \Big)^2  \vert \sZin  \Big] \\
& = \frac{1}{L^2}  \E \Big[ \Big \lVert  \mathsf{z}_{\nu,\delta}^\top    \bZin    \Big(   \sum_{w = 1}^V ( \frac{n_w }{N}- \frac{1}{V})   \frac{n_w}{N}  \be_w \be_w^\top   +   ( \frac{n_w }{N}- \frac{1}{V}) ^2  \tfrac{L - 1}{V} \bs{I}_V \Big)   \bZin^\top   \Big \rVert_2^2  \vert \sZin  \Big] \\
& \leq  \frac{V^2}{L^2 d^2}  \E \Big[ \Big \lVert     \sum_{w = 1}^V ( \frac{n_w }{N}- \frac{1}{V})   \frac{n_w}{N}  \be_w \be_w^\top   +   ( \frac{n_w }{N}- \frac{1}{V}) ^2  \tfrac{L - 1}{V} \bs{I}_V   \Big \rVert_2^2  \vert \sZin  \Big] \\
& \leq   \frac{C V^2}{L^2 d^2}  \E \Big[   \sup_{w \in [N]}  \Big \lvert   ( \frac{n_w }{N}- \frac{1}{V})   \frac{n_w}{N}  \Big \rvert^2  \Big] +     \frac{C }{d^2}   \E \Big[   \Big(  \sum_{w = 1}^V   ( \frac{n_w }{N}- \frac{1}{V}) ^2      \Big )^2   \Big]  \leq    \frac{C }{d^2 N^2} .
}
Therefore,   
\eq{
 \E \Big[  \tilde \nu_{11}^2   \vert \sZin  \Big] \lesssim \frac{1}{d^2N^2}. 
}
\item  Moreover,
\eq{
 \tilde \nu_{12}^2 & \leq   \frac{C}{N} \Big \lVert   \frac{1}{N  L^2}    \sum_{ i = 1}^N      \bx_i   \mathsf{z}_{k}^\top    \bZin     \bx_i \bx_i ^\top     \bZin^\top   \bw_k \bw_k^\top  \bZin    \bX_i^\top \mathbbm{1}_L     \psi_1  \big(  \tfrac{1}{L} \bw_k^\top  \bZin    \bX_i^\top \mathbbm{1}_L    \big) \Big \rVert_2^2 .
}
We have for any $i \in [N]$,  
\eq{
\Big \lvert  \mathsf{z}_{k}^\top    \bZin     \bx_i \bx_i ^\top     \bZin^\top   \bw_k \bw_k^\top  \bZin    \bX_i^\top \mathbbm{1}_L     \psi_1  \big(  \tfrac{1}{L} \bw_k^\top  \bZin    \bX_i^\top \mathbbm{1}_L    \big) \Big \rvert \lesssim \sqrt{L} \Big(\indic{\bx_i = \be_{k}} + \frac{1}{\sqrt{d}} \Big)
}
Then,
\eq{
&\Big \lVert   \frac{1}{N  L^2}    \sum_{ i = 1}^N      \bx_i  \mathsf{z}_{k}^\top    \bZin     \bx_i \bx_i ^\top     \bZin^\top   \bw_k \bw_k^\top  \bZin    \bX_i^\top \mathbbm{1}_L     \psi_1  \big(  \tfrac{1}{L} \bw_k^\top  \bZin    \bX_i^\top \mathbbm{1}_L    \big) \Big  \rVert_2^2  \\
& \lesssim \frac{1}{   L^3}  \Big \lVert \frac{1}{N}      \sum_{ i = 1}^N      \bx_i  \Big(\indic{\bx_i = \be_{k}} + \frac{1}{\sqrt{d}} \Big) \Big \rVert^2  \lesssim  \frac{1}{ V d   L^3} 
}
Then,  
\eq{
\E [  \tilde \nu_{12}^2  \vert \sZin ]  \lesssim   \frac{1}{ N V d   L^3}.  
}
\item  Moreover,
\eq{
  \tilde \nu_{13}^2 
 & \leq \frac{C}{N} \Big \lVert   \frac{1}{N  L^2}    \sum_{ i = 1}^N      \bx_i  \mathsf{z}_{k}^\top    \bZin     \bN_i^\top   \bN_i    \bZin^\top   \bw_k \bw_k^\top  \bZin    \bX_i^\top \mathbbm{1}_L     \psi_1  \big(  \tfrac{1}{L} \bw_k^\top  \bZin    \bX_i^\top \mathbbm{1}_L    \big) \Big \rVert_2^2 .
}
We have for any $i \in [N]$,   
\eq{
\Big \lvert  & \mathsf{z}_{k}^\top    \bZin      \bN_i^\top   \bN_i       \bZin^\top   \bw_k \bw_k^\top  \bZin    \bX_i^\top \mathbbm{1}_L     \psi_1  \big(  \tfrac{1}{L} \bw_k^\top  \bZin    \bX_i^\top \mathbbm{1}_L    \big) \Big \rvert \\
& \lesssim \sqrt{L}  \lVert   \bZin      \bN_i^\top   \bN_i       \bZin^\top   \mathsf{z}_{k}  \rVert_2 
  \lesssim \sqrt{L}  \big(  \be_{k}^\top  \bN_i^\top  \mathbbm{1}_{L - 1} + 1 + \frac{L}{d}   \big)
}
Then, 
\eq{
 &  \Big \lVert   \frac{1}{N  L^2}    \sum_{ i = 1}^N      \bx_i   \mathsf{z}_{k}^\top    \bZin     \bN_i^\top   \bN_i    \bZin^\top   \bw_k \bw_k^\top  \bZin  \bX_i^\top \mathbbm{1}_L     \psi_1  \big(  \tfrac{1}{L} \bw_k^\top  \bZin    \bX_i^\top \mathbbm{1}_L    \big) \Big \rVert_2^2 \\
 & \lesssim    \frac{1}{N^2  L^3}    \Big \lVert  \sum_{ i = 1}^N      \bx_i    \Big(    \mathbbm{1}_{L - 1} ^\top \bN_i \be_{k}  + 1 + \frac{L}{d}  \Big)     \Big \rVert_2^2  \lesssim    \frac{1}{V  L d (L \wedge d)} +   \frac{1}{N^2  L^3}      \Big \lVert  \sum_{ i = 1}^N    \bx_i        \mathbbm{1}_{L - 1} ^\top \bN_i \be_{k}    \Big \rVert_2^2
}
We have
\eq{
 \frac{1}{N^2  L^3}   \E \Big[    \Big \lVert  \sum_{ i = 1}^N    \bx_i        \mathbbm{1}_{L - 1} ^\top \bN_i \be_{k}    \Big \rVert_2^2   \Big] & \lesssim  \frac{1}{V^3  L}  +  \frac{1}{N  V  L^2}.
}
Then,  
\eq{
\E [  \tilde \nu_{13}^2  \vert \sZin ]  \lesssim   \frac{1}{N V  L d (L \wedge d)}.   
}
\item Lastly, we have
\eq{
\lvert  \tilde \nu_{14} \rvert & \leq \Big \lVert    \frac{1}{N  L}    \sum_{ i = 1}^N     ( \bx_i -   \tfrac{1}{V} \mathbbm{1}_{V} )  \mathsf{z}_{k}^\top    \bZin       \bX_i^\top   \bX_i  \bZin^\top   \bw_k   \phi^\prime  \big(  \tfrac{1}{L} \bw_k^\top  \bZin    \bX_i^\top \mathbbm{1}_L    \big)  \Big \rVert_2 \\
& \qquad \times 
\Big \lVert     \frac{1}{N} \sum_{j = 1}^N          \psi \Big( \tfrac{1}{L}   \bw_k^\top  \bZin    \bX_j^\top \mathbbm{1}_L  \Big)   \frac{1}{L}   \bw_k^\top  \bZin    \bX_j^\top \mathbbm{1}_L       ( \bx_j -   \tfrac{1}{V} \mathbbm{1}_{V} )^\top         \Big \rVert_2. \label{eq:scorek14}
}
By using the derivations in the two previous items, we have
\eq{
 &   \Big \lVert    \frac{1}{N  L}    \sum_{ i = 1}^N  ( \bx_i  -   \tfrac{1}{V} \mathbbm{1}_{V} )^\top   \mathsf{z}_{k}^\top    \bZin     \bX_i^\top   \bX_i  \bZin^\top   \bw_k   \phi^\prime  \big(  \tfrac{1}{L} \bw_k^\top  \bZin    \bX_i^\top \mathbbm{1}_L    \big)  \Big \rVert_2 \\
&  \leq  \Big \lVert    \frac{1}{N  L}    \sum_{ i = 1}^N   \bx_i   \mathsf{z}_{k}^\top    \bZin       \bx_i  \bx_i^\top   \bZin^\top   \bw_k   \phi^\prime  \big(  \tfrac{1}{L} \bw_k^\top  \bZin    \bX_i^\top \mathbbm{1}_L    \big)  \Big \rVert_2 \\
&  +  \Big \lVert    \frac{1}{N  L^2}    \sum_{ i = 1}^N   \bx_i   \mathsf{z}_{k}^\top    \bZin       \bN_i^\top   \bN_i  \bZin^\top   \bw_k  \bw_k^\top  \bZin    \bX_i^\top \mathbbm{1}_L     \psi_1  \big(  \tfrac{1}{L} \bw_k^\top  \bZin    \bX_i^\top \mathbbm{1}_L    \big)  \Big \rVert_2 \\
&  + \lvert \phi^\prime(0)  \rvert \Big \lVert    \frac{1}{N  L}    \sum_{ i = 1}^N  ( \bx_i - \tfrac{1}{V} \mathbbm{1}_V)   \mathsf{z}_{k}^\top    \bZin       \bN_i^\top   \bN_i  \bZin^\top   \bw_k     \Big \rVert_2 \\
& \lesssim \frac{1}{ \sqrt{ V  L d (L \wedge d) } }  +   \phi^\prime(0)  \Big \lVert    \frac{1}{N  L}    \sum_{ i = 1}^N  ( \bx_i - \tfrac{1}{V} \mathbbm{1}_V)   \mathsf{z}_{k}^\top    \bZin       \bN_i^\top   \bN_i  \bZin^\top   \bw_k     \Big \rVert_2
}
We have
\eq{
& \E \Big[  \Big \lVert    \frac{1}{N  L}    \sum_{ i = 1}^N  ( \bx_i - \tfrac{1}{V} \mathbbm{1}_V)   \mathsf{z}_{k}^\top    \bZin       \bN_i^\top   \bN_i  \bZin^\top   \bw_k     \Big \rVert_2^2 \Big \vert \sZin  \Big]   \\
& \leq  \frac{1}{N^2 L^2}  \E \Big[  \sum_{ i, j = 1}^N  (\indic{\bx_i = \bx_j} - \tfrac{1}{V})  \mathsf{z}_{k}^\top    \bZin       \bN_i^\top   \bN_i  \bZin^\top   \bZin       \bN_i^\top   \bN_i  \bZin^\top   \mathsf{z}_{k} \Big \vert \sZin  \Big] \\
& \leq   \frac{1}{N  L^2}   \frac{L}{V}    \mathsf{z}_{k}^\top    \bZin      \mathrm{diag}(  \bZin^\top \bZin )  \bZin^\top   \mathsf{z}_{k}  +    \frac{1}{N  L^2}   \frac{L^2}{V^2}    \mathsf{z}_{k}^\top    \bZin       \bZin^\top \bZin  \bZin^\top   \mathsf{z}_{k}   =   \frac{1}{N d (L \wedge d)}.  ~~~ \label{eq:sk14normbound1}
}
Moreover,
\eq{
 \Big \lVert   \frac{1}{N L } \sum_{j = 1}^N     &     \psi \Big( \tfrac{1}{L}   \bw_k^\top  \bZin    \bX_j^\top \mathbbm{1}_L  \Big)    \bw_k^\top  \bZin    \bX_j^\top \mathbbm{1}_L       ( \bx_j -   \tfrac{1}{V} \mathbbm{1}_{V} )^\top         \Big \rVert_2 \\
& = \lvert \psi(0) \rvert \Big \lVert     \frac{1}{N L } \sum_{j = 1}^N         \bw_k^\top  \bZin    \bX_j^\top \mathbbm{1}_L       ( \bx_j -   \tfrac{1}{V} \mathbbm{1}_{V} )^\top         \Big \rVert_2 \\
& + \Big \lVert     \frac{1}{N L^2 } \sum_{j = 1}^N       \bx_j         \psi_2 \Big( \tfrac{1}{L}   \bw_k^\top  \bZin    \bX_j^\top \mathbbm{1}_L  \Big)   ( \bw_k^\top  \bZin    \bX_j^\top \mathbbm{1}_L   )^2         \Big \rVert_2 . 
}
We have  
\eq{
\Big \lvert \psi_2 \Big( \tfrac{1}{L}   \bw_k^\top  \bZin    \bX_j^\top \mathbbm{1}_L  \Big)   ( \bw_k^\top  \bZin    \bX_j^\top \mathbbm{1}_L   )^2  \Big   \rvert \lesssim L. 
}
Therefore, 
\eq{
 \Big \lVert     \frac{1}{N L^2 } \sum_{j = 1}^N       \bx_j         \psi_2 \Big( \tfrac{1}{L}   \bw_k^\top  \bZin    \bX_j^\top \mathbbm{1}_L  \Big)   ( \bw_k^\top  \bZin    \bX_j^\top \mathbbm{1}_L   )^2         \Big \rVert_2   \lesssim \frac{1}{\sqrt{V} L}.
}
Moreover,
\eq{
\Big \lVert     \frac{1}{N L } \sum_{j = 1}^N    &      \bw_k^\top  \bZin    \bX_j^\top \mathbbm{1}_L       ( \bx_j -   \tfrac{1}{V} \mathbbm{1}_{V} )^\top         \Big \rVert_2 \\
 & \leq   \Big \lVert     \frac{1}{N L } \sum_{j = 1}^N         \bw_k^\top  \bZin  (  \bX_j^\top - \tfrac{1}{V} \mathbbm{1}_V \mathbbm{1}_L^\top) \mathbbm{1}_L       ( \bx_j -   \tfrac{1}{V} \mathbbm{1}_{V} )^\top         \Big \rVert_2  \\
 &+ \frac{1}{V} \lvert  \bw_k^\top  \bZin  \mathbbm{1}_V  \rvert  \Big \lVert     \frac{1}{N } \sum_{j = 1}^N           ( \bx_j -   \tfrac{1}{V} \mathbbm{1}_{V} )^\top         \Big \rVert_2. 
}
We have 
\begin{itemize}
\item  $\frac{1}{V} \lvert  \bw_k^\top  \bZin  \mathbbm{1}_V  \rvert  \Big \lVert     \frac{1}{N } \sum_{j = 1}^N           ( \bx_j -   \frac{1}{V} \mathbbm{1}_{V} )^\top         \Big \rVert_2 \leq \frac{C \log^2 V}{\sqrt{VN}} $ 
\item Moreover,
\eq{
& \Big \lVert     \frac{1}{N L } \sum_{j = 1}^N         \bw_k^\top  \bZin  (  \bX_j^\top - \frac{1}{V} \mathbbm{1}_V \mathbbm{1}_L^\top) \mathbbm{1}_L       ( \bx_j -   \frac{1}{V} \mathbbm{1}_{V} )^\top         \Big \rVert_2 ^2  \\
&  =  \bw_k^\top  \bZin   \Big(      \frac{1}{N L } \sum_{j = 1}^N   (  \bX_j^\top - \tfrac{1}{V} \mathbbm{1}_V \mathbbm{1}_L^\top) \mathbbm{1}_L       ( \bx_j -   \tfrac{1}{V} \mathbbm{1}_{V} )^\top      \Big) \\
& \hspace{7em} \times \Big(      \frac{1}{N L } \sum_{j = 1}^N   (  \bX_j^\top - \tfrac{1}{V} \mathbbm{1}_V \mathbbm{1}_L^\top) \mathbbm{1}_L       ( \bx_j -   \tfrac{1}{V} \mathbbm{1}_{V} )^\top      \Big)^\top   \bZin^\top  \bw_k \lesssim  \frac{1}{N L}.  ~~~~~  \label{eq:sk14normbound2}
}
\end{itemize}
Then,  for $N \ll VL$
\eq{
\E [  \bar \nu_{14}^2  \vert \sZin ]  \lesssim   \frac{1}{N^2  L d (L \wedge d)}     
}
\item On the other hand,  we have
\eq{
 \lvert \nu_2 \rvert & \leq 
  \frac{1}{\sqrt{d}} \Big \lVert  \frac{1}{N  L}    \sum_{ i = 1}^N  ( \bx_i -   \tfrac{1}{V} \mathbbm{1}_{V} )   \mathsf{z}_{k}^\top    \bZin       \bX_i^\top   \bX_i  \bZin^\top   \bw_k   \phi^\prime  \big(  \tfrac{1}{L} \bw_k^\top  \bZin    \bX_i^\top \mathbbm{1}_L    \big) \Big \rVert_2 \\
&  \hspace{10em} \times  \Big \lVert   \frac{1}{N} \sum_{j = 1}^N      \phi \Big( \tfrac{1}{L}   \bw_k^\top  \bZin    \bX_j^\top \mathbbm{1}_L  \Big)       ( \bx_j -   \tfrac{1}{V} \mathbbm{1}_{V} )   \Big \rVert_2  
 } 
Note that   
\eq{
\Big \lVert   \frac{1}{N} \sum_{j = 1}^N      \phi \Big( \tfrac{1}{L}   & \bw_k^\top  \bZin     \bX_j^\top \mathbbm{1}_L  \Big)       ( \bx_j -   \tfrac{1}{V} \mathbbm{1}_{V} )   \Big \rVert_2    =
\lvert  \phi(0) \rvert  \Big \lVert   \frac{1}{N} \sum_{j = 1}^N     ( \bx_j -   \frac{1}{V} \mathbbm{1}_{V} )   \Big \rVert_2  \\
&  +  \Big \lVert   \frac{1}{N} \sum_{j = 1}^N      \psi \Big( \tfrac{1}{L}   \bw_k^\top  \bZin    \bX_j^\top \mathbbm{1}_L  \Big)     \frac{1}{L}   \bw_k^\top  \bZin    \bX_j^\top \mathbbm{1}_L     ( \bx_j -   \tfrac{1}{V} \mathbbm{1}_{V} )   \Big \rVert_2    \lesssim  \frac{1}{\sqrt{N}}. 
}
Therefore by \eqref{eq:sk14normbound1},  we have
\eq{
\E \Big[  \nu_2^2 \vert \sZin \Big] \lesssim   \frac{1}{N^2 d^2 (L \wedge d) }   
}
\end{itemize}

Therefore,  we have
\eq{
\E[  \nu \vert  \sZin  ] = 0  ~~ \text{and} ~~  \mathrm{Variance}(   \nu \vert  \sZin )  \lesssim \frac{1 }{N^2 d^2 m}.
}

\section{Lower Bound}
\label{sec:lowerbound}
To prove a lower bound,  we construct  a Bayesian setting  with the same likelihood distribution in our setting. In particular, the ground truth permutation is chosen from the set of permutation matrices:
\eq{
\mathcal{H} \coloneqq \{ \bs{P} \in \{0, 1\}^{V \times V} ~ \vert ~  \bs{\Pi} ~ \text{is a permutation matrix} \}.
}
We describe our Bayesian setting as a game between  \texttt{Environment} and \texttt{Learner} as follows:
\begin{itemize}[leftmargin=*]
\item  At the beginning,  \texttt{Environment} samples   $\bs{P}_*  \sim \mathrm{Unif}(\mathcal{H})$, probability vectors without revealing them to the learner.
\item  \texttt{Learner} observes $L + 1$ channel that generates words from the set $\mathcal{V} = \{ \bs{e}_1,  \bs{e}_2, \cdots,  \bs{e}_V \}$ sequentially for $t = 1, 2, \cdots, N$ with distributions:
\begin{itemize}
\item At every round,   \texttt{Environment} randomly picks a channel $\ell_t$ 
\item   \emph{Label:}  Channel  $0$  generates   $\bs{p}_{t} \sim_{iid} \mathrm{Unif}(\mathcal{V})$     
\item   \emph{Input:}  Given $\ell_t$ and $\bs{p}_{t}, $ Channel $\ell_t$ generates  $\bX_{\ell_t, t} =    \bs{P}_* \bs{p}_{t}$
\item  \emph{Noise distribution:} Channel $j \in [L] \setminus \{ \ell_t \}$ generate $\bX_{j, t} \sim\mathrm{Unif}(\mathcal{V})$ independent of  Channel  $0$. 
\end{itemize}
\item Let $\mathcal{D} \coloneqq \{ (\bX_{t}, \bs{p}_t) \}_{t \leq N}$ be the dataset.  We study the Bayes estimator with $0-1$ loss given the representation of the past: $S = f(\mathcal{D},  \ell_{1:N})$:
\eq{
\hat{\bs{P}} = \argmax_{\bs{P} \in \mathcal{H}} \mpr[  \bs{P} =  \bs{P}_*  \vert S,  \sZin ]. \label{eq:decisionrule}
}
In the following we consider the empirical mean and covariance of embedded words as the given data, i.e.,   $S \coloneqq \{ (  \bs{\mu}_t,   \bs{\Sigma}_t, \bs{p}_t) \}_{t \leq N}$, where
\eq{
\bs{\mu}_t \coloneqq  \frac{1}{L} \bZin \bX_{t}^\top \mathbbm{1}_L + \frac{\sigma_{\bs{\mu}}}{\sqrt{L}} \bs{g}_t   ~~ \text{and} ~~  \bs{\Sigma}_t \coloneqq    \frac{1}{L}     \bZin \bX_{t}^\top \bX_{t}    \bZin ^\top + \frac{\sigma_{\bs{\Sigma}}}{\sqrt{d L}} \bs{G}_t    .
}
where  $\{ (\bs{g}_t, \bs{G}_t) \}_{t \leq N}$ are i.i.d.  measurement noise with distributions $\bs{g}_t \sim \cN(0, \frac{1}{d} \bs{I}_d)$ and  $\bs{G}_{t, ij} = \bs{G}_{t, ji}$ with $\bs{G}_{t, ij} \sim  \cN \big(0,  \frac{(1+\delta_{ij})}{d}  \big)$  i.i.d. for $i < j$. 
\end{itemize}

\begin{theorem}
\label{thm:lowerbound}
The following lower bound holds:
\eq{
\mpr[\hat{\bs{P}} \neq \bs{P}_*  \vert \sZin ] \geq 1 - o_V(1) - \frac{\Omega(N)}{V} \Bigg( 1 \wedge     \Big(  \frac{1}{ \sigma_{\bs{\mu}}^2 } \frac{ d}{L \log V} +   \frac{C}{ \sigma_{\bs{\Sigma}}^2 } \frac{  d^2}{L \log V} \Big) \Bigg)
}
\end{theorem}

We use an information-theoretic argument to prove Theorem~\ref{thm:lowerbound}. For the proof, let $H(A)$ and $H(A \vert C)$ denote the entropy and conditional entropy of $A$ given $C$; let $I(A;B) =  H(A) - H(A \vert B) $ and $I(A;B \vert C) = H(A \vert C)  -  H(A \vert B, C) $ denote the mutual information  between random variables $A$ and $B$ and the conditional mutual  given $C$, respectively.  We let $D_{\mathrm{KL}}$ denote the Kullback-Leibler (KL) divergence. We start with an auxiliary statement  for the proof.   

\begin{lemma}
\label{lem:entropycond}
Let  $A, B, C, D$ be  discrete random variables defined on the same probability space. The following statements hold:
\begin{itemize}
\item  In general, $H(A \vert B, C) \leq H(A \vert B)$. The equality is satisfied if and only if $A \independent C \vert B$.  
\item  If  $B \independent D ~ \vert ~ (A,C)$,  we have $I(A,B \vert C, D)  \leq  I(A,B \vert C)$.
\item  Let $S = g(A,C)$  be a measurable function of $(A,C)$.  If $B \independent A \vert (S,C,D)$,  then $I(A; B \vert C,D) = I(S ; B \vert C, D)$.
\item  Given,  $\bs{\mu}, \bs{\mu}^\prime \in \R^d$, positive definite $\bs{\Sigma} \in \R^{d \times d}$ and $\mathrm{supp}(A) \subseteq \R^d$, we have   
\eq{
D_{\mathrm{KL}}( \cN(  \bs{\mu} + A,  \bs{\Sigma})\lvert \rvert   \cN(  \bs{\mu}^\prime + A,  \bs{\Sigma}) ) \leq  \frac{1}{2}  (\bs{\mu} - \bs{\mu}^\prime )^\top \bs{\Sigma}^{-1}  (\bs{\mu} - \bs{\mu}^\prime ).  
}
\end{itemize} 
\end{lemma}
 
\begin{proof}
We have
\eq{
H(A \vert B)  -  H(A \vert B, C)  =  \E \Big[ \log \frac{\mpr(A \vert B, C)}{\mpr(A \vert B)} \Big] =   \E \Big[ \log \frac{\mpr(A, C \vert B)}{\mpr(A \vert B) \mpr(C \vert B)} \Big] = I(A, C \vert B). 
}
Since the mutual information is non-negative, the first item follows.  Moreover,  since   $I(A, C \vert B) = 0$ if and only if  $A \independent C \vert B$.  For the  second item,  by using the first item, 
\eq{
I(A,B \vert C, D)  = H(B \vert C, D) - H(B \vert A, C, D) \leq   H(B \vert C) - H(B \vert A, C ) =   I(A,B \vert C).
}
For the third item,  since $S$ is a function of $(A,C)$, we have
\eq{
I(A; B \vert C,D) = I( (A, S); B \vert C,D) &  = H( B \vert C, D) - H( B \vert A, S, C, D) \\
& =  H( B \vert C, D) - H( B \vert S, C, D) 
 = I(S ; B \vert C, D).
}
Let $f$ denotes the  Gaussian pdf with 0 and covariance $\bs{\Sigma}$. For any $\bs{x} \in \R^d$,  since $t \to t \log t$ is convex
\eq{
 \Big( \sum_{\bs{a} \in \mathrm{supp}(A) } p(\bs{a}) f( \bs{x} - \bs{\mu} - \bs{a}  ) \Big)  & \log \frac{ \Big( \sum_{\bs{a} \in \mathrm{supp}(A) } p(\bs{a}) f( \bs{x} -  \bs{\mu} -  \bs{a}  ) \Big) }{ \Big( \sum_{\bs{a} \in \mathrm{supp}(A) } p(a) f( \bs{x} - \bs{\mu}^\prime - \bs{a}  ) \Big) } \\
 &   \leq   \sum_{\bs{a} \in \mathrm{supp}(A) } p(\bs{a})    f( \bs{x} - \bs{\mu} - \bs{a}  )   \log \frac{ f( \bs{x} - \bs{\mu} - \bs{a}  )  }{ f( \bs{x} - \bs{\mu}^\prime - \bs{a}  ) } .
}
Therefore, we have
\eq{
D_{\mathrm{KL}}( \cN(  \bs{\mu} + A,  \bs{\Sigma})\lvert \rvert    \cN(  \bs{\mu}^\prime + A,  \bs{\Sigma}) ) 
& \leq \sum_{\bs{a} \in \mathrm{supp}(A) } p(\bs{a})   D_{\mathrm{KL}}( \cN(  \bs{\mu} + \bs{a},  \bs{\Sigma})\lvert \rvert   \cN(  \bs{\mu}^\prime + \bs{a},  \bs{\Sigma}) )  \\
& =  D_{\mathrm{KL}}( \cN(  \bs{\mu},  \bs{\Sigma})\lvert \rvert   \cN(  \bs{\mu}^\prime,  \bs{\Sigma}) ),
}
where the last inequality follows the invariance of KL divergence in the second line  to constant shifts. The final bound follows the known formula for the KL divergence between Gaussian distributions.
\end{proof}
 
The proof of Theorem \ref{thm:lowerbound} is given in the following:

\begin{proof}[Proof of Theorem \ref{thm:lowerbound}]
Since we assume $\sZin$ is known by the learner, we will fix it in the following without explicitly conditioning thte terms on it.
Note that we consider the Bayes decision rule in \eqref{eq:decisionrule} and use Fano's inequality \citep{scarlett2019introductoryguidefanosinequality} to lower bound its error probability:
\eq{
\mpr[ \hat{\bs{P}} \neq \bs{P}_*  \vert \sZin ] \geq 1 - \frac{I(\bs{P}_*; S   ) + \log 2}{\log  \lvert \mathcal{H} \rvert} . \label{eq:fano}
}
We have
\eq{
I(\bs{P}_*; S  )   = I(\bs{P}_*;  \{ (  \bs{\mu}_t,   \bs{\Sigma}_t, \bs{p}_t) \}_{t \leq N}  )  
&  = I(\bs{P}_*;    \{ \bs{p}_t \}_{t \leq N}    )  +  I(\bs{P}_*;  \{ ( \bs{\mu}_t,   \bs{\Sigma}_t) \}_{t \leq N}   \vert \{ \bs{p}_t \}_{t \leq N} ,  ) \\
& \labelrel={lowerbound:eqq0}      I(\bs{P}_*;  \{ ( \bs{\mu}_t,   \bs{\Sigma}_t) \}_{t \leq N}   \vert  \{ \bs{p}_t \}_{t \leq N}   )  \\
& = \sum_{t = 1}^N I(\bs{P}_*;  ( \bs{\mu}_t,   \bs{\Sigma}_t)  \vert   \{ ( \bs{\mu}_u,   \bs{\Sigma}_u) \}_{u < t} ,    \{ \bs{p}_t \}_{t \leq N}  ) 
}
Given fixed  $\sZin $, we observe that    $( \bs{\mu}_t ,   \bs{\Sigma}_t ) \independent   \{ ( \bs{\mu}_u,   \bs{\Sigma}_u) \}_{u < t}  ~\big \vert ~  \bs{P}_*,  \{ \bs{p}_t \}_{t \leq N} $  and     $( \bs{\mu}_t,   \bs{\Sigma}_t)  \independent      \{ \bs{p}_u \}_{u \neq t} \vert  \bs{P}_*,  $.Therefore, by Lemma \ref{lem:entropycond},
\eq{
I(\bs{P}_*; S   )   \leq   \sum_{t = 1}^N I(\bs{P}_*;   ( \bs{\mu}_t,   \bs{\Sigma}_t)   \vert  \{ \bs{p}_t  \}_{t \leq N} ) \leq  \sum_{t = 1}^N I(\bs{P}_*;   ( \bs{\mu}_t,   \bs{\Sigma}_t)   \vert  \bs{p}_t ) .
}
Moreover,  we have  $\bs{P}_*  \independent  ( \bs{\mu}_t,   \bs{\Sigma}_t) ~ \vert ~ \bX_{\ell_t, t} , \bs{p}_t $, where  $\bX_{\ell_t, t}$ is a function of  $(\bs{P}_*,\bs{p}_t).$  Therefore,  by Lemma \ref{lem:entropycond},
\eq{
I(\bs{P}_*; S )   \leq    \sum_{t = 1}^N I( \bX_{\ell_t, t};   ( \bs{\mu}_t,   \bs{\Sigma}_t)   \vert \bs{p}_t ) .
}
We have
\eq{
 I( \bX_{\ell_t, t};   ( \bs{\mu}_t,   \bs{\Sigma}_t)   \vert \bs{p}_t , \sZin)  = \frac{1}{V} \sum_{k = 1}^V   D_{\mathrm{KL}} ( \mpr^k_{( \bs{\mu}_t,   \bs{\Sigma}_t)  }  \vert\vert  \mpr_0 )    \labelrel \leq{lowerbound:ineqq1}  \frac{1}{V^2} \sum_{j, k = 1}^V  D_{\mathrm{KL}} ( \mpr^k_{( \bs{\mu}_t,   \bs{\Sigma}_t)  }   \vert\vert  \mpr^j_{( \bs{\mu}_t,   \bs{\Sigma}_t)  }  )  
}
where  $\mpr^k_{( \bs{\mu}_t,   \bs{\Sigma}_t)  } $ denotes the distribution of $( \bs{s}_t,   \bs{\Sigma}_t)   \vert \bX_{\ell_t, t} \! = \! \bs{e}_k$,  $\mpr_0$ denotes  $\mpr_0 = \frac{1}{V} \sum_{k = 1}^V \mpr^k_{( \bs{\mu}_t,   \bs{\Sigma}_t)  } $, and \eqref{lowerbound:ineqq1} follows the convexity of KL divergence in its second argument.  For $k \neq j$,  by the last item of Lemma \ref{lem:entropycond}, we have
\eq{
 D_{\mathrm{KL}} ( \mpr^k_{( \bs{\mu}_t,   \bs{\Sigma}_t)  }   \lvert\rvert  \mpr^j_{( \bs{\mu}_t,   \bs{\Sigma}_t)  }  )  \leq  \frac{C}{ \sigma_{\bs{\mu}}^2 } \frac{ d}{L} \lVert \mathsf{z}_k - \mathsf{z}_j \rVert_2^2 +  \frac{C}{ \sigma_{\bs{\Sigma}}^2 } \frac{  d^2}{L} \lVert \mathsf{z}_k  \mathsf{z}_k^\top - \mathsf{z}_j  \mathsf{z}_j^\top  \rVert_F^2 \leq   \frac{C}{ \sigma_{\bs{\mu}}^2 } \frac{ d}{L} +   \frac{C}{ \sigma_{\bs{\Sigma}}^2 } \frac{  d^2}{L}.
}
Therefore, we have
\eq{
 I(\bs{P}_*; S  )   \leq   N  \Big(  \frac{C}{ \sigma_{\bs{\mu}}^2 } \frac{ d}{L} +   \frac{C}{ \sigma_{\bs{\Sigma}}^2 } \frac{  d^2}{L} \Big).
}
Moreover, we can write  
\eq{
I(\bs{P}_*; S )  \leq I(\bs{P}_*; \mathcal{D}, \ell_{1:N} )  & =   I(\bs{P}_*;  \{ \bX_{t} \}_{t \leq N} \vert \{  \bs{p}_t \}_{t \leq N}, \ell_{1:N} )  \\
& \leq  \sum_{t = 1}^N  I(\bs{P}_*;   \bX_{\ell_t, t}  \vert   \{  \bs{p}_t, \ell_t \}_{t \leq N}  )  \\
& \leq  \sum_{t = 1}^N  I(\bs{P}_*;   \bX_{\ell_t, t}  \vert     \bs{p}_t, \ell_t    ) 
}
where the first inequality follows data processing inequality,  third and fourth inequalities follow  the first  and second items in Lemma \ref{lem:entropycond}. We have
\eq{
 I(\bs{P}_*;   \bX_{\ell_t, t}  \vert     \bs{p}_t, \ell_t    ) =  \underbrace{ H(  \bX_{\ell_t, t}  \vert     \bs{p}_t, \ell_t   ) }_{  \log V} - \underbrace{ H(  \bX_{\ell_t, t}  \vert     \bs{p}_t, \ell_t  , \bs{P}_*) }_{= 0} = \log V.
}
Therefore, we have  $I(\bs{P}_*; S )  \leq  N \log V$.  Finally, we have
\eq{
I(\bs{P}_*; S )  \leq  N \Bigg( \log V \wedge     \Big(  \frac{C}{ \sigma_{\bs{\mu}}^2 } \frac{ d}{L} +   \frac{C}{ \sigma_{\bs{\Sigma}}^2 } \frac{  d^2}{L} \Big) \Bigg).
}
The result follows  from \eqref{eq:fano}.
\end{proof}

\section{Auxiliary Statements}

\subsection{Gaussian matrices and related statements}
\begin{lemma}
\label{lem:gaussianvectormoment}
Let $\bs{z} \sim \cN(0, \bs{I}_d).$ We have $\E[\norm{\bs{z}}_2^{2k}] = d (d+ 2) \cdots (d + 2k - 2)$.
\end{lemma}

\begin{proof}
We observe that  $\norm{\bs{z}}_2 \sim \chi^2_d$.  By using the moment formula for chi-squared distribution,  we have the result.
\end{proof}

\begin{lemma}
\label{lem:hansonwright}
Let  $\bs{z} \sim \cN(0,\bs{I}_{d})$ and $\bs{S} \in \R^{d \times d}$ be a symmetric matrix.  For $u > 0$,
\eq{
\mpr \left[  \abs{ \bs{z}^\top  \bs{S}  \bs{z} - \tr(\bs{S})  } \geq 2 \norm{\bs{S}}_F u + 2 \norm{\bs{S}}_2 u^2 \right] \leq 2 e^{- u^2}.
}
\end{lemma}

\begin{proof}
We note that   $\bs{z}^\top  \bs{S}  \bs{z}  - \tr(\bs{S})$ has the same distribution with $\sum_{i = 1}^d \lambda_i(\bs{S}) (Z_i^2 - 1)$, where   $Z_i \sim_{iid} \cN(0,1)$.  By using the Laurent-Massart lemma, we have the result.
\end{proof}

\begin{proposition}
\label{prop:gaussquadraticform}
Let $\bs{S} \in \R^{V \times V}$ be a symmetric positive semidefinite matrix.  Let
\eq{
\bs{M} =  \bZin \bs{S} \bZin^\top.
}
For $\mathrm{poly}(d) \gg V \gg d$, We have
\eq{
\mpr \left[  \Big \lVert \bs{M}  - \frac{  \tr(\bs{S})}{d}   \bs{I}_d \Big \rVert_2  \geq   \max \Big\{   \frac{\norm{\bs{S}}_F}{\sqrt{d}} \log V,    \norm{\bs{S}}_2  \log^2 V \Big\}   \right] \leq    \exp  (  - c \log^2 V ) .
}
\end{proposition}

\begin{proof}
Without loss of generality, we can assume that $\bs{S}$ is diagonal, i.e.,  $\bs{S} = \mathrm{diag}(s_1, \cdots, s_V)$.   We have
\eq{
\bs{M}  - \frac{  \tr(\bs{S})}{d} \bs{I}_d   = \sum_{i = 1}^V s_i \big( \bs{z}_i  \bs{z}_i^\top - \frac{1}{d} \bs{I}_d  \big).
}
We have  
\eq{
\E \Big[ \Big(  \sum_{i = 1}^V s_i \big( \bs{z}_i  \bs{z}_i^\top - \frac{1}{d} \bs{I}_d  \big) \Big)^2   \Big] = \frac{1}{d}  (1 + \frac{1}{d} )  \lVert \bs{S} \rVert_F^2 \bs{I}_d
}
Moreover, for $p \leq \frac{d}{2}$
\eq{
\E \Big[  \lVert    \bs{z}_i  \bs{z}_i^\top - \frac{1}{d} \bs{I}_d    \rVert_2^p \Big]  \leq \E[ \lVert \bs{z}_i \rVert_2^{2p} ] \leq 2^p.
}
By Proposition \ref{prop:rosenthal}, we have  $2 \leq p \leq \frac{d}{2}$
\eq{
\E \Big[  \lVert   \bs{M}  - \frac{  \tr(\bs{S})}{d}  \rVert_2^p \Big]  \leq C  \Big( \sqrt{ p \vee \log d } \frac{ \lVert \bs{S} \rVert_F }{\sqrt{d}} +  ( p \vee \log d) V^{\frac{1}{p}} \lVert \bs{S} \rVert_2 \Big).
}
For $p = \frac{1}{e^2 C^2} \log^2 V$, we have the result.
\end{proof}

\begin{proposition}
\label{prop:gaussquadraticformtrace}
Let $\bs{S} \in \R^{V \times V}$ be a square matrix and let $\bs{M} =  \bZin \bs{S} \bZin^\top$.
For $\mathrm{poly}(d) \gg V \gg d$, We have
\eq{
\mpr \big[    \lvert \tr(\bs{M})  -   \tr(\bs{S})   \rvert  \geq    \log^2 V \frac{\lVert \bs{S} \rVert_F}{\sqrt{d}} \big] \leq    \exp  (  - c \log^2 V ) .
}
\end{proposition}

\begin{proof}
    Without loss of generality, we can assume that $\bs{S}$ is diagonal, i.e.,  $\bs{S} = \mathrm{diag}(s_1, \cdots, s_V)$. We have
\eq{
\tr(\bs{M})  -   \tr(\bs{S})   = \sum_{i = 1}^V s_i \big( \lVert \bs{z}_i  \rVert_2^2 - 1 \big).
}
We have
\eq{
\E \big[ \exp(\lambda  s_i \big( \lVert \bs{z}_i  \rVert_2^2 - 1 \big) ) \big] \leq \exp \Big( \frac{4 \lambda^2 s_i^2}{d} \Big), ~~ \lvert \lambda \rvert \leq \frac{d}{4 \abs{s_i}}.
}
Then,
\eq{
\E \big[ \exp(\lambda    \big(  \tr(\bs{M})  -   \tr(\bs{S}) \big) )   \big] \leq  \exp \Big( \frac{4 \lambda^2 \lVert \bs{S} \rVert_F^2}{d} \Big),  ~~  \lvert \lambda \rvert \leq \frac{d}{4 \lVert \bs{S} \rVert_2}
}
We have
\eq{
\mpr \Big[ \lvert \tr(\bs{M})  -   \tr(\bs{S}) \rvert \geq    \log^2 V \frac{\lVert \bs{S} \rVert_F}{\sqrt{d}}  \Big] \leq \exp( - c \log^2 V ).
}
\end{proof}

\begin{proposition}
\label{prop:gausssquare}
Let $\bs{S} \in \R^{V \times V}$ be a square matrix.  For $\bs{u}, \bs{v} \in S^{d - 1}$ and $\bs{M} =  \bZin \bs{S} \bZin^\top$, we have
\eq{
\mpr \Big[ \Big \lvert    \big( \bs{v}^\top  \bs{M}   \bs{u}    -  \frac{\tr(\bs{S})}{d}   \bs{v}^\top   \bs{u}  \big) \Big \rvert \geq \frac{\norm{\bs{u}}_2   \norm{\bs{v}}_2 }{d}  \max \Big \{    \norm{ \sym(\bs{S}) }_F  t,  & \norm{ \sym(\bs{S}) }_2  t^2 \Big \}  \Big]   \leq 2 \exp  ( - c t^2 ).
}
\end{proposition}

\begin{proof}
Consider $\bs{g} = \sqrt{d} \mathrm{vec}(\bs{Z})$, where $\bs{g} \sim \cN(0, \bs{I}_{dV})$. We have
\eq{
 \bs{v}^\top  \bs{M}   \bs{u} = \frac{1}{d}   \bs{g}^\top (\bs{u} \bs{v}^\top) \otimes \bs{S} \bs{g} =  \frac{1}{d} \bs{g}^\top \sym( \bs{u} \bs{v}^\top) \otimes \sym( \bs{S}) \bs{g}
}
By using Proposition \ref{prop:kroneckerproduct}, we have
\eq{
\E[ \bs{g}^\top \sym( \bs{u} \bs{v}^\top) \otimes \sym( \bs{S}) \bs{g} ] = \tr(\bs{S})  \bs{u}^\top \bs{v}.
}
Moreover,
\eq{
  \Big( \bs{g}^\top \sym( \bs{u} \bs{v}^\top) \otimes \sym( \bs{S}) \bs{g}  -  \tr(\bs{S})  \bs{u}^\top \bs{v} \Big) =_d  \sum_{i = 1}^{dV} \lambda_i (g_i^2 - 1) 
}
where $g_i \sim N(0,1)$.  By using the subexponential concentration,  we have the result.
\end{proof}

\begin{proposition}
\label{prop:wishartsquare}
For $\bs{u}, \bs{v} \in \R^{V}$, we have
\eq{
 \mpr \Big[ \Big \lvert     \bs{v}^\top  \bZin^\top   \bZin  \bZin^\top   \bZin     \bs{u}  -    \bs{u}^\top \bs{v}  \Big( 1 +   \frac{V - 1}{d} \Big) \Big \rvert \geq   C  \lVert   \bs{u}& \rVert_2  \lVert   \bs{v} \rVert_2  \log V \Big( \frac{\sqrt{V}}{d}  +  \frac{V}{d^{3/2}} \Big)  \Big]  \leq  10 \exp  ( - c \log^2 V ).
}
\end{proposition}

\begin{proof}
Without loss of generality,  we assume that $\bs{u}$  and $\bs{v}$ have a unit norm.  Let 
\eq{
\bs{v}_{\perp} \coloneqq \frac{1}{\sqrt{1 - ( \bs{u}^\top \bs{v})^2 }} (\bs{I}_V -   \bs{v} \bs{v}^\top )  \bs{u}. 
}
We have 
\eq{
  \bs{v}^\top  \bZin^\top   \bZin  \bZin^\top   \bZin     \bs{u}   =  (\bs{u}^\top \bs{v}) \bs{v}^\top    \bZin^\top   \bZin  \bZin^\top   \bZin     \bs{v}   +   \sqrt{1 -  (\bs{u}^\top \bs{v})^2 }  \bs{v}^\top    \bZin^\top   \bZin  \bZin^\top   \bZin    \bs{v}_{\perp} .
}
Without loss of generality,  we consider $\bs{v} = \bs{e}_1$ and     $\bs{v}_{\perp} =  \bs{e}_2$.  For the second term,  we write $\bs{z}_i   \coloneqq \bZin \bs{e}_i$ and let $\tilde{\bs{Z}} \coloneqq \{ \bs{z}_i \}_{i = 3}^V$ and $\bs{g}  = \sqrt{d} \mathrm{vec}(\tilde{\bs{Z}} )$.
\eq{
 \bs{e}_1 ^\top    \bZin^\top   \bZin  \bZin^\top   \bZin    \bs{e}_2  & = ( \lVert \bs{z}_1  \rVert_2^2 +  \lVert \bs{z}_2  \rVert_2^2 )  \bs{z}_1^\top  \bs{z}_2  +  \bs{z}_1^\top \tilde{\bs{Z}}  \tilde{\bs{Z}}^\top  \bs{z}_2 \\
&  =  ( \lVert \bs{z}_1  \rVert_2^2 +  \lVert \bs{z}_2  \rVert_2^2 )  \bs{z}_1^\top  \bs{z}_2   + \frac{1}{d}  \bs{g}^\top  \sym( \bs{z}_1 \bs{z}_2^\top) \otimes \bs{I}_{V-2}  \bs{g}.
 } 
 We have
\begin{itemize}[leftmargin=*]
\item  By  Lemma \ref{lem:hansonwright}, and Proposition \ref{prop:gausssquare}
\eq{
\mpr  \Big[  \big \lvert   \lVert \bs{z}_1  \rVert^2_2 - 1 \big \rvert \leq  \frac{5 \log V}{\sqrt{d}}  ~ \text{and} ~  \big \lvert   \lVert \bs{z}_2  \rVert^2_2 - 1 \big \rvert  \leq    \frac{5 \log V}{\sqrt{d}} & ~ \text{and} ~ \lvert  \bs{z}_1^\top  \bs{z}_2  \rvert \leq \frac{\log V}{\sqrt{d}}   \Big]    \leq  1 - 6 \exp(- c \log^2 V).
}
\item  By Proposition \ref{prop:kroneckerproduct}, we have
\begin{itemize}
\item $\lVert  \sym( \bs{z}_1 \bs{z}_2^\top) \otimes \bs{I}_{V-2} \rVert_2 \leq  \lVert  \bs{z}_1 \rVert_2  \lVert  \bs{z}_2 \rVert_2 $
\item $\lVert  \sym( \bs{z}_1 \bs{z}_2^\top) \otimes \bs{I}_{V-2} \rVert_F \leq \sqrt{V} \lVert  \bs{z}_1 \rVert_2  \lVert  \bs{z}_2 \rVert_2 $
\item $\tr \big( \sym( \bs{z}_1 \bs{z}_2^\top) \otimes \bs{I}_{V-2}  \big) = (V-2)   \bs{z}_1^\top  \bs{z}_2. $
\end{itemize}
Therefore, by Lemma \ref{lem:hansonwright}, we have
\eq{
 \mpr \Big[ \Big \lvert   \frac{1}{d}  \bs{g}^\top  \sym( \bs{z}_1 \bs{z}_2^\top) \otimes \bs{I}_{V-2}  \bs{g} \! - \!  \frac{(V \! - \! 2)}{d}   \bs{z}_1^\top  \bs{z}_2 \Big \rvert \! \leq \! 2  \lVert  \bs{z}_1 \rVert_2  \lVert \bs{z}_2 \rVert_2\Big( \frac{\log V}{d}  \sqrt{V} \!  + \!    \frac{\log^2 V}{d}    \Big)  \Big] \! \leq \! 1  \! - \! 2 \exp(- c \log^2 V).
}
\end{itemize}
By union bound of the precious two items, we have
\eq{
 \mpr \Big[  \Big \lvert  \bs{e}_1 ^\top    \bZin^\top   \bZin  \bZin^\top   \bZin    \bs{e}_2 \Big \rvert \leq  2    \log V \Big(    \frac{V}{d^{3/2}} +\frac{\sqrt{V}}{d}  \Big)   \Big]  \geq 1 - 8 \exp(- c \log^2 V). \label{eq:orthogonalpart}
}
Next, we redefine the notation:  $\tilde{\bs{Z}} \coloneqq \{ \bs{z}_i \}_{i = 2}^V$.   We write
\eq{
\bs{z}_1^\top     \bZin  \bZin^\top    \bs{z}_1  -  1 - \frac{V - 1}{d} =    \lVert \bs{z}_1  \rVert^4_2 - 1 +  \bs{z}_1^\top  \Big( \tilde{\bs{Z}}  \tilde{\bs{Z}}^\top    - \frac{V - 1}{d} \bs{I}_d \Big) \bs{z}_1   - \frac{V - 1}{d}   (  \lVert \bs{z}_1  \rVert^2_2 - 1 )
}
By Proposition bla, we have
\eq{
\mpr \Big[     \bs{z}_1^\top  \Big( \tilde{\bs{Z}}  \tilde{\bs{Z}}^\top    - \frac{V - 1}{d} \bs{I}_d \Big) \bs{z}_1 \leq \log V \norm{ \bs{z}_1}_2^2  \frac{\sqrt{V}}{d} \Big] \leq  1 - 2 \exp(- c \log^2 V)
}
By using the first item above, we have
\eq{
\mpr \Big[  \Big \lvert  \bs{z}_1^\top     \bZin  \bZin^\top    \bs{z}_1  -  1 - \frac{V - 1}{d} \Big \rvert \geq 6  \log V \Big( \frac{\sqrt{V}}{d}  +  \frac{V}{d^{3/2}} \Big) \Big]  \leq   1 - 2 \exp(- c \log^2 V).  \label{eq:alignedpart}
}
The result follows \eqref{eq:orthogonalpart} and \eqref{eq:alignedpart}.
\end{proof}

\subsection{Multinomial distribution and related statements}

\begin{lemma}
\label{lem:multinomialmgf}
Let $(n_1, \cdots, n_V) \in  \mathrm{Mult}\big(N; (p_1, \cdots, p_V) \big)$.  For $\bs{t} \in \R^V$,
\eq{
\E \Big[  \exp \Big( \sum_{i = 1}^V t_i n_i \Big) \Big] = \Big(  \sum_{i = 1}^V p_i e^{t_i}  \Big)^N.
}
Then, if  $p_i  = \frac{1}{V}$, $i \in [V]$,
\begin{itemize}[leftmargin = *]
\item We have 
\begin{itemize}[leftmargin=*]
\item[-] $\E \left[ \prod_{i = 1}^V  n_i (n_i - 1) \cdots (n_i - k_i + 1)  \right] =  \frac{ N (N-1) \cdots (N - K + 1)}{V^K}$,  where  $K \coloneqq \sum_{i = 1}^V k_i$.
\end{itemize}
\item By the previous item, we can write
\begin{itemize}[leftmargin=*]
\item $ \E \left[  n_i^2  \right] =  \frac{N}{V}  + \frac{N (N - 1)}{V^2}$
\item  $ \E \left[  \big( \tfrac{n_i}{N} - \frac{1}{V} \big)^2 n_i  \right] = \frac{(V - 1) (N + V - 2)}{N V^3}$.
\item  $ \E \left[  n_i^3  \right] = \frac{N}{V} + \frac{3 N (N - 1)}{V^2} + \frac{N (N - 1) (N - 2)}{V^3}$
\item   $ \E \left[  n_i^4  \right] = \frac{N}{V} + \frac{7 N (N - 1)}{V^2} + \frac{ 6N (N - 1) (N - 2)}{V^3} + \frac{N (N - 1) (N - 2) (N - 3)}{V^4}  $
\item  For $i \neq i^\prime$, $\E \left[  n_i^2  n_{i^\prime}^2  \right] = \frac{N (N - 1)}{V^2} + \frac{2 N(N - 1) (N - 2)}{V^3} + \frac{N (N - 1) (N - 2) (N - 3)}{V ^4}. $
\item  $\E \left[   \big(  \sum_{i = 1}^V n_i^2 \big) ^2  \right]  = N^2 + \frac{ 2  (N + 1) N (N - 1)}{V} +  \frac{  (N + 1) N (N - 1)(N  - 2) }{V^2}$
\end{itemize}
\end{itemize}
\end{lemma}

\begin{proof}
Let $\bs{x}_j$ sampled from $\{ \bs{e}_1, \cdots, \bs{e}_V \}$ with  $(p_1, \cdots, p_V).$  We have $n_i =   \sum_{j = 1}^N \bs{e}_i^\top \bs{x}_i$. We have
\eq{
\E \Big[  \exp \Big( \sum_{i = 1}^V t_i n_i \Big) \Big] =\E \Big[  \exp \Big( \sum_{j = 1}^N \inner{\bs{t}}{\bs{x}_j} \Big) \Big]   & = \left(  \E \Big[  \exp \Big(   \inner{\bs{t}}{\bs{x}_1} \Big) \Big] \right)^N  
 = \Big(  \sum_{i = 1}^V p_i e^{t_i}  \Big)^N.
}
The later statements can be derived by using $z_i = e^{t_i}$ and taking derivatives of both sides with respect $(z_1, \cdots, z_V)$.  
\end{proof}

\begin{proposition}
\label{prop:multimatrixexpectation}
Let  $\bs{n} \coloneqq (n_1, \cdots, n_V) \sim  \mathrm{Mult}\big(L, \frac{1}{V}  \mathbbm{1}_V\big)$ and  $\bs{S} \in \R^{V \times V}$ be a symmetric matrix. The following statements hold: 
\begin{itemize}[leftmargin = *]
\item  We have
\begin{itemize}[leftmargin=*]
    \item[-] $\E[ \diag(\bs{n}) \bs{S}  \diag(\bs{n})  ] = L \E[   \bx_1^\top \bs{S}   \bx_1 \bx_1  \bx_1^\top ] + \frac{L(L-1)}{V^2} \bs{S}$
    \item[-] $\E[ \diag(\bs{n} - \tfrac{L}{V} \mathbbm{1}_V ) \bs{S}  \diag(\bs{n} - \tfrac{L}{V} \mathbbm{1}_V)  ] = L \E[   \bx_1^\top \bs{S}   \bx_1 \bx_1  \bx_1^\top ] - \frac{L}{V^2} \bs{S}.$
\end{itemize}
\item We have
\begin{itemize}[leftmargin=*]
    \item[-]  $\E[ \bs{n}  \bs{n}^\top \bs{S}  \bs{n}   ]  = \frac{2 L (L-1)}{V^2}   \bs{S}  \mathbbm{1}_V +  L  \E \big[ \bx_1   \bx_1^\top \bs{S}   \bx_1 \big]  + \Big(  \frac{L(L-1)}{V^2} \tr(\bs{S}) +    \frac{L(L-1) (L-2)}{V^3} \mathbbm{1}_V^\top \bs{S}  \mathbbm{1}_V   \Big)  \mathbbm{1}_V.$
\end{itemize}
\item  We have
\eq{
 \E \Big[ \Big(   \big(\bs{n} - \tfrac{L}{V} \mathbbm{1}_V \big)^\top    \bs{S} \big(\bs{n} - \tfrac{L}{V} \mathbbm{1}_V \big)   \Big)^2 \Big]&  = 
 \frac{L}{V} \Big \lVert \diag(\bs{S}) - \frac{2}{V} \bs{S} \mathbbm{1}_V + \frac{1}{V^2}   \big( \mathbbm{1}_V^\top  \bs{S}  \mathbbm{1}_V \big)  \mathbbm{1}_V \Big \rVert_2^2   \\
 & + \frac{L (L-1)}{V^2}  \tr \Big( \big(\bs{I}_ V - \tfrac{1}{V} \mathbbm{1}_V \mathbbm{1}_V^\top \big)     \bs{S} \Big)^2 \\
&  +  \frac{ L (L-1)}{V^2}     \tr \Big( \big(\bs{I}_ V - \tfrac{1}{V} \mathbbm{1}_V \mathbbm{1}_V^\top \big)     \bs{S} \big(\bs{I}_ V - \tfrac{1}{V} \mathbbm{1}_V \mathbbm{1}_V^\top \big)     \bs{S} \Big) 
}
\end{itemize}
\end{proposition}

\begin{proof}
For the first item, we observe that
\eq{
\bs{e}_j^\top  \E[ \diag(\bs{n}) \bs{S}  \diag(\bs{n})  ]  \bs{e}_i = \E[n_j n_i]  \bs{S}_{ij} = \Big( \frac{L}{V} \delta_{ij} + \frac{L(L-1)}{V^2} \Big)    \bs{S}_{ij},
}
from which the first equation follows.   For the second equation, 
\eq{
 \bs{e}_j^\top \E[ \diag(\bs{n} - \tfrac{L}{V} \mathbbm{1}_V ) \bs{S}  \diag(\bs{n} - \tfrac{L}{V} \mathbbm{1}_V)  ]  \bs{e}_i
= \E[ (n_j - \frac{L}{V}) (n_i -   \frac{L}{V} )]  \bs{S}_{ij} 
= \Big( \frac{L}{V} \delta_{ij} - \frac{L}{V^2} \Big)    \bs{S}_{ij}.
}
For the second item, we have 
\eq{
& ( \E[ \bs{n}  \bs{n}^\top \bs{S}  \bs{n}   ] )_{i}  = \sum_{jk}   \bs{S}_{jk}  \E[ n_i n_j n_k ] \\
& =  \frac{L (L-1) (L-2)}{V^3} ( \sum_{i \neq j \neq k}   \bs{S}_{jk}) +  \Big( \frac{L (L-1) (L-2)}{V^3} + \frac{L (L-1)}{V^2} \Big)   ( 2  \sum_{i  \neq   k}   \bs{S}_{ik} +     \sum_{i  \neq   k}   \bs{S}_{kk}     )  \\
& +   \Big( \frac{L}{V} + \frac{3 L (L - 1)}{V^2} + \frac{L (L - 1) (L - 2)}{V^3}  \Big)   \bs{S}_{ii}     \\
& = \frac{L}{V}   \bs{S}_{ii}  +  \frac{L (L-1)}{V} \tr(\bs{S}) +     \frac{2 L (L-1)}{V} \sum_{k}  \bs{S}_{ik}  +\frac{L (L-1) (L-2)}{V^3} ( \sum_{jk}   \bs{S}_{jk}).
}
For the third item,  we have     $\big(\bs{n} - \frac{L}{V} \mathbbm{1}_V \big) = \sum_{i = 1}^L (\bs{x}_{i} - \frac{1}{V} \mathbbm{1}_V)$ in distribution.  For notational convenience, let 
\eq{
 s_{ij} \coloneqq (\bs{x}_{i} - \frac{1}{V} \mathbbm{1}_V)^\top \bs{S}    (\bs{x}_{j} - \frac{1}{V} \mathbbm{1}_V).
}
Then,
\eq{ 
\E \Big[  \Big(  \big(\bs{n} - \frac{L}{V} \mathbbm{1}_V \big)^\top    \bs{S} \big(\bs{n} - \frac{L}{V} \mathbbm{1}_V \big) \Big)^2 \Big] =  \sum_{i,j,k,l = 1}^L  \E [  s_{ij}  s_{kl} ]
}
By independence,  only $(i,j,k,l)$ where each index occur even times contribute.  The possible cases are:
\begin{itemize}[leftmargin = *]
\item All four indices equal $(i = j = k =  l)$: There are $L$ many terms here with contribution
\eq{
\E [ s_{ii}^2 ] &= \frac{1}{V}	 \Big \lVert \mathrm{diag} \Big( (\bs{I}_V - \frac{1}{V} \mathbbm{1}_V \mathbbm{1}_V^\top )  \bs{S}  (\bs{I}_V - \frac{1}{V} \mathbbm{1}_V \mathbbm{1}_V^\top )   \Big) \Big \rVert_2^2  \\
& =  \frac{1}{V}	 \Big \lVert \mathrm{diag}(\bs{S}) - \frac{2}{V}   \bs{S} \mathbbm{1}_V + \frac{1}{V^2}  (\mathbbm{1}_V^\top    \bs{S} \mathbbm{1}_V)   \mathbbm{1}_V  \Big \rVert_2^2 .
}
\item Two distinct indices, both pairs diagonal ($i = j$ and  $k =   l$ and $i \neq k$):  There are $L (L-1)$ many terms here with contribution
\eq{
\E [  s_{ii} s_{ kk} ]=    \E [   s_{ii}  ]^2
= \frac{1}{V^2} \tr \Big(   (\bs{I}_V - \frac{1}{V} \mathbbm{1}_V \mathbbm{1}_V^\top ) \bs{S}   \Big)^2
}
\item Two distinct indices, paired off-diagonal: ($i = k$ and  $j = l$ and $i \neq j$):  There are $L (L-1)$ many terms here with contribution
\eq{
\E [  s_{ij}^2 ] & =   \tr \Big( \E [   (\bs{x}_{1} - \frac{1}{V} \mathbbm{1}_V) (\bs{x}_{1} - \frac{1}{V} \mathbbm{1}_V)^\top \bs{S}   (\bs{x}_{2} - \frac{1}{V} \mathbbm{1}_V) (\bs{x}_{2} - \frac{1}{V} \mathbbm{1}_V)^\top  ] \bs{S} \Big) \\
& = \frac{1}{V^2} \tr \Big(    (\bs{I}_V - \frac{1}{V} \mathbbm{1}_V \mathbbm{1}_V^\top ) \bs{S}     (\bs{I}_V - \frac{1}{V} \mathbbm{1}_V \mathbbm{1}_V^\top ) \bs{S}   \Big).
}
\end{itemize}
\end{proof}

\begin{proposition}
\label{prop:multrandommatrix}
Let $V^3 \gg L$.  There exists a universal $C > 0$ such that the following holds:
\begin{itemize}[leftmargin = *]
\item    Let $m_{ij} \coloneqq (1 - \frac{1}{V}) \indic{i = j} + \frac{L}{V}$.  For   $K > 0$  and $p \geq \log V$,
\begin{itemize}[leftmargin=*]
    \item[-]  $\E \Big[ \Big \lvert  \frac{1}{L}  \mathbbm{1}_L^\top \bX_i \bX_j^\top \mathbbm{1}_L -  m_{ij} \Big \rvert^p \Big]^{\frac{1}{p}} \leq   C  \Big (      \frac{p^{\frac{3}{2}}}{\sqrt{V}}  + \frac{p^2}{L}      \Big)$
    \item[-] $\mpr \Big[  \Big \lvert  \frac{1}{L}  \mathbbm{1}_L^\top \bX_i \bX_j^\top \mathbbm{1}_L -  m_{ij} \Big \rvert \geq   C  K^2 \frac{\log^2 V}{\sqrt{V} \wedge L}    \Big] \leq  \frac{1}{V^K}$
\end{itemize}
\item For $K > 0$  and $p \geq \log V$,
\begin{itemize}[leftmargin=*]
    \item[-]  $\E \Big[  \Big \lVert   \frac{1}{N L} \! \sum_{i = 1}^N \!\! \big( \bX_i^\top \! - \! \frac{1}{V} \mathbbm{1}_V   \mathbbm{1}_L^\top \big) \mathbbm{1}_L    \mathbbm{1}_L  ^\top  \big( \bX_i^\top - \frac{1}{V} \mathbbm{1}_V   \mathbbm{1}_L^\top \big)^\top  \!\!  - \!\frac{1}{V} (\bs{I} - \frac{1}{V} \mathbbm{1}_V \mathbbm{1}_V^\top )  \Big \rVert_2^p \Big]^{\frac{1}{p}}  \leq C \Bigg(  \sqrt{\frac{p}{ N V}}  +  \frac{p}{N} \Big(   1 +     \frac{p^2}{\sqrt{V} \wedge L}   \Big)  \Bigg)$
    \item[-] $\mpr \Big[  \Big \lVert   \frac{1}{N L} \! \sum_{i = 1}^N  \!\! \big( \bX_i^\top \! - \! \frac{1}{V} \mathbbm{1}_V   \mathbbm{1}_L^\top \big)\! \mathbbm{1}_L \mathbbm{1}_L^\top   \!  \big( \bX_i^\top \! - \!  \frac{1}{V} \mathbbm{1}_V   \mathbbm{1}_L^\top \big)^\top \!\!\!\!  -  \frac{1}{V} (\bs{I} \! -  \!\frac{1}{V} \mathbbm{1}_V \mathbbm{1}_V^\top )  \Big \rVert_2 \!\! > \! C K \!  \Big(   \frac{ \log^2 \! V}{  \sqrt{N V}} \! + \! \frac{  \log^2 \! V}{N} ( 1  \! +  \!       \frac{\log^2 \! V}{\sqrt{V} \wedge L}     )  \Big) \Big]  \! \leq  \! \tfrac{1}{V^K}.$
\end{itemize}
\end{itemize}

\end{proposition}

\begin{proof}
Let $\bs{x}_{il}$ be i.i.d. copies of $\bs{x}_1$. We note that $\bs{X}_i  \mathbbm{1}_L = \sum_{l = 1}^L \bs{x}_{il}$ in distribution.  For $i = j$, we have 
\eq{
 \frac{1}{L}  \mathbbm{1}_L^\top \bX_i \bX_i^\top \mathbbm{1}_L = 1 +  \frac{2}{L} \sum_{1 \leq l < r \leq L} \indic{\bs{x}_{ir} = \bs{x}_{il}} 
=  1  + \frac{ (L-1)}{V} +  \frac{2}{L}  \sum_{l = 2}^L  \sum_{r = 1}^{l -1}  \big( \indic{\bs{x}_{ir} = \bs{x}_{il}}   - \tfrac{1}{V} \big)
}
Define
\eq{
Y_l  \coloneqq     \sum_{r = 1}^{l -1}  \big( \indic{\bs{x}_{ir} = \bs{x}_{il}}   - \tfrac{1}{V} \big)   ~~ \text{and} ~~ \mathcal{F}_l \coloneqq \sigma(Y_1, \cdots, Y_l).
}
Given that  
\eq{
\sum_{r = 1}^{l -1}   \indic{\bs{x}_{ir} = \bs{x}_{il}}    \vert   \bs{x}_{il}  \sim \mathrm{Binomial}(l - 1, \frac{1}{V}) \Rightarrow \E[ \abs{Y_k}^p]^{\frac{1}{p}} \leq C  (\sqrt{p} \sqrt{\frac{L}{V}}  + p), ~~ p \geq \log V. \label{eq:Ypnormb}
} 
where we used Corollary \ref{cor:rosenthalresults}.
As for the quadratic variation
\eq{
Q_L   \coloneqq \sum_{l = 1}^L \E[ Y_l^2 \vert \mathcal{F}_{l-1} ] & =  \sum_{l = 1}^L  \frac{1}{V} \Big(   \big \lVert  \sum_{r = 1}^{l - 1} \bs{x}_{ir} \big \rVert_2^2 - \frac{(l - 1)^2}{V} \big) \Big)    =   \frac{1}{V}  \sum_{l = 1}^L     \big \lVert  \sum_{r = 1}^{l - 1} \bs{x}_{ir}  - \tfrac{l -1}{V}  \mathbbm{1}_{V} \big    \rVert_2^2. 
 }
For   $p \geq \log V$,   by using triangle inequality,
\eq{
\E[ \abs{Q_L}^{\frac{p}{2}} ]^{\frac{2}{p}} & \leq \frac{1}{V} \sum_{l = 1}^L \E \Big[   \Big \lVert \sum_{r = 1}^{l - 1} \bs{x}_{ir}  - \tfrac{l -1}{V}  \mathbbm{1}_{V} \Big  \rVert_2^p   \Big]^{\frac{2}{p}}\\
&  \labelrel\leq{mult:ineqq0}    \frac{1}{V} \sum_{l = 1}^V (l-1) \E \Big[ \Big  \lVert    \sum_{r = 1}^{l - 1} \bs{x}_{ir} \Big  \rVert _p^p   \Big]^{\frac{2}{p}}  +   \sum_{l = V + 1}^L \E \Big[  \Big  \lVert  \sum_{r = 1}^{l - 1} \bs{x}_{ir}  - \tfrac{l -1}{V}  \mathbbm{1}_{V}   \Big \rVert_p^p   \Big]^{\frac{2}{p}} \\
&   \labelrel\leq{mult:ineqq1}     C p^2 \frac{1}{V} \sum_{l = 1}^L l  \\
& =  C p^2 \frac{L^2}{V},
}
where  we used H\"{o}lder's inequality in \eqref{mult:ineqq0} and   Corollary \ref{cor:rosenthalresults} in \eqref{mult:ineqq1}.
By using \eqref{eq:Ypnormb} and Proposition \ref{prop:rosenthal},   for $p \geq \log V$,  we have
\eq{
 \E\Big[ \Big \lvert \sum_{l = 1}^L Y_k \Big \rvert^p \Big]^{\frac{1}{p}} \leq  C  \Big (  p \sqrt{p}   \frac{L}{\sqrt{V}}  + p^2      \Big).
}
By using $p = \log V$, we have
\eq{
\mpr \left[  \Big \lvert \frac{1}{L} \sum_{l = 1}^L Y_k  \Big \rvert   > C e K^2 \frac{\log^2 V}{\sqrt{V} \wedge L} \right] \leq  \frac{1}{V^K}.
}
Hence, we have the $i = j$ case.  For $i \neq j$, we have  
\eq{
\frac{1}{L}  \mathbbm{1}_L^\top \bX_j \bX_i^\top \mathbbm{1}_L   = \frac{L}{V} +  \frac{1}{L}  \sum_{l = 1}^L    \sum_{r = 1}^L\indic{\bs{x}_{il} = \bs{x}_{jr} } - \frac{1}{V}
}
We redefine the martingale difference sequence as
\eq{
Y_{l} \coloneqq    \sum_{r = 1}^L \indic{\bs{x}_{il} = \bs{x}_{jr} } - \frac{1}{V}.
}
Conditioned on $\bX_{j},$  we have $\{ Y_1, \cdots, Y_L \}$ are i.i.d. and
\eq{
\E[  Y_{k} \vert  \bX_{j}  ] = 0 ~~ \text{and} ~~  \E[  Y^p_{k} \vert  \bX_{j}  ]  = \frac{1}{V}   \norm{  (\bX_{j}^\top - \frac{1}{V} \mathbbm{1}_V \mathbbm{1}_L^\top) \mathbbm{1}_L }_p^p  
}
By Proposition \ref{prop:rosenthal},  for $p \geq \log V$, we have
\eq{
\E \Big[ \Big \lvert \frac{1}{L} \sum_{l = 1}^L Y_l  \Big \rvert^p \Big]^{\frac{1}{p}} \leq C \Big( \frac{\sqrt{p}}{\sqrt{V}} +  \frac{p^{\frac{3}{2}}}{\sqrt{LV}} + \frac{  p^2 }{L}   \Big).
}
By using $p = \log V$, we have
\eq{
\mpr \Big[  \Big \lvert   \frac{1}{L}  \mathbbm{1}_L^\top \bX_j \bX_i^\top \mathbbm{1}_L   -  \frac{L}{V}  \Big \rvert \geq    \frac{ C  K^2 \log^2 V}{\sqrt{V} \vee L}      \Big]  \leq   \frac{1}{V^K}.
}
For the second item,  we define 
\eq{
\bs{Y}_i \coloneqq \frac{1}{L} \big( \bX_i^\top - \frac{1}{V} \mathbbm{1}_V   \mathbbm{1}_L^\top \big) \mathbbm{1}_L    \mathbbm{1}_L  ^\top  \big( \bX_i^\top - \frac{1}{V} \mathbbm{1}_V   \mathbbm{1}_L^\top \big)^\top - \frac{1}{V} (\bs{I}_V - \frac{1}{V} \mathbbm{1}_V \mathbbm{1}_V^\top)   
}
and  $\bs{Q}_N\coloneqq N \E[ \bs{Y}_1^2 ]$.
We have
\eq{
\bs{Q}_N & \preceq  N \E \Big[  \Big \lVert   \frac{1}{\sqrt{L}} \big( \bX_1^\top - \tfrac{1}{V} \mathbbm{1}_V   \mathbbm{1}_L^\top \big) \mathbbm{1}_L   \Big \rVert_2^2  \frac{1}{L} \big( \bX_1^\top - \tfrac{1}{V} \mathbbm{1}_V   \mathbbm{1}_L^\top \big) \mathbbm{1}_L    \mathbbm{1}_L  ^\top  \big( \bX_1^\top - \tfrac{1}{V} \mathbbm{1}_V   \mathbbm{1}_L^\top \big)^\top  \Big ] \\
& = N \E \Big[ ( 1 - \frac{1}{V} ) \frac{1}{L} \big( \bX_1^\top - \tfrac{1}{V} \mathbbm{1}_V   \mathbbm{1}_L^\top \big) \mathbbm{1}_L    \mathbbm{1}_L  ^\top  \big( \bX_1^\top - \tfrac{1}{V} \mathbbm{1}_V   \mathbbm{1}_L^\top \big)^\top  \Big ] \\
& + N \E \Big[ \Big(  \Big \lVert   \frac{1}{\sqrt{L}} \big( \bX_1^\top - \tfrac{1}{V} \mathbbm{1}_V   \mathbbm{1}_L^\top \big) \mathbbm{1}_L   \Big \rVert_2^2 -  (1 - \frac{1}{V}) \Big) \\
& \hspace{2em}  \times \Big(    \frac{1}{L}   \big( \bX_1^\top - \tfrac{1}{V} \mathbbm{1}_V   \mathbbm{1}_L^\top \big) \mathbbm{1}_L    \mathbbm{1}_L  ^\top  \big( \bX_1^\top - \tfrac{1}{V} \mathbbm{1}_V   \mathbbm{1}_L^\top \big)^\top    -   \frac{1}{V} (\bs{I}_V - \frac{1}{V} \mathbbm{1}_V \mathbbm{1}_V^\top \Big)\Big ] \\
& \labelrel\preceq{mult:ineqq2}  \frac{C N}{V}  \bs{I}_V  + \frac{1}{2} \bs{Q}_N, 
}
where we use Proposition \ref{prop:matrixCS} in \eqref{mult:ineqq2}. 
Therefore, we have $\norm{ \bs{Q}_N  }_2 \leq \frac{CN}{V}$.  Moreover, we observe that
\eq{
\norm{ \bs{Y}_i }_2 & \leq \frac{1}{V} +    \Big \lVert  \frac{1}{\sqrt{L}}  \big( \bX_i^\top - \frac{1}{V} \mathbbm{1}_V   \mathbbm{1}_L^\top \big) \mathbbm{1}_L \Big \rVert_2^{2}.
}
By using the first item,
\eq{
\E[ \norm{ \bs{Y}_i }_2^p]^{\frac{1}{p}} \leq  \frac{1}{V} +  \E \Big[ \Big \lVert  \frac{1}{\sqrt{L}}  \big( \bX_i^\top - \frac{1}{V} \mathbbm{1}_V   \mathbbm{1}_L^\top \big) \mathbbm{1}_L \Big \rVert_2^{2p}   \Big]^{\frac{1}{p}}  \leq  1+ C  \Big (      \frac{p^{\frac{3}{2}}}{\sqrt{V}}  + \frac{p^2}{L}      \Big).
}
Therefore, by using Proposition \ref{prop:rosenthal}, we have
\eq{
 \E \Big[  &   \Big \lVert \frac{1}{N} \sum_{i = 1}^N  \bs{Y}_i    \Big \rVert_2^p \Big]  \leq  C \Bigg( \sqrt{p \vee \log V} ~ \sqrt{\frac{1}{ N V}}  + (p \vee \log V) N^{\frac{1}{p} - 1} \Big(  1+  \frac{p^{\frac{3}{2}}}{\sqrt{V}}  + \frac{p^2}{L}        \Big)  \Bigg).
}
By using $p = \log V$, we have
\eq{
\mpr \Big[   \Big \lVert \frac{1}{N} \sum_{i = 1}^N  \bs{Y}_i  \Big \rVert_2 >  C K \log^2 V \Big(   \frac{1}{  \sqrt{N V}}  + \frac{1}{N}  \Big( 1 +        \frac{\log^2 V}{\sqrt{V} \wedge L}     \Big) \Big)   \Big] \leq \frac{1}{V^K}.
}
\end{proof}

\begin{proposition}
\label{prop:sbound}
We consider $\bs{S}_1$,  $\bs{S}_2$ and $\bs{S}_3$ defined in  \eqref{def:sone}, \eqref{def:stwo} and \eqref{def:sthree} in the regime $V^3 \gg N \gg V $ and $L \asymp V^{\varepsilon}$, $\varepsilon \in (0,1)$.  For any $K > 0$ and $V \geq \Omega_{K, \varepsilon}(1)$,   the following holds:
\begin{enumerate}[leftmargin = *]
\item We have
\eq{
\mpr \left[ \Big \lvert \tr(\bs{S}_1)  -  \frac{1 - 1/V}{L^2} \big(   \frac{1}{V} + (1- \frac{2}{V})  \frac{1}{N} \big)  \Big  \rvert >    C  K^2 \frac{\log^2 V}{L^2 N \sqrt{V}}   ~~ \text{or} ~~    \norm{\bs{S}_1}_2 > \frac{e^2}{L^2 V^2}  \right] \leq \frac{2}{V^K}.
}
\item We have
\eq{
\mpr \left[ \Big \lvert \tr(\bs{S}_2)  -  (1 - \frac{1}{V})^2 \frac{L - 1}{L^2 N}  \Big  \rvert >   C \frac{K^{\frac{3}{2}} \log^3 V}{L N V}   ~~ \text{or} ~~    \norm{\bs{S}_2 }_2 > C  \frac{K^{\frac{3}{2}} \log^2 V}{NLV} \right] \leq \frac{4}{V^K}.
}
\item We have
\eq{
\mpr \Big[ \tfrac{-  C  K^2 \log^2 V}{N \sqrt{V}}   \tfrac{1}{V^2 L^2} \mathbbm{1}_V \mathbbm{1}_V^\top  \preceq   \bs{S}_3 - \tfrac{1}{N} \tfrac{1}{V^2 L^2} \mathbbm{1}_V \mathbbm{1}_V^\top  \preceq  \tfrac{ C  K^2  \log^2 V}{N \sqrt{V}}    \tfrac{1}{V^2 L^2} \mathbbm{1}_V \mathbbm{1}_V^\top  \Big] \leq \frac{1}{V^K}.
}
\end{enumerate}
\end{proposition}

\begin{proof}
We define $n_i \coloneqq \lvert \{ j \leq N ~ \vert ~  \bx_j  = \be_i \} \rvert$. We observe that
\eq{
\tr(\bs{S}_1) = (1 - \frac{2}{V}) \frac{1}{L^2 N^2} \sum_{i = 1}^V n_i^2 + \frac{1}{V^2 L^2} ~~ \text{and} ~~ \norm{\bs{S}_1}_2 \leq  \sup_{i \leq N} \frac{ n_i^2}{L^2 N^2}.
}
By using Proposition  \ref{prop:multrandommatrix} and Corollary \ref{cor:rosenthalresults}, we have the first item.  For the second item, we write
\eq{
\bs{S}_2 & =  \frac{(1 - \frac{1}{V})}{L^2 N^2} \sum_{j = 1}^N    (\bN_j^\top  - \frac{1}{V} \mathbbm{1}_V  \mathbbm{1}_{L -1}^\top  )   \mathbbm{1}_{L -1}  \mathbbm{1}_{L -1}^\top (\bN_j^\top  - \frac{1}{V} \mathbbm{1}_V  \mathbbm{1}_{L -1}^\top  ) ^\top \\
& +   \frac{2}{L^2 N^2} \sum_{ j < k } \big( \indic{\bx_j = \bx_k} \! - \!\tfrac{1}{V} \big) \sym \Big( (\bN_j^\top  - \tfrac{1}{V} \mathbbm{1}_V  \mathbbm{1}_{L -1}^\top  )   \mathbbm{1}_{L -1}  \mathbbm{1}_{L -1}^\top (\bN_k^\top  - \tfrac{1}{V} \mathbbm{1}_V  \mathbbm{1}_{L -1}^\top  ) ^\top \Big) \\
 & \eqqcolon \bs{S}_{21} +   \bs{S}_{22}
}
We will analyze  $\bs{S}_{21}$ and $\bs{S}_{22}$ separately.   

\paragraph{Bounds for $\bs{S}_{21}$:}   We have
\eq{
& \tr(  \bs{S}_{21} ) -   (1 - \frac{1}{V})^2 \frac{L - 1}{L^2 N} \\
& =  (1 - \frac{1}{V}) \frac{L - 1}{L^2  N^2} \sum_{j = 1}^N \underbrace{ \lVert  \tfrac{1}{\sqrt{L - 1}} (\bN_j^\top  - \frac{1}{V} \mathbbm{1}_V  \mathbbm{1}_{L -1}^\top  )   \mathbbm{1}_{L -1} \rVert_2^2 - (1 - \frac{1}{V}) }_{\coloneqq Y_{1,j}}.
}
We have $\E[Y_{1,j}^2] \leq \frac{2}{V}$ and by  Proposition \ref{prop:multrandommatrix}, 
\eq{
\E[ \lvert Y_{1,j} \rvert^p]^{\frac{1}{p}} \leq \frac{C p^2}{\sqrt{V} \wedge L}.
}
Therefore, by Proposition \ref{prop:rosenthal},
\eq{
\E \Big[ \Big \lvert \tr(  \bs{S}_{21} ) -  (1 - \frac{1}{V})^2 \frac{L - 1}{L^2 N}  \Big \rvert^p \Big]^{\frac{1}{p}} \leq \frac{C}{L N^2} \Big(  \sqrt{ \frac{p N}{V} } + p  N^{\frac{1}{p}}   \frac{p^2}{\sqrt{V} \wedge L} \Big)
}
By using $p = \log V$, we have
\eq{
\mpr \Big[  \Big \lvert \tr(  \bs{S}_{21} ) -  (1 - \frac{1}{V})^2 \frac{L - 1}{L^2 N}  \Big \rvert  > C   \frac{K \log^3 V}{L  N \sqrt{NV}}    \Big]  \leq \frac{1}{V^K}.  \label{eq:s21trace}
}
Moreover, by Proposition \ref{prop:multrandommatrix}, we have  
\eq{  
& \mpr \Big[  \Big \lVert \bs{S}_{21}  -  \big(1 - \frac{1}{V} \big) \frac{L - 1}{L^2 N} \frac{1}{V} ( \bs{I}_V - \mathbbm{1}_V \mathbbm{1}_V^\top ) \Big \rVert_2 >        C \frac{K \log^2 V}{LN} \Big ( \frac{1}{\sqrt{NV}} +  \frac{1}{N} \Big(1 + \frac{\log^2 V}{\sqrt{V} \wedge L} \Big) \Big) \Big]   \leq \frac{1}{V^K}.     \label{eq:s21norm}
}
\paragraph{Bounds for $\bs{S}_{22}$:} We write
\eq{
 \tr(  \bs{S}_{22} )  
& =  \frac{2 }{L^2 N^2} \sum_{k = 2}^N \sum_{j = 1}^{k - 1} \big( \indic{\bx_j = \bx_k} - \tfrac{1}{V} \big)     \mathbbm{1}_{L -1}^\top (\bN_j^\top  - \tfrac{1}{V} \mathbbm{1}_V  \mathbbm{1}_{L -1}^\top  ) ^\top  (\bN_k^\top  - \tfrac{1}{V} \mathbbm{1}_V  \mathbbm{1}_{L -1}^\top  )   \mathbbm{1}_{L -1} \\
& \eqqcolon \sum_{k = 2}^N Y_{2,k}.
}
Let   $\mathcal{F}_k \coloneqq \sigma(  \bN_{1:k} )$ and $Y_{2,1} = 0$.
We have 
\eq{
\E[ Y_{2,k}^2 \vert \mathcal{F}_{k-1} ] 
& = \frac{4 (L - 1)}{L^4 N^4} \frac{1}{V} \E \Big[ \Big \lVert  \sum_{j = 1}^{k - 1} \big( \indic{\bx_j = \bx_k} - \tfrac{1}{V} \big)      (\bN_j^\top  - \tfrac{1}{V} \mathbbm{1}_V  \mathbbm{1}_{L -1}^\top  )  \mathbbm{1}_{L -1}  \Big \rVert_2^2 \Big \vert \mathcal{F}_{k-1} \Big ] \\
& =    (1 - \frac{1}{V}) \frac{4 (L - 1)}{L^4 N^4}   \frac{1}{V^2}  \sum_{j = 1}^{k - 1}   \lVert   (\bN_j^\top  - \frac{1}{V} \mathbbm{1}_V  \mathbbm{1}_{L -1}^\top  )    \mathbbm{1}_{L -1}  \rVert_2^2
}
Then,
\eq{
Q_N   =    \sum_{k = 1}^N \E[ Y_{2,k}^2 \vert \mathcal{F}_{k-1} ]   
& =     (1 - \frac{1}{V}) \frac{4 (L - 1)}{L^4 N^4 V^2}     \sum_{k = 2}^N  \sum_{j = 1}^{k - 1}   \lVert   (\bN_j^\top  - \frac{1}{V} \mathbbm{1}_V  \mathbbm{1}_{L -1}^\top  )    \mathbbm{1}_{L -1}  \rVert_2^2 \\
& =     (1 - \frac{1}{V}) \frac{4 (L - 1)}{L^4 N^4 V^2}     \sum_{k = 1}^{N - 1} (N - k)     \lVert  (\bN_k^\top  - \frac{1}{V} \mathbbm{1}_V  \mathbbm{1}_{L -1}^\top  )    \mathbbm{1}_{L -1}  \rVert_2^2.
}
Then,  for $p \geq \log V$,
\eq{
\E  \big[ \lvert Q_N \rvert^{\frac{p}{2}} \big]^{\frac{2}{p}}  \leq \frac{5}{L^3 N^3 V^2}    \sum_{k = 1}^{N}  \E \Big[    \lVert   \bN_k^\top   \mathbbm{1}_{L -1}  \rVert_2^p \Big]^{\frac{2}{p}}   & \labelrel\leq{sbound:ineqq0} \frac{5 L^{1 - \frac{2}{p}}}{L^3  N^3 V^2}    \sum_{k = 1}^{N}   \E \Big[    \lVert  \bN_k^\top   \mathbbm{1}_{L -1}  \rVert_p^p \Big]^{\frac{2}{p}}  \\
& \labelrel\leq{sbound:ineqq1} \frac{5 p^2}{L^2  N^2 V^2}, \label{eq:qnbound}
}
where we used H\"{o}lder's inequality in \eqref{sbound:ineqq0}
and Corollary \ref{cor:rosenthalresults} in  \eqref{sbound:ineqq1}.
By using Proposition \ref{prop:rosenthal},  we show the following:
\begin{itemize}[leftmargin=*]
\item To bound $\E[ \lvert Y_{2,k} \rvert^p  ]^{\frac{1}{p}}$ for $p \geq \log V$, by using the conditional independence of $\{ \bs{x}_j \}_{j = 1}^{k - 1}$,  we write
\eq{
 & \E[ \lvert Y_{2,k} \rvert^p \vert \bN_{1:k} ,  \bx_k ]^{\frac{1}{p}} \\
 &  \leq \frac{C}{L  N^2}   \frac{\sqrt{p}}{  \sqrt{V}} \Big( \sum_{j = 1}^{k - 1} \Big \lvert \frac{1}{L-1} \Big \langle  (\bN_k^\top  - \frac{1}{V} \mathbbm{1}_V  \mathbbm{1}_{L -1}^\top  )  \mathbbm{1}_{L -1} ,  (\bN_j^\top  - \frac{1}{V} \mathbbm{1}_V  \mathbbm{1}_{L -1}^\top  ) \mathbbm{1}_{L -1} \Big \rangle \Big \rvert^2    \Big)^{\frac{1}{2}} \\
& +   \frac{C p k^{\frac{1}{p}}}{L N^2} \Big(  \sum_{j = 1}^{k - 1} \Big \lvert  \frac{1}{L- 1}   \Big \langle  (\bN_k^\top  - \frac{1}{V} \mathbbm{1}_V  \mathbbm{1}_{L -1}^\top  ) \mathbbm{1}_{L -1}  ,  (\bN_j^\top  - \frac{1}{V} \mathbbm{1}_V  \mathbbm{1}_{L -1}^\top  ) \mathbbm{1}_{L -1}  \Big \rangle \Big \rvert^p \Big)^\frac{1}{p}
}
Therefore,
\eq{
\E[ \lvert Y_{2,k} \rvert^p  ]^{\frac{1}{p}}  & \leq \frac{C}{L N^2} \Big(   \frac{\sqrt{p} \sqrt{k} }{  \sqrt{V}}   +    p k^{\frac{2}{p}}  \Big)   \E \Big[ \Big \lvert  \frac{1}{L-1}   \Big \langle  (\bN_k^\top  - \frac{1}{V} \mathbbm{1}_V  \mathbbm{1}_{L -1}^\top  ) \mathbbm{1}_{L -1}   ,  (\bN_1^\top  - \frac{1}{V} \mathbbm{1}_V  \mathbbm{1}_{L -1}^\top  ) \mathbbm{1}_{L -1}  \Big \rangle \Big \rvert^p  \Big]^{\frac{1}{p}}   \\
  & \leq    \frac{C  }{L N^2}   \Big(  \frac{\sqrt{p} \sqrt{k} }{  \sqrt{V}}   +   p k^{\frac{2}{p}}    \Big) \E \Big[ \Big \lvert  \frac{1}{L- 1}   \Big \langle \bN_k^\top    \mathbbm{1}_{L -1}   ,   \bN_1^\top   \mathbbm{1}_{L -1}  \Big \rangle - \tfrac{L - 1}{V} \Big \rvert^p  \Big]^{\frac{1}{p}} \\
    & \labelrel\leq{sbound:ineqq3}    \frac{C p^2 }{L  N^2}   \Big(  \frac{\sqrt{p} \sqrt{k} }{  \sqrt{V}}   +   p k^{\frac{2}{p}}    \Big)    \frac{1}{\sqrt{V} \wedge L} , \label{eq:ypbound}
}
where we used Proposition \ref{prop:multrandommatrix} in \eqref{sbound:ineqq3}.
\item Then by using \eqref{eq:qnbound} and \eqref{eq:ypbound}, we have for $p = \log V$
\eq{
& \E[ \lvert \tr(  \bs{S}_{22} ) \rvert^p ]^{\frac{1}{p}}  \leq C  \Big(     \frac{p^{\frac{3}{2}}}{N L V} +  \frac{p^4 N^{\frac{1}{p}}}{L N^{\frac{3}{2}} \sqrt{V}}   \frac{1}{\sqrt{V} \wedge L}  \Big)   \leq \frac{C p^{\frac{3}{2}}}{N L V} . 
}
\end{itemize}
Then, we have
\eq{
\mpr \left[  \lvert \tr(  \bs{S}_{22} ) \rvert > \frac{C K^\frac{3}{2} \log^\frac{3}{2} V}{N L V} \right] \leq \frac{1}{V^K}.   \label{eq:s22trace}
}
To bound  $\lVert \bs{S}_{22} \rVert_2,$  we define
\eq{
\bs{Y}_{k} \coloneqq     \sum_{j = 1}^{k - 1} \big( \indic{\bx_j = \bx_k} - \frac{1}{V} \big) \sym \Big( (\bN_j^\top  - \frac{1}{V} \mathbbm{1}_V  \mathbbm{1}_{L -1}^\top  )   \mathbbm{1}_{L -1}  \mathbbm{1}_{L -1}^\top (\bN_k^\top  - \frac{1}{V} \mathbbm{1}_V  \mathbbm{1}_{L -1}^\top  ) ^\top \Big).
}
We have
\eq{
 \E[ \bs{Y}_{k}^2 \vert  \mathcal{F}_{k-1}] 
& \! \preceq \! \frac{2}{V}   \sum_{j = 1}^{k - 1}  \E \Big [      (\bN_k^\top \! - \!  \tfrac{1}{V} \mathbbm{1}_V  \mathbbm{1}_{L -1}^\top  )     \mathbbm{1}_{L -1}  \mathbbm{1}_{L -1}^\top (\bN_k^\top  \! - \! \tfrac{1}{V} \mathbbm{1}_V  \mathbbm{1}_{L -1}^\top  ) ^\top  \Big \lVert   (\bN_j^\top \! - \!  \tfrac{1}{V} \mathbbm{1}_V  \mathbbm{1}_{L -1}^\top  )   \mathbbm{1}_{L -1} \Big \rVert_2^2  \Big \vert  \mathcal{F}_{k-1}  \Big]   \\
& +  \frac{2}{V}   \sum_{j = 1}^{k - 1}  \E \! \Big[     \Big \lVert    (\bN_k^\top \! - \! \tfrac{1}{V} \mathbbm{1}_V  \mathbbm{1}_{L -1}^\top  )     \mathbbm{1}_{L -1}    \Big \rVert_2^2     (\bN_j^\top  - \tfrac{1}{V} \mathbbm{1}_V  \mathbbm{1}_{L -1}^\top  )    \mathbbm{1}_{L -1} \! \mathbbm{1}_{L -1}^\top (\bN_j^\top   \! - \! \tfrac{1}{V} \mathbbm{1}_V  \mathbbm{1}_{L -1}^\top  )^{\top}   \Big  \vert  \mathcal{F}_{k-1}  \Big]   \\
& \preceq   \frac{2L}{V^2}   \sum_{j = 1}^{k - 1}     \Big \lVert   (\bN_j^\top  \! - \!  \tfrac{1}{V} \mathbbm{1}_V  \mathbbm{1}_{L -1}^\top  )   \mathbbm{1}_{L -1} \Big \rVert_2^2  \bs{I}_V  \\
& + \frac{2L}{V}  \sum_{j = 1}^{k - 1}  (\bN_j^\top   \! - \!  \tfrac{1}{V} \mathbbm{1}_V  \mathbbm{1}_{L -1}^\top  )    \mathbbm{1}_{L -1}  \mathbbm{1}_{L -1}^\top (\bN_j^\top   \! - \!  \tfrac{1}{V} \mathbbm{1}_V  \mathbbm{1}_{L -1}^\top  )^{\top}.
}
Therefore, we have
\eq{
\bs{Q}_N   \coloneqq \sum_{k = 1}^N  \E[ \bs{Y}_{k}^2 \vert  \mathcal{F}_{k-1}] 
& \preceq  \frac{2   L}{V^2}   \sum_{k = 1}^{N - 1} (N - k)  \Big \lVert    (\bN_k^\top  - \frac{1}{V} \mathbbm{1}_V  \mathbbm{1}_{L -1}^\top  )   \mathbbm{1}_{L -1}   \Big \rVert_2^2  \bs{I}_V   \\
& +   \frac{2 L^2 N}{V}  \frac{1}{L N}  \sum_{k = 1}^{N - 1} (N - k)  (\bN_k^\top  - \frac{1}{V} \mathbbm{1}_V  \mathbbm{1}_{L -1}^\top  )    \mathbbm{1}_{L -1}  \mathbbm{1}_{L -1}^\top (\bN_k^\top  - \frac{1}{V} \mathbbm{1}_V  \mathbbm{1}_{L -1}^\top  )^{\top} \\
& \preceq  \frac{2 N  L}{V^2}   \sum_{k = 1}^{N - 1}  \Big \lVert    (\bN_k^\top  - \frac{1}{V} \mathbbm{1}_V  \mathbbm{1}_{L -1}^\top  )   \mathbbm{1}_{L -1}   \Big \rVert_2^2  \bs{I}_V   \\
& +   \frac{2 L^2 N^2}{V}  \frac{1}{L N}  \sum_{k = 1}^{N - 1}   (\bN_k^\top  - \frac{1}{V} \mathbbm{1}_V  \mathbbm{1}_{L -1}^\top  )    \mathbbm{1}_{L -1}  \mathbbm{1}_{L -1}^\top (\bN_k^\top  - \frac{1}{V} \mathbbm{1}_V  \mathbbm{1}_{L -1}^\top  )^{\top}. 
}
Then,
\eq{
\E[ \norm{ \bs{Q}_N}_2^{\frac{p}{2}} ]^{\frac{2}{p}}  &   \leq    \frac{2 N L}{V^2}  \E \Big[ \Big(    \sum_{k = 1}^{N - 1}   \Big \lVert    (\bN_k^\top  - \frac{1}{V} \mathbbm{1}_V  \mathbbm{1}_{L -1}^\top  )   \mathbbm{1}_{L -1}   \Big \rVert_2^2 \Big)^{\frac{p}{2}} \Big]^{\frac{2}{p}} \\
& + \frac{2 L^2 N^2}{V}  \E \Big[ \Big \lVert \frac{1}{N (L-1)} \sum_{k = 1}^{N - 1}   (\bN_k^\top  - \frac{1}{V} \mathbbm{1}_V  \mathbbm{1}_{L -1}^\top  )    \mathbbm{1}_{L -1}  \mathbbm{1}_{L -1}^\top (\bN_k^\top  - \frac{1}{V} \mathbbm{1}_V  \mathbbm{1}_{L -1}^\top  )^{\top} \Big \rVert_2^{\frac{p}{2}}  \Big ]^{\frac{2}{p}}  \\
 & \leq    \frac{2 N^2 L^2}{V^2}  \E \Big[    \Big \lVert  \frac{1}{\sqrt{L - 1}}  (\bN_1^\top  - \frac{1}{V} \mathbbm{1}_V  \mathbbm{1}_{L -1}^\top  )   \mathbbm{1}_{L -1}   \Big \rVert_2^p \Big]^{\frac{2}{p}} \\
& + \frac{2 L^2 N^2}{V}  \E \Big[ \Big \lVert \frac{1}{N (L-1)}  \sum_{k = 1}^{N - 1}   (\bN_k^\top  - \frac{1}{V} \mathbbm{1}_V  \mathbbm{1}_{L -1}^\top  )    \mathbbm{1}_{L -1}  \mathbbm{1}_{L -1}^\top (\bN_k^\top  - \frac{1}{V} \mathbbm{1}_V  \mathbbm{1}_{L -1}^\top  )^{\top} \Big \rVert_2^{\frac{p}{2}}  \Big ]^{\frac{2}{p}}  \\
& \leq  \frac{C N^2 L^2}{V^2} \Big(1   + \frac{p^2}{\sqrt{V} \vee L}   \Big) 
 + \frac{C L^2 N^2}{V} \Big(  \frac{1}{V} + \sqrt{\frac{p}{ N V}}  +  \frac{p}{N} \Big(   1 +     \frac{p^2}{\sqrt{V} \wedge L}   \Big)    \Big)  \\
 & \leq  \frac{C N^2 L^2}{V^2} \Big(1   + \frac{p}{N/V} + \frac{p^2}{\sqrt{V} \vee L}   \Big) .
}
To bound $\E[ \norm{ \bs{Y}_{k} }_2^p]$ , we observe that
\begin{itemize}[leftmargin=*]
\item We have
\eq{
\E \Big[   \Big(  & \big( \indic{\bx_j = \bx_k}  - \tfrac{1}{V} \big)  \sym \Big( (\bN_j^\top  - \tfrac{1}{V} \mathbbm{1}_V  \mathbbm{1}_{L -1}^\top  )   \mathbbm{1}_{L -1}  \mathbbm{1}_{L -1}^\top (\bN_k^\top  - \tfrac{1}{V} \mathbbm{1}_V  \mathbbm{1}_{L -1}^\top  ) ^\top \Big) \Big)^2  \Big \vert  \bx_k, \bN_k \Big] \\
& \preceq  \frac{L}{V^2} \lVert  (\bN_k^\top  - \frac{1}{V} \mathbbm{1}_V  \mathbbm{1}_{L -1}^\top  )   \mathbbm{1}_{L -1}  \rVert_2^2  \bs{I}_V 
 + \frac{L}{V} (\bN_k^\top  - \tfrac{1}{V} \mathbbm{1}_V  \mathbbm{1}_{L -1}^\top  )   \mathbbm{1}_{L -1}  \mathbbm{1}_{L -1}^\top (\bN_k^\top  - \tfrac{1}{V} \mathbbm{1}_V  \mathbbm{1}_{L -1}^\top  ) ^\top. 
}
Moreover,  
\eq{
&  \E \Big[ \Big   \lVert   \big( \indic{\bx_j = \bx_k} \! - \! \tfrac{1}{V} \big) \sym \Big( (\bN_j^\top \! - \! \tfrac{1}{V} \mathbbm{1}_V  \mathbbm{1}_{L -1}^\top  )   \mathbbm{1}_{L -1}  \mathbbm{1}_{L -1}^\top (\bN_k^\top  \! - \! \tfrac{1}{V} \mathbbm{1}_V  \mathbbm{1}_{L -1}^\top  ) ^\top \Big)  \Big \rVert_2^p  \Big \vert  \bx_k, \bN_k \Big]^{\frac{1}{p}} \\
&  \leq \! 2   \E \! \Big[ \Big \lvert  \big( \indic{\bx_j = \bx_k} \! - \!  \tfrac{1}{V} \big) \Big \rvert^p  \Big \vert \bx_k \Big]^{\frac{1}{p}} \!  \big \lVert ( \bN_k^\top \! - \!  \tfrac{1}{V} \mathbbm{1}_V  \mathbbm{1}_{L -1}^\top  )  \mathbbm{1}_{L -1}    \big \rVert_2    \E \! \Big[  \big \lVert ( \bN_j^\top  \! - \!  \tfrac{1}{V} \mathbbm{1}_V  \mathbbm{1}_{L -1}^\top  )  \mathbbm{1}_{L -1}    \big \rVert_2^p   \Big]^{\frac{1}{p}}  \\
& \leq  C \sqrt{L} \big \lVert ( \bN_k^\top  - \tfrac{1}{V} \mathbbm{1}_V  \mathbbm{1}_{L -1}^\top ) \mathbbm{1}_{L -1}  \big \rVert_2   \Big( \sqrt{\frac{p}{V}} + p \Big) \Big( 1 + \frac{p^2}{\sqrt{V} \wedge L}\Big).
}
\item By Proposition \ref{prop:rosenthal}, we have for $p = \log V$,
\eq{
\E[  \norm{ \bs{Y}_{k} }_2^p \vert \bx_k, \bN_k  ]^{\frac{1}{p}} &  \leq C \lVert  (\bN_k^\top  - \frac{1}{V} \mathbbm{1}_V  \mathbbm{1}_{L -1}^\top  )   \mathbbm{1}_{L -1}  \rVert_2 \Big(   \sqrt{p} \sqrt{ \frac{ L k }{V}  }    +   p^2 \sqrt{L} k^{\frac{1}{p}}    \Big), 
}
which implies
\eq{
\E[  \norm{ \bs{Y}_{k} }_2^p ]^{\frac{1}{p}} \leq   C L p \Big(   \sqrt{p} \sqrt{ \frac{ k }{V}  }    +   p^2  k^{\frac{1}{p}}    \Big).
}
\end{itemize}
Therefore,  for $p = \log V$, we have
\eq{
\E [ \lVert \bs{S}_{22} \rVert_2^p ] \leq C \Big(   \frac{\sqrt{p} }{N LV}  +  \frac{p^{5/2}}{L N \sqrt{NV}}   + \frac{p^4}{L N^2}   \Big)
}
Therefore, we have
\eq{
\mpr \Big[  \lVert \bs{S}_{22} \rVert_2 > C  \frac{K^{3/2} \log^{3/2} V}{N LV}    \Big] \leq \frac{1}{V^K}.  \label{eq:s22norm}
}
By  \eqref{eq:s21trace},  \eqref{eq:s21norm},   \eqref{eq:s22trace}, and   \eqref{eq:s22norm}, we have the second item.  For the last item, we have
\eq{
 \bs{S}_3 - \frac{1}{N} \frac{1}{V^2 L^2} \mathbbm{1}_V \mathbbm{1}_V^\top
 =  \frac{1}{V^2 L^2} \mathbbm{1}_V \mathbbm{1}_V^\top \Bigg(    \Big \lVert \frac{1}{N} \sum_{j = 1}^N    ( \bx_j -   \frac{1}{V} \mathbbm{1}_{V}    )    \Big \rVert_2^2       - \frac{1}{N}   \Bigg)
}
By Proposition \ref{prop:multrandommatrix}, 
\eq{
\mpr \Big[ ~ \Big \lvert    \Big \lVert \frac{1}{N} \sum_{j = 1}^N    ( \bx_j -   \frac{1}{V} \mathbbm{1}_{V}    )    \Big \rVert_2^2       - \frac{1}{N}  \Big \rvert > \frac{ C  K^2 \log^2 V}{N \sqrt{V}}  ~  \Big] \leq \frac{1}{V^K}.
}
The displayed equation implies the third item.
\end{proof}

\section{Miscellaneous}

\begin{proposition}
\label{prop:kroneckerproduct}
Let $\bs{A} \in \R^{d \times d}$ and $\bs{B} \in \R^{V \times V}$. Let $ \bs{M} \coloneqq \bs{A} \otimes \bs{B}$.  We have
\eq{
\norm{\bs{M}}_2 = \norm{\bs{A}}_2 \norm{\bs{B}}_2 ~~ \text{and} ~~  \norm{\bs{M}}_F = \norm{ \bs{A} }_F \norm{ \bs{B}}_F ~~ \text{and} ~~  \tr(\bs{M}) = \tr( \bs{A} ) \tr( \bs{B}) .
}
\end{proposition}

\begin{proof}
The Frobenius norm and trace are straightforward.  For the $\ell_2$ norm,  let   $\bs{A} \eqqcolon \sum_{i = 1}^d \sigma_i \bs{u}_i \bs{v}_i^\top$ and   $\bs{B} \eqqcolon \sum_{j = 1}^V \tilde{\sigma}_j \tilde{\bs{u}}_j \tilde{\bs{v}}_j^\top$.  We have
\eq{
\bs{M} & =   \sum_{i = 1}^d \sum_{j = 1}^V \sigma_i  \tilde{\sigma}_j (\bs{u}_i \bs{v}_i^\top) \otimes ( \tilde{\bs{u}}_j \tilde{\bs{v}}_j^\top )  =  \sum_{i = 1}^d \sum_{j = 1}^V \sigma_i  \tilde{\sigma}_j     (\bs{u}_i  \otimes  \tilde{\bs{u}}_j)  ( \bs{v}_i \otimes  \tilde{\bs{v}}_j)^\top.
}
For any $(i,j) \neq (i^\prime, j^\prime)$, we have
\eq{
 (\bs{u}_i  \otimes  \tilde{\bs{u}}_j)^\top       (\bs{u}_{i^\prime}  \otimes  \tilde{\bs{u}}_{j^\prime})   =   (\bs{v}_i  \otimes  \tilde{\bs{v}}_j)^\top       (\bs{v}_{i^\prime}  \otimes  \tilde{\bs{v}}_{j^\prime})    = 0 .
}
Therefore,
\eq{
\norm{\bs{M}}_2  = \max_{i,j }  \sigma_i  \tilde{\sigma}_j = \max_{i}  \sigma_i \max_{j} \tilde{\sigma}_j.
}
\end{proof}

\begin{proposition}
\label{prop:HCpositivepolynomial}
Let $\bs{z} \sim \cN(0, I_d)$ and $P_k: \R^d \to [0,\infty)$ denotes a degree $k$ polynomial which takes nonnegative values.  For $p \geq 1$, we have
\eq{
\E[  \lvert  P_k(\bs{z} ) \rvert^p ]^{\frac{1}{p}} \leq  \big( 8 (p-1) \big)^{\frac{k}{2}} \E[   P_k(\bs{z})   ].
}
\end{proposition}

\begin{proof}
By hypercontractivity,  it is sufficient to prove   that  $\frac{\E[  \lvert  P_k(\bs{z} )^2 ]^{\frac{1}{2}} }{\E[   P_k(\bs{z})   ]} \leq 8^{\frac{k}{2}}$. We have
\eq{
 \E[  \lvert  P_k(\bs{z} )^2 ]^2 \leq \E[  \lvert  P_k(\bs{z} ) ]  \E[  \lvert  P_k(\bs{z} )^3 ] 
 \leq  2^{\frac{3 k}{2}}   \E[  \lvert  P_k(\bs{z} ) ] \E[  \lvert  P_k(\bs{z} )^2 ]^{\frac{3}{2}} 
} 
which proves the result.
\end{proof}

\begin{proposition}
\label{prop:scaledhermiteexpansion}
Let $k \in \N$ and $\bw \sim N(0, \bs{I}_d)$.  For $L > 0$ and $\bs{u}, \bs{v} \in S^{d-1}$, we have
\eq{
\E \Big[ H_{e_k} \Big(  \frac{1}{\sqrt{L}} \bw^\top \bs{u} \Big)  H_{e_k} \Big(   \frac{1}{\sqrt{L}} \bw^\top \bs{v} \Big) \Big ] =  \frac{k!}{L^k} \sum_{i = 0}^{\floor{k/2}}  \frac{(2i - 1)!!}{2i!!} \binom{k}{2i}  (L - 1)^{2i}  \inner{\bs{u}}{\bs{v}}^{k - 2i} 
}
\end{proposition}

\begin{proof}
For $a \in \R$, we have
\eq{
H_{e_k} ( a x) = \sum_{i = 0}^{\floor{k/2}} \frac{k!}{2^i i! (k - 2i)!} (a^2 - 1)^i a^{k - 2i}  H_{e_{k - 2i}}( x) 
}
Therefore, for $a = 1/\sqrt{L}$, we have
\eq{
 & \E \Big[ H_{e_k} \Big(  \frac{1}{\sqrt{L}} \bw^\top \bs{u} \Big)  H_{e_k} \Big(   \frac{1}{\sqrt{L}} \bw^\top \bs{v} \Big) \Big ] \\
&= 
\E \Bigg[ \Big(  \sum_{i = 0}^{\floor{k/2}} \frac{k!}{2^i i! (k - 2i)!} (a^2 - 1)^i a^{k - 2i}  H_{e_{k - 2i}}( \bw^\top \bs{u} )  \Big)  \Big(  \sum_{i = 0}^{\floor{k/2}} \frac{k!}{2^i i! (k - 2i)!} (a^2 - 1)^i a^{k - 2i}  H_{e_{k - 2i}}(  \bw^\top \bs{v} )  \Big) \Bigg] \\
& =   \sum_{i = 0}^{\floor{k/2}} \Big( \frac{k!}{2^i i! (k - 2i)!}  \Big)^2  (a^2 - 1)^{2i}  a^{2(k - 2i)} (k - 2i)! \inner{\bs{u}}{\bs{v}}^{k - 2i} \\
& =  \frac{k!}{L^k} \sum_{i = 0}^{\floor{k/2}}  \frac{(2i - 1)!!}{2i!!} \binom{k}{2i}  (L - 1)^{2i}  \inner{\bs{u}}{\bs{v}}^{k - 2i}. 
}
\end{proof}

\begin{proposition}
\label{prop:matrixCS}
    Let  $Z$ be a random variable and $\bs{X}$ be a $d \times d$ symmetric matrix valued random matrix. We have
    \eq{
    - \E[Z^2]^{\frac{1}{2}} \E[\bs{X}^2 ]^{\frac{1}{2}} \preceq \E[Z \bs{X}] \preceq \E[Z^2]^{\frac{1}{2}} \E[\bs{X}^2 ]^{\frac{1}{2}}.
    }
\end{proposition}

\begin{proof}
    We observe that
    \eq{
    \begin{bmatrix}
        Z \bs{I}_d \\
         \bs{X} 
    \end{bmatrix}  \begin{bmatrix}
        Z \bs{I}_d & \bs{X}
    \end{bmatrix} = \begin{bmatrix}
        Z^2 \bs{I}_d &  Z \bs{X} \\
        Z \bs{X} & \bs{X}^2
    \end{bmatrix} \succeq 0 ~ \Rightarrow ~  \begin{bmatrix}
        \E[ Z^2] \bs{I}_d &  \E[ Z \bs{X}] \\
        \E[ Z \bs{X}] &   \E [\bs{X}^2 ] 
    \end{bmatrix}  \succeq 0. \label{eq:matrixCSineq}
    }
    By \cite[Proposition 24]{Benarous2025learningquadraticneuralnetworks}, we know that \eqref{eq:matrixCSineq} is equivalent to 
    $\E[Z \bs{X}]^2 \preceq \E[Z^2] \E[\bs{X}^2 ]$. Since $\bs{X} \to \sqrt{\bs{X}}$ is monotone in matrix order, we have the result.
\end{proof}

\subsection{Rosenthal-Burkholder inequality and corollaries}
We will rely on the following inequality:
\begin{proposition}[{\citep[Theorem 2.1]{peng2025matrixrosenthalconcentrationinequalities}}]
\label{prop:rosenthal}
Let  $\{ \bs{M}_{k} \}_{k = 1}^N$ be a d-dimensional symmetric matrix valued martingale adapted to the filtration  $\{ \cF_k \}_{k = 0}^N$.  Let    $\bs{Y}_k \coloneqq  \bs{M}_{k} -  \bs{M}_{k - 1}$   be its corresponding difference sequence and the quadratic variation is defined as
\eq{
\bs{Q}_N \coloneqq \sum_{k = 1}^N \E[ \bs{Y}_k^2 \vert \cF_{k - 1} ].
}
For any $p \geq 2$, suppose
\eq{
\E \Big[ \lVert  \bs{Q}_N \rVert_2^{\frac{p}{2}} \Big]^{\frac{1}{p}} < \infty ~~ \text{and} ~~  \sup_{k \in [N]}   \E \Big[\lVert  \bs{Y}_k \rVert_2^p   \Big]^{\frac{1}{p}} < \infty.
}
Then it holds that
\eq{
\E \Big[ \lVert  \bs{M}_N \rVert_2^{p} \Big]^{\frac{1}{p}}  \leq C \Big( \sqrt{p \vee \log d}  ~ \E \Big[ \lVert  \bs{Q}_N \rVert_2^{\frac{p}{2}} \Big]^{\frac{1}{p}} + (p \vee \log d) N^{\frac{1}{p}}  \sup_{k \in [N]}   \E \Big[\lVert  \bs{Y}_k \rVert_2^p   \Big]^{\frac{1}{p}}    \Big).
}
\end{proposition}  

We have the following corollaries:

\begin{corollary}
\label{cor:rosenthalresults}
The following statements holds for general $L, V > 0$:
\begin{enumerate}
\item  For  $X \sim \mathrm{Binomial}(L, \frac{1}{V})$,  we have
\eq{
\E[ \abs{X - kq}^p]^{\frac{1}{p}} \leq C \Big( \sqrt{p}  \sqrt{\frac{L}{V}} +   p \Big( \frac{L}{V} \Big)^{\frac{1}{p}}  \Big).
}
\item Let $\bs{c} = (c_1, \cdots, c_V) \sim   \mathrm{Multinomial}(L,\frac{1}{V} \mathbbm{1}_V)$. For $p \geq 1$, we have
\eq{
\E[ \norm{\bs{c}}_p^p ] \leq   C^p V   \Bigg(    \Big( \frac{L}{V}  \Big)^{p}   +  \Big( \frac{p L}{V}  \Big)^{\frac{p}{2}}   + p^p \frac{L}{V}  \Bigg).
}
\item By following the notation in the second item, 
\begin{itemize}
\item  If  $V \gg L,$  we have for $L \geq e^{2e} + 1$,
\eq{
\mpr \left[  \norm{\bs{c}}_\infty \geq \log L \right] \leq \Big( \frac{2e}{\log L - 1} \Big)^{\frac{\log L - 1}{2}}	  \Big( \frac{L}{V} \Big)^{\log L - 2}	 
}
\item  If  $L \gg V$,  we have  
\eq{
\mpr \left[  \norm{\bs{c}}_\infty \geq \frac{e L}{V} \right] \leq 2 V e^{-  L/V } .
}
\end{itemize}
\end{enumerate}
\end{corollary}

\begin{proof}
The first two items are direct consequence of  Proposition \ref{prop:rosenthal}.
For the third item, using $\indic{c_w \geq k} \leq \frac{c_w (c_w - 1) \cdots  (c_w - k + 1)}{k!}$ and linearity of expectation
\eq{
\mpr[  \norm{\bs{c}}_\infty \geq k ] \leq \sum_{w = 1}^V \mpr[ c_w  \geq k ]& \leq \sum_{w = 1}^V \frac{\E[ c_w (c_w - 1) \cdots  (c_w - k + 1) ]}{k!}  = \frac{L (L-1) \cdots  (L - k + 1)}{k! V^{k - 1}}.
}
For $V \gg L$, by choosing $k =\floor{ \log L }$, the result follows. For $L \gg V$, by choosing $k = \floor{ \frac{eL}{V} }$, the result follows.
\end{proof}

\end{document}